%% file: opt-descent-overview.tex
\documentclass[twoside,11pt]{book}

\usepackage{jmlr2e}
\usepackage{bm}

\usepackage[english]{babel}

\usepackage{imakeidx}
\makeindex[columns=2, title=Alphabetical Index]

\usepackage{stackengine}

\usepackage{enumitem}
\setenumerate[1]{itemsep=0pt,partopsep=0pt,parsep=\parskip,topsep=5pt}
\setitemize[1]{itemsep=0pt,partopsep=0pt,parsep=\parskip,topsep=5pt}
\setdescription{itemsep=0pt,partopsep=0pt,parsep=\parskip,topsep=5pt}

\usepackage{dsfont}

\usepackage[framemethod=tikz]{mdframed}

\usepackage{amsmath}
\usepackage{arydshln}

\numberwithin{equation}{section}

\usepackage{amsmath}
\usepackage{blkarray}

\usepackage[final]{pdfpages}

\usepackage[noframe]{showframe}
\usepackage{framed}
\usepackage{lipsum}
\definecolor{color0}  {RGB}{174,225,254} 
\definecolor{color1}  {RGB}{220,227,248} 
\definecolor{color2}  {RGB}{28,130,185} 
\definecolor{color3}  {RGB}{255,253,250} 

\definecolor{colormiddleright}  {RGB}{245,253,250} 
\definecolor{colorbottomleft}  {RGB}{255,243,250} 
\definecolor{coloruppermiddle}  {RGB}{255,253,230} 
\definecolor{colormiddleleft}  {RGB}{255,253,250}

\definecolor{colorcr}  {RGB}{249,253,232} 
\definecolor{colorreduction}  {RGB}{255,235,254} 
\definecolor{colorqr}  {RGB}{254,221,199} 
\definecolor{colorbiconjugate}  {RGB}{251,149,161} 

\definecolor{colorsvd}  {RGB}{215,247,235} 

\definecolor{colorupperright}  {RGB}{239,246,251} 
\definecolor{colorspectral}  {RGB}{206,226,243} 

\definecolor{colorbottomright}  {RGB}{220,224,236} 
\definecolor{coloreigenvalue}  {RGB}{197,203,224} 

\definecolor{colorupperleft}  {RGB}{235,243,240} 
\definecolor{colorsemidefinite}  {RGB}{217,232,226} 

\definecolor{colormiddle} {RGB}{235, 240,255}
\definecolor{colorlu}  {RGB}{220,227,255} 

\definecolor{colorals}  {RGB}{240,230,255} 
\definecolor{coloralsbkg}  {RGB}{248,243,255} 

\definecolor{canaryyellow}{rgb}{1.0, 0.75, 0.0}
\definecolor{bluepigment}{rgb}{0.0, 0.0, 1.0}

\definecolor{canarypurple}{RGB}{208, 13, 241}

{\endMakeFramed}
\definecolor{shadecolor}{gray}{0.75}

\usepackage{xcolor}

\usepackage{tcolorbox}

\usepackage{graphicx}  
\usepackage[hang]{subfigure}

\usepackage{algorithm}
\usepackage{algpseudocode}
\usepackage{amsmath}
\usepackage{graphics}
\usepackage{epsfig}

\usepackage{blkarray}

\usepackage{listings}

\usepackage{tikz}
\usetikzlibrary{mindmap,trees,backgrounds}
\usetikzlibrary{arrows.meta}

\usetikzlibrary{decorations.text}

\usetikzlibrary{calc,backgrounds}

\usepackage{changepage}                 

\newlength{\offsetpage}
	{\end{adjustwidth}}

\usepackage{hhline}

\usetikzlibrary{positioning,shapes,shadows,arrows}
\tikzstyle{condition}=[rectangle, draw=black, rounded corners, fill=colorqr, drop shadow,
text centered, anchor=north, text=black, text width=3cm]
\tikzstyle{abstract}=[rectangle, draw=black, rounded corners, fill=blue!30, drop shadow,
text centered, anchor=north, text=black, text width=3cm]
\tikzstyle{comment}=[rectangle, draw=black, rounded corners, fill=color1, drop shadow,
text centered, anchor=north, text=black, text width=3cm]
\tikzstyle{myarrow}=[->, >=open triangle 90, thick]
\tikzstyle{line}=[-, thick]

\usepackage{sidecap}
\sidecaptionvpos{figure}{t}
\usepackage{verbatimbox}

\usepackage[labelfont=bf]{caption}
\newcommand\myhrulefill[1]{\leavevmode\leaders\hrule height#1\hfill\kern0pt}
\DeclareCaptionFormat{myformat}{{\color[RGB]{0,0,0}\myhrulefill{0.08em}}\\#1#2#3}
\usepackage{graphicx}
\usepackage{pythonhighlight}

\usepackage{mathtools}
\DeclarePairedDelimiter\ceil{\lceil}{\rceil}
\DeclarePairedDelimiter\floor{\lfloor}{\rfloor}

\makeatletter \renewcommand*\cleardoublepage{
	\clearpage
	\if@twoside   
	\ifodd\c@page 
	\hbox{}\newpage
	\if@twocolumn\hbox{}   
	\newpage
	\fi
	\fi
	\fi
} \makeatother
\let\originalpart=\part
\def\part#1{\cleardoublepage\clearpage \pagecolor{\partcolor} \originalpart{#1}\nopagecolor }

\usepackage{yfonts,color,lettrine}
\definecolor{caligraphcolor}{HTML}{7F0000}

\usepackage{setspace}

\AtBeginDocument{\setlength{\DefaultFindent}{0.5em}}
\usepackage[Bjornstrup]{fncychap}

\usepackage{fancyhdr, blindtext}
\newcommand{\changefontt}{%
	\fontsize{11}{13}\selectfont
}
\newcommand{\changefonts}{%
	\fontsize{9}{11}\selectfont
}

\fancypagestyle{plain}{%
	\fancyhf{}%
	\fancyfoot[C]{\color{winestain} \thepage}%
}

\usepackage[colorlinks]{hyperref}

\usepackage{color}   
\definecolor{winestain}{rgb}{0.5,0,0}
\definecolor{mydarkblue}{rgb}{0,0.08,0.45}
\usepackage{hyperref}
\hypersetup{
	colorlinks=true, 
	linktoc=all,     
	linkcolor=winestain,  
	anchorcolor=blue,
	citecolor=winestain,
}
\usepackage[hyperpageref]{backref} 

\usepackage{booktabs}  

\newcommand{\clearchapter}{%
	\ifodd\thepage\else\newpage\mbox{}\thispagestyle{empty}\fi
	\chapter
}

\makeatletter \renewcommand*\cleardoublepage{
	\clearpage
	\if@twoside   
	\ifodd\c@page 
	\hbox{}\newpage
	\if@twocolumn\hbox{}   
	\newpage
	\fi
	\fi
	\fi
} \makeatother
\let\originalpart=\part
\def\part#1{\cleardoublepage\clearpage \pagecolor{\partcolor} \originalpart{#1}\nopagecolor }

\newtheorem{definitionT}{Definition}[chapter]
\newmdenv[skipabove=7pt,
skipbelow=7pt,
rightline=false,
leftline=true,
topline=false,
bottomline=false,
linecolor=mydefinitionred,
innerleftmargin=5pt,
innerrightmargin=5pt,
innertopmargin=0pt,
leftmargin=2cm,
rightmargin=0cm,
linewidth=4pt,
innerbottommargin=0pt]{dBox}
\newenvironment{definition}{\begin{dBox}\begin{definitionT}}{\end{definitionT}\end{dBox}}	

\newtheorem{exerciseC}{Exercise}[chapter]
\newmdenv[skipabove=7pt,
skipbelow=7pt,
rightline=false,
leftline=true,
topline=false,
bottomline=false,
linecolor=mydarkgreen,
innerleftmargin=5pt,
innerrightmargin=5pt,
innertopmargin=0pt,
leftmargin=2cm,
rightmargin=0cm,
linewidth=4pt,
innerbottommargin=0pt]{eBox}
\newenvironment{exercise}{\begin{eBox}\begin{exerciseC}}{\end{exerciseC}\end{eBox}}

\usepackage{adforn}  
\newcommand{\xchaptertitle}{Chapter~\thechapter~}
\newcommand{\problemname}{Problems}

\makeatletter
\renewcommand{\l@section}{\@dottedtocline{1}{1.5em}{2.25em}}
\renewcommand{\l@subsection}{\@dottedtocline{2}{3.7em}{3.1em}}
\renewcommand{\l@subsubsection}{\@dottedtocline{3}{4.5em}{3.4em}}
\makeatother

\numberwithin{equation}{section}
\renewcommand\thesection{\thechapter.\arabic{section}}

\usepackage{minitoc}   
\let\cleardoublepage\clearpage
\usepackage[titletoc]{appendix}

\input{symbols_dl}

\input{symbols.tex}

\jmlrheading{1}{2022}{1-48}{4/00}{10/00}{Jun Lu}
\firstpageno{1}
\begin{document}

\newcommand{\mytitle}{Gradient Descent, Stochastic Optimization, and Other Tales}
\title{\mytitle}

\frontmatter
\input{front.tex}

\clearpage
\author{
\begin{center}
\name Jun Lu \\ 
\email jun.lu.locky@gmail.com
\end{center}
}
\clearpage
\maketitle
\chapter*{\centering \begin{normalsize}Preface\end{normalsize}}
The goal of this book is to debunk and dispel the magic behind the black-box optimizers and stochastic optimizers. It aims to build a solid foundation on how and why these techniques work. This
manuscript crystallizes this knowledge by deriving from simple intuitions, the mathematics
behind the strategies. This book doesn't shy away from addressing both the formal
and informal aspects of gradient descent and stochastic optimization methods. By doing so, it hopes to provide readers with a deeper understanding of these techniques as well as the when, the how and the why of applying these algorithms.

Gradient descent stands out as one of the most popular algorithms to perform optimization and by far the
most common way to optimize machine learning tasks. Its stochastic version receives attention in recent years, and this is particularly true for optimizing deep neural networks. In deep neural networks, the gradient followed by a single sample or a batch of samples is employed to save computational resources and escape from saddle points. In 1951, Robbins and Monro published \textit{A stochastic approximation method}, one of the first modern treatments on stochastic optimization that estimates local gradients with a new batch of samples. And now, stochastic optimization has become a core technology in machine learning, largely due to the development of the back propagation algorithm in fitting a neural network. The sole aim of this article is to give a self-contained introduction to concepts and mathematical tools in gradient descent and stochastic optimization.
However, we clearly realize our inability to cover all the useful and interesting results concerning optimization methods and given the paucity of scope to present this discussion, e.g., the separated analysis of trust region methods, convex optimization, and so on. We refer the reader to literature in the field of numerical optimization for a more detailed introduction to the related fields. 

The article is primarily a summary of purpose, significance of important concepts in optimization methods, e.g., vanilla gradient descent, gradient descent with momentum, conjugate descent, conjugate gradient, and the origin and rate of convergence of the methods, which shed light on their applications. The mathematical prerequisite is a first course in linear algebra and calculus. Other than this modest background, the development is self-contained, with rigorous proof provided throughout.




\paragraph{Keywords: }
Gradient descent, Stochastic gradient descent, Steepest descent, Conjugate descent and conjugate gradient, Learning rate annealing, Adaptive learning rate, Second-order methods.

\newpage
\begingroup
\hypersetup{
	linkcolor=winestain,
	linktoc=page,  
}
\dominitoc
\tableofcontents 
\listoffigures 
\endgroup

\newpage

\mainmatter

\input{chapter-intro}
\input{chapter-GD}

\input{chapter-lineSearch}

\input{chapter-lrate}

\input{chapter-stocopt}
\input{chapter-second}

\input{chapter-append}

\newpage
\vskip 0.2in
\bibliography{bib}

\clearpage
\printindex

\end{document}

%% file: symbols_dl.tex
\def\ceil#1{\lceil #1 \rceil}
\def\floor#1{\lfloor #1 \rfloor}
\def\1{\bm{1}}



%% file: symbols.tex
\newcommand{\bone}{\mathbf{1}}
\newcommand{\complex}{\mathbb{C}}

\newcommand{\rank}{\mathrm{rank}}
\newcommand{\trace}{\mathrm{tr}}

\newcommand{\rms}{\text{RMS}}

\newcommand{\real}{\mathbb{R}}

\newcommand{\cspace}{\mathcal{C}}
\newcommand{\nspace}{\mathcal{N}}

\newcommand{\leadto}{\qquad\underrightarrow{ \text{leads to} }\qquad}

\newcommand{\gap}{\,\,\,\,\,\,\,\,}

\definecolor{mylightbluetitle}{RGB}{60,113,183}
\definecolor{mylightbluetext}{RGB}{27,54,189}
\definecolor{structurecolorblue}{RGB}{60,113,183}
\definecolor{structurecolorgreen}{RGB}{63,145,182}
\colorlet{structurecolor}{structurecolorblue}
\definecolor{structurecolorelegant}{RGB}{60,113,183}
\definecolor{structurecolorlt}{RGB}{31,119,185}

\definecolor{mydarkblue}{rgb}{0,0.08,0.45}
\definecolor{mydarkred}{rgb}{0.70,0.00,0.00}
\definecolor{mydarkgreen}{rgb}{0.00,0.30,0.00}
\definecolor{mydarkyellow}{RGB}{197,151,13}
\definecolor{mydarkpurple}{RGB}{149,18,192}
\definecolor{mydarkgray}{RGB}{64,64,64}

\definecolor{mydefinitionred}{rgb}{0.30,0.00,0.00}

\newcommand{\bepsilon}{\boldsymbol\epsilon}
\newcommand{\bLambda}{\boldsymbol\Lambda}

\newcommand{\blambda}{\boldsymbol\lambda}

\newcommand{\bmu}{\boldsymbol\mu}

\newcommand\norm[1]{\left\lVert#1\right\rVert}
\newcommand{\diag}{\mathrm{diag}}

\newcommand{\ba}{\bm{a}}
\newcommand{\bA}{\bm{A}}
\newcommand{\bb}{\bm{b}}
\newcommand{\bB}{\bm{B}}
\newcommand{\bc}{\bm{c}}
\newcommand{\bC}{\bm{C}}
\newcommand{\bd}{\bm{d}}
\newcommand{\bD}{\bm{D}}
\newcommand{\be}{\bm{e}}

\newcommand{\bg}{\bm{g}}

\newcommand{\bH}{\bm{H}}

\newcommand{\bI}{\bm{I}}

\newcommand{\bL}{\bm{L}}
\newcommand{\bmm}{\bm{m}}
\newcommand{\bM}{\bm{M}}
\newcommand{\bn}{\bm{n}}

\newcommand{\bP}{\bm{P}}
\newcommand{\bq}{\bm{q}}
\newcommand{\bQ}{\bm{Q}}
\newcommand{\br}{\bm{r}}
\newcommand{\bR}{\bm{R}}
\newcommand{\bs}{\bm{s}}

\newcommand{\bu}{\bm{u}}

\newcommand{\bv}{\bm{v}}

\newcommand{\bw}{\bm{w}}

\newcommand{\bx}{\bm{x}}
\newcommand{\bX}{\bm{X}}
\newcommand{\by}{\bm{y}}
\newcommand{\bY}{\bm{Y}}
\newcommand{\bz}{\bm{z}}

\newcommand{\whbx}{\widehat{\bx}}
\newcommand{\whbd}{\widehat{\bd}}
\newcommand{\whbg}{\widehat{\bg}}
\newcommand{\whL}{\widehat{L}}

\newcommand{\mathcalV}{\mathcal{V}}

\newcommand{\bzero}{\mathbf{0}}

\def\sS{{\mathbb{S}}}

%% file: front.tex
\begingroup
\clearpage

\pagestyle{empty}
\begin{center}
	\bfseries
\vspace{1cm}
	\Huge
	{\Huge
		{\mytitle}}
	
\vspace{1cm}
	\normalsize

\vspace{1cm}
	
\vspace{1cm}
	
\vspace{1cm}
	\normalsize
	
\vspace{1cm}
\end{center}

\clearpage
\newpage\null\thispagestyle{empty}\newpage

\clearpage
\pagestyle{empty}
\begin{center}
	\bfseries
\vspace{1cm}
	\Huge
	{\Huge
		\mytitle}
	
	\vspace{1cm}
	\normalsize
	
\vspace{1cm}
	\small BY\\
	\Large Jun Lu\\[0.5em]
	
\vspace{1cm}

\vspace{1cm}
	\normalsize
	
\vspace{10cm}

\large
\vspace{1cm}
\end{center}

\clearpage
\pagestyle{empty}

\clearpage
\endgroup

%% file: chapter-intro.tex
\newpage
\clearchapter{Background}
\newpage

In this chapter, we offer a concise overview of fundamental concepts in linear algebra and calculus.  
Additional significant concepts will be introduced and elaborated upon as needed for clarity.
Note that this chapter does not aim to provide an exhaustive treatment of these subjects. Interested readers wishing to delve deeper into these topics are advised to consult advanced texts on linear algebra and calculus.
For the sake of simplicity, we restrict our consideration to real matrices throughout this text. Unless explicitly stated otherwise, the eigenvalues of the matrices under discussion are assumed to be real as well.

\paragraph{The vector space $\real^n$ and matrix space $\real^{m\times n}$.}
The vector space $\real^n$ is the set of $n$-dimensional column vectors with real components. 
Throughout the book, our primary focus will be on problems within the $\real^n$ vector space. However, in a few instances, we will also explore other vector spaces, e.g., the nonnegative vector space.
Correspondingly, the matrix space $\real^{m\times n}$ represents the set of all real-valued matrices with dimensions $m\times n$.

In all cases, scalars will be denoted in a non-bold font, possibly with subscripts (e.g., $a$, $\alpha$, $\alpha_i$). We will use \textbf{boldface} lowercase letters, possibly with subscripts, to denote vectors (e.g., $\bmu$, $\bx$, $\bx_n$, $\bz$), and
\textbf{boldface} uppercase letters, possibly with subscripts, to denote matrices (e.g., $\bA$, $\bL_j$). 
The $i$-th element of a vector $\bz$ will be denoted by $z_i$ in the non-bold font.


Subarrays are formed when fixing a subset of indices.
The $i$-th row and $j$-th column value of matrix $\bA$ (referred to as entry ($i,j$) of $\bA$) will be denoted by  $a_{ij}$. 
Furthermore, it will be helpful to utilize the \textbf{Matlab-style notation}: the submatrix of  matrix $\bA$ spanning from the $i$-th row to the $j$-th row and the $k$-th column to the $m$-th column  will be denoted by $\bA_{i:j,k:m}$. 
A colon is used to indicate all elements of a dimension, e.g., $\bA_{:,k:m}$ denotes the $k$-th column to the $m$-th column of  matrix $\bA$, and $\bA_{:,k}$ denotes the $k$-th column of $\bA$. Alternatively, the $k$-th column of matrix $\bA$ may be denoted more compactly as $\ba_k$.

\index{Matlab-style notation}

When the index is non-continuous, given ordered subindex sets $I$ and $J$, $\bA[I, J]$ denotes the submatrix of $\bA$ obtained by extracting the rows and columns of $\bA$ indexed by $I$ and $J$, respectively; and $\bA[:, J]$ denotes the submatrix of $\bA$ obtained by extracting the columns of $\bA$ indexed by $J$, where again the colon operator implies all indices, and  the $[:, J]$ syntax in this expression selects all rows from $\bA$ and only the columns specified by the indices in $J$.

\begin{definition}[Matlab Notation]\label{definition:matlabnotation}
Suppose $\bA\in \real^{m\times n}$, and $I=[i_1, i_2, \ldots, i_k]$ and $J=[j_1, j_2, \ldots, j_l]$ are two index vectors, then $\bA[I,J]$ denotes the $k\times l$ submatrix
$$
\bA[I,J]=
\begin{bmatrix}
	a_{i_1,j_1} & a_{i_1,j_2} &\ldots & a_{i_1,j_l}\\
	a_{i_2,j_1} & a_{i_2,j_2} &\ldots & a_{i_2,j_l}\\
	\vdots & \vdots&\ddots & \vdots\\
	a_{i_k,j_1} & a_{i_k,j_2} &\ldots & a_{i_k,j_l}\\
\end{bmatrix}.
$$
Whilst, $\bA[I,:]$ denotes a $k\times n$ submatrix, and $\bA[:,J]$ denotes a $m\times l$ submatrix analogously.

We note that it does not matter whether the index vectors $I$ and $J$ are row vectors or column vectors. 
What's important is which axis they index (either rows of $\bA$ or columns of $\bA$).
Note that the ranges of the indices are given as folows:
$$
\left\{
\begin{aligned}
	0&\leq \min(I) \leq \max(I)\leq m;\\
	0&\leq \min(J) \leq \max(J)\leq n.
\end{aligned}
\right.
$$
\end{definition}

In all instances, vectors are presented in column form rather than as rows.
A row vector will be denoted by the transpose of a column vector such as $\ba^\top$. 
A column vector with specific values is separated by the semicolon symbol $``;"$, for instance, $\bx=[1;2;3]$ is a column vector in $\real^3$. 
Similarly, a  row vector with specific values is split by the comma symbol $``,"$, e.g., $\by=[1,2,3]$ is a row vector with 3 values. 
Additionally, a column vector can also be denoted by the transpose of a row vector, e.g., $\by=[1,2,3]^\top$ is also a column vector.

The transpose of a matrix $\bA$ will be denoted as $\bA^\top$, and its inverse will be denoted by $\bA^{-1}$. 
We will denote the $p \times p$ identity matrix by $\bI_p$ (or simply by $\bI$ when the size is clear from context). A vector or matrix of all zeros will be denoted by a \textbf{boldface} zero, $\bzero$, with the size clear from context, or we denote $\bzero_p$ to be the vector of all zeros with $p$ entries.
Similarly, a vector or matrix of all ones will be denoted by a \textbf{boldface} one $\bone$ whose size is clear from  context, or we denote $\bone_p$ to be the vector of all ones with $p$ entries.
We will frequently omit the subscripts of these matrices when the dimensions are clear from  context.

We will  use $\be_1, \be_2, \ldots, \be_n$ to represent the standard basis of $\real^n$, where $\be_i$ is the vector whose $i$-th component is one while all the others are zero.


%

\index{Eigenvalue}
\begin{definition}[Eigenvalue]
Given any vector space $E$ and any linear map $\bA: E \rightarrow E$ (or simply real matrix $\bA\in\real^{m\times n}$), a scalar $\lambda \in K$ is called an eigenvalue, or proper value, or characteristic value of $\bA$, if there exists some nonzero vector $\bu \in E$ such that
\begin{equation*}
	\bA \bu = \lambda \bu.
\end{equation*}
\end{definition}
In fact, real-valued matrices can have complex eigenvalues. However, all the eigenvalues of symmetric matrices are real (see Theorem~\ref{theorem:spectral_theorem}, p.~\pageref{theorem:spectral_theorem}).

\index{Spectrum}
\index{Spectral radius}
\begin{definition}[Spectrum and Spectral Radius]\label{definition:spectrum}
The set of all eigenvalues of $\bA$ is called the spectrum of $\bA$ and is denoted by $\Lambda(\bA)$. The largest magnitude of the eigenvalues is known as the spectral radius $\rho(\bA)$:
$$
\rho(\bA) = \mathop{\max}_{\lambda\in \Lambda(\bA)}  |\lambda|.
$$
\end{definition}

\index{Eigenvector}
\begin{definition}[Eigenvector]
A vector $\bu \in E$ is called an eigenvector, or proper vector, or characteristic vector of $\bA$, if $\bu \neq 0$ and if there exists some $\lambda \in K$ such that
\begin{equation*}
	\bA \bu  = \lambda \bu,
\end{equation*}
where the scalar $\lambda$ is then an eigenvalue. And we say that $\bu$ is an eigenvector associated with $\lambda$.
\end{definition}

Moreover, the tuple $(\lambda, \bu)$ mentioned above is termed an  \textbf{eigenpair}. Intuitively, these definitions imply that multiplying matrix $\bA$ by the vector $\bu$ results in a new vector that is in the same direction as $\bu$, but its magnitude scaled by a factor $\lambda$. For any eigenvector $\bu$, we can scale it by a scalar $s$ such that $s\bu$ remains an eigenvector of $\bA$; and for this reason we call the eigenvector an eigenvector of $\bA$ associated with eigenvalue $\lambda$. 
To avoid ambiguity, it is customary to assume that the eigenvector is normalized to have length $1$ and the first entry is positive (or negative) since both $\bu$ and $-\bu$ are eigenvectors.

In linear algebra, every vector space has a basis, and any vector within the space can be expressed as  a linear combination of the basis vectors. 
We then define the span and dimension of a subspace via the basis.

\index{Subspace}
\begin{definition}[Subspace]
A nonempty subset $\mathcal{V}$ of $\real^n$ is called a subspace if $x\ba+y\ba\in \mathcal{V}$ for every $\ba,\bb\in \mathcal{V}$ and every $x,y\in \real$.
\end{definition}

\index{Span}
\begin{definition}[Span]
If every vector $\bv$ in subspace $\mathcal{V}$ can be expressed as a linear combination of $\{\ba_1, \ba_2, \ldots,$ $\ba_m\}$, then $\{\ba_1, \ba_2, \ldots, \ba_m\}$ is said to span $\mathcal{V}$.
\end{definition}

\index{Linearly independent}
In this context, we will often use the idea of the linear independence of a set of vectors. Two equivalent definitions are given as follows.
\begin{definition}[Linearly Independent]
A set of vectors $\{\ba_1, \ba_2, \ldots, \ba_m\}$ is called linearly independent if there is no combination that  can yield $x_1\ba_1+x_2\ba_2+\ldots+x_m\ba_m=0$ unless all $x_i$'s are equal to  zero. An equivalent definition is that $\ba_1\neq \bzero$, and for every $k>1$, the vector $\ba_k$ does not belong to the span of $\{\ba_1, \ba_2, \ldots, \ba_{k-1}\}$.
\end{definition}


\index{Basis}
\index{Dimension}
\begin{definition}[Basis and Dimension]
A set of vectors $\{\ba_1, \ba_2, \ldots, \ba_m\}$ is called a basis of $\mathcal{V}$ if they are linearly independent and span $\mathcal{V}$. Every basis of a given subspace has the same number of vectors, and the number of vectors in any basis is called the dimension of the subspace $\mathcal{V}$. By convention, the subspace $\{\bzero\}$ is said to  have a dimension of zero. 
Furthermore, every subspace of nonzero dimension has a basis that is orthogonal, i.e., the basis of a subspace can be chosen orthogonal.
\end{definition}

\index{Column space}
\begin{definition}[Column Space (Range)]
If $\bA$ is an $m \times n$ real matrix, we define the column space (or range) of $\bA$ as the set spanned by its columns:
\begin{equation*}
	\mathcal{C} (\bA) = \{ \by\in \mathbb{R}^m: \exists\, \bx \in \mathbb{R}^n, \, \by = \bA \bx \}.
\end{equation*}
And the row space of $\bA$ is the set spanned by its rows, which is equivalent to the column space of $\bA^\top$:
\begin{equation*}
	\mathcal{C} (\bA^\top) = \{ \bx\in \mathbb{R}^n: \exists\, \by \in \mathbb{R}^m, \, \bx = \bA^\top \by \}.
\end{equation*}
\end{definition}

\index{Null space}
\begin{definition}[Null Space (Nullspace, Kernel)]
If $\bA$ is an $m \times n$ real matrix, we define the null space (or kernel, or nullspace) of $\bA$ as the set:
\begin{equation*}
	\nspace (\bA) = \{\by \in \mathbb{R}^n:  \, \bA \by = \bzero \}.
\end{equation*}
And the null space of $\bA^\top$ is defined as 	\begin{equation*}
	\nspace (\bA^\top) = \{\bx \in \mathbb{R}^m:  \, \bA^\top \bx = \bzero \}.
\end{equation*}
\end{definition}

Both the column space of $\bA$ and the null space of $\bA^\top$ are subspaces of $\real^m$. In fact, every vector in $\nspace(\bA^\top)$ is orthogonal  to vectors in $\cspace(\bA)$, and vice versa.\footnote{Every vector in $\nspace(\bA)$ is also perpendicular to vectors in $\cspace(\bA^\top)$, and vice versa.}

\index{Rank}
\begin{definition}[Rank]
The rank of a matrix $\bA\in \real^{m\times n}$ is the dimension of its column space. 
That is, the rank of $\bA$ is equal to the maximum number of linearly independent columns of $\bA$, and is also the maximum number of linearly independent rows of $\bA$. The matrix $\bA$ and its transpose $\bA^\top$ have the same rank. 
A matrix $\bA$ is considered to have full rank if its rank equals $min\{m,n\}$.
Specifically, given a vector $\bu \in \real^m$ and a vector $\bv \in \real^n$, then the $m\times n$ matrix $\bu\bv^\top$ obtained by the outer product of vectors is of rank 1. In short, the rank of a matrix is equal to:
\begin{itemize}
	\item number of linearly independent columns;
	\item number of linearly independent rows.
\end{itemize}
And remarkably, these two quantities are always equal (see \citet{lu2022matrix}).
\end{definition}

\index{Orthogonal complement}
\begin{definition}[Orthogonal Complement in General]
The orthogonal complement $\mathcalV^\perp\subseteq \real^m$ of a subspace $\mathcalV\subseteq\real^m$ contains any vector that is perpendicular to $\mathcalV$. That is,
$$
\mathcalV^\perp = \{\bv\in \real^m : \bv^\top\bu=0, \,\,\, \forall \bu\in \mathcalV  \}.
$$
These two subspaces are mutually exclusive yet collectively span the entire space.
The dimensions of $\mathcalV$ and $\mathcalV^\perp$ sum up to the dimension of the entire space. Furthermore, it holds that $(\mathcalV^\perp)^\perp=\mathcalV$.
\end{definition}

\index{Orthogonal complement}
\begin{definition}[Orthogonal Complement of Column Space]
If $\bA$ is an $m \times n$ real matrix, the orthogonal complement of $\mathcal{C}(\bA)$, $\mathcal{C}^{\bot}(\bA)$, is the subspace defined as:
\begin{equation*}
	\begin{aligned}
		\mathcal{C}^{\bot}(\bA) &= \{\by\in \mathbb{R}^m: \, \by^\top \bA \bx=\bzero, \, \forall \bx \in \mathbb{R}^n \} \\
		&=\{\by\in \mathbb{R}^m: \, \by^\top \bv = \bzero, \, \forall \bv \in \mathcal{C}(\bA) \}.
	\end{aligned}
\end{equation*}
\end{definition}
We can then identify the four fundamental spaces associated with any matrix $\bA\in \real^{m\times n}$ of rank $r$:
\begin{itemize}
\item $\cspace(\bA)$: Column space of $\bA$, i.e., linear combinations of columns with dimension $r$;

\item $\nspace(\bA)$: Null space of $\bA$, i.e., all $\bx$ with $\bA\bx=\bzero$ with dimension $n-r$;

\item  $\cspace(\bA^\top)$: Row space of $\bA$, i.e., linear combinations of rows with dimension $r$;

\item  $\nspace(\bA^\top)$: Left null space of $\bA$, i.e., all $\by$ with $\bA^\top \by=\bzero$ with dimension $m-r$, 
\end{itemize}
where $r$ represents the rank of the matrix. Furthermore, $\nspace(\bA)$ is the orthogonal complement of $\cspace(\bA^\top)$, and $\cspace(\bA)$ is the orthogonal complement of $\nspace(\bA^\top)$. The proof can be found in \citet{lu2022matrix}.

\index{Fundamental subspace}

\index{Orthogonal matrix}
\begin{definition}[Orthogonal Matrix, Semi-Orthogonal Matrix]
A real square matrix $\bQ\in\real^{n\times n}$ is an orthogonal matrix if the inverse of $\bQ$ is equal to its  transpose, that is, $\bQ^{-1}=\bQ^\top$ and $\bQ\bQ^\top = \bQ^\top\bQ = \bI$. 
In other words, suppose $\bQ=[\bq_1, \bq_2, \ldots, \bq_n]$, where $\bq_i \in \real^n$ for all $i \in \{1, 2, \ldots, n\}$, then $\bq_i^\top \bq_j = \delta(i,j)$, where $\delta(i,j)$ is the Kronecker delta function. 
For any vector $\bx$, the orthogonal matrix will preserve the length: $\norm{\bQ\bx}= \norm{\bx}$.

If $\bQ$ contains only $\gamma$ of these columns with $\gamma<n$, then $\bQ^\top\bQ = \bI_\gamma$ stills holds, where $\bI_\gamma$ is the $\gamma\times \gamma$ identity matrix. 
But $\bQ\bQ^\top=\bI$ will not hold. 
In this case, $\bQ$ is called semi-orthogonal.

\end{definition}

\index{Normal matrix}
\index{Hermitian matrix}
\index{Orthogonal matrix}
\index{Unitary matrix}

From an introductory course on linear algebra, we have the following remark regarding  the equivalent claims of nonsingular matrices.
\begin{remark}[List of Equivalence of Nonsingularity for a Matrix]
For a square matrix $\bA\in \real^{n\times n}$, the following claims are equivalent:
\begin{itemize}
	\item $\bA$ is nonsingular;~\footnote{The source of the name is a result of the singular value decomposition (SVD).}
	\item $\bA$ is invertible, i.e., $\bA^{-1}$ exists; 
	\item $\bA\bx=\bb$ has a unique solution $\bx = \bA^{-1}\bb$;
	\item $\bA\bx = \bzero$ has a unique, trivial solution: $\bx=\bzero$;
	\item Columns of $\bA$ are linearly independent;
	\item Rows of $\bA$ are linearly independent;
	\item $\det(\bA) \neq 0$; 
	\item $\dim(\nspace(\bA))=0$;
	\item $\nspace(\bA) = \{\bzero\}$, i.e., the null space is trivial;
	\item $\cspace(\bA)=\cspace(\bA^\top) = \real^n$, i.e., the column space or row space span the whole $\real^n$;
	\item $\bA$ has full rank $r=n$;
	\item The reduced row echelon form is $\bR=\bI$;
	\item $\bA^\top\bA$ is symmetric positive definite;
	\item $\bA$ has $n$ nonzero (positive) singular values;
	\item All eigenvalues are nonzero.
\end{itemize}
\end{remark}

It will be shown important to take the above equivalence into mind. On the other hand, the following remark  shows the equivalent claims for singular matrices.
\begin{remark}[List of Equivalence of Singularity for a Matrix]
For a square matrix $\bA\in \real^{n\times n}$ with eigenpair $(\lambda, \bu)$, the following claims are equivalent:
\begin{itemize}
	\item $(\bA-\lambda\bI)$ is singular; 
	\item $(\bA-\lambda\bI)$ is not invertible;
	\item $(\bA-\lambda\bI)\bx = \bzero$ has nonzero $\bx\neq \bzero$ solutions, and $\bx=\bu$ is one of such solutions;
	\item $(\bA-\lambda\bI)$ has linearly dependent columns;
	\item $\det(\bA-\lambda\bI) = 0$; 
	\item $\dim(\nspace(\bA-\lambda\bI))>0$;
	\item Null space of $(\bA-\lambda\bI)$ is nontrivial;
	\item Columns of $(\bA-\lambda\bI)$ are linearly dependent;
	\item Rows of $(\bA-\lambda\bI)$ are linearly dependent;
	\item $(\bA-\lambda\bI)$ has rank $r<n$;
	\item Dimension of column space = dimension of row space = $r<n$;
	\item $(\bA-\lambda\bI)^\top(\bA-\lambda\bI)$ is symmetric semidefinite;
	\item $(\bA-\lambda\bI)$ has $r<n$ nonzero (positive) singular values;
	\item Zero is an eigenvalue of $(\bA-\lambda\bI)$.
\end{itemize}
\end{remark}


\index{Vector norm}
\index{Matrix norm}
\begin{definition}[Vector $\ell_2$-Norm]\label{definition:vec_l2_norm}
For a vector $\bx\in\real^n$, the \textbf{$\ell_2$ vector norm} is defined as $\norm{\bx}_2 = \sqrt{x_1^2+x_2^2+\ldots+x_n^2}$.
\end{definition}

For a matrix $\bA\in\real^{m\times n}$, we  define the (matrix) Frobenius norm as follows.
\begin{definition}[Matrix Frobenius Norm\index{Frobenius norm}]\label{definition:frobernius-in-svd}
The \textbf{Frobenius norm} of a matrix $\bA\in \real^{m\times n}$ is defined as 
$$
\norm{\bA}_F = \sqrt{\sum_{i=1,j=1}^{m,n} (a_{ij})^2}=\sqrt{\trace(\bA\bA^\top)}=\sqrt{\trace(\bA^\top\bA)} = \sqrt{\sigma_1^2+\sigma_2^2+\ldots+\sigma_r^2}, 
$$
where $\sigma_1, \sigma_2, \ldots, \sigma_r$ are nonzero singular values of $\bA$.
\end{definition}

The spectral norm is defined as follows.
\begin{definition}[Matrix Spectral Norm]\label{definition:spectral_norm}
The \textbf{spectral norm} of a matrix $\bA\in \real^{m\times n}$ is defined as 
$$
\norm{\bA}_2 = \mathop{\max}_{\bx\neq\bzero} \frac{\norm{\bA\bx}_2}{\norm{\bx}_2}  =\mathop{\max}_{\bu\in \real^n: \norm{\bu}_2=1}  \norm{\bA\bx}_2 ,
$$
which is also the maximal singular value of $\bA$, i.e., $\norm{\bA}_2 = \sigma_1(\bA)$.
\end{definition}

We note that the Frobenius norm serves as the matrix counterpart of vector $\ell_2$-norm.
For simplicity, we do not give the full subscript of the norm for the vector $\ell_2$-norm or Frobenius norm when it is clear from the context which one we are referring to: $\norm{\bA}=\norm{\bA}_F$ and $\norm{\bx}=\norm{\bx}_2$.
However, for the spectral norm, the subscript $\norm{\bA}_2$ should \textbf{not} be omitted.


\subsection*{Differentiability and Differential Calculus}

\begin{definition}[Directional Derivative, Partial Derivative]\label{definition:partial_deri}
Given a function $f$ defined over a set $\sS\subseteq \real^n$ and a nonzero vector $\bd\in\real^n$. Then the \textbf{directional derivative} of $f$ at $\bx$ w.r.t. the direction $\bd$ is given by, if the limit exists, 
$$
\mathop{\lim}_{t\rightarrow 0^+}
\frac{f(\bx+t\bd) - f(\bx)}{t}.
$$
And it is denoted by $f^\prime(\bx; \bd)$ or $D_{\bd}f(\bx)$. 
The directional derivative is sometimes called the \textbf{G\^ateaux derivative}.

For any $i\in\{1,2,\ldots,n\}$, the directional derivative at $\bx$ w.r.t. the direction of the $i$-th standard basis $\be_i$ is called the $i$-th \textbf{partial derivative} and is denoted by $\frac{\partial f}{\partial x_i} (\bx)$, $D_{\be_i}f(\bx)$, or $\partial_i f(\bx)$.
\end{definition}

If all the partial derivatives of a function $f$ exist at a point $\bx\in\real^n$, then the \textit{gradient} of $f$ at $\bx$ is defined as the column vector containing all the partial derivatives:
$$
\nabla f(\bx)=
\begin{bmatrix}
\frac{\partial f}{\partial x_1} (\bx)\\
\frac{\partial f}{\partial x_2} (\bx)\\
\vdots \\
\frac{\partial f}{\partial x_n} (\bx)	
\end{bmatrix}
\in \real^n.
$$
A function $f$ defined over an open set $\sS\subseteq \real^n$ is called \textit{continuously differentiable} over $\sS$ if all the partial derivatives exist and are continuous on $\sS$.
In the setting of continuous differentiability, the directional derivative and gradient have the following relationship:
\begin{equation}
f^\prime(\bx; \bd) = \nabla f(\bx)^\top \bd, \gap \text{for all }\bx\in\sS \text{ and }\bd\in\real^n.
\end{equation} 
And in the setting of continuously differentiability, we also have
\begin{equation}
\mathop{\lim}_{\bd\rightarrow \bzero}
\frac{f(\bx+\bd) - f(\bx) - \nabla f(\bx)^\top \bd}{\norm{\bd}} = 0\gap 
\text{for all }\bx\in\sS,
\end{equation}
or 
\begin{equation}
f(\by) = f(\bx)+\nabla f(\bx)^\top (\by-\bx) + o(\norm{\by-\bx}),
\end{equation}
where $o(\cdot): \real_+\rightarrow \real$ is a one-dimensional function satisfying $\frac{o(t)}{t}\rightarrow 0$ as $t\rightarrow 0^+$.

The partial derivative $\frac{\partial f}{\partial x_i} (\bx)$ is also a real-valued function of $\bx\in\sS$ that can be partially differentiated. The $j$-th partial derivative of $\frac{\partial f}{\partial x_i} (\bx)$ is defined as 
$$
\frac{\partial^2 f}{\partial x_j\partial x_i} (\bx)=
\frac{\partial \left(\frac{\partial f}{\partial x_i} (\bx)\right)}{\partial x_j} (\bx).
$$
This is called the ($j,i$)-th \textit{second-order partial derivative} of function $f$.
A function $f$ defined over an open set $\sS\subseteq$ is called \textit{twice continuously differentiable} over $\sS$ if all the second-order partial derivatives exist and are continuous over $\sS$. In the setting of twice continuously differentiability, the second-order partial derivative are symmetric:
$$
\frac{\partial^2 f}{\partial x_j\partial x_i} (\bx)=
\frac{\partial^2 f}{\partial x_i\partial x_j} (\bx).
$$
The \textit{Hessian} of the function $f$ at a point $\bx\in\sS$ is defined as the symmetric $n\times n$ matrix 
$$
\nabla^2f(\bx)=
\begin{bmatrix}
\frac{\partial^2 f}{\partial x_1^2} (\bx) & 
\frac{\partial^2 f}{\partial x_1\partial x_2} (\bx) & \ldots &
\frac{\partial^2 f}{\partial x_1\partial x_n} (\bx)\\
\frac{\partial^2 f}{\partial x_2\partial x_1} (\bx) & 
\frac{\partial^2 f}{\partial x_2\partial x_2} (\bx) & \ldots &
\frac{\partial^2 f}{\partial x_2\partial x_n} (\bx)\\
\vdots & 
\vdots & \ddots &
\vdots\\
\frac{\partial^2 f}{\partial x_n\partial x_1} (\bx) & 
\frac{\partial^2 f}{\partial x_n\partial x_2} (\bx) & \ldots &
\frac{\partial^2 f}{\partial x_n^2} (\bx)
\end{bmatrix}.
$$
We provide a simple proof of Taylor's expansion in Appendix~\ref{appendix:taylor-expansion} (p.~\pageref{appendix:taylor-expansion}) for one-dimensional functions. 
In the case of high-dimensional functions, we have the following two approximation results.
\begin{theorem}[Linear Approximation Theorem]\label{theorem:linear_approx}
Let $f(\bx):\sS\rightarrow \real$ be a twice continuously differentiable function over an open set $\sS\subseteq\real^n$, and given two points $\bx, \by$. Then there exists $\bx^\star\in[\bx,\by]$ such that 
$$
f(\by) = f(\bx)+ \nabla f(\bx)^\top (\by-\bx) + \frac{1}{2} (\by-\bx)^\top \nabla^2 f(\bx^\star) (\by-\bx).
$$ 
\end{theorem}

\begin{theorem}[Quadratic Approximation Theorem]
Let $f(\bx):\sS\rightarrow \real$ be a twice continuously differentiable function over an open set $\sS\subseteq\real^n$, and given two points $\bx, \by$. Then it follows that 
$$
f(\by) = f(\bx)+ \nabla f(\bx)^\top (\by-\bx) + \frac{1}{2} (\by-\bx)^\top \nabla^2 f(\bx) (\by-\bx)
+
o(\norm{\by-\bx}^2).
$$ 
\end{theorem}

%% file: chapter-GD.tex
\newpage
\clearchapter{Gradient Descent}\label{chapter:gradient-descent} 
\begingroup
\hypersetup{linkcolor=winestain,
linktoc=page,  
}
\minitoc \newpage
\endgroup

\index{Gradient descent}
\section{Gradient Descent}\label{section:gradient-descent-all} 
\lettrine{\color{caligraphcolor}G}
The \textit{gradient descent} (GD) method is employed to find the minimum of a differentiable, convex or non-convex function, commonly referred to as the ``cost" or ``loss" function (also known as the ``objective" function).
It stands out as one of the most popular algorithms to perform optimization and by far the
most common way to optimize machine learning, deep learning, and various optimization problems. 
And this is particularly true for optimizing neural networks. 
In the context of machine learning, the cost function measures the difference between the predicted output of a model and the actual output.
The neural networks or machine learning in general find the set of parameters $\bx\in \real^d$ (also known as weights) 
in order to optimize an objective function $L(\bx)$. 
The gradient descent aims to find a sequence of parameters:
\begin{equation}
\bx_1, \bx_2, \ldots, \bx_T,
\end{equation}
such that as $T\rightarrow \infty$, the objective function $L(\bx_T)$ attains the optimal minimum value.
At each iteration $t$, a step $\Delta \bx_t$ is applied to update the parameters. Denoting the parameters at the $t$-th iteration by $\bx_t$. 
Then the update rule becomes 
\begin{equation}
\bx_{t+1} = \bx_t + \Delta \bx_t.
\end{equation}

\index{Learning rate}
\index{Strict SGD}
\index{Mini-batch SGD}
\index{Local minima}
The most straightforward gradient descents is the \textit{vanilla update}: the parameters move in the opposite direction of the gradient, which finds the steepest descent direction since the gradients are orthogonal to level curves (also known as  level surfaces, see Lemma~\ref{lemm:direction-gradients}): 
\begin{equation}\label{equation:gd-equaa-gene}
\Delta \bx_{t} = -\eta \bg_t= -\eta \frac{\partial L(\bx_t)}{\partial \bx_t} = -\eta \nabla L(\bx_t),
\end{equation}
where the positive value $\eta$ denotes the \textit{learning rate} and depends on specific problems, and $\bg_t=\frac{\partial L(\bx^t)}{\partial \bx^t} \in 
\real^d$ represents the gradient of the parameters.
The learning rate $\eta$ controls how large of a step to take in the direction of negative gradient so that we can reach a (local) minimum.
While if we follow the negative gradient of a single sample or a batch of samples iteratively, the local estimate of the direction can be obtained and is known as the \textit{stochastic gradient descent} (SGD) \citep{robbins1951stochastic}.
The SGD can be categorized  into two types:
\begin{itemize}
\item \textbf{The strict SGD:} Computes the gradient for only one randomly chosen data point in each iteration.
\item \textbf{The mini-batch SGD:} Represents a compromise between GD and strict SGD, using a subset (mini-batch) of the dataset to compute the gradient.
\end{itemize}
The SGD method is particular useful when the number of training entries (i.e., the data used for updating the model, while the data used for final evaluation is called the test entries or test data) are substantial, resulting in that the gradients from different input samples may cancel out and the final update is small.
In the SGD framework, the objective function is stochastic, composed of a sum of subfunctions evaluated at different subsamples of the data. However, the drawback of the vanilla update (both GD and SGD) lies in its susceptibility to getting trapped in local minima \citep{rutishauser1959theory}.

For a small step size, gradient descent makes a monotonic improvement at every iteration, ensuring convergence, albeit to a local minimum. However, the speed of the vanilla gradient descent method is generally slow, and it can exhibit an exponential rate in case of poor curvature conditions. 
While choosing a rate higher than this  may lead to divergence in terms of the objective function. 
Determining an optimal learning rate (whether global or per-dimension) becomes more of an art than science for many problems. 
Previous work has been done to alleviate the need for selecting a global learning rate \citep{zeiler2012adadelta}, while it remains sensitive to other hyper-parameters.

In the following sections, we embark on a multifaceted exploration of the gradient descent method, unraveling its intricacies and adaptations through distinct lenses.
This comprehensive journey is designed to foster a nuanced comprehension of the algorithms, elucidating its various formulations, challenges, and contextual applications.

\section{Gradient Descent by Calculus}


An intuitive analogy for gradient descent is to envision a river's course flowing from a mountaintop.
The objective of gradient descent aligns with the river's goal—to descend to the lowest point at the foothill from the mountain's summit.

To restate the problem, the objective function is $L(\bx)$, where $\bx$ is a $d$-dimensional input variable; our goal is to use an algorithm to get the (local) minimum of $L(\bx)$.
To provide further precision, let's think about what happens when we move the ball a small amount $\Delta x_1$ in the $x_1$ direction, a small amount $\Delta x_2$ in the $x_2$ direction, \ldots, and a small amount $\Delta x_d$ in the $x_d$ direction. Calculus informs us of the variation in the objective function  $L(\bx)$ as follows:
$$
\Delta L(\bx) \approx \frac{\partial L}{\partial x_1}\Delta x_1 + \frac{\partial L}{\partial x_2}\Delta x_2 + \ldots + \frac{\partial L}{\partial x_d}\Delta x_d.
$$
In this sense, the challenge is to determine $\Delta x_1$, \ldots, $\Delta x_d$ in a manner that induces a negative change in $\Delta L(\bx)$, i.e., we'll make the objective function decrease, aiming for minimization.
Let $\Delta \bx=[\Delta x_1,\Delta x_2, \ldots , \Delta x_d]^\top$ represent the vector of  changes in $\bx$, and  $\nabla L(\bx) = \frac{\partial L(\bx)}{\partial \bx}=[\frac{\partial L}{\partial x_1},\frac{\partial L}{\partial x_2}, \ldots , \frac{\partial L}{\partial x_d}]^\top$ denote the gradient vector of $L(\bx)$ \footnote{Note the difference between $\Delta L(\bx)$ and $\nabla L(\bx)$.}. Then it follows that
$$
\Delta L(\bx) \approx \frac{\partial L}{\partial x_1}\Delta x_1 + \frac{\partial L}{\partial x_2}\Delta x_2 +\ldots  + \frac{\partial L}{\partial x_d}\Delta x_d = \nabla L(\bx)^\top \Delta \bx.
$$ 
In the  context of descent, our objective is to ensure that $\Delta L(\bx)$ is negative. This condition ensures that a step $\bx_{t+1} = \bx_t+\Delta \bx_t$ (from $t$-th iteration to $(t+1)$-th iteration) results in a decrease of the loss function $L(\bx_{t+1}) = L(\bx_t) + \Delta L(\bx_t)$, given that $\Delta L(\bx_t) \leq 0$.
It can be demonstrated that if the update step is defined as $\Delta \bx_t=-\eta \nabla L(\bx_t)$, where $\eta$ is the learning rate, the following relationship holds:
$$
\Delta  L(\bx_t) \approx -\eta \nabla L(\bx_t)^\top\nabla L(\bx_t) = -\eta\norm{\nabla L}_2^2 \leq 0. 
$$
To be precise,  $\Delta  L(\bx_t) < 0$ in the above equation; otherwise, we would have reached the optimal point with zero gradients.
This analysis confirms the validity of gradient descent. The update rule for the next parameter $\bx_{t+1}$ is given by:
$$
\bx_{t+1} = \bx_t - \eta \nabla L(\bx_t).
$$
This update rule will make the objective function drop to the minimum point steadily in a convex setting or local minima in a non-convex setting.

\index{Descent condition}
\index{Descent direction}
\index{Search direction}
\index{Taylor's formula}
\begin{remark}[Descent Condition]\label{remark:descent_condition}
In the above construction, we define $\Delta \bx_t = -\eta \nabla L(\bx_t)$, where $-\nabla L(\bx_t)$ is the \textit{descent direction} such that $\Delta L \approx -\eta \nabla L(\bx_t)^\top \nabla L(\bx_t) < 0$ (assuming $\nabla L(\bx_t) \neq \bzero $). More generally, any \textit{search direction} $\bd_t \in \real^d{\setminus}\{\bzero\}$ that satisfies the \textit{descent condition} can be chosen as the descent direction:
$$
\frac{d L(\bx_t + \eta \bd_t)}{d \eta}\bigg|_{\eta=0} = \nabla L(\bx_t)^\top \bd_t <0.
$$
In other words, according to  Taylor's formula (Appendix~\ref{appendix:taylor-expansion}, p.~\pageref{appendix:taylor-expansion}),
$$
L(\bx_t+\eta\bd_t) \approx L(\bx_t) + \eta \nabla L(\bx_t)^\top\bd_t 
$$
implies $L(\bx_t+\eta\bd_t)  < L(\bx_t)$ when $\eta$ is sufficiently small. When $\bd_t = -\nabla L(\bx_t)$, the descent direction is known as the \textit{steepest descent direction}. When the learning rate $\eta$ is not fixed and decided by exact line search, the method is called \textit{steepest descent method} (see Section~\ref{section:quadratic-in-steepestdescent}, p.~\pageref{section:quadratic-in-steepestdescent}).
\end{remark}

\index{Convex functions}
\index{Jensen's inequality}
\subsection*{Gradient Descent in Convex Problems}
We further consider the application of gradient descent in convex problems. The notion of convexity for a function is defined as follows.
\begin{definition}[Convex Functions]
A function $f: \sS \rightarrow \real$ defined on a convex set $\sS \subseteq \real^n$ is called convex if 
$$
f(\lambda \bx +(1-\lambda)\by) \leq \lambda f(\bx) +(1-\lambda) f(\by),
$$
where $\bx,\by\in \sS$, and $\lambda\in[0,1]$.
And the function $f$ is called strictly convex if 
$$
f(\lambda \bx +(1-\lambda)\by) < \lambda f(\bx) +(1-\lambda) f(\by),
$$
where $\bx\neq \by\in \sS$, and $\lambda\in(0,1)$.
\end{definition}

There are several inequalities in convex functions.
\begin{lemma}[Inequalities in Convex Functions]
A convex function satisfies the following inequalities \citep{beck2017first}.
\paragraph{Jensen's inequality.}
Let $f: \sS \rightarrow \real$ be a convex function defined on a convex set $\sS\subseteq \real^n$. Then, given any $\bx_1, \bx_2, \ldots, \bx_k\in\sS$ and $\blambda\in\Delta_k$,  it follows that 
$$
f\left(\sum_{i=1}^{k} \lambda_i\bx_i\right) \leq \sum_{i=1}^{k}\lambda_if(\bx_i).
$$
\paragraph{Gradient inequality.}
Suppose further $f$ is continuously differentiable. Then, given any $\bx,\by\in\sS$, $f$ is convex over $\sS$ if and only if 
$$
f(\bx) -f(\by) \leq  \nabla f(\bx)^\top (\bx-\by).
$$
Given any $\bx\neq \by\in\sS$, $f$ is strictly convex over $\sS$ if and only if 
$$
f(\bx) -f(\by) <  \nabla f(\bx)^\top (\bx-\by).
$$
\paragraph{Monotonicity of the gradient.} Suppose again $f$ is continuously differentiable. Then, given any $\bx,\by\in\sS$, $f$ is convex over $\sS$ if and only if 
$$
(\nabla f(\bx) - \nabla f(\by))^\top (\bx-\by)\geq 0.
$$
\end{lemma}

If the objective function $L(\bx)$ is (continuously differentiable) convex, then the relationship $\nabla L(\bx_t)^\top(\bx_{t+1}-\bx_{t})\geq 0$ implies $L(\bx_{t+1}) \geq L(\bx_t)$. This can be derived from the gradient inequality of a continuously differentiable convex function, i.e., $L(\bx_{t+1})- L(\bx_{t})\geq L(\bx_t)^\top(\bx_{t+1}-\bx_t)$. 

In this sense, to ensure a reduction in the objective function, it is imperative to ensure  $\nabla L(\bx_t)^\top (\bx_{t+1}-\bx_{t})\leq 0$. 
In the context of  gradient descent, the choice of $\Delta \bx_t = \bx_{t+1}-\bx_t$ aligns with the negative gradient $-\nabla L(\bx_t)$. However, there are many other descent methods, such as \textit{steepest descent}, \textit{normalized steepest descent}, \textit{Newton step}, and so on. The main idea of these methods are undertaken to ensure $\nabla L(\bx_t)^\top(\bx_{t+1}-\bx_{t})= \nabla L(\bx_t)^\top \Delta \bx_t \leq 0$ if the objective function is convex.

\index{Greedy search}
\section{Gradient Descent by Greedy Search}\label{section:als-gradie-descent-taylor}

We now consider  the greedy search such that $\bx_{t+1}    \leftarrow \mathop{\arg \min}_{\bx_t} L(\bx_t)$. Suppose we want to approximate $\bx_{t+1}$ by a linear update on $\bx_t$, the expression takes the following form:
$$
\bx_{t+1} = \bx_t + \eta \bv.
$$
The problem now revolves around finding a solution for $\bv$ to minimize the expression:
$$
\bv=\mathop{\arg \min}_{\bv} L(\bx_{t} + \eta \bv) .
$$
By Taylor's formula (Appendix~\ref{appendix:taylor-expansion}, p.~\pageref{appendix:taylor-expansion}), $L(\bx_t  + \eta \bv)$ can be approximated by 
$$
L(\bx_t + \eta \bv) \approx L(\bx_t ) + \eta \bv^\top \nabla  L(\bx_t ),
$$
when $\eta$ is sufficiently small. Considering the condition $\norm{\bv}=1$ for positive $\eta$, we formulate the descent search as:
$$
\bv=\mathop{\arg \min}_{\norm{\bv}=1} L(\bx_{t} + \eta \bv) \approx\mathop{\arg \min}_{\norm{\bv}=1}
\left\{L(\bx_{t} ) + \eta \bv^\top \nabla  L(\bx_{t} )\right\}.
$$
This is known as the \textit{greedy search}. This process leads to the optimal $\bv$ determined by
$$
\bv = -\frac{\nabla L(\bx_{t} )}{\norm{\nabla L(\bx_{t} )}},
$$
i.e., $\bv$ lies in the opposite direction of $\nabla L(\bx_{t} )$. Consequently, the update for $\bx_{t+1}$ is reasonably expressed as:
$$
\bx_{t+1} =\bx_{t} + \eta \bv = \bx_{t} - \eta \frac{\nabla L(\bx_{t})}{\norm{\nabla L(\bx_{t} )}},
$$
which is usually called the \textit{gradient descent}, as aforementioned. If we further absorb the denominator into the step size $\eta$, the gradient descent can be simplified to the trivial way:
$$
\bx_{t+1} = \bx_{t} - \eta {\nabla L(\bx_{t})}.
$$

\section{Geometrical Interpretation of Gradient Descent} 
\begin{lemma}[Direction of Gradients]\label{lemm:direction-gradients}
An important fact is that gradients are orthogonal to level curves (also known as level surfaces).
\end{lemma}
\begin{proof}[of Lemma~\ref{lemm:direction-gradients}]
This is equivalent to proving that the gradient is orthogonal to the tangent of the level curve. For simplicity, let's first look at the two-dimensional case. Suppose the level curve has the form $f(x,y)=c$. This implicitly establishes a relationship between $x$ and $y$ such that $y=y(x)$, where $y$ can be thought of as a function of $x$. Therefore, the level curve can be written as 
$$
f(x, y(x)) = c.
$$
The chain rule indicates
$$
\frac{\partial f}{\partial x} \underbrace{\frac{dx}{dx}}_{=1} + \frac{\partial f}{\partial y} \frac{dy}{dx}=0.
$$
Therefore, the gradient is perpendicular to the tangent:
$$
\left\langle \frac{\partial f}{\partial x}, \frac{\partial f}{\partial y}\right\rangle
\cdot 
\left\langle \frac{dx}{dx}, \frac{dy}{dx}\right\rangle=0.
$$
Let's now treat the problem in full generality, consider the level curve of a vector $\bx\in \real^n$: $f(\bx) = f(x_1, x_2, \ldots, x_n)=c$. Each variable $x_i$ can be regarded as a function of a variable $t$ on the level curve $f(\bx)=c$: $f(x_1(t), x_2(t), \ldots, x_n(t))=c$. Differentiate the equation with respect to $t$ by chain rule:
$$
\frac{\partial f}{\partial x_1} \frac{dx_1}{dt} + \frac{\partial f}{\partial x_2} \frac{dx_2}{dt}
+\ldots + \frac{\partial f}{\partial x_n} \frac{dx_n}{dt}
=0.
$$
Therefore, the gradients is perpendicular to the tangent in $n$-dimensional case:
$$
\left\langle \frac{\partial f}{\partial x_1}, \frac{\partial f}{\partial x_2}, \ldots, \frac{\partial f}{\partial x_n}\right\rangle
\cdot 
\left\langle \frac{dx_1}{dt}, \frac{dx_2}{dt}, \ldots \frac{dx_n}{dt}\right\rangle=0.
$$
This completes the proof.
\end{proof}
The lemma above offers a profound geometrical interpretation of gradient descent. 
In the pursuit of minimizing a convex function $L(\bx)$, the gradient descent strategically navigates in the direction opposite to the gradient that can decrease the loss. Figure~\ref{fig:alsgd-geometrical} depicts a two-dimensional scenario, where $-\nabla L(\bx)$ pushes the loss to decrease for the convex function $L(\bx)$. 

\begin{figure}[h]
\centering  
\vspace{-0.35cm} 
\subfigtopskip=2pt 
\subfigbottomskip=2pt 
\subfigcapskip=-5pt 
\subfigure[A two-dimensional convex function $L(\bx)$.]{\label{fig:alsgd1}
	\includegraphics[width=0.47\linewidth]{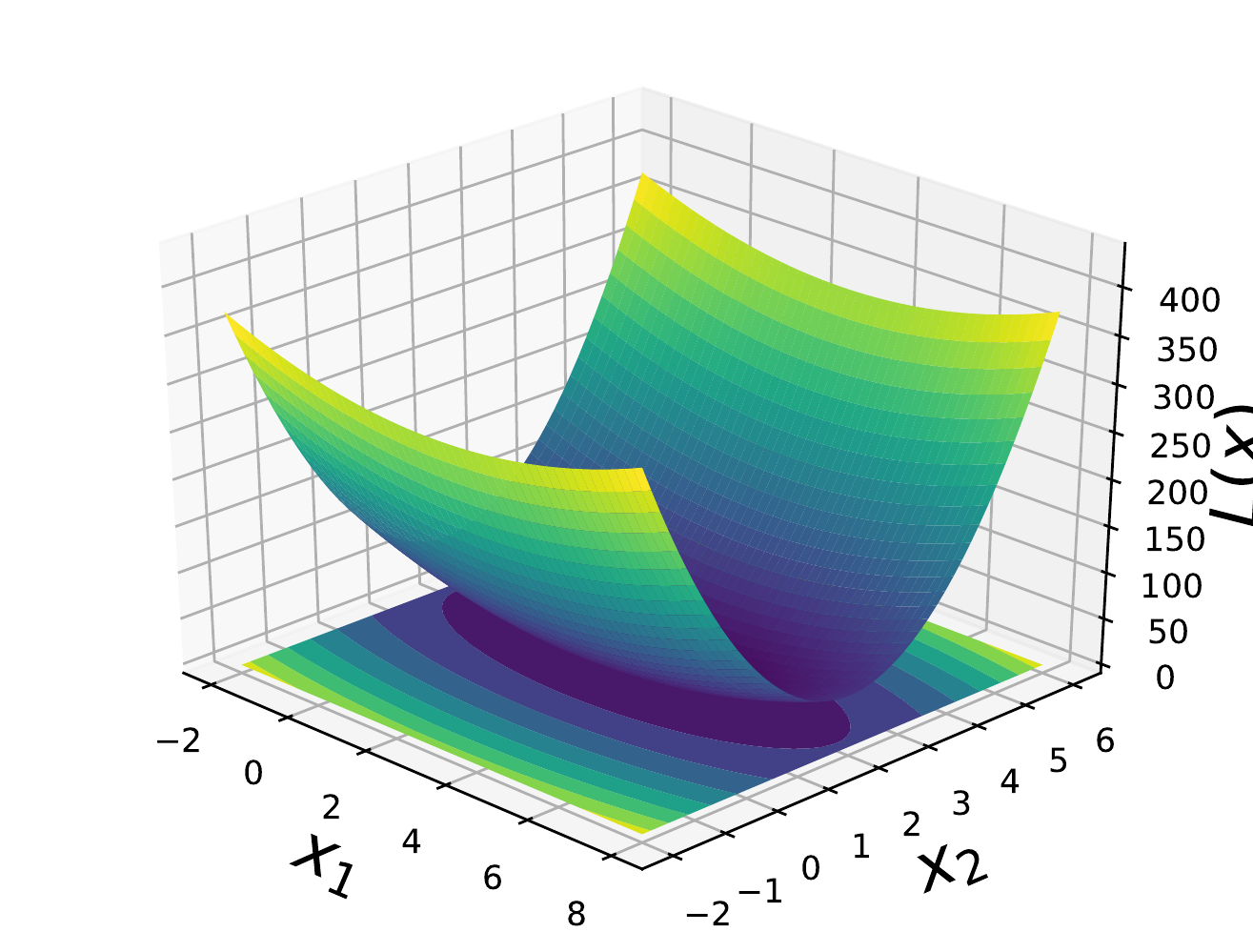}}
\subfigure[$L(\bx)=c$ is a constant.]{\label{fig:alsgd2}
	\includegraphics[width=0.44\linewidth]{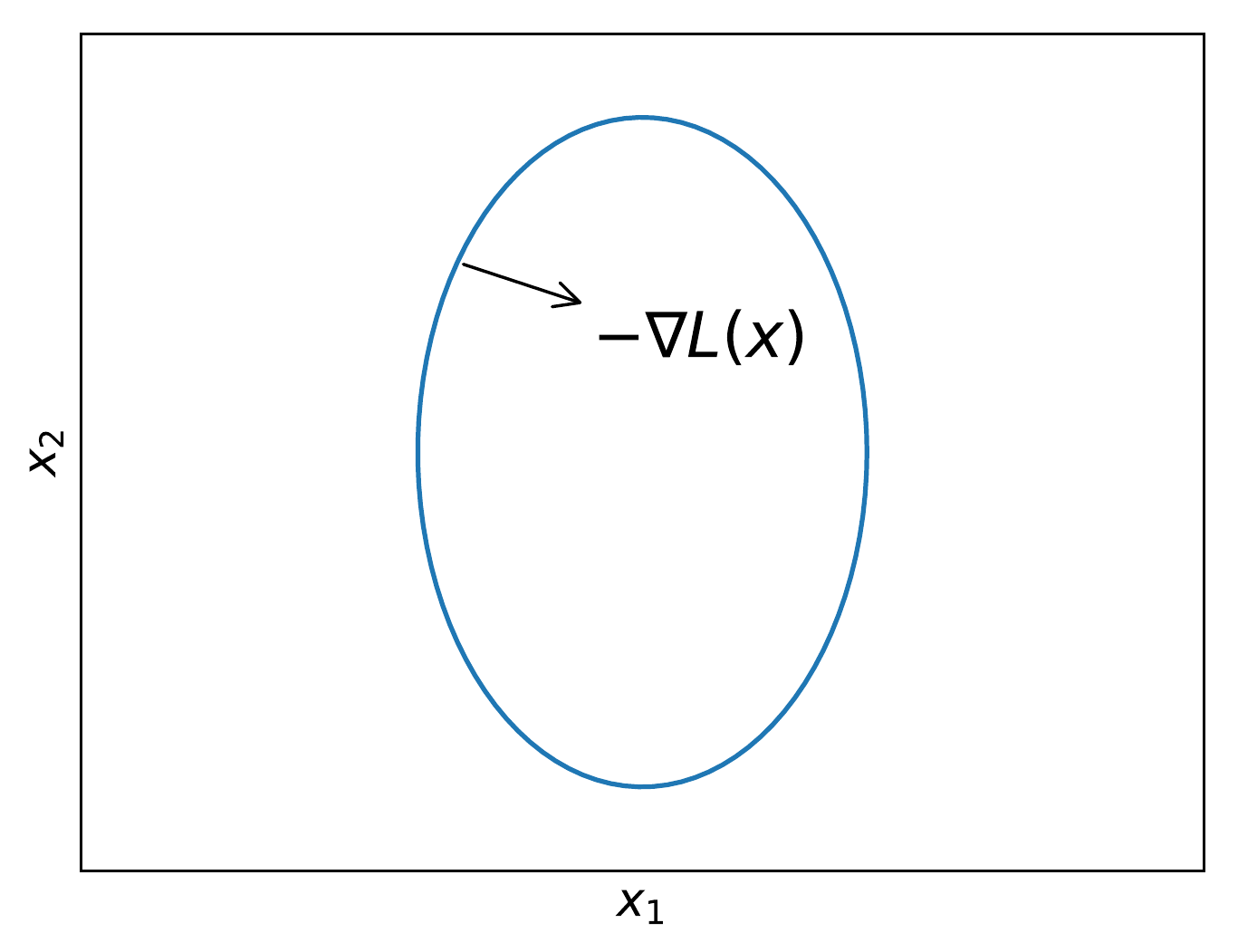}}
\caption{Figure~\ref{fig:alsgd1} shows a convex function surface plot and its contour plot (\textcolor{mylightbluetext}{blue}=low, \textcolor{mydarkyellow}{yellow}=high), where the upper graph is the surface plot, and the lower one is the projection of it (i.e., contour). Figure~\ref{fig:alsgd2}: $-\nabla L(\bx)$ pushes the loss to decrease for the convex function $L(\bx)$.}
\label{fig:alsgd-geometrical}
\end{figure}

\index{Regularization}
\section{Regularization: A Geometrical Interpretation}
\begin{figure}[h]
\centering
\includegraphics[width=0.95\textwidth]{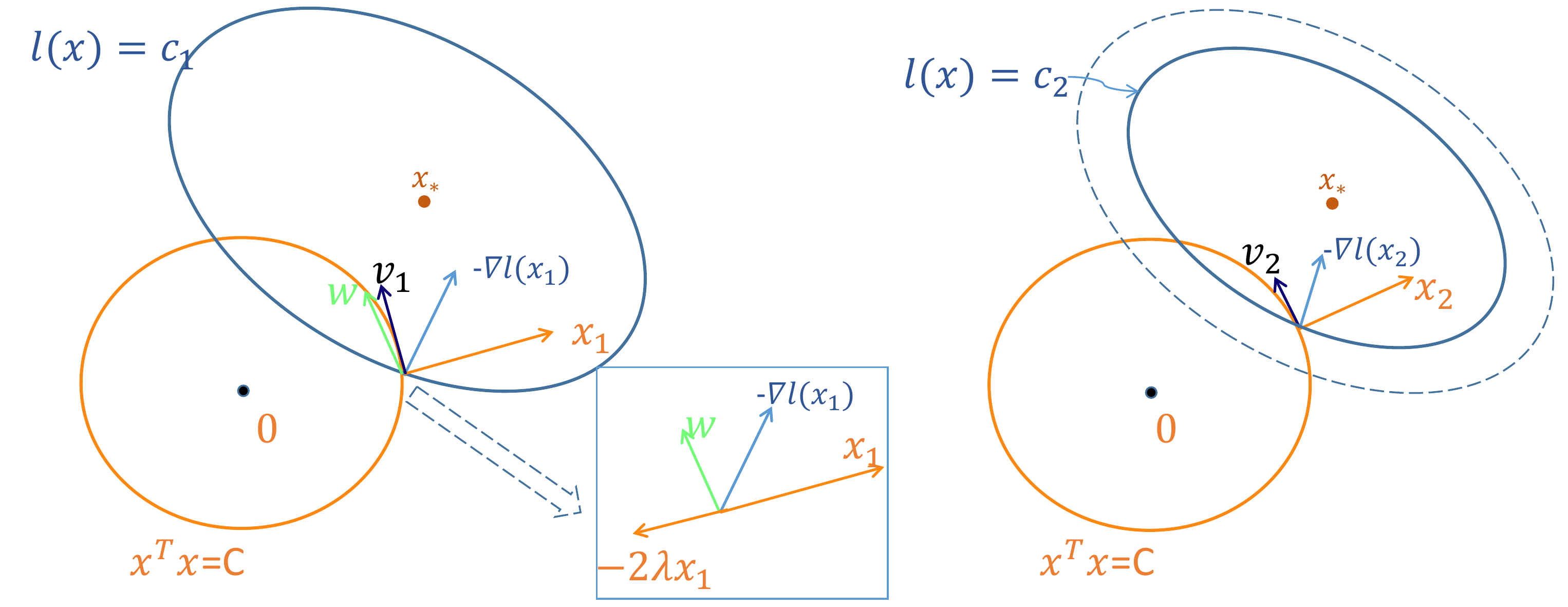}
\caption{Constrained gradient descent with $\bx^\top\bx\leq C$. The \textcolor{mydarkgreen}{green} vector $\bw$ is the projection of $\bv_1$ into $\bx^\top\bx\leq C$ where $\bv_1$ is the component of $-\nabla l(\bx)$ perpendicular to $\bx_1$. The right picture is the next step after the update in the left picture. $\bx_\star$ denotes the optimal solution of \{$\min l(\bx)$\}.}
\label{fig:alsgd3}
\end{figure}
The gradient descent also unveils the geometrical significance of regularization. 
To avoid confusion, we denote the loss function without regularization by $l(\bz)$ and the loss with regularization by $L(\bx) = \l(\bx)+\lambda_x \norm{\bx}^2$, where $l(\bx): \real^d \rightarrow \real$ (this notation is exclusive to this section). When minimizing $l(\bx)$, the descent method will search in $\real^d$ for a solution. 
However, in machine learning, an exhaustive search across the entire space may lead to overfitting. A partial remedy involves searching within a subset of the vector space, such as searching in $\bx^\top\bx < C$ for some constant $C$. That is,
$$
\mathop{\arg\min}_{\bx} \,\, l(\bx), \gap  \text{s.t.,} \gap \bx^\top\bx\leq C.
$$
We will see that this constrained search helps prevent overfitting by introducing regularization through the addition of a penalty term in the optimization process.
In the previous discussion, a trivial gradient descent approach proceeds in the direction of $-\nabla l(\bx)$,  updating $\bx$ by $\bx\leftarrow \bx-\eta \nabla l(\bx)$ for a small step size $\eta$. 
When the level curve is $l(\bx)=c_1$ and the descent approach is situated at $\bx=\bx_1$, where $\bx_1$ is the intersection of $\bx^\top\bx=C$ and $l(\bx)=c_1$, the descent direction $-\nabla l(\bx_1)$ will be perpendicular to the level curve of $l(\bx_1)=c_1$, as shown in the left picture of Figure~\ref{fig:alsgd3}. 
However, if we further restrict that the optimal value can only be in the subspace $\bx^\top\bx\leq C$, the trivial descent direction $-\nabla l(\bx_1)$ will lead $\bx_2=\bx_1-\eta \nabla l(\bx_1)$ outside of $\bx^\top\bx\leq C$. 
To address this,  the step $-\nabla l(\bx_1)$ is decomposed  into 
$$
-\nabla l(\bx_1) = a\bx_1 + \bv_1,
$$ 
where $a\bx_1$ is the component perpendicular to the curve of $\bx^\top\bx=C$, and $\bv_1$ is the component parallel to the curve of $\bx^\top\bx=C$. Keeping only the step $\bv_1$, then the update 
$$
\bx_2 = \text{project}(\bx_1+\eta \bv_1) = \text{project}\left(\bx_1 + \eta
\underbrace{(-\nabla l(\bx_1) -a\bx_1)}_{\bv_1}\right)\footnote{where the project($\bx$) 
will project the vector $\bx$ to the closest point inside $\bx^\top\bx\leq C$. Notice here the direct update $\bx_2 = \bx_1+\eta \bv_1$ can still make $\bx_2$ outside the curve of $\bx^\top\bx\leq C$.}
$$ 
will lead to a smaller loss from $l(\bx_1)$ to $l(\bx_2)$ while still satisfying the prerequisite of $\bx^\top\bx\leq C$. 
This technique is known as the \textit{projection gradient descent}. It is not hard to see that the update $\bx_2 = \text{project}(\bx_1+\eta \bv_1)$ is equivalent to finding a vector $\bw$ (depicted in \textcolor{mydarkgreen}{green} vector in the left panel of Figure~\ref{fig:alsgd3}) such that $\bx_2=\bx_1+\bw$ lies inside the curve of $\bx^\top\bx\leq C$. Mathematically, the $\bw$ can be obtained as $-\nabla l(\bx_1) -2\lambda \bx_1$ for some $\lambda$, as shown in the middle panel of Figure~\ref{fig:alsgd3}. This aligns with the negative gradient of $L(\bx)=l(\bx)+\lambda\norm{\bx}^2$ such that 
$$
-\nabla L(\bx) = -\nabla l(\bx) - 2\lambda \bx,
$$
and 
$$
\begin{aligned}
\bw &= -\nabla L(\bx) \leadto \bx_2 &= \bx_1+ \bw =\bx_1 -  \nabla L(\bx).
\end{aligned}
$$
And in practice, a small step size $\eta$ can be applied to prevent crossing  the curve boundary of $\bx^\top\bx\leq C$:
$$
\bx_2  =\bx_1 -  \eta\nabla L(\bx).
$$

\index{Quadratic form}
\index{Fisher information matrix}
\index{Positive definite}
\index{Positive semidefinite}
\index{Symmetry}
\section{Quadratic Form in Gradient Descent}\label{section:quadratic_vanilla_GD}

We delve deeper into (vanilla) gradient descent applied to the simplest model, the convex quadratic function,
\begin{equation}\label{equation:quadratic-form-general-form}
L(\bx) = \frac{1}{2} \bx^\top \bA \bx - \bb^\top \bx + c, \gap \bx\in \real^d,
\end{equation}
where $\bA\in \real^{d\times d}$, $\bb \in \real^d$, and $c$ is a scalar constant. Though the quadratic form in Eq.~\eqref{equation:quadratic-form-general-form} is an extremely simple model, it is rich enough to approximate many other functions, e.g., the Fisher information matrix \citep{amari1998natural}, and capture key features of pathological curvature. The gradient of $L(\bx)$ at point $\bx$ is given by 
\begin{equation}\label{equation:unsymmetric_gd_gradient}
\nabla L(\bx) = \frac{1}{2} (\bA^\top +\bA) \bx - \bb.
\end{equation}
The unique minimum of the function is the solution of the linear system $\frac{1}{2} (\bA^\top +\bA) \bx=  \bb $:
\begin{equation}\label{equation:gd_solution_unsymmetric}
\bx_\star = 2(\bA^\top +\bA)^{-1}\bb.
\end{equation}
If $\bA$ is symmetric (for most of our discussions, we will restrict to symmetric $\bA$ or even \textit{positive definite}, see definition below), the equation reduces to 
\begin{equation}\label{equation:symmetric_gd_gradient}
\nabla L(\bx) = \bA \bx - \bb.
\end{equation}
Then the unique minimum of the function is the solution of the linear system $\bA\bx=\bb$~\footnote{This represents the \textit{first-order optimality condition} for local optima points. Note the proof of this condition for multivariate functions heavily relies on the first-order optimality conditions for one-dimensional functions, which is also known as the \textit{Fermat's theorem}. Refer to Exercise~\ref{problem:fist_opt}.}, where $\bA,\bb$ are known matrix or vector, $\bx$ is an unknown vector; and the optimal point of $\bx$ is thus given by 
$$
\bx_\star = \bA^{-1}\bb
$$
if $\bA$ is nonsingular.

\index{Fermat's theorem}
\begin{exercise}\label{problem:fist_opt}
\textbf{First-order optimality condition for local optima points.} 
Consider the \textit{Fermat's theorem}: for a one-dimensional function $g(\cdot)$ defined and differentiable over an interval ($a, b$), if a point $x^\star\in(a,b)$ is a local maximum or minimum, then $g^\prime(x^\star)=0$. 
Prove the first-order optimality conditions for multivariate functions based on this Fermat's theorem for one-dimensional functions.
That is, consider function $f: \sS\rightarrow \real$ as a function defined on a set $\sS\subseteq \real^n$. Suppose that $\bx^\star\in\text{int}(\sS)$, i.e., in the interior point of the set, is a local optimum point and that all the partial derivatives (Definition~\ref{definition:partial_deri}, p.~\pageref{definition:partial_deri}) of $f$ exist at $\bx^\star$. Then $\nabla f(\bx^\star)=\bzero$, i.e., the gradient vanishes at all local optimum points. (Note that, this optimality condition is a necessary condition; however, there could be vanished points which are not local maximum or minimum point.)
\end{exercise}

\paragraph{Symmetric matrices.} A symmetric matrix can be further categorized into positive definite, positive semidefinite, negative definite, negative semidefinite, and indefinite types as follows.
\begin{definition}[Positive Definite and Positive Semidefinite\index{Positive definite}\index{Positive semidefinite}]\label{definition:psd-pd-defini}
A matrix $\bA\in \real^{n\times n}$ is considered positive definite (PD) if $\bx^\top\bA\bx>0$ for all nonzero $\bx\in \real^n$.
And a matrix $\bA\in \real^{n\times n}$ is called positive semidefinite (PSD) if $\bx^\top\bA\bx \geq 0$ for all $\bx\in \real^n$. 
\footnote{
In this book a positive definite or a semidefinite matrix is always assumed to be symmetric, i.e., the notion of a positive definite matrix or semidefinite matrix is only interesting for symmetric matrices.
}
\footnote{Similarly, a complex matrix $\bA$ is said to be \textit{Hermitian positive definite} (HSD), if $\bA$ is Hermitian  and $\bz^\ast\bA\bz>0$ for all $\bz\in\complex^n$ with $\bz\neq \bzero$.}
\footnote{A symmetric matrix $\bA\in\real^{n\times n}$ is called \textit{negative definite} (ND) if $\bx^\top\bA\bx<0$ for all nonzero $\bx\in\real^n$; 
a symmetric matrix $\bA\in\real^{n\times n}$ is called \textit{negative semidefinite} (NSD) if $\bx^\top\bA\bx\leq 0$ for all $\bx\in\real^n$;
and a symmetric matrix $\bA\in\real^{n\times n}$ is called \textit{indefinite} (ID) if there exist $\bx$ and $\by\in\real^n$ such that $\bx^\top\bA\bx<0$ and $\by^\top\bA\by>0$.
}
\end{definition}

\begin{lemma}[Positive Definite Properties]
Given a negative definite matrix $\bA$, then $-\bA$ is a positive definite matrix; if $\bA$ is negative semidefintie matrix, then $-\bA$ is positive semidefinite.
Positive definite, positive semidefinite, and indefinite matrices admit the following properties.
\paragraph{Eigenvalue.} A matrix  is positive definite if and only if it has exclusively \textit{positive eigenvalues}. Similarly, a matrix  is positive semidefinite if and only if it exhibits solely \textit{nonnegative eigenvalues}. And a matrix is indefinite if and only if it possesses at least one positive eigenvalue and at least one negative eigenvalue.

\paragraph{Diagonals.} The diagonal elements of a positive definite matrix are all \textit{positive}. And similarly, the diagonal elements of a positive semidefinite matrix are all \textit{nonnegative}. And the diagonal elements of a indefinite matrix contains at least one positive diagonal and at least one negative diagonal.
\end{lemma}
The proof for the lemma is trivial and can be found in \citet{lu2022matrix}

\begin{figure}[htp]
\centering  
\vspace{-0.35cm} 
\subfigtopskip=2pt 
\subfigbottomskip=2pt 
\subfigcapskip=-5pt 
\subfigure[Positive definite matrix: $\bA = \begin{bmatrix}
	200 & 0 \\ 0 & 200
\end{bmatrix}$.]{\label{fig:quadratic_PD}
	\includegraphics[width=0.485\linewidth]{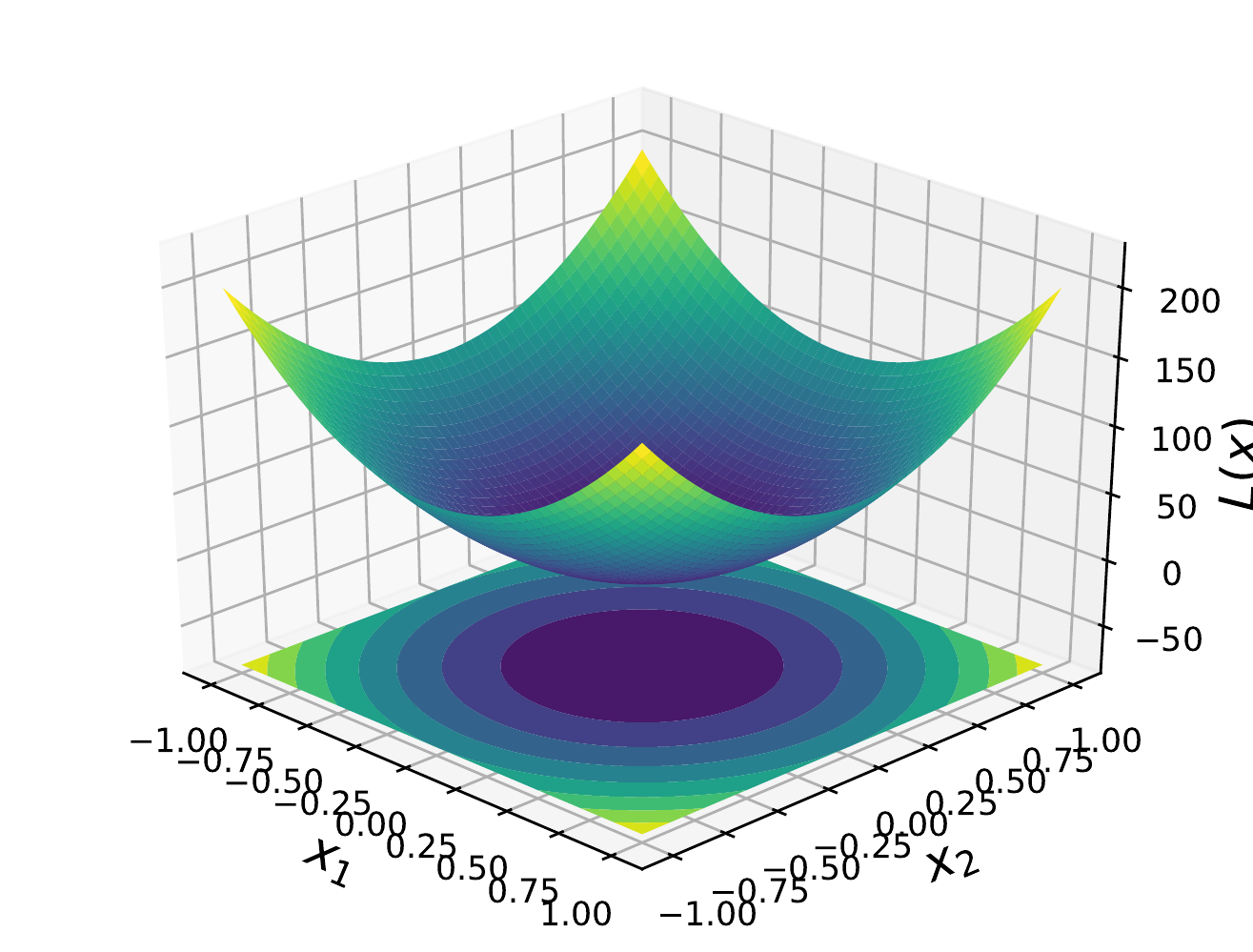}}
\subfigure[Negative definite matrix: $\bA = \begin{bmatrix}
	-200 & 0 \\ 0 & -200
\end{bmatrix}$.]{\label{fig:quadratic_ND}
	\includegraphics[width=0.485\linewidth]{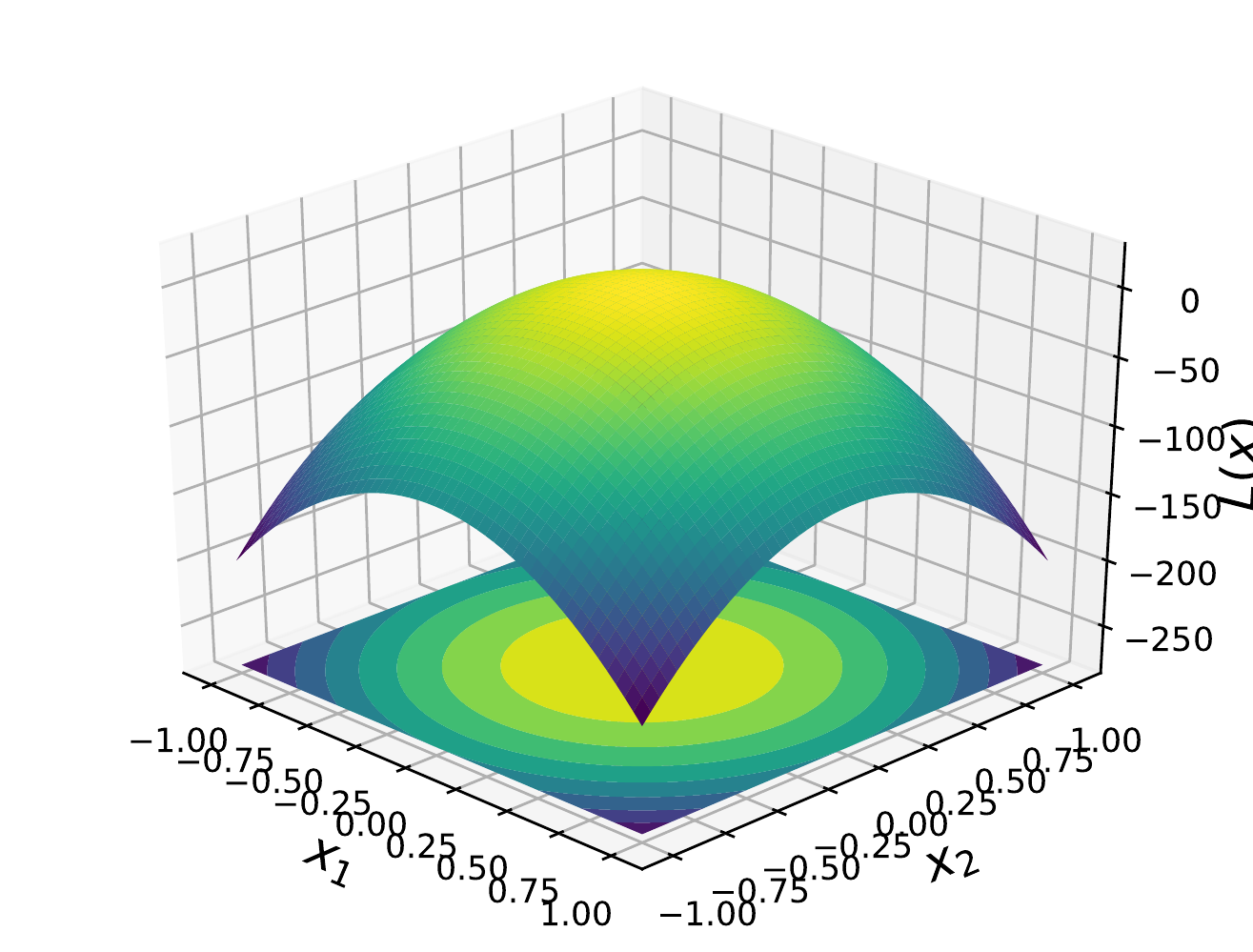}}
\subfigure[Semidefinite matrix: $\bA = \begin{bmatrix}
	200 & 0 \\ 0 & 0
\end{bmatrix}$. A line runs through the bottom of the valley is the set of solutions.]{\label{fig:quadratic_singular}
	\includegraphics[width=0.485\linewidth]{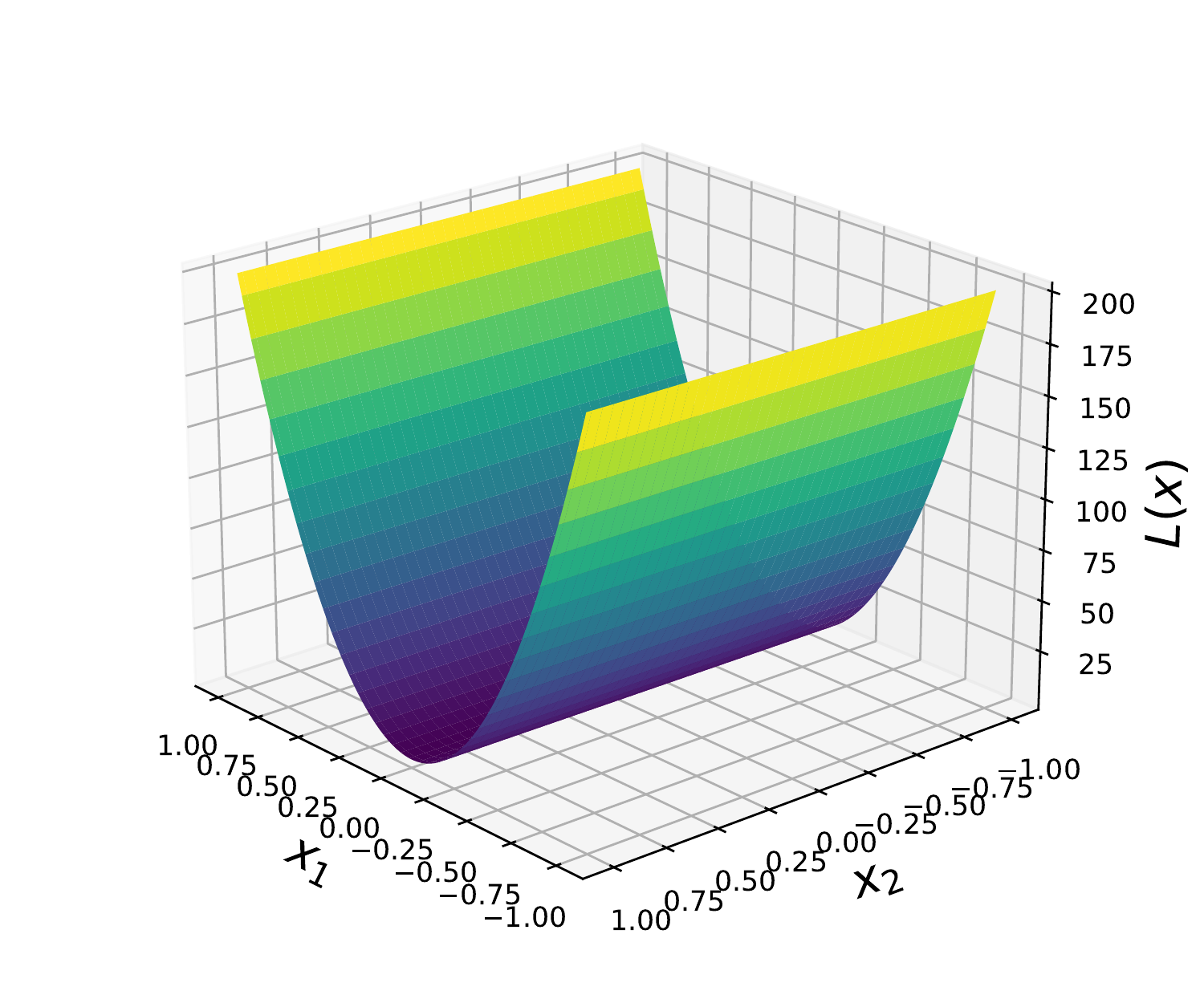}}
\subfigure[Indefinte matrix: $\bA = \begin{bmatrix}
	200 & 0 \\ 0 & -200
\end{bmatrix}$.]{\label{fig:quadratic_saddle}
	\includegraphics[width=0.485\linewidth]{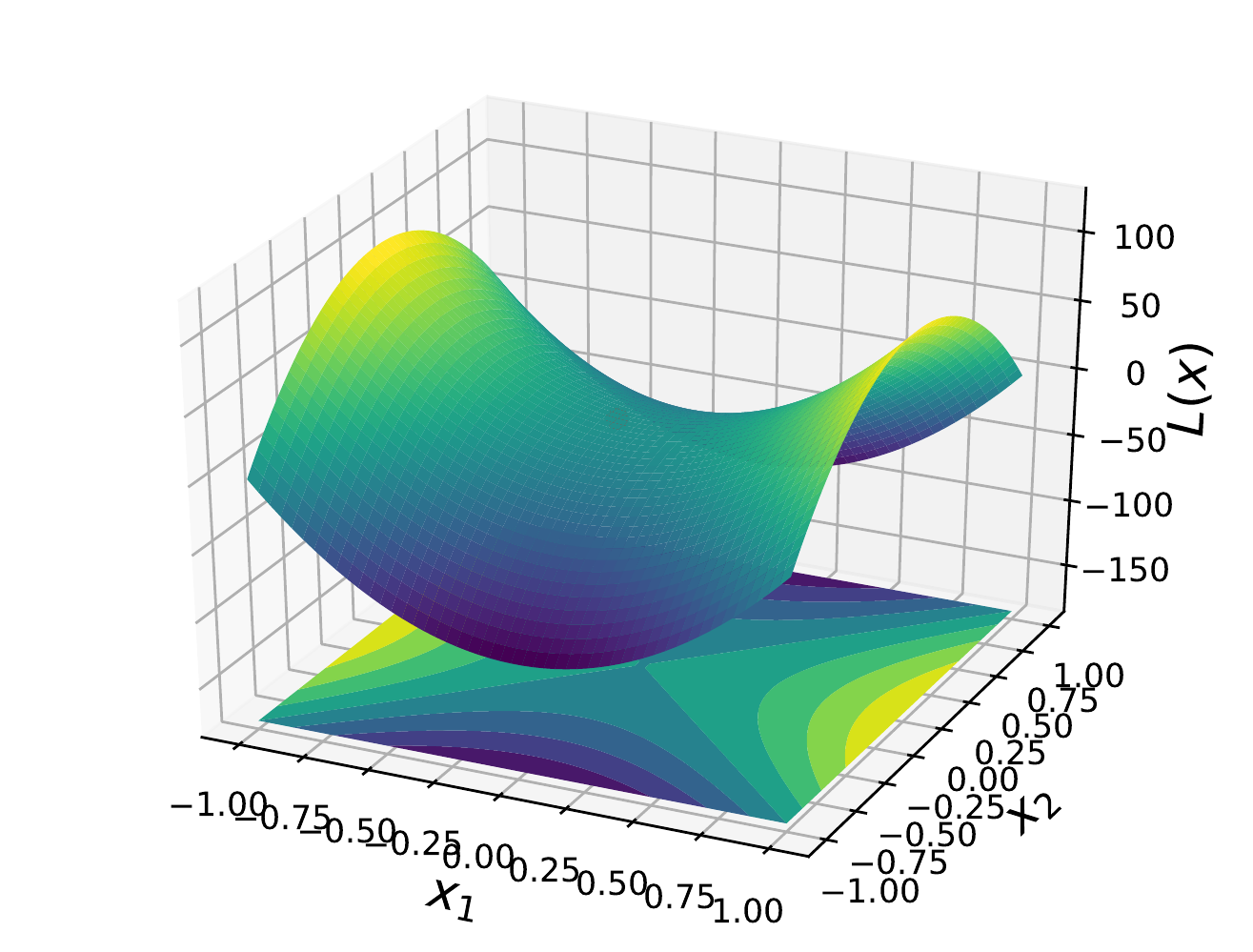}}
\caption{Loss surface for different quadratic forms.}
\label{fig:different_quadratics}
\end{figure}
For different types of matrix $\bA$, the loss surface of $L(\bx)$ will be different, as illustrated in Figure~\ref{fig:different_quadratics}. When $\bA$ is positive definite, the surface forms a convex bowl; when $\bA$ is negative definite, on the contrary, the surface becomes a concave bowl. $\bA$ also could be singular, in which case $\bA\bx-\bb=\bzero$ has more than one solution, and the set of solutions is a line (in the two-dimensional case ) or a hyperplane (in the high-dimensional case). This situation is similar to the case of a semidefinite quadratic form, as shown in Figure~\ref{fig:quadratic_singular}. If $\bA$ does not fall into any of these categories, a saddle point emerges (see Figure~\ref{fig:quadratic_saddle}), posing a challenge for gradient descent.  
In such cases, alternative methods, e.g., perturbed GD \citep{jin2017escape, du2017gradient}, can be employed to navigate away from saddle points.

\begin{figure}[htp]
\centering  
\vspace{-0.35cm} 
\subfigtopskip=2pt 
\subfigbottomskip=2pt 
\subfigcapskip=-5pt 
\subfigure[Contour and the descent direction. The red dot is the optimal point.]{\label{fig:quadratic_vanillegd_contour}
	\includegraphics[width=0.31\linewidth]{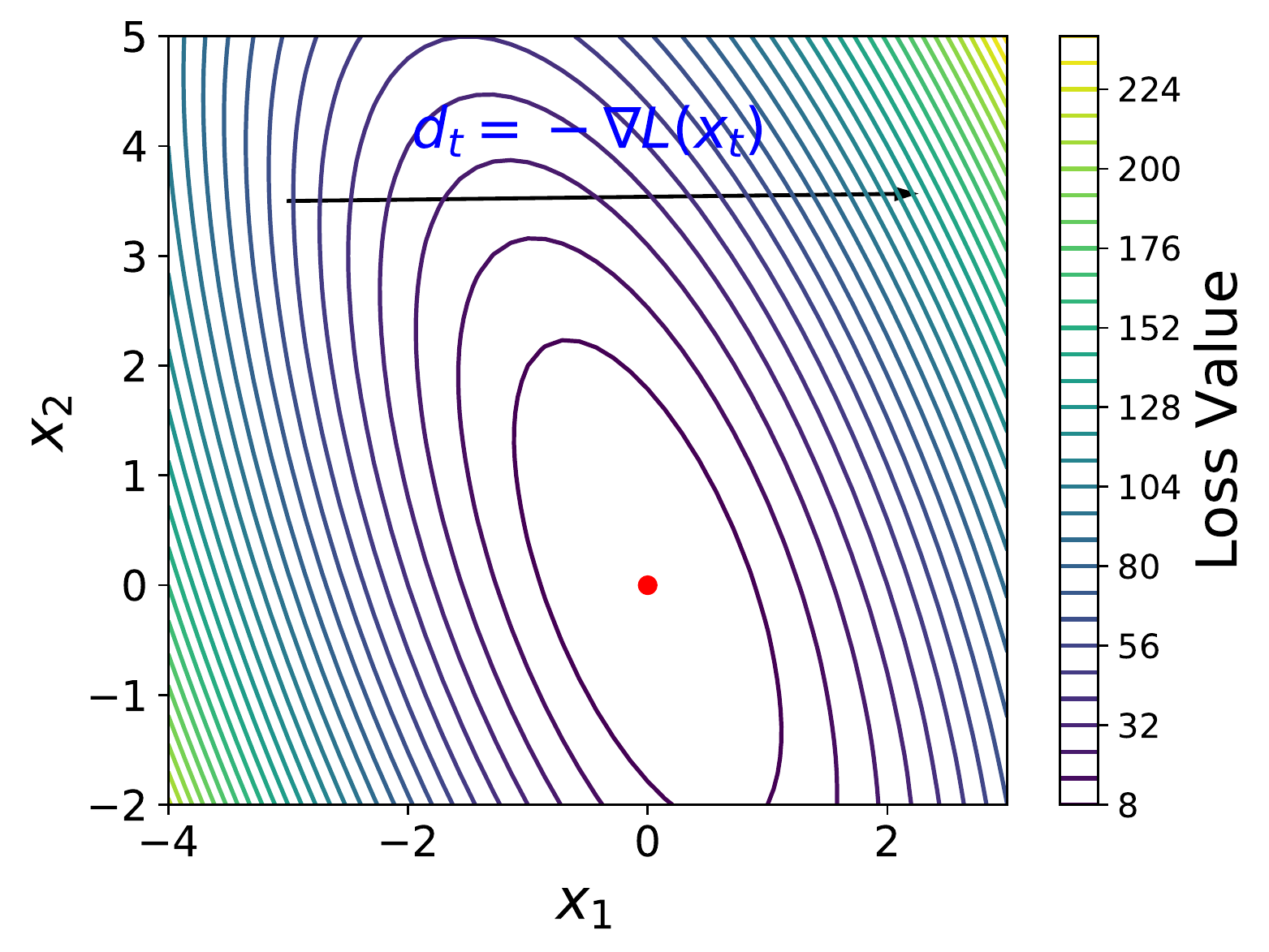}}
\subfigure[Vanilla GD, $\eta=0.02$.]{\label{fig:quadratic_vanillegd_contour2}
	\includegraphics[width=0.31\linewidth]{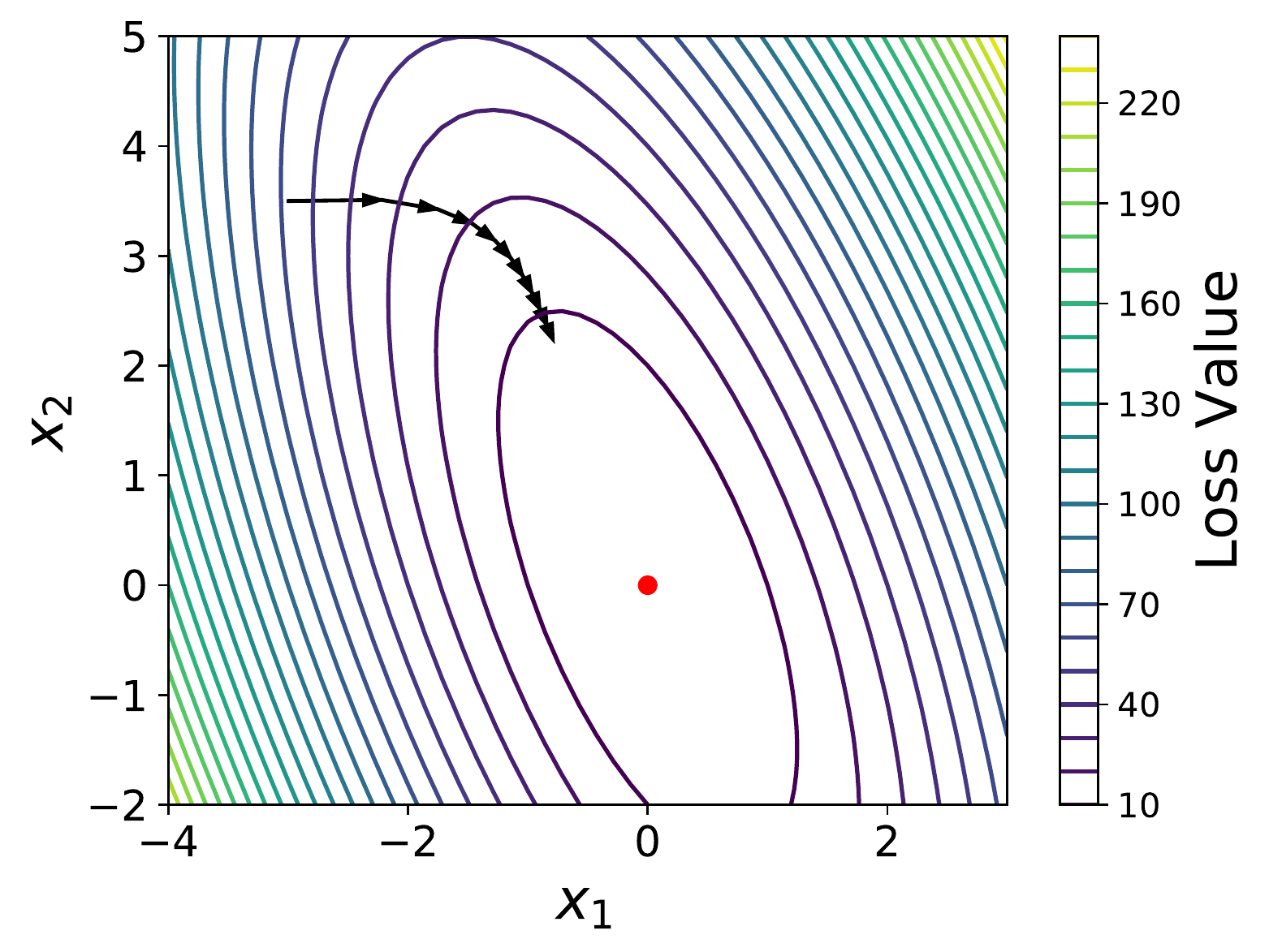}}
\subfigure[Vanilla GD, $\eta=0.08$.]{\label{fig:quadratic_vanillegd_contour8}
	\includegraphics[width=0.31\linewidth]{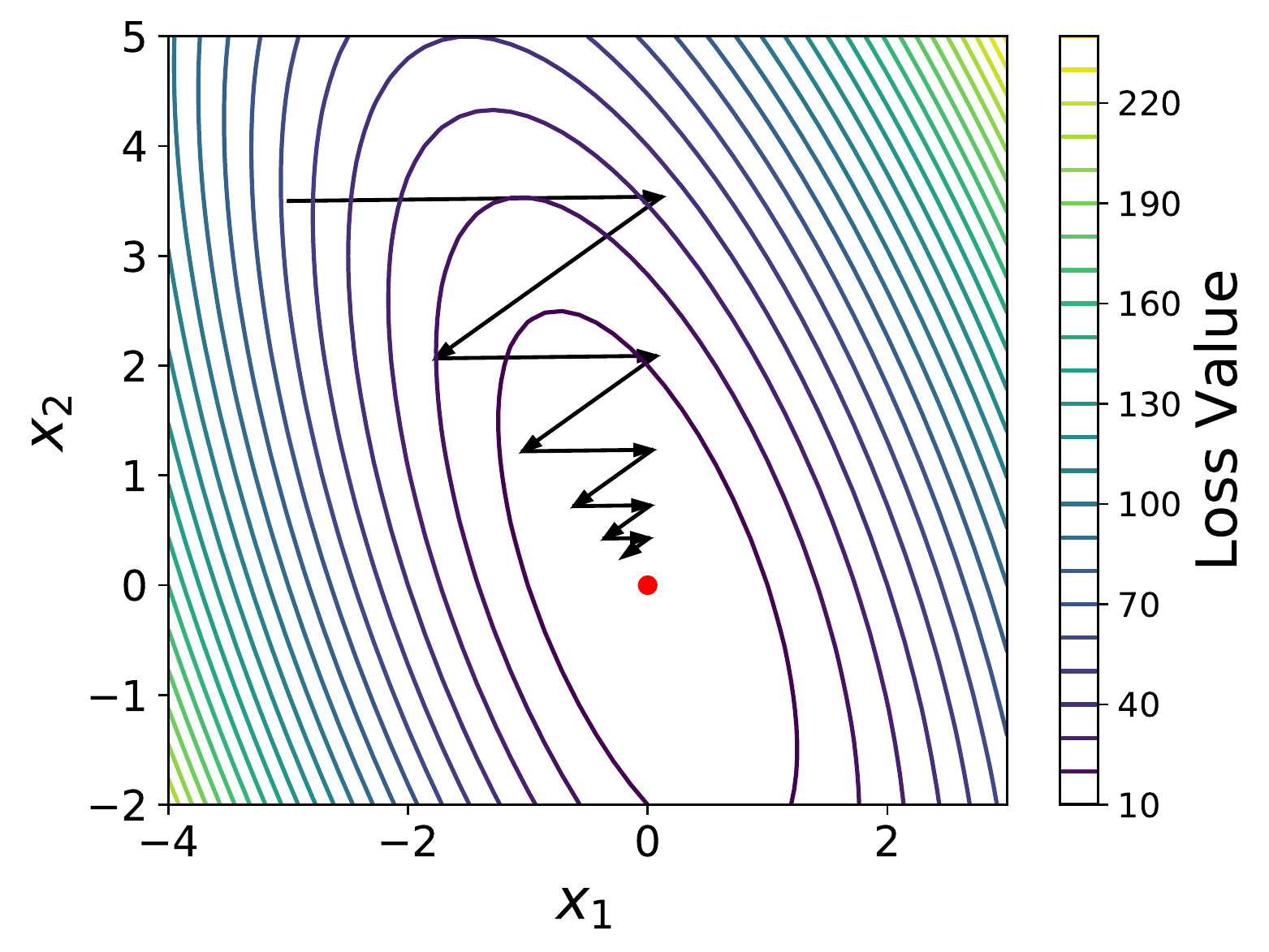}}
\caption{Illustration for the linear search of quadratic form with $\bA=\begin{bmatrix}
		20 & 7 \\ 5 & 5
	\end{bmatrix}$, $\bb=\bzero$, and $c=0$. The procedure is at $\bx_t=[-3,3.5]^\top$ for the $t$-th iteration.}
\label{fig:quadratic_vanillegd}
\end{figure}

\index{Quadratic form}
\index{Saddle point}

Note that in this context, gradient descent is not necessary; we can directly proceed to the minimum (if the data matrix $\bA$ is nonsingular, and we have an algorithm to compute its inverse).
However, we are concerned with the iterative updates of the convex quadratic function. Suppose we pick up a starting point $\bx_1\in \real^d$ \footnote{In some texts, the starting point is denoted as $\bx_0$, however, we will take it as $\bx_1$ in this article.}. The trivial way for the update at time step $t$ involves fixing the learning rate $\eta$ and choosing a descent direction $\bd_t$; and the gradient descent update becomes:
$$
\bx_{t+1} =  \bx_t + \eta \bd_t. 
$$
This results in a monotonically decreasing sequence of $\{L(\bx_t)\}$. 
Specifically, when the descent direction is chosen to be the negative gradient $\bd_t=\bA\bx_t-\bb$ ($\bA$ is symmetric), the update becomes 
\begin{equation}\label{equation:vanilla-gd-update}
\text{Vanilla GD: \gap } \bx_{t+1} = \bx_t - \eta (\bA\bx_t-\bb).
\end{equation}
A concrete example is given in Figure~\ref{fig:quadratic_vanillegd}, where $\bA=\begin{bmatrix}
20 & 7 \\ 5 & 5
\end{bmatrix}$, $\bb=\bzero$, and $c=0$. Suppose at $t$-th iteration, $\bx_t=[-3,3.5]^\top$. Figure~\ref{fig:quadratic_vanillegd_contour} shows the descent direction given by the negative gradient at the point of $\bx_t$; Figure~\ref{fig:quadratic_vanillegd_contour2} and Figure~\ref{fig:quadratic_vanillegd_contour8} present 10 iterations afterwards with $\eta=0.02$ and $\eta=0.08$, respectively.

\index{Spectral decomposition}
\paragraph{Closed form for vanilla GD.}
When $\bA$ is symmetric, it admits spectral decomposition (Theorem 13.1 in \citet{lu2022matrix} or Appendix~\ref{appendix:spectraldecomp}, p.~\pageref{appendix:spectraldecomp}):
$$
\bA=\bQ\bLambda\bQ^\top \in \real^{d\times d} \leadto \bA^{-1} = \bQ\bLambda^{-1}\bQ^\top,
$$ 
where $\bQ = [\bq_1, \bq_2, \ldots , \bq_d]$ comprises mutually orthonormal eigenvectors of $\bA$, and $\bLambda = \diag(\lambda_1, \lambda_2, \ldots , \lambda_d)$ contains the corresponding real eigenvalues of $\bA$.  If we further assume $\bA$ is positive definite, then the eigenvalues are all positive. By convention, we order the eigenvalues such that $\lambda_1\geq \lambda_2\geq \ldots \geq \lambda_d$. Define the following iterate vector at iteration $t$ as
\begin{equation}\label{equation:vanilla-yt}
\by_t = \bQ^\top(\bx_t - \bx_\star),
\end{equation}
where $\bx_\star = \bA^{-1}\bb$ if we further assume $\bA$ is nonsingular, as aforementioned. 
It then follows that
$$
\begin{aligned}
\by_{t+1} &= \bQ^\top(\bx_{t+1} - \bx_\star) = \bQ^\top(\bx_{t} - \eta(\bA\bx_{t}-\bb) - \bx_\star) \gap &\text{($\bx_{t+1} = \bx_{t}-\eta\nabla L(\bx_{t})$)}\\
&=\bQ^\top(\bx_{t} - \bx_\star) - \eta \bQ^\top (\bA\bx_{t}-\bb) \\
&= \by_{t} - \eta \bQ^\top (\bQ\bLambda\bQ^\top\bx_{t}-\bb) \gap &\text{($\bA=\bQ\bLambda\bQ^\top$)}\\
&= \by_{t} - \eta  (\bLambda\bQ^\top\bx_{t}-\bQ^\top\bb) \\
&= \by_{t} - \eta \bLambda\bQ^\top (\bx_{t}-\bx_\star) = \by_{t} - \eta \bLambda \by_{t} \\
&= (\bI - \eta \bLambda)\by_{t}  = (\bI - \eta \bLambda)^t\by_{1} \\
\end{aligned}
$$
where the second equality is from Eq.~\eqref{equation:vanilla-gd-update}. This reveals the error term at each iteration:
\begin{equation}\label{equation:vanilla-gd-closedform}
\norm{\bx_{t+1} - \bx_\star}^2 = \norm{\bQ\by_{t+1}}^2 = \norm{\bQ(\bI - \eta \bLambda)^t\by_{1}}^2 = \bigg|\bigg|\sum_{i=1}^{d} y_{1,i} \cdot (1-\eta \lambda_i)^t \bq_i\bigg|\bigg|^2,
\end{equation}
where $\by_1$ depends on the initial parameter $\bx_1$, and $y_{1,i}$ is the $i$-th element of $\by_1$. An intuitive interpretation for $\by_{t+1}$ is the error in the $\bQ$-basis at iteration $t+1$. By Eq.~\eqref{equation:vanilla-gd-closedform}, we realize that the learning rate should be chosen such that 
\begin{equation}\label{equation:vanillagd-quandr-rate-chgoices}
|1-\eta\lambda_i| \leq 1, \gap \forall \,\, i\in \{1,2,\dots, d\}.
\end{equation}
And the error is a sum of $d$ terms, each has its own dynamics and depends on the rate of $1-\eta\lambda_i$; the closer the rate is to 1, the slower it converges in that dimension \citep{shewchuk1994introduction, o2015adaptive, goh2017momentum}.

\begin{figure}[h]
\centering
\includegraphics[width=0.5\textwidth]{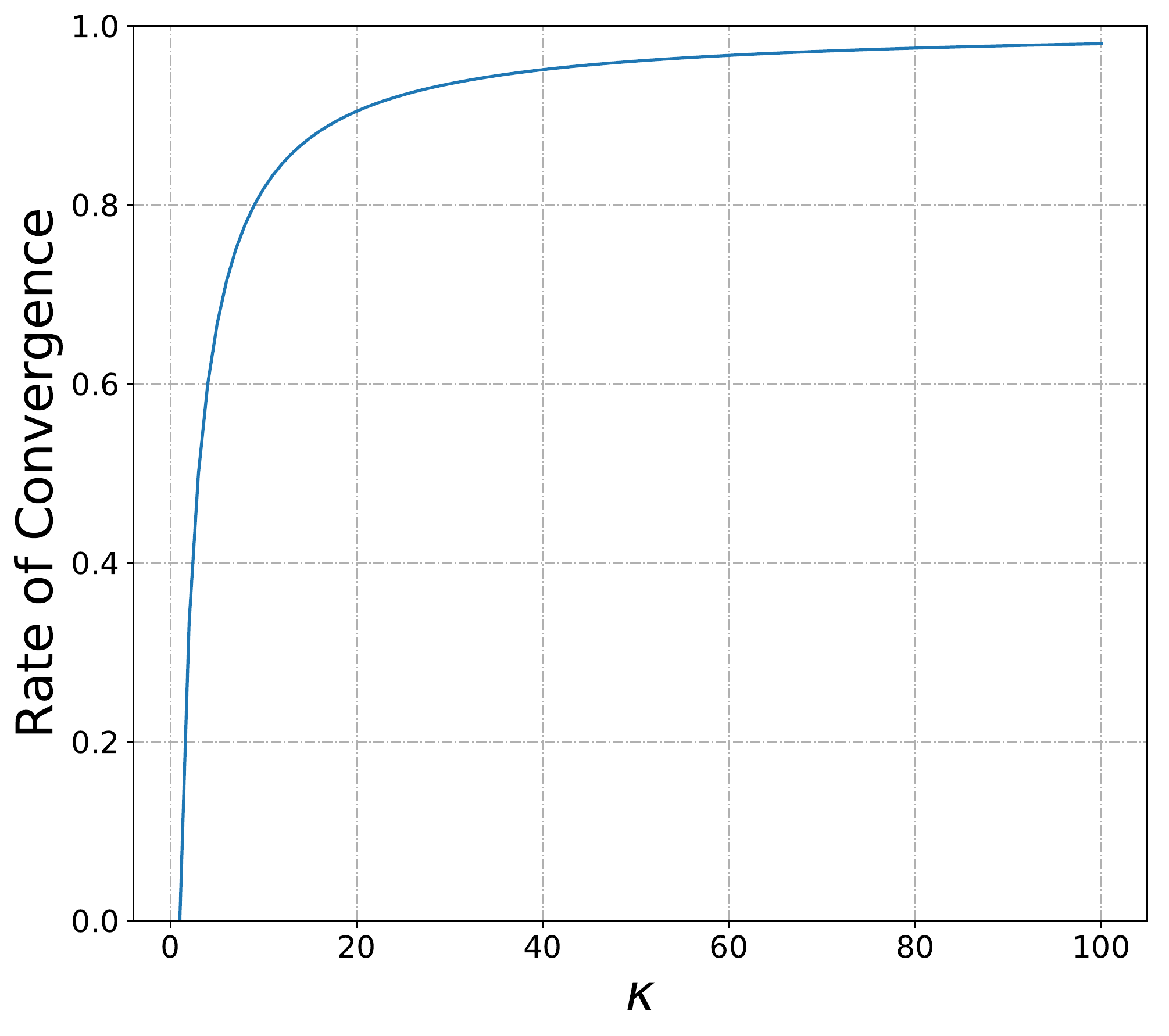}
\caption{Rate of convergence (per iteration) in vanilla GD method. The $y$-axis is $\frac{\kappa-1}{\kappa+1}$.}
\label{fig:rate_convergen_vanillaGD}
\end{figure}
\index{Rate of convergence}

To ensure convergence, the learning rate must satisfy that $|1-\eta\lambda_i| \leq 1$. This condition implies $0<\eta\lambda_i <2$ for $i$ in $\{1,2,\ldots, d\}$. Therefore, the overall rate of convergence is determined by the slowest component:
$$
\text{rate}(\eta) = \max \{|1-\eta\lambda_1|,  |1-\eta\lambda_d|\},
$$
since $\lambda_1\geq \lambda_2\geq \ldots \geq \lambda_d$. The optimal learning rate occurs when the first and the last eigenvectors converge at the same rate, i.e., $\eta\lambda_1-1 =1- \eta\lambda_d$:
\begin{equation}\label{equation:eta-vanilla-gd}
\text{optimal } \eta = \underset{\eta}{\arg\min} \text{ rate}(\eta) = \frac{2}{\lambda_1+\lambda_d},
\end{equation}
and 
\begin{equation}\label{equation:vanialla-gd-rate}
\text{optimal rate}  = \underset{\eta}{\min} \text{ rate}(\eta) =
\frac{\lambda_1/\lambda_d - 1}{\lambda_1/\lambda_d + 1}
=\frac{\kappa - 1}{\kappa + 1},
\end{equation}
where $\kappa = \frac{\lambda_1}{\lambda_d}$ is known as the \textit{condition number} (see \citet{lu2021numerical} for more information). When $\kappa=1$, the convergence is fast with just one step; as the condition number increases, the gradient descent becomes slower. The rate of convergence (per iteration) is plotted in Figure~\ref{fig:rate_convergen_vanillaGD}. The more \textit{ill-conditioned} the matrix, i.e., the larger its condition number, the slower the convergence of vanilla GD.

%% file: chapter-lineSearch.tex
\newpage
\clearchapter{Line Search}
\begingroup
\hypersetup{linkcolor=winestain,
linktoc=page,  
}
\minitoc \newpage
\endgroup

\index{Line search}
\index{Steepest descent}
\section{Line Search}\label{section:line-search}
\lettrine{\color{caligraphcolor}I}
In the last section, we derive the gradient descent, where the update step at step $t$ is $ -\eta\bg_t:=-\eta \nabla  L(\bx_t)$, and the learning rate $\eta$ controls how large of a step to take in the direction of negative gradient. Line search is a method that directly determines the optimal learning rate in order to provide the most significant improvement in the gradient movement. Formally, the line search solves the following problem at the $t$-th step of gradient descent:
$$
\eta_t = \underset{\eta}{\arg\min}\,\,   L(\bx_t - \eta \bg_t).
$$
After performing the gradient update $\bx_{t+1} = \bx_t - \eta_t \bg_t$, the gradient is computed at $\bx_{t+1}$ for the next step $t+1$. More generally, let $\bd_t$ be the descent direction; then, the gradient descent with line search (to differentiate, we call it \textit{steepest descent} when $\bd_t=-\bg_t$ in this article, and the fixed learning rate GD is known as the \textit{vanilla GD}) can be described by:
$$
\eta_t = \underset{\eta}{\arg\min}\,\,   L(\bx_t + \eta \bd_t).
$$

\begin{lemma}[Orthogonality in Line Search]\label{lemm:linear-search-orghonal}
The gradient of optimal point $\bx_{t+1}=\bx_t + \eta_t \bd_t $ of a line search is orthogonal to the current update direction $\bd_t$:
$$
\nabla L(\bx_{t+1})^\top \bd_t = 0.
$$
\end{lemma}
\begin{proof}[of Lemma~\ref{lemm:linear-search-orghonal}]
Suppose $\nabla L(\bx_{t+1})^\top \bd_t \neq 0$, then there exists a $\delta$ and it follows by Taylor's formula (Appendix~\ref{appendix:taylor-expansion}, p.~\pageref{appendix:taylor-expansion}) that 
\begin{equation}\label{equation:orthogonal-line-search}
	L(\bx_t +\eta_t \bd_t \pm \delta \bd_t) \approx L(\bx_t +\eta_t \bd_t)\pm \delta \bd_t^\top \nabla L(\bx_t +\eta_t \bd_t).
\end{equation}
Since $\bx_t +\eta_t \bd_t$ is the optimal move such that $ L(\bx_t +\eta_t \bd_t) \leq L(\bx_t +\eta_t \bd_t \pm \delta \bd_t)$ and $\delta \neq 0$. This leads to the claim 
$$
\bd_t^\top \nabla L(\bx_t +\eta_t \bd_t)=0.$$
We complete the proof.
\end{proof}

In line search methods, the loss function at iteration $t$ can be  expressed in terms of $\eta$ as follows:
$$
J(\eta) = L(\bx_t + \eta \bd_t).
$$
Consequently, the problem can be formulated as finding
$$
\eta_t = \underset{\eta}{\arg\min} \,\, L(\bx_t + \eta \bd_t)=\underset{\eta}{\arg\min} \,\,  J(\eta).
$$
This indicates that the (local) minimum of $\eta$ can be obtained by finding the solution of $J^\prime(\eta)=0$ if $J(\eta)$ is differentiable, according to Fermat's theorem (see Exercise~\ref{problem:fist_opt}, p.~\pageref{problem:fist_opt}). The solution then follows that
\begin{equation}\label{equation:j_eta_ajmijo}
J^\prime(\eta) =\bd_t^\top  \nabla L(\bx_t+\eta\bd_t)=0,
\end{equation}
which reaffirms  Lemma~\ref{lemm:linear-search-orghonal}.
When $\eta=0$, we have (by Remark~\ref{remark:descent_condition}, p.~\pageref{remark:descent_condition})
\begin{equation}\label{equation:linesearc-eta0}
J^\prime (0) = \bd_t^\top \bg_t\leq 0.
\end{equation}
A crucial property in typical line search settings is that the loss function
$J(\eta)$, when expressed in terms of $\eta$, is often a unimodal function. 
If a value $\eta_{\max}$ is identified such that $J^\prime(\eta_{\max}) > 0$, the optimal learning rate is then in the range of $[0, \eta_{\max}]$. Line search methods are followed to find the optimal $\eta$ within this range, satisfying the optimal condition $J^\prime(\eta)=0$.
Now, we introduce some prominent line search approaches: bisection linear search, Golden-Section line search, and Armijo rule search.


\index{Bisection line search}
\section{Bisection Line Search}
In the \textit{bisection line search} method, we start by setting the interval $[a,b]$ as $[\eta_{\min}, \eta_{\max}]$, where $\eta_{\min}$ and $ \eta_{\max}$ serve as the lower and upper bounds, respectively, for the learning rate $\eta$ ($\eta_{\min}$ can be set to 0 as specified by Eq.~\eqref{equation:linesearc-eta0}). The bisection line search involves evaluating the loss function $J(\eta)$ at the midpoint $\frac{a+b}{2}$. Given the information that $J^\prime(a)<0$ and $J^\prime(b)>0$, the bisection line search follows that
$$
\left\{
\begin{aligned}
\text{set } a &:= \frac{a+b}{2} \text{, \gap if $J^\prime\left(\frac{a+b}{2}\right)<0$}; \\
\text{set } b &:= \frac{a+b}{2} \text{, \gap if $J^\prime\left(\frac{a+b}{2} \right)>0$}.
\end{aligned}
\right.
$$
The procedure is repeated until the interval between $a$ and $b$ becomes sufficiently small.

The bisection line search is also known as the \textit{binary line search}. And in some cases, the derivative of $J(\eta)$ cannot be easily obtained; then the interval is narrowed by evaluating the objective function at two closely spaced
points around $\frac{a+b}{2} $. To be more concrete, assume $J(\eta)$ is convex (since we are in the descent setting), we evaluate the loss function at $\frac{a+b}{2}$ and $\frac{a+b}{2}+\epsilon$, where $\epsilon$ is a numerically small value, e.g.,  $\epsilon=1e-8$.  
This allows us to evaluate whether the function is increasing or decreasing at $\frac{a+b}{2}$ by determining which of the two evaluations is larger. If the function is increasing
at $\frac{a+b}{2}$, the interval is narrowed to $[a,\frac{a+b}{2}+\epsilon]$; otherwise, it is narrowed to
$[\frac{a+b}{2}, b]$.
$$
\left\{
\begin{aligned}
\text{set } b &= \frac{a+b}{2}+\epsilon,  &\text{ \gap if increasing at $\frac{a+b}{2}$};\\
\text{set } a &= \frac{a+b}{2}, &\text{ \gap otherwise}. \\
\end{aligned}
\right.
$$
This iterative process continues until the range is sufficiently small or the required level of accuracy is achieved in the interval.

\begin{figure}[h]
\centering  
\vspace{-0.35cm} 
\subfigtopskip=2pt 
\subfigbottomskip=2pt 
\subfigcapskip=-5pt 
\subfigure[$\eta=a$ yields the minimum.]{\label{fig:golden_1}
	\includegraphics[width=0.23\linewidth]{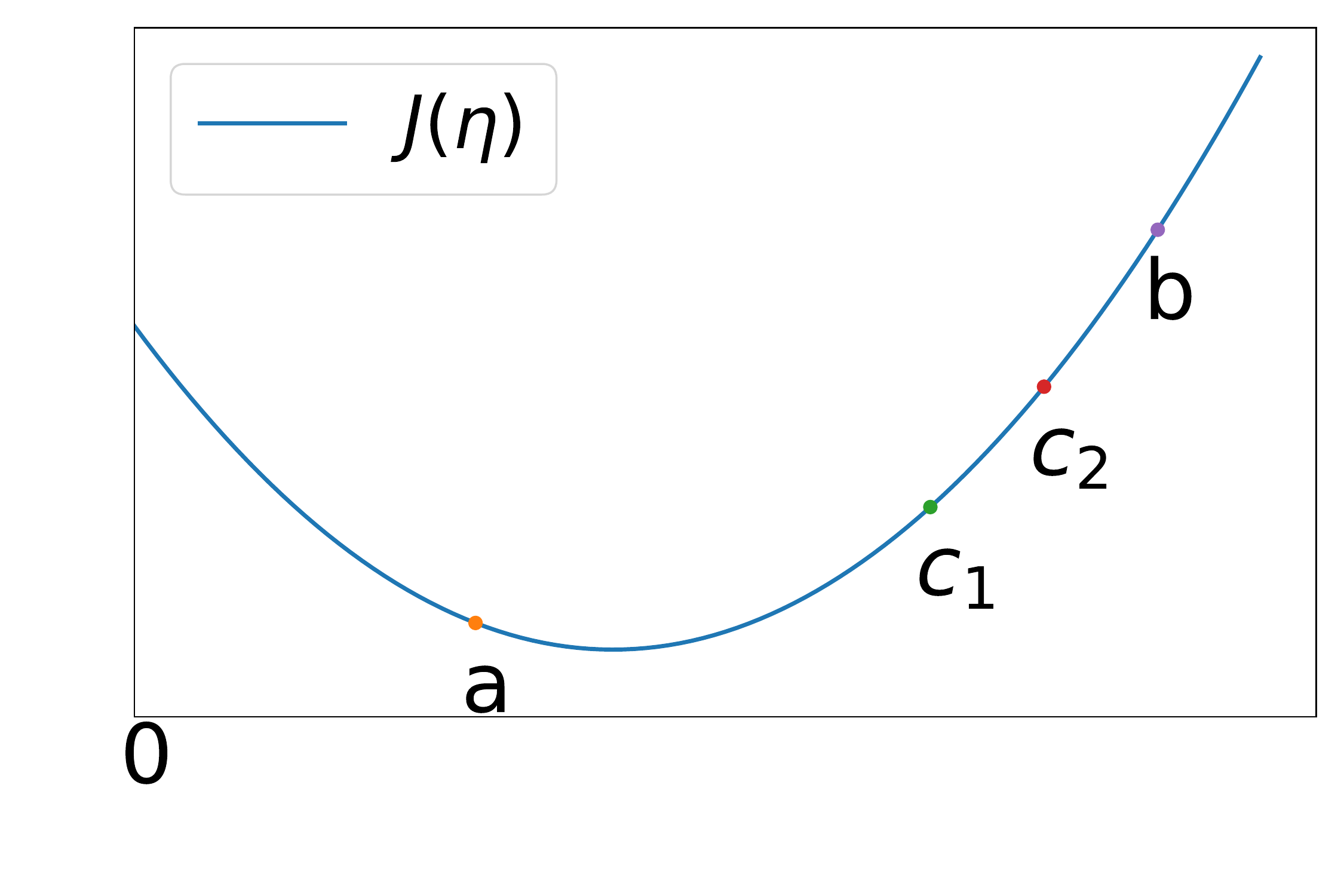}}
\subfigure[$\eta=c_1$ yields the minimum.]{\label{fig:golden_2}
	\includegraphics[width=0.23\linewidth]{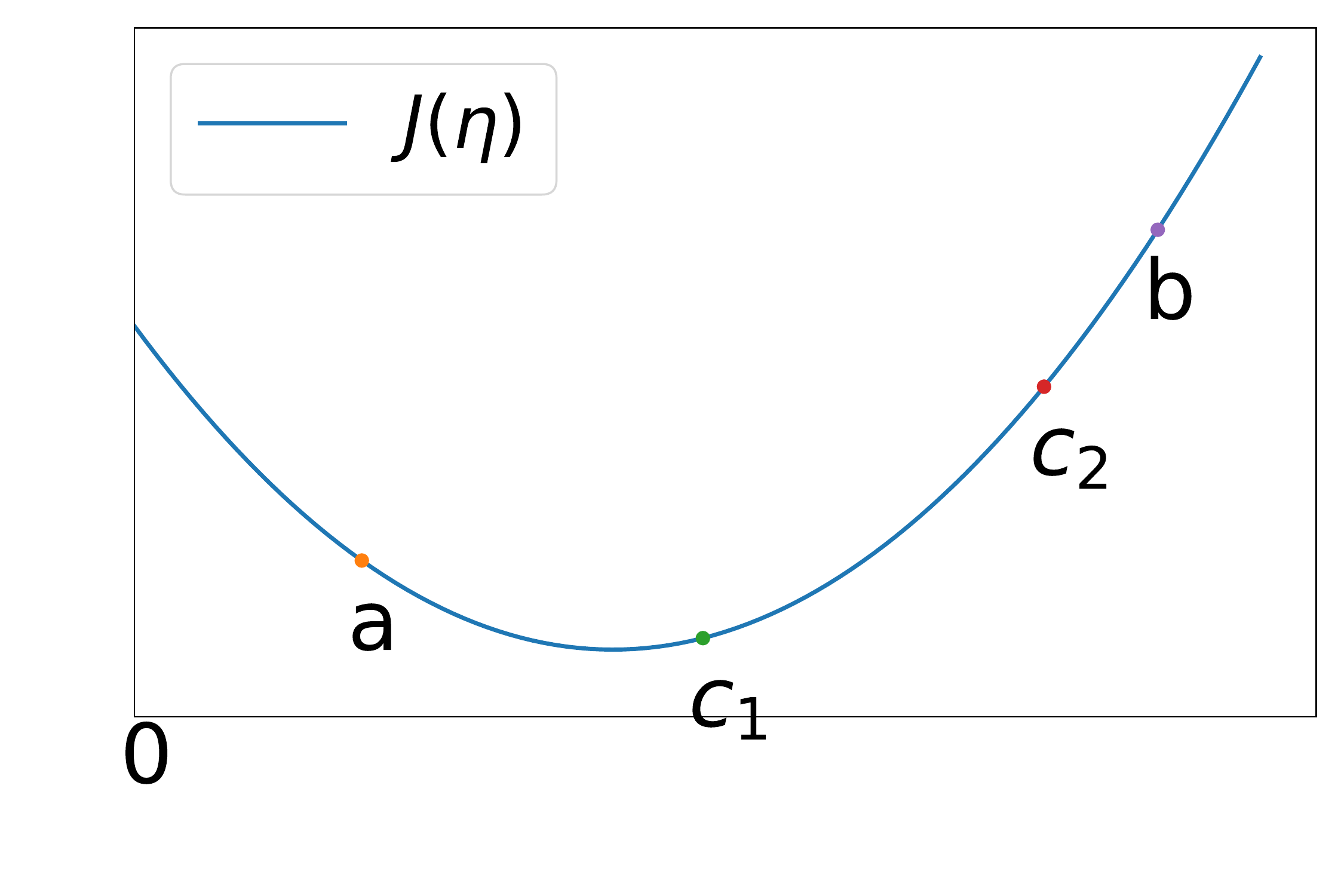}}
\subfigure[$\eta=c_2$ yields the minimum.]{\label{fig:golden_3}
	\includegraphics[width=0.23\linewidth]{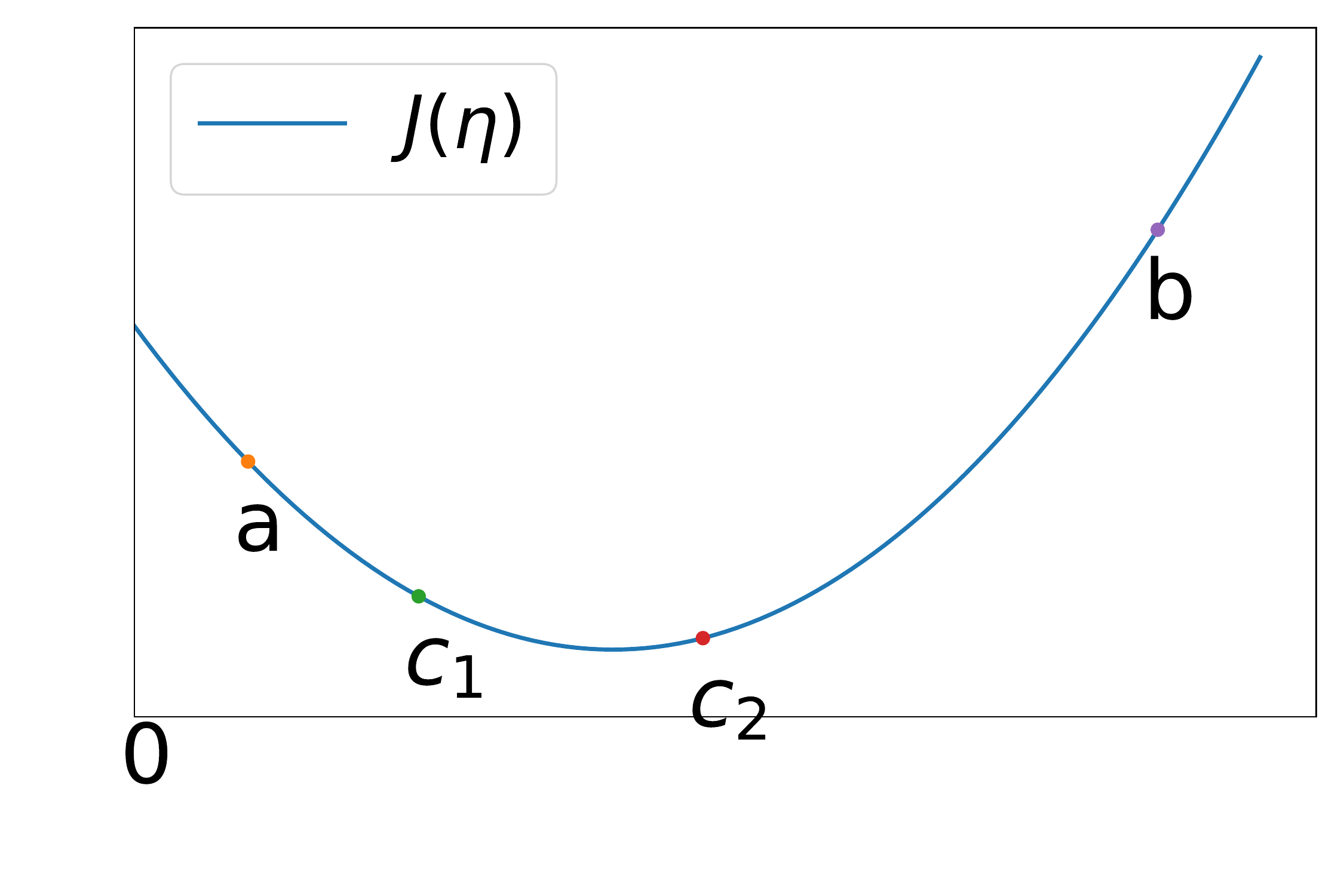}}
\subfigure[$\eta=b$ yields the minimum.]{\label{fig:golden_4}
	\includegraphics[width=0.23\linewidth]{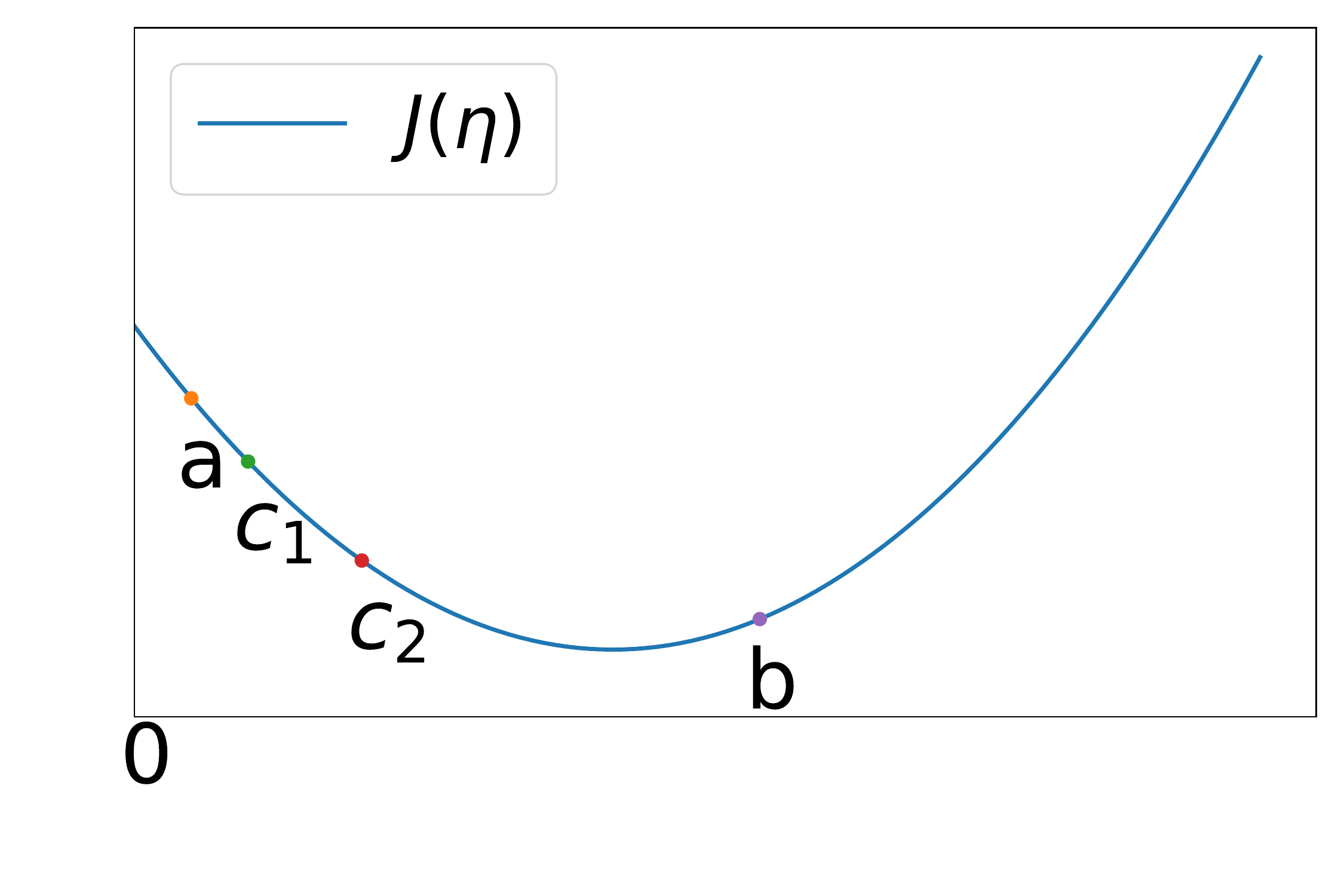}}
\caption{Demonstration of 4 different update ways in golden-section line search.}
\label{fig:conjguatecy-golden_1234}
\end{figure}

\index{Golden-section line search}
\section{Golden-Section Line Search}
Similar to the bisection line search, the \textit{golden-section line search} also identifies the best learning rate $\eta$ for a unimodal function $J(\eta)$. Again, it starts with the interval $[a,b]$ as $[0, \eta_{\max}]$. However, instead of selecting a midpoint, the golden-section search designates a pair of $c_1, c_2$ satisfying $a<c_1<c_2<b$. 
The procedure is as follows: if $\eta=a$ results in the minimum value for $J(\eta)$ (among the four values $J(a), J(c_1), J(c_2)$, and $J(b)$), we can exclude the interval $(c_1, b]$; if $\eta=c_1$ yields the minimum value, we can exclude the interval $(c_2, b]$; if $\eta=c_2$ yields the minimum value, we can exclude the interval $[a, c_1)$; and if $\eta=b$ yields the minimum value, we can exclude the interval $[a, c_2)$. The four situations are shown in Figure~\ref{fig:conjguatecy-golden_1234}.
In other words, at least one of the intervals $[a,c_1]$ and $[c_2, b]$ can be discarded  in the golden-section search
method. In summary, we have
$$
\begin{aligned}
&\text{when $J(a)$ is the minimum, exclude $(c_1, b]$};\\
&\text{when $J(c_1)$ is the minimum, exclude $(c_2, b]$};\\
&\text{when $J(c_2)$ is the minimum, exclude $[a, c_1)$};\\
&\text{when $J(b)$ is the minimum, exclude $[a, c_2)$}.\\
\end{aligned}
$$
By excluding one of the four intervals, the new bounds $[a, b]$ are adjusted accordingly, and the process iterates until the range is sufficiently small.

\begin{figure}[h]
\centering
\includegraphics[width=0.6\textwidth]{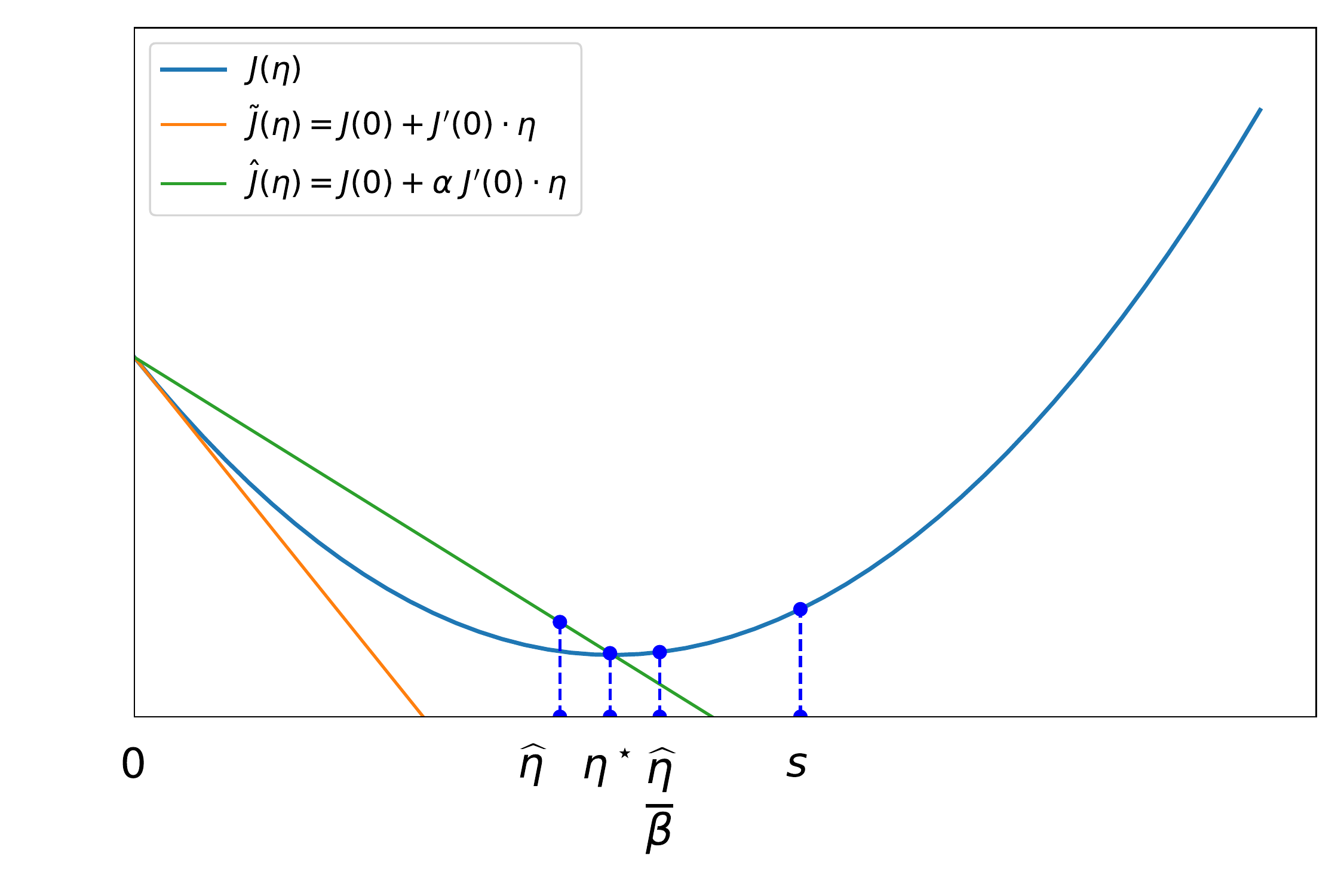}
\caption{Demonstration of Armijo rule in a convex setting.}
\label{fig:armijo_convex}
\end{figure}

\index{Armijo rule}
\section{Armijo Rule}
Similar to Eq.~\eqref{equation:orthogonal-line-search}, using Taylor's formula once again, we have
$$
J(\eta)=L(\bx_t +\eta_t \bd_t ) \approx L(\bx_t )+ \eta  \bd_t^\top \nabla L(\bx_t ).
$$
Since $\bd_t^\top \nabla L(\bx_t ) \leq 0$ (by Remark~\ref{remark:descent_condition}, p.~\pageref{remark:descent_condition}), it follows that 
\begin{equation}\label{equation:armijo_step_approx}
L(\bx_t + \eta \bd_t) \leq L(\bx_t) + \alpha \eta \cdot  \bd_t^\top \nabla L(\bx_t),  \gap \alpha \in (0,1).
\end{equation}
Let 
$\widetilde{J}(\eta) = J(0)+  J^\prime(0) \cdot  \eta$ \footnote{The tangent of $J(\eta)$ at $\eta=0$.}
and 
$\widehat{J}(\eta) = J(0)+ \alpha J^\prime(0) \cdot  \eta$, the relationship between the two functions is depicted in Figure~\ref{fig:armijo_convex} for the case where $J(\eta)$ is a convex function; and we note that $\widehat{J}(\eta) > \widetilde{J}(\eta)$ when $\eta>0$.

The Armijo rule states that an acceptable $\eta$ should satisfy $J(\widehat{\eta}) \leq \widehat{J}(\widehat{\eta})$ to ensure sufficient decrease and $J( \widehat{\eta}/\beta) > \widehat{J}(\widehat{\eta}/\beta)$ to prevent the step size from being too small, where $\beta\in (0,1)$. 
This ensures that  the (local) optimal learning rate is in the range of $[\widehat{\eta}, \widehat{\eta}/\beta)$. By Eq.~\eqref{equation:armijo_step_approx}, the two criteria above can also be described by:
$$
\left\{
\begin{aligned}
J(\widehat{\eta}) &\leq \widehat{J}(\widehat{\eta}); \\
J( \widehat{\eta}/\beta) &> \widehat{J}(\widehat{\eta}/\beta),
\end{aligned}
\right.
\Longrightarrow \gap 
\left\{
\begin{aligned}
L(\bx_t + \widehat{\eta} \bd_t) -  L(\bx_t)&\leq \alpha \widehat{\eta} \cdot \bd_t^\top\nabla L(\bx_t);  \\
L(\bx_t + \widehat{\eta}/\beta \bd_t) -  L(\bx_t)&> \alpha \widehat{\eta}/\beta \cdot  \bd_t^\top\nabla L(\bx_t) .
\end{aligned}
\right.
$$
The complete algorithm for calculating the learning rate at the $t$-th iteration is outlined in Algorithm~\ref{alg:als_armijo}. In practice, the parameters   are typically set as $\beta \in [0.2, 0.5]$ and $\alpha\in [1e-5, 0.5]$. Additionally, it's worth noting that the Armijo rule is inexact, and it works even when $J(\eta)$ is not unimodal.

After developing the Armijo algorithm, the underlying concept of the Armijo rule becomes apparent: the descent direction $J^\prime(0)$ at that starting point $\eta=0$ often deteriorates in terms of rate of improvement since it moves further along this direction. However, a fraction $\alpha\in [1e-5, 0.5]$ of this improvement is acceptable. By Eq.~\eqref{equation:armijo_step_approx}, the descent update at $(t+1)$-th iteration $L(\bx_t + \eta \bd_t)$ is at least $\alpha \eta \cdot  \bd_t^\top \nabla L(\bx_t)$ smaller than that of $t$-th iteration.

\begin{algorithm}[h] 
\caption{Armijo Rule at $t$-th Iteration}
\label{alg:als_armijo}
\begin{algorithmic}[1] 
	\Require Start with $\eta_t=s$, $0<\beta<1$, and $0<\alpha<1$;
	\State isStop = False;
	\While{isStop is False} 
	\If{$L(\bx_t + \eta_t \bd_t) - L(\bx_t) \leq \alpha \eta_t \cdot \bd_t^\top \nabla L(\bx_t)$}
	\State isStop = True;
	\Else
	\State $\eta_t= \beta \eta_t$;
	\EndIf
	\EndWhile
	\State Output $\eta_t$;
\end{algorithmic} 
\end{algorithm}

\index{Steepest descent}
\index{Quadratic form}
\section{Quadratic Form in Steepest Descent}\label{section:quadratic-in-steepestdescent}
Following the discussion of the quadratic form in the gradient descent section (Section~\ref{section:quadratic_vanilla_GD}, p.~\pageref{section:quadratic_vanilla_GD}), we now discuss the quadratic form in gradient descent with line search. 
By definition, we express $J(\eta)$ as follows:
$$
\begin{aligned}
J(\eta) = L(\bx_t+\eta \bd_t) &= \frac{1}{2} (\bx_t+\eta \bd_t)^\top \bA (\bx_t+\eta \bd_t) - \bb^\top (\bx_t+\eta \bd_t) +c\\
&= L(\bx_t) + \eta \bd_t^\top \underbrace{\left(\frac{1}{2} (\bA+\bA^\top )\bx_t - \bb \right)}_{=\nabla L(\bx_t)=\bg_t} +\frac{1}{2} \eta^2 \bd_t^\top \bA\bd_t,
\end{aligned}
$$
which is a quadratic function with respect to $\eta$ for which a closed   form for the line search exists:
\begin{equation}\label{equation:eta-gd-steepest}
\eta_t = - \frac{\bd_t^\top \bg_t}{ \bd_t^\top \bA\bd_t }.
\end{equation}
We observe that $\bd_t^\top \bA\bd_t >0 $ when $\bd_t \neq 0 $. As previously mentioned, when the search direction is the negative gradient $\bd_t=-\bg_t$, the method is known as the \textit{steepest descent}. Consequently, the descent update becomes
\begin{equation}\label{equation:steepest-quadratic}
\begin{aligned}
	\text{Steepest Descent: \gap }\bx_{t+1} &=  \bx_t + \eta_t \bd_t \\
	&= \bx_t  -  \frac{\bd_t^\top \bg_t}{ \bd_t^\top \bA\bd_t } \bd_t \\
	& =\bx_t   - \frac{\bg_t^\top \bg_t}{ \bg_t^\top \bA\bg_t } \bg_t,
\end{aligned}
\end{equation}
where $\bd_t = -\bg_t$ for the gradient descent case.

\begin{figure}[h]
\centering  
\vspace{-0.35cm} 
\subfigtopskip=2pt 
\subfigbottomskip=2pt 
\subfigcapskip=-5pt 
\subfigure[Contour and the descent direction. The red dot is the optimal point.]{\label{fig:quadratic_steepest_contour_tilt}
	\includegraphics[width=0.485\linewidth]{imgs/quadratic_steepest_contour_tilt.pdf}}
\subfigure[Intersection of the loss surface and vertical plane through the descent direction.]{\label{fig:quadratic_steepest_surface_tilt}
	\includegraphics[width=0.485\linewidth]{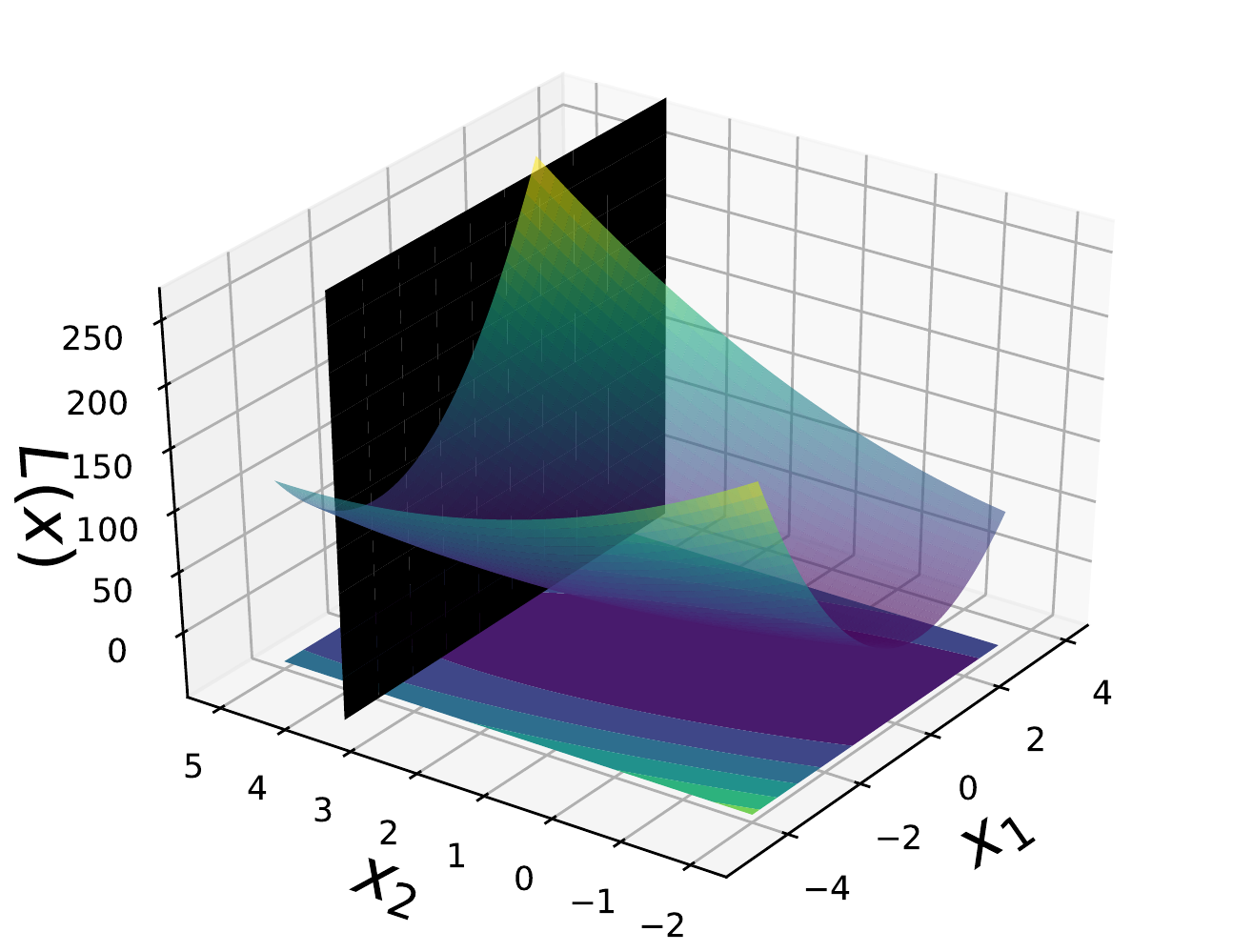}}
\subfigure[Intersection of the loss surface and vertical plane through the descent direction in two-dimensional space.]{\label{fig:quadratic_steepest_tilt_intersection}
	\includegraphics[width=0.485\linewidth]{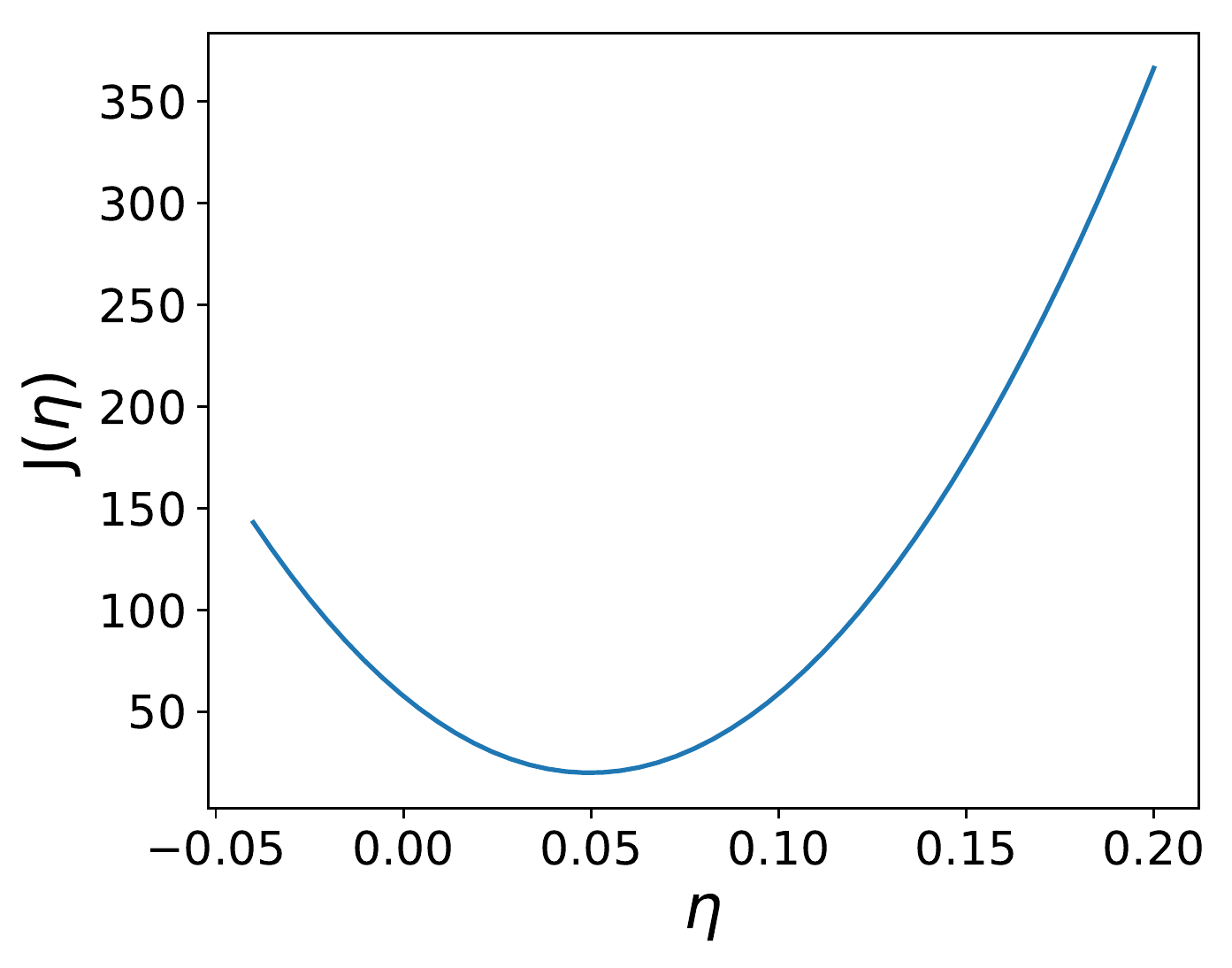}}
\subfigure[Various gradients on the line through the descent direction, where the gradient at the bottommost point is orthogonal to the gradient of the previous step.]{\label{fig:quadratic_steepest_tilt_gradient_direction}
	\includegraphics[width=0.485\linewidth]{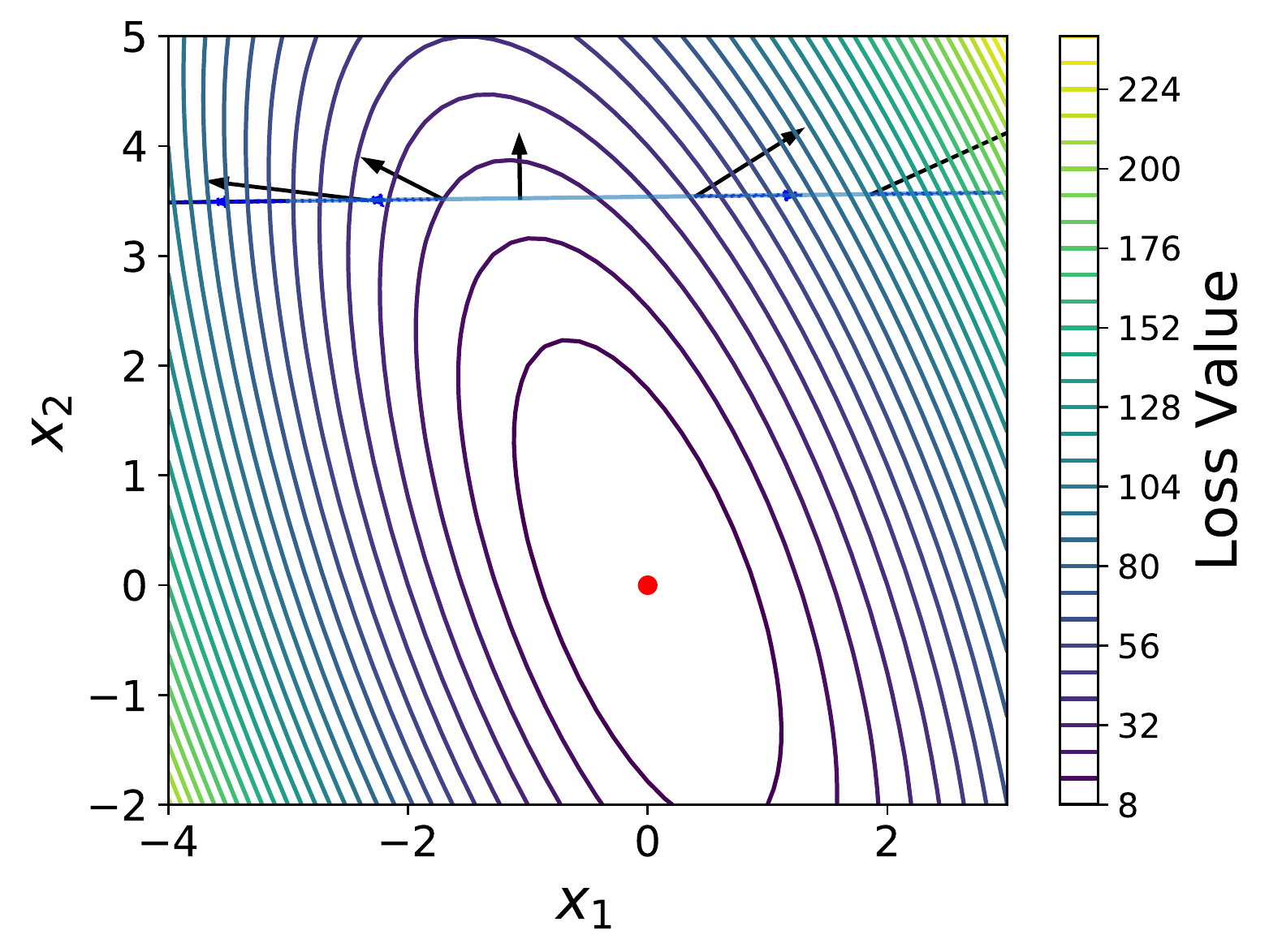}}
\caption{Illustration for the line search of quadratic form with $\bA=\begin{bmatrix}
		20 & 7 \\ 5 & 5
	\end{bmatrix}$, $\bb=\bzero$, and $c=0$. The procedure is at $\bx_t=[-3,3.5]^\top$ for $t$-th iteration.
}
\label{fig:quadratic_steepest_tilt}
\end{figure}

\begin{figure}[h]
\centering  
\vspace{-0.35cm} 
\subfigtopskip=2pt 
\subfigbottomskip=2pt 
\subfigcapskip=-5pt 
\subfigure[Vanilla GD, $\eta=0.02$.]{\label{fig:momentum_gd_conjugate2}
	\includegraphics[width=0.31\linewidth]{./imgs/steepest_gd_mom-0_lrate-2.pdf}}
\subfigure[Vanilla GD, $\eta=0.08$.]{\label{fig:momentum_gd_conjugate8}
	\includegraphics[width=0.31\linewidth]{./imgs/steepest_gd_mom-0_lrate-8.pdf}}
\subfigure[Steepest descent.]{\label{fig:conjguatecy_zigzag2}
	\includegraphics[width=0.31\linewidth]{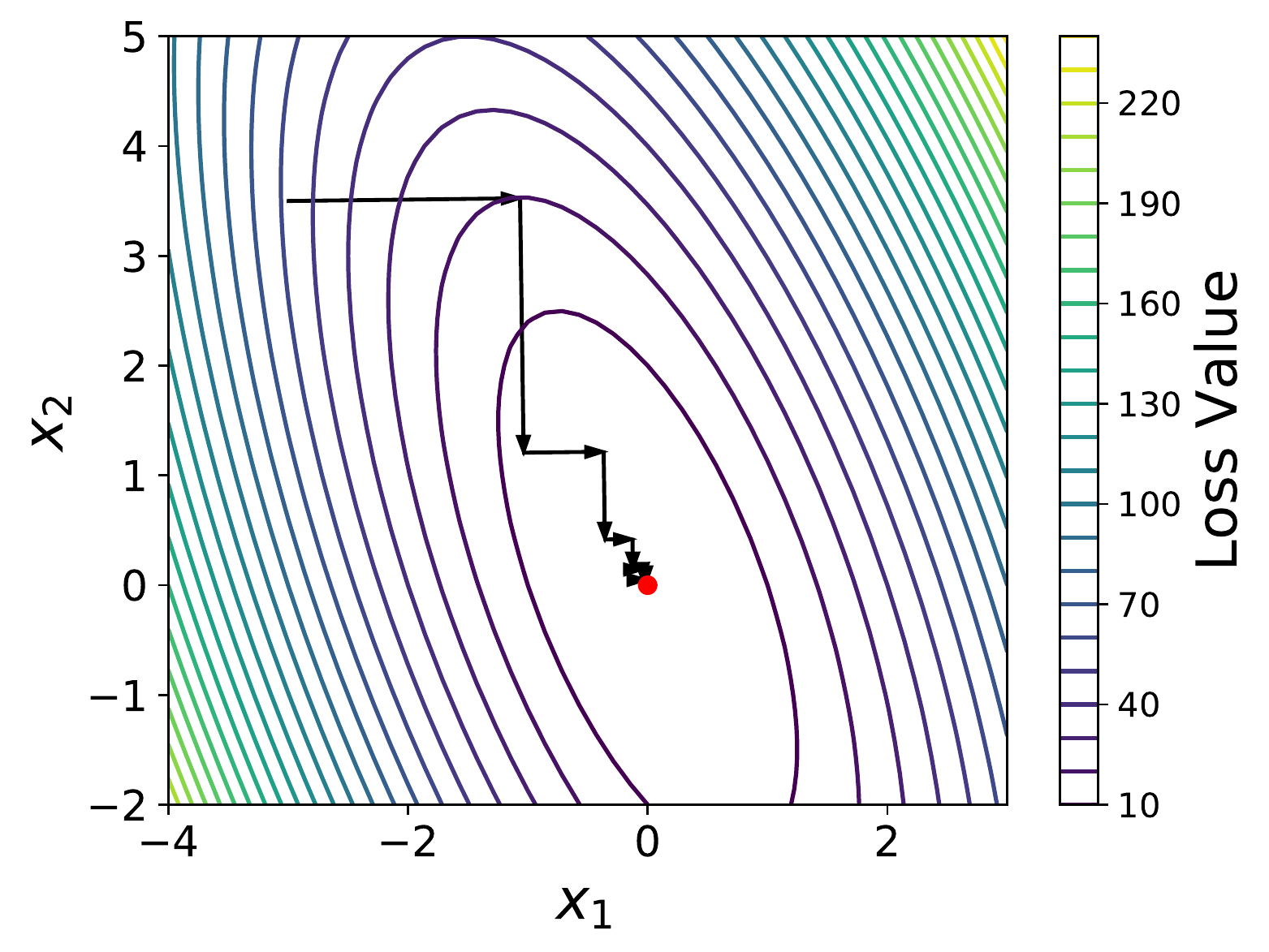}}
\caption{Illustration for the line search of quadratic form with $\bA=\begin{bmatrix}
		20 & 7 \\ 5 & 5
	\end{bmatrix}$, $\bb=\bzero$, and $c=0$. The procedure is at $\bx_t=[-3,3.5]^\top$ for $t$-th iteration. The example is executed for 10 iterations by vanilla GD with $\eta=0.02$, vanilla GD with $\eta=0.08$, and steepest descent, respectively. We notice the tedious choices in vanilla GD; and the zigzag path in the steepest GD with line search due to the orthogonality between each gradient and the previous gradient (Lemma~\ref{lemm:linear-search-orghonal}).
}
\label{fig:quadratic_steepest_tilt22}
\end{figure}

A concrete example is presented in Figure~\ref{fig:quadratic_steepest_tilt}, where $\bA=\begin{bmatrix}
20 & 7 \\ 5 & 5
\end{bmatrix}$, $\bb=\bzero$, and $c=0$. 
Suppose at the $t$-th iteration, the parameter is positioned at $\bx_t=[-3,3.5]^\top$. 
Additionally, Figure~\ref{fig:quadratic_steepest_contour_tilt} presents the descent direction via the negative gradient. The line search involves selecting the learning rate $\eta_t$ by minimizing $J(\eta) = L(\bx_t + \eta\bd_t)$, which is equivalent to determining the point on the intersection of the vertical plane through the descent direction and the paraboloid defined by the loss function $L(\bx)$, as shown in Figure~\ref{fig:quadratic_steepest_surface_tilt}. Figure~\ref{fig:quadratic_steepest_tilt_intersection} further shows the parabola defined by the intersection of the two surfaces. 
In Figure~\ref{fig:quadratic_steepest_tilt_gradient_direction}, various gradients on the line through the descent direction are displayed, where the gradient at the bottommost point is orthogonal to the gradient of the previous step $\nabla L(\bx_{t+1})^\top \bd_t$, as proved in Lemma~\ref{lemm:linear-search-orghonal}, where the black arrows are the gradients and the blue arrows are the projection of these gradients along $\bd_t = -\nabla L(\bx_t)$. 
An intuitive explanation for this orthogonality at the minimum is as follows: the slope of the parabola (Figure~\ref{fig:quadratic_steepest_tilt_intersection}) at any point is equal to the magnitude of the projection of the gradients onto the search direction (Figure~\ref{fig:quadratic_steepest_tilt_gradient_direction}) \citep{shewchuk1994introduction}. These projections represent the rate of increase of the loss function $L(\bx)$ as the point traverses the search line; and the minimum of $L(\bx)$ occurs where the projection is zero, and corresponding to when  the gradient is orthogonal to the search line.

The example is executed for 10 iterations in Figure~\ref{fig:momentum_gd_conjugate2}, Figure~\ref{fig:momentum_gd_conjugate8}, and Figure~\ref{fig:conjguatecy_zigzag2} using vanilla GD with $\eta=0.02$, vanilla GD with $\eta=0.08$, and steepest descent, respectively. It is evident that vanilla GD involves cumbersome choices of learning rates; and the zigzag trajectory in the steepest GD with line search, resulting from the orthogonality between each gradient and the previous gradient (Lemma~\ref{lemm:linear-search-orghonal}), can be observed. However, this drawback and limitation will be partly addressed in conjugate descent (Section~\ref{section:conjugate-descent}, p.~\pageref{section:conjugate-descent}).

\index{Spectral decomposition}
\index{Quadratic form}
\subsubsection{Special Case: Symmetric Quadratic Form}
To delve into the convergence results of steepest descent, we explore specific scenarios.
Following \citet{shewchuk1994introduction}, we first introduce some key definitions: the \textit{error vector} between the parameter $\bx_t$ at $t$-th iteration and the optimal parameter point $\bx_\star$, and the \textit{residual vector} between target $\bb$ and prediction at $t$-th iteration $\bA\bx_t$.

\begin{definition}[Error and Residual Vector]\label{definition:error-gd-}
At iteration $t$, the \textit{error vector} is defined as $\be_t = \bx_t - \bx_\star$, a vector indicates how far the iterate is from the solution, where $\bx_\star =\bA^{-1}\bb$ when $\bA$ is symmetric and nonsingular. 
Substituting into Eq.~\eqref{equation:steepest-quadratic}, the update for the error vector is
\begin{equation}\label{equation:steepest-quadratic-error}
	\be_{t+1} = \be_t - \frac{\bg_t^\top \bg_t}{ \bg_t^\top \bA\bg_t } \bg_t.
\end{equation}
Furthermore, the \textit{residual} $\br_t = \bb - \bA\bx_t$ indicates how far the iterate is from the correct value of $\bb$. Note in this case, the residual is equal to the negative gradient and the descent direction, i.e., $\br_t = \bd_t = -\bg_t$, when $\bA$ is symmetric (we may use $-\bg_t$ and $\br_t$ interchangeably when $\bA$ is symmetric). 
\end{definition}

We first consider the case where the error vector $\be_t$ at iteration $t$ is an eigenvector of $\bA$ corresponding to eigenvalue $\lambda_t$, i.e., $\bA\be_t = \lambda_t\be_t $. Then the gradient vector (for symmetric $\bA$ by Eq.~\eqref{equation:symmetric_gd_gradient}, p.~\pageref{equation:symmetric_gd_gradient})
$$
\bg_t = \bA\bx_t-\bb = \bA\left(\bx_t-\underbrace{\bA^{-1}\bb}_{\bx_\star}\right) =\bA\be_t = \lambda_t\be_t
$$ 
is also an eigenvector of $\bA$ corresponding to eigenvalue  $\lambda_t$, i.e., $\bA\bg_t = \lambda_t \bg_t$.
By Eq.~\eqref{equation:steepest-quadratic-error}, the update for $(t+1)$-th iteration is 
$$
\begin{aligned}
\be_{t+1} &= \be_t   - \frac{\bg_t^\top \bg_t}{ \bg_t^\top \bA\bg_t } \bg_t\\
&=\be_t - \frac{\bg_t^\top \bg_t}{ \lambda_t \bg_t^\top \bg_t } (\lambda_t\be_t) = \bzero .
\end{aligned}
$$
Therefore, convergence to the solution is achieved in just one additional step when   $\be_t$ is an eigenvector of $\bA$. A concrete example is shown in Figure~\ref{fig:steepest_gd_bisection_eigenvector}, where $\bA=\begin{bmatrix}
20 & 5 \\ 5 & 5
\end{bmatrix}$, $\bb=\bzero$, and $c=0$.

\begin{figure}[h]
\centering  
\vspace{-0.35cm} 
\subfigtopskip=2pt 
\subfigbottomskip=2pt 
\subfigcapskip=-5pt 
\subfigure[Steepest GD, $\bA\be_t=\lambda_t\be_t$.]{\label{fig:steepest_gd_bisection_eigenvector}
	\includegraphics[width=0.481\linewidth]{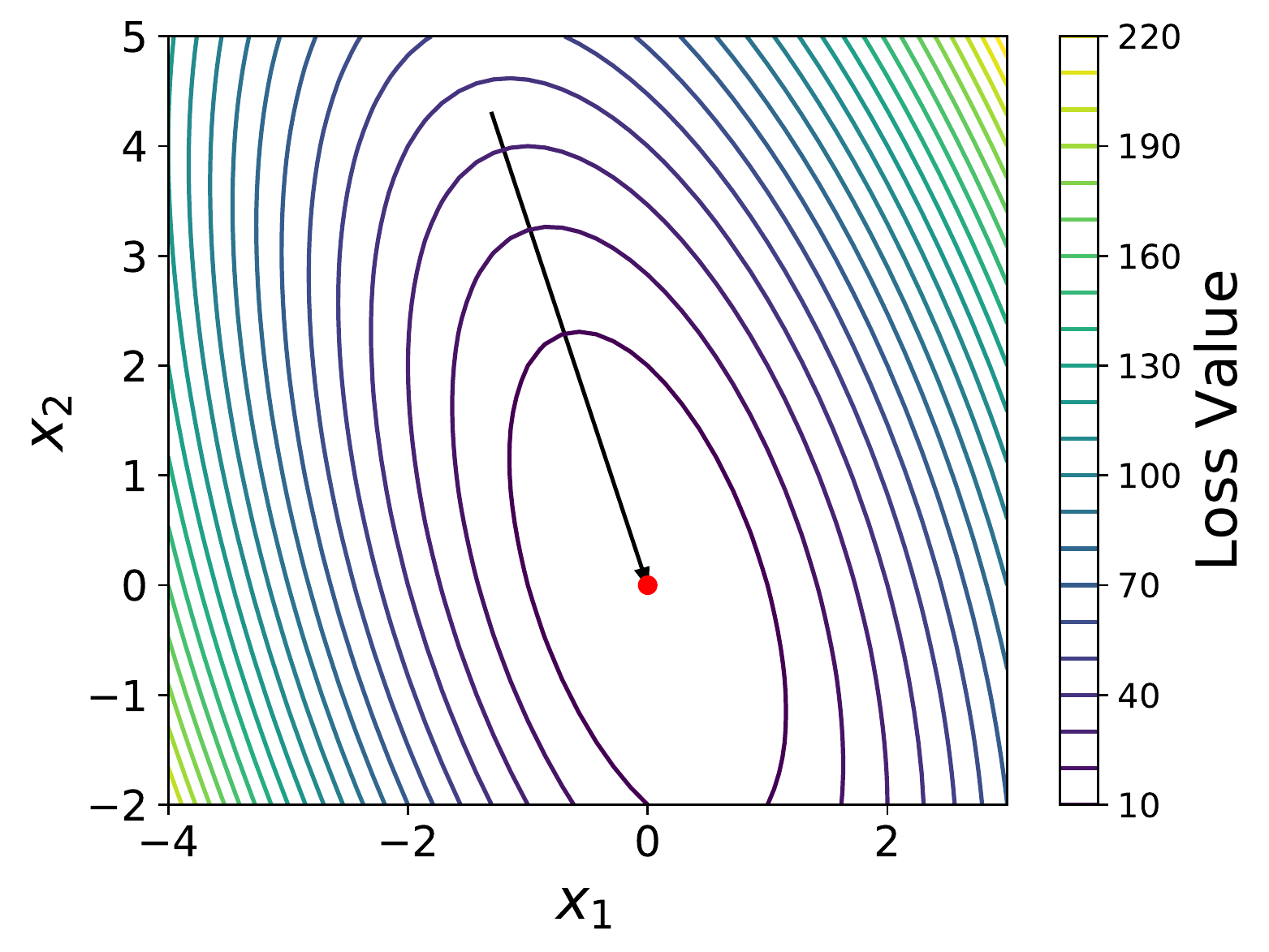}}
\subfigure[Steepest GD, $\lambda_1=\lambda_2$.]{\label{fig:steepest_gd_bisection_eigenvector_sameeigenvalue}
	\includegraphics[width=0.481\linewidth]{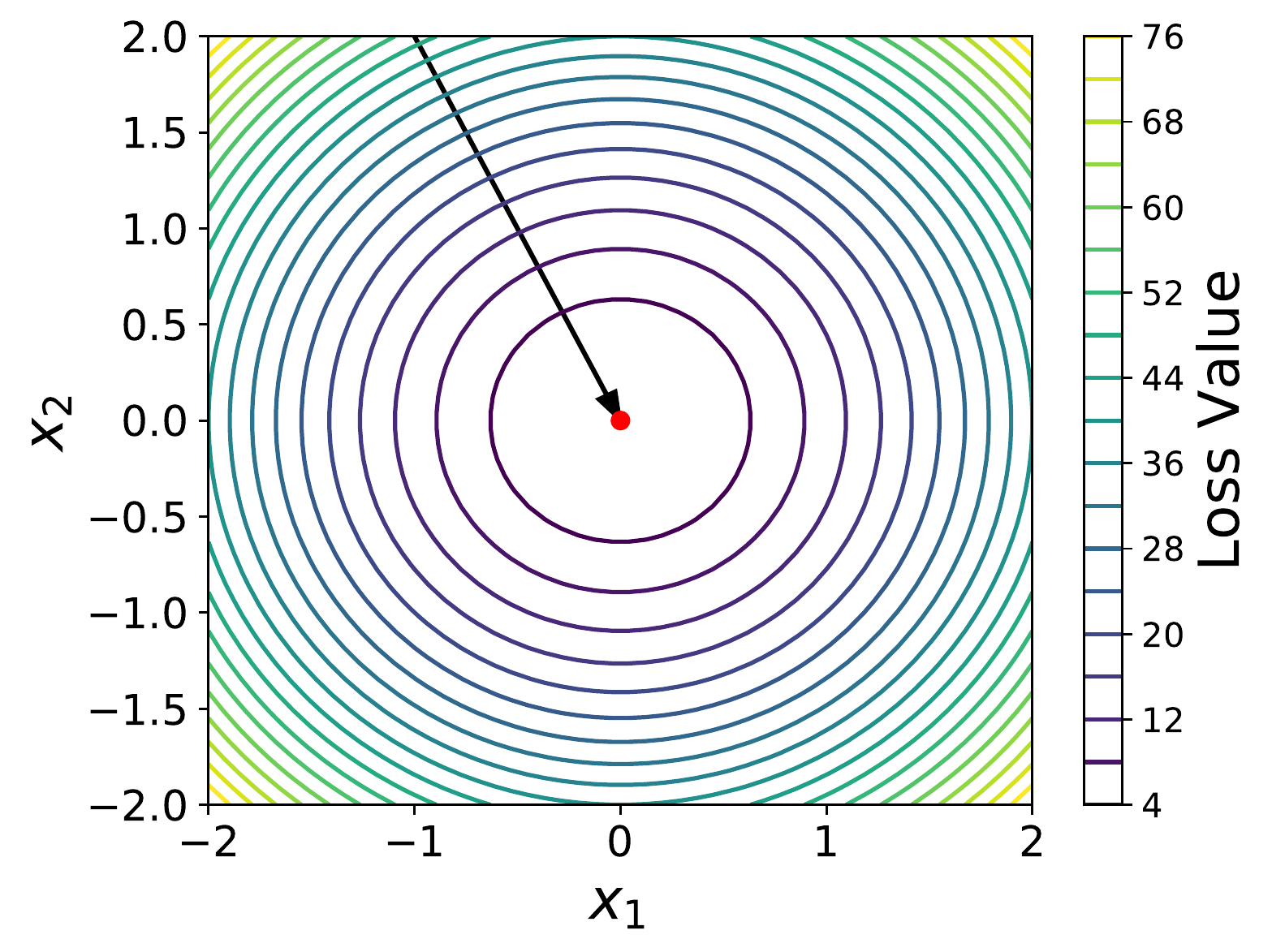}}
\caption{Illustration of special cases for GD with line search of quadratic form. $\bA=\begin{bmatrix}
		20 & 5 \\ 5 & 5
	\end{bmatrix}$, $\bb=\bzero$, $c=0$, and starting point to descent is $\bx_t=[-1.3, 4.3]^\top$ for Fig~\ref{fig:steepest_gd_bisection_eigenvector}. $\bA=\begin{bmatrix}
		20 & 0 \\ 0 & 20
	\end{bmatrix}$, $\bb=\bzero$, $c=0$, and starting point to descent is $\bx_t=[-1, 2]^\top$ for Fig~\ref{fig:steepest_gd_bisection_eigenvector_sameeigenvalue}.}
\label{fig:steepest_specialcases}
\end{figure}

\subsubsection{Special Case: Symmetric with Orthogonal Eigenvectors}
When $\bA$ is symmetric, it admits spectral decomposition (Theorem 13.1 in \citet{lu2022matrix} or Appendix~\ref{appendix:spectraldecomp}, p.~\pageref{appendix:spectraldecomp}):
$$
\bA=\bQ\bLambda\bQ^\top \in \real^{d\times d} \leadto \bA^{-1} = \bQ\bLambda^{-1}\bQ^\top,
$$ 
where the columns of $\bQ = [\bq_1, \bq_2, \ldots , \bq_d]$ are eigenvectors of $\bA$ and are mutually orthonormal, and the entries of $\bLambda = \diag(\lambda_1, \lambda_2, \ldots , \lambda_d)$ with $\lambda_1\geq \lambda_2\geq \ldots \geq \lambda_d$ are the corresponding eigenvalues of $\bA$, which are real. Since the eigenvectors are chosen to be mutually orthonormal:
$$
\bq_i^\top \bq_j =
\left\{
\begin{aligned}
&1, \gap i=j;\\
&0, \gap i\neq j,
\end{aligned}
\right.
$$
the eigenvectors also span the entire space $\real^d$ such that every error vector $\be_t  \in \real^d$ can be expressed as a linear combination of the eigenvectors:
\begin{equation}\label{equation:steepest-et-eigen-decom}
\be_t = \sum_{i=1}^{d} \alpha_i \bq_i,
\end{equation}
where $\alpha_i$ indicates the component of $\be_t$ in the direction of $\bq_i$.
Then the gradient vector (for symmetric $\bA$ by Eq.~\eqref{equation:symmetric_gd_gradient}, p.~\pageref{equation:symmetric_gd_gradient}) can be obtained by  
\begin{equation}\label{equation:steepest-eigen-decom-part}
\begin{aligned}
\bg_t &= \bA\bx_t-\bb  =\bA\be_t = 
\bA  \sum_{i=1}^{d} \alpha_i \bq_i\\
&= 
\sum_{i=1}^{d} \alpha_i \lambda_i\bq_i,
\end{aligned}
\end{equation} 
i.e., a linear combination of eigenvectors with length at the $i$-th dimension being $\alpha_i\lambda_i$.
Again by Eq.~\eqref{equation:steepest-quadratic-error}, the update for the $(t+1)$-th iteration is 
$$
\begin{aligned}
\be_{t+1} &= \be_t   - \frac{\bg_t^\top \bg_t}{ \bg_t^\top \bA\bg_t } \bg_t\\
&=\be_t - \frac{\sum_{i=1}^{d}\alpha_i^2\lambda_i^2}{ \sum_{i=1}^{d}\alpha_i^2\lambda_i^3 } \sum_{i=1}^{d} \alpha_i \lambda_i\bq_i \\
&= \bzero .
\end{aligned}
$$
The above equation indicates  that when only one component of $\alpha_i$'s is nonzero, the convergence is achieved in only one step, as illustrated in Figure~\ref{fig:steepest_gd_bisection_eigenvector}.
More specially, when $\lambda_1=\lambda_2=\ldots =\lambda_d=\lambda$, i.e., all the eigenvalues are the same, it then follows that 
$$
\begin{aligned}
\be_{t+1} &= \be_t   - \frac{\bg_t^\top \bg_t}{ \bg_t^\top \bA\bg_t } \bg_t\\
&=\be_t - \frac{\sum_{i=1}^{d}\alpha_i^2}{ \sum_{i=1}^{d}\alpha_i^2 }  \be_t \\
&= \bzero .
\end{aligned}
$$
Therefore, it takes only one step further to converge to the solution for any arbitrary $\be_t$. A specific example is shown in Figure~\ref{fig:steepest_gd_bisection_eigenvector_sameeigenvalue}, where $\bA=\begin{bmatrix}
20 & 0 \\ 0 & 20
\end{bmatrix}$, $\bb=\bzero$, and $c=0$.

\subsubsection{General Convergence Analysis for Symmetric PD Quadratic}\label{section:general-converg-steepest}
To delve into the general convergence results, we further define the \textit{energy norm} for the error vector, denoted by $\norm{\be}_{\bA} = (\be^\top\bA\be)^{1/2}$. It can be shown that minimizing $\norm{\be_t}_{\bA}$ is equivalent to minimizing $L(\bx_t)$ due to the relation:
\begin{equation}\label{equation:energy-norm-equivalent}
\norm{\be}_{\bA}^2  = 2L(\bx_t) \underbrace{- 2L(\bx_\star) -2\bb^\top\bx_\star}_{\text{constant}}.
\end{equation}
With  the definition of the energy norm, Eq.~\eqref{equation:steepest-quadratic-error}, and the symmetric positive definiteness of $\bA$, the update on the energy norm sequence is expressed as follows:
\begin{equation}\label{equation:energy-norm-leq}
\begin{aligned}
	\norm{\be_{t+1}}_{\bA}^2 &= \be_{t+1}^\top \bA\be_{t+1} \\
	&= \left(\be_t - \frac{\bg_t^\top \bg_t}{ \bg_t^\top \bA\bg_t } \bg_t\right)^\top \bA \left(\be_t - \frac{\bg_t^\top \bg_t}{ \bg_t^\top \bA\bg_t } \bg_t\right)\\
	&=\norm{\be_t}_{\bA}^2 + \left(\frac{\bg_t^\top \bg_t}{ \bg_t^\top \bA\bg_t }\right)^2 \bg_t^\top \bA\bg_t 
	- 2 \frac{\bg_t^\top \bg_t}{ \bg_t^\top \bA\bg_t } \bg_t^\top \bA \be_t   \\
	&=\norm{\be_t}_{\bA}^2 - \frac{(\bg_t^\top \bg_t)^2}{ \bg_t^\top \bA\bg_t }    \gap\gap&(\bA\be_t = \bg_t)\\
	&=\norm{\be_t}_{\bA}^2 \cdot\left(1- \frac{(\bg_t^\top \bg_t)^2}{ \bg_t^\top \bA\bg_t \cdot \be_t^\top \bA\be_t}\right)\\
	&=\norm{\be_t}_{\bA}^2  \cdot\left(1- \frac{(\sum_{i=1}^{d}\alpha_i^2\lambda_i^2)^2}{(\sum_{i=1}^{d}\alpha_i^2\lambda_i^3) \cdot (\sum_{i=1}^{d}\alpha_i^2\lambda_i)}\right)
	\gap\gap&(\text{by Eq.~\eqref{equation:steepest-et-eigen-decom}, Eq.~\eqref{equation:steepest-eigen-decom-part}} )\\
	&=\norm{\be_t}_{\bA}^2 \cdot r^2, 
\end{aligned}
\end{equation}
where $r^2 =\left(1- \frac{(\sum_{i=1}^{d}\alpha_i^2\lambda_i^2)^2}{(\sum_{i=1}^{d}\alpha_i^2\lambda_i^3) \cdot (\sum_{i=1}^{d}\alpha_i^2\lambda_i)}\right)$
determines the rate of convergence. As per convention, we assume $\lambda_1 \geq \lambda_2\geq \ldots \geq \lambda_d>0$, i.e., the eigenvalues are ordered in magnitude and positive since $\bA$ is positive definite. Then the condition number is defined as $\kappa = \frac{\lambda_1}{\lambda_d}$. 
Additionally, let $\kappa_i = \lambda_i/\lambda_d$ and $\sigma_i = \alpha_i/\alpha_1$. It follows that 
$$
r^2 = 
\left(1- 
\frac{ (\kappa^2+ \sum_{i=\textcolor{mylightbluetext}{2}}^{d}\sigma_i^2\kappa_i^2)^2}
{ (\kappa^3 + \sum_{i=\textcolor{mylightbluetext}{2}}^{d}\sigma_i^2\kappa_i^3) \cdot 
( \kappa +\sum_{i=\textcolor{mylightbluetext}{2}}^{d}\sigma_i^2\kappa_i)}
\right).
$$
Therefore, the rate of convergence is further controlled by $\kappa$, $\sigma_i$'s, and $\kappa_i$'s, where $|\kappa_i|\geq  1$ for $i\in \{2,3,\ldots,d\}$. 

\paragraph{Two-dimensional case.} Specifically, when $d=2$, we have 
\begin{equation}\label{equation:2d-rate-steepest}
r^2=
1- 
\frac{ (\kappa^2+ \sigma_2^2)^2}
{ (\kappa^3 + \sigma_2^2) \cdot 
	( \kappa +\sigma_2^2)} .
\end{equation}
Figure~\ref{fig:converge_contour_steepest} depicts the value $r^2$ as a function of $\kappa$ and $\sigma_2$. When $d=2$, from Eq.~\eqref{equation:steepest-et-eigen-decom}, we have 
\begin{equation}\label{equation:steepest-et-eigen-decom-2d}
\be_t = \alpha_1 \bq_1+\alpha_2\bq_2. 
\end{equation}
This confirms the two special examples shown in Figure~\ref{fig:steepest_specialcases}: when $\be_t$ is an eigenvector of $\bA$, it follows that:
$$
\begin{aligned}
\text{case 1: }	&\alpha_2=0\leadto  \sigma_2 = \alpha_2/\alpha_1 \rightarrow 0;\\
\text{case 2: }	&\alpha_1=0\leadto \sigma_2 = \alpha_2/\alpha_1 \rightarrow \infty,
\end{aligned}
$$
i.e., the slope of $\sigma_2$ is either zero or infinite, the rate of convergence approaches  zero and it converges instantly in just one step (example in Figure~\ref{fig:steepest_gd_bisection_eigenvector}). While if the eigenvalues are identical, $\kappa=1$, once again, the rate of convergence is zero (example in Figure~\ref{fig:steepest_gd_bisection_eigenvector_sameeigenvalue}).

\begin{figure}[h]
\centering
\includegraphics[width=0.6\textwidth]{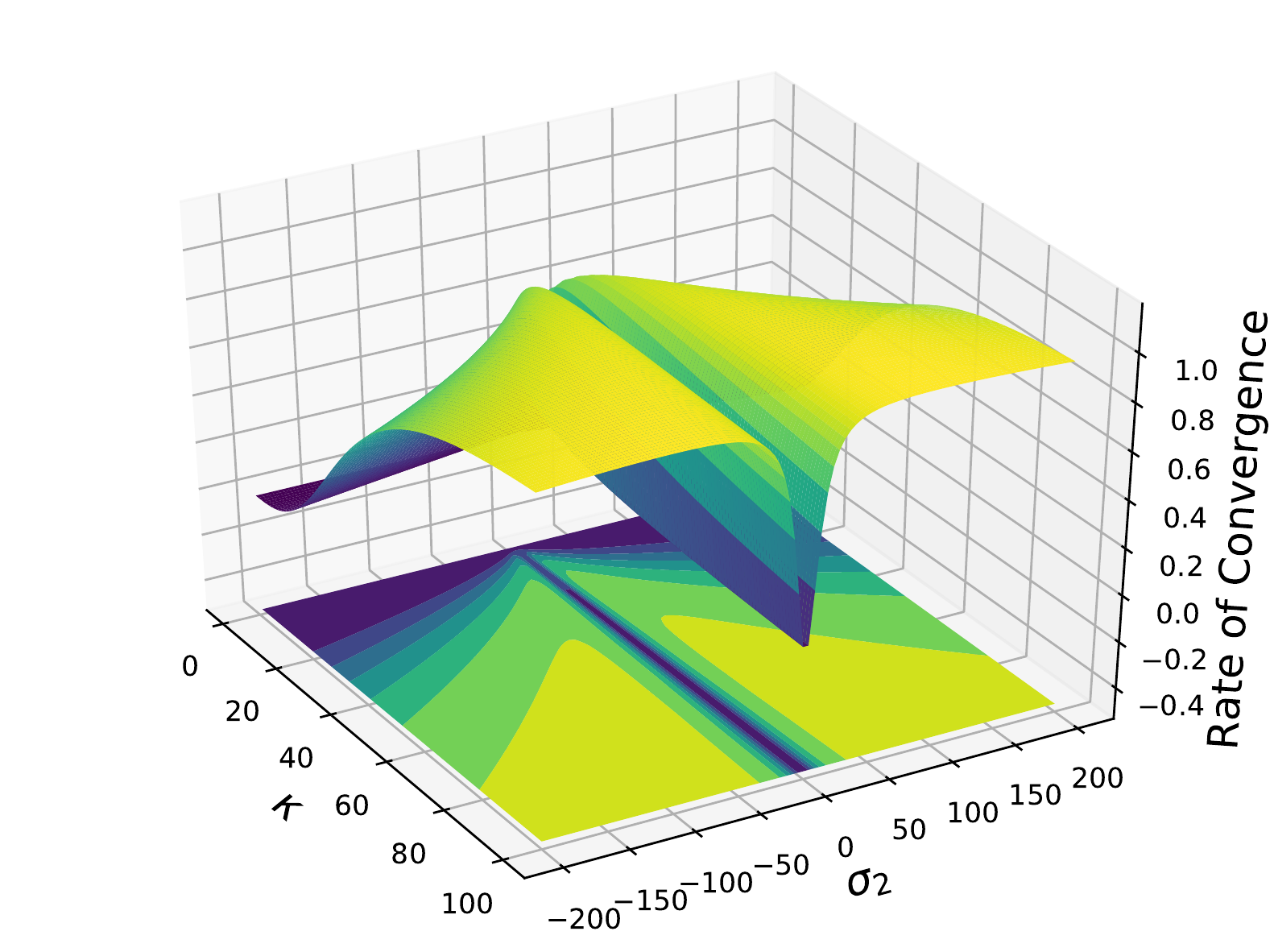}
\caption{Demonstration of the rate of convergence in steepest descent method with two-dimensional parameter. When $\sigma_2=0, \infty$, or $\kappa=1$, the rate of convergence is 0, which makes the update converges instantly in one step. The two cases correspond to $\be_t$ being an eigenvector of $\bA$ and the eigenvalues being identical, respectively, as examples shown in Figure~\ref{fig:steepest_specialcases}.}
\label{fig:converge_contour_steepest}
\end{figure}

\begin{figure}[h]
\centering
\includegraphics[width=0.5\textwidth]{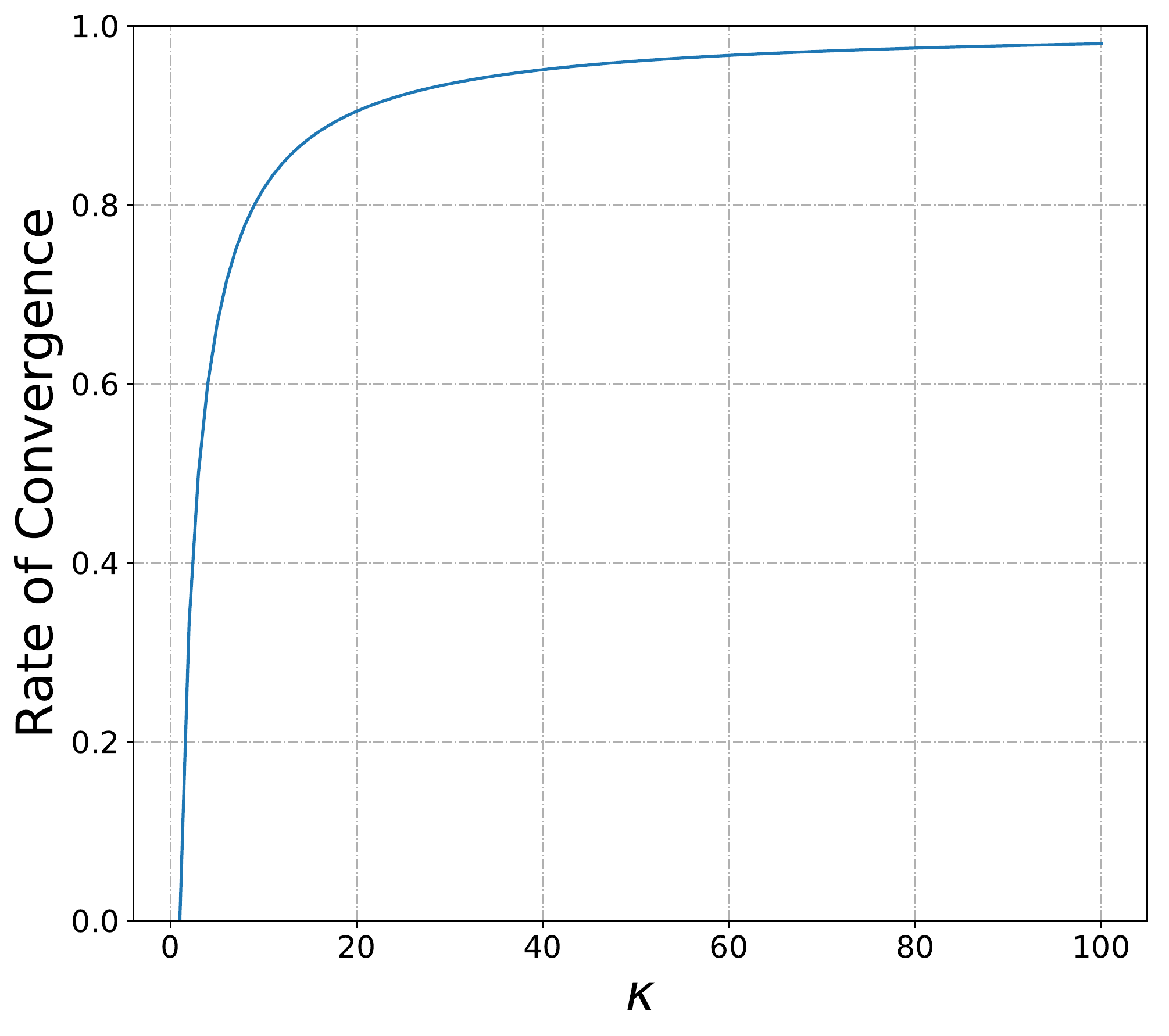}
\caption{Upper bound on the rate of convergence (per iteration) in steepest descent method with two-dimensional parameter. The $y$-axis is $\frac{\kappa-1}{\kappa+1}$.}
\label{fig:rate_convergen_steepest}
\end{figure}
\index{Rate of convergence}

\paragraph{Worst case.} We recall that $\sigma_2 = \alpha_2/\alpha_1$ determines the error vector $\be_t$ (Eq.~\eqref{equation:steepest-et-eigen-decom} or Eq.~\eqref{equation:steepest-et-eigen-decom-2d}), which, in turn, decides the point $\bx_t$ in the two-dimensional case. It is then interesting to identify the worst point to descent. Holding $\kappa$ fixed (i.e., $\bA$ and the loss function $L(\bx)$ are fixed), suppose further $t=1$, we aim to see the worst starting point $\bx_1$ to descent. It can be shown that the rate of convergence in Eq.~\eqref{equation:2d-rate-steepest} is maximized when $\sigma_2=\pm \kappa$:
$$
\begin{aligned}
r^2 &\leq 1-\frac{4 \kappa^2}{\kappa^5+2\kappa^4+\kappa^3}= \frac{(\kappa-1)^2}{(\kappa+1)^2}.
\end{aligned}
$$
Substituting into Eq.~\eqref{equation:energy-norm-leq}, we have 
$$
\begin{aligned}
\norm{\be_{t+1}}_{\bA}^2 &\leq \norm{\be_t}_{\bA}^2 \cdot \frac{(\kappa-1)^2}{(\kappa+1)^2},\\
\leadto 
&\norm{\be_{t+1}}_{\bA} \leq \norm{\be_1}_{\bA} \cdot\left(\frac{\kappa-1}{\kappa+1}\right)^t.
\end{aligned}
$$
The upper bound of the rate of convergence (per iteration) is plotted in Figure~\ref{fig:rate_convergen_steepest}. Once again, the more \textit{ill-conditioned} the matrix, the slower the convergence of steepest descent. We may notice that the (upper bound of the) rate of convergence is the same as that of the vanilla GD in Eq.~\eqref{equation:vanialla-gd-rate} (p.~\pageref{equation:vanialla-gd-rate}). However, the two are different in that the rate of the vanilla GD is described in terms of the $\by_t$ vector in Eq.~\eqref{equation:vanilla-yt} (p.~\pageref{equation:vanilla-yt}); while the rate of steepest descent is presented in terms of the energy norm. Moreover, the rate of vanilla GD in Eq.~\eqref{equation:vanialla-gd-rate} (p.~\pageref{equation:vanialla-gd-rate}) is obtained by selecting a specific learning rate, as shown in Eq.~\eqref{equation:eta-vanilla-gd} (p.~\pageref{equation:eta-vanilla-gd}), which is not practical in vanilla GD since the learning rate is fixed throughout all iterations. This makes the rate of vanilla GD more of  a tight bound. In practice, vanilla GD converges slower than steepest descent, as evident in the examples shown in Figure~\ref{fig:momentum_gd_conjugate2}, Figure~\ref{fig:momentum_gd_conjugate8}, and Figure~\ref{fig:conjguatecy_zigzag2}.

%% file: chapter-lrate.tex
\newpage
\clearchapter{Learning Rate Annealing and Warmup}
\begingroup
\hypersetup{linkcolor=winestain,
	linktoc=page,  
}
\minitoc \newpage
\endgroup

\index{Learning rate annealing}
\index{Learning rate warmup}
\section{Learning Rate Annealing and Warmup}\label{section:learning-rate-annealing}
\lettrine{\color{caligraphcolor}W}
We have discussed in Eq.~\eqref{equation:gd-equaa-gene} (p.~\pageref{equation:gd-equaa-gene}) that the learning rate $\eta$ controls how large of a step to take in the direction of the negative gradient so that we can reach a (local) minimum. In a wide range of applications, a fixed learning rate works well in practice. While there are alternative learning rate schedules that change the learning rate during learning, and it is most often changed between epochs. We shall see in the sequel that per-dimension optimizers can change the learning rate in each dimension adaptively, e.g., AdaGrad, AdaDelta, RMSProp, and AdaSmooth \citep{duchi2011adaptive, hinton2012neural, zeiler2012adadelta, lu2022adasmooth}; while in this section, we focus on strategies for decaying or annealing the global learning rate, i.e., the value of $\eta$ in Eq.~\eqref{equation:gd-equaa-gene} (p.~\pageref{equation:gd-equaa-gene}).

A constant learning rate often poses a dilemma to the analyst: a small learning rate used will cause the algorithm to take too long to reach anywhere close to an optimal solution. While a large initial learning rate will allow the algorithm to come reasonably close to a good (local) minimum in the cost surface at first; however, the algorithm will then oscillate back and forth around the point for a very long time. One method to prevent this challenge is to slow down the parameter updates by decreasing the learning rate. This can be done manually when the validation accuracy appears to plateau. 
On the other hand, decaying the learning rate over time, based on how many epochs through the data have been done, can naturally address these issues. 
Common decay functions include \textit{step decay}, \textit{inverse decay}, and \textit{exponential decay}. 
The subsequent section delves into the mathematical formulations of various learning rate annealing schemes.

\index{Learning rate annealing}
\section{Learning Rate Annealing}\label{section:learning-rate-anneal}
\paragraph{Step decay.} The \textit{step decay} scheduler drops the learning rate by a factor every epoch or every few epochs. For a given iteration $t$, number of iterations to drop $n$, initial learning rate $\eta_0$, and decay factor $d<1$, the form of step decay is given by
$$
\eta_t = \eta_0 \cdot  d^{\floor{\frac{t}{n}} }=\eta_0 \cdot  d^s,
$$
where $s={\floor{\frac{t}{n}} }$ is called the \textit{step stage} to decay. Therefore, the step decay policy decays the learning rate every $n$ iterations.

\paragraph{Multi-step decay.}
The \textit{multi-step decay} scheduler is a slightly different version of the step decay, wherein the step stage is the index where the iteration $t$ falls in the milestone vector $\bmm=[m_1,m_2,\ldots, m_k]^\top$ with $0\leq m_1\leq m_2\leq\ldots\leq m_k\leq T$ and $T$ being the total number of iterations (or epochs) \footnote{When $T$ is the total number of iterations, it can be obtained by the product of the number of epochs and the number of steps per epoch.}. To be more concrete, the step stage $s$ at iteration $t$ is obtained by
$$
s = 
\left\{
\begin{aligned}
	&0, \gap &t<m_1;\\
	&1, \gap &m_1\leq t < m_2; \\
	&\ldots\\
	&k, \gap &m_k\leq t \leq T. 
\end{aligned}
\right.
$$
As a result, given the iteration $t$, initial learning rate $\eta_0$, and decay factor $d<1$, the learning rate at iteration $t$ is calculated as 
$$
\eta_t = \eta_0 \cdot  d^s.  
$$

\begin{figure}[h]
	\centering
	\includegraphics[width=0.6\textwidth]{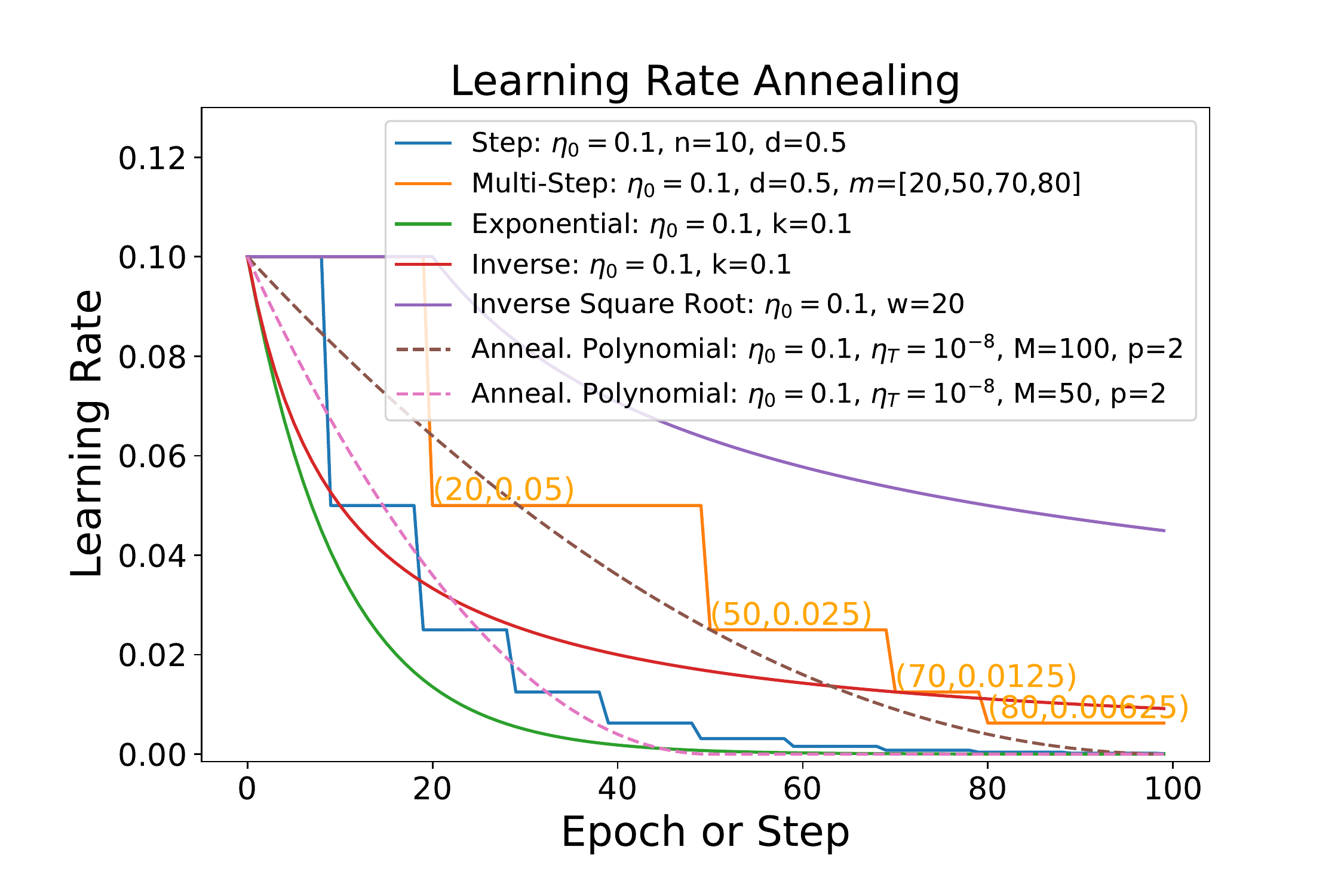}
	\caption{Demonstration of step decay, multi-step decay, annealing polynomial, inverse decay, inverse square root, and exponential decay schedulers. One may find that among the six, exponential decay exhibits the smoothest behavior, while multi-step decay is characterized by the least smoothness.}
	\label{fig:lr_step_decay}
\end{figure}

\paragraph{Exponential decay.}
Given the iteration $t$, the initial learning rate $\eta_0$, and the exponential decay factor $k$, the form of the \textit{exponential decay} is given by
$$
\eta_t = \eta_0 \cdot \exp(-k \cdot t),
$$
where the parameter $k$ controls the rate of the decay.


\paragraph{Inverse decay.}
The \textit{inverse decay} scheduler is a variant of exponential decay in that the decaying effect is applied by the inverse function. Given the iteration number $t$, the initial learning rate $\eta_0$, and the decay factor $k$, the form of the inverse decay is obtained by
$$
\eta_t = \frac{\eta_0}{1+ k\cdot t},
$$
where, again, the parameter $k$ controls the rate of the decay.

\paragraph{Inverse square root.} 
The \textit{inverse square root} scheduler is a learning rate schedule 
$$
\eta_t = \eta_0 \cdot \sqrt{w} \cdot \frac{1}{\sqrt{\max(t, w)}},
$$
where $t$ represents the current training iteration,  $w$ is the number of warm-up steps, and $\eta_0$ is the initial learning rate. This scheduler maintains a constant learning rate for the initial  steps, then exponentially decays the learning rate until the pre-training phase concludes.

\paragraph{Annealing polynomial decay.}
Given the iteration $t$, the \textit{max decay iteration} $M$, the power factor $p$, the initial learning rate $\eta_0$, and the final learning rate $\eta_T$, the \textit{annealing polynomial decay} at iteration $t$ can be obtained by 
\begin{equation}\label{equation:annealing_polynomial}
	\begin{aligned}
		& decay\_batch = \min(t, M) ;\\
		& \eta_t = (\eta_0-\eta_T)\cdot \left(1-\frac{t}{decay\_batch}\right)^{p}+\eta_T.
	\end{aligned}
\end{equation}
In practice, the default values for the parameters are: initial rate $\eta_0=0.001$, end rate $\eta_T=1e-10$, the warm up steps $M=T/2$ where $T$ is the maximal iteration number, and power rate $p=2$.

Figure~\ref{fig:lr_step_decay} compares step decay, multi-step decay, annealing polynomial decay, inverse decay, inverse square root, and exponential decay with a specified  set of parameters. 
The smoothness varies among these methods, with exponential decay exhibiting the smoothest behavior and multi-step decay being the least smooth.  
In  annealing polynomial decay, the max decay iteration parameter $M$ determines the decay rate:
\begin{itemize}
	\item When $M$ is small, the decay gets closer to that of the exponential scheduler or the step decay; however, the exponential decay has a longer tail. That is, the exponential scheduler decays slightly faster in the beginning iterations but slows down in the last few iterations.
	\item When $M$ is large, the decay gets closer to that of the multi-step decay; however, the multi-step scheduler exhibits a more aggressive behavior.
\end{itemize}

\begin{figure}[h]
	\centering  
	\vspace{-0.35cm} 
	\subfigtopskip=2pt 
	\subfigbottomskip=2pt 
	\subfigcapskip=-5pt 
	\subfigure[Training loss.]{\label{fig:mnist_scheduler_1}
		\includegraphics[width=0.47\linewidth]{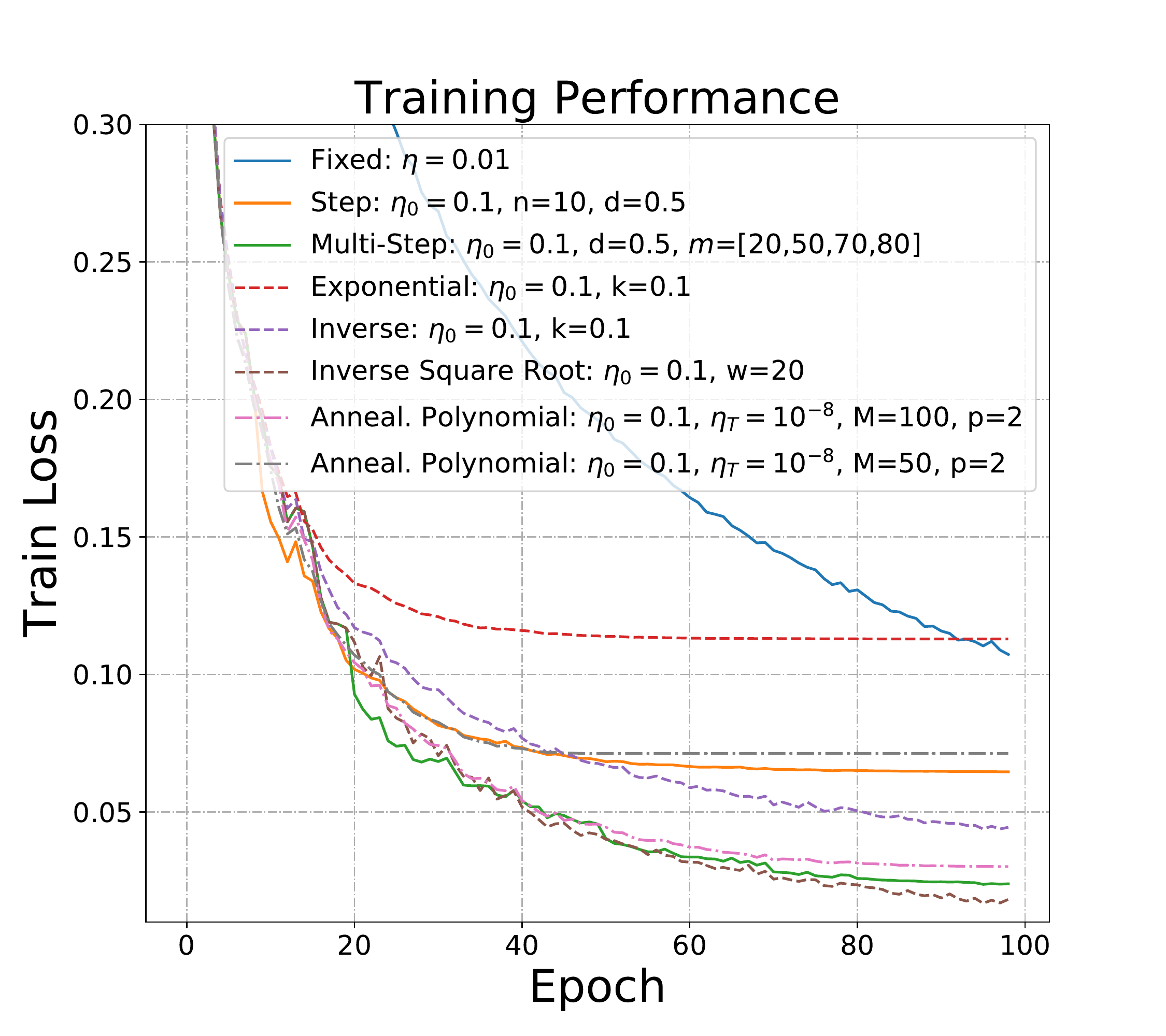}}
	\subfigure[Training accuracy.]{\label{fig:mnist_scheduler_2}
		\includegraphics[width=0.47\linewidth]{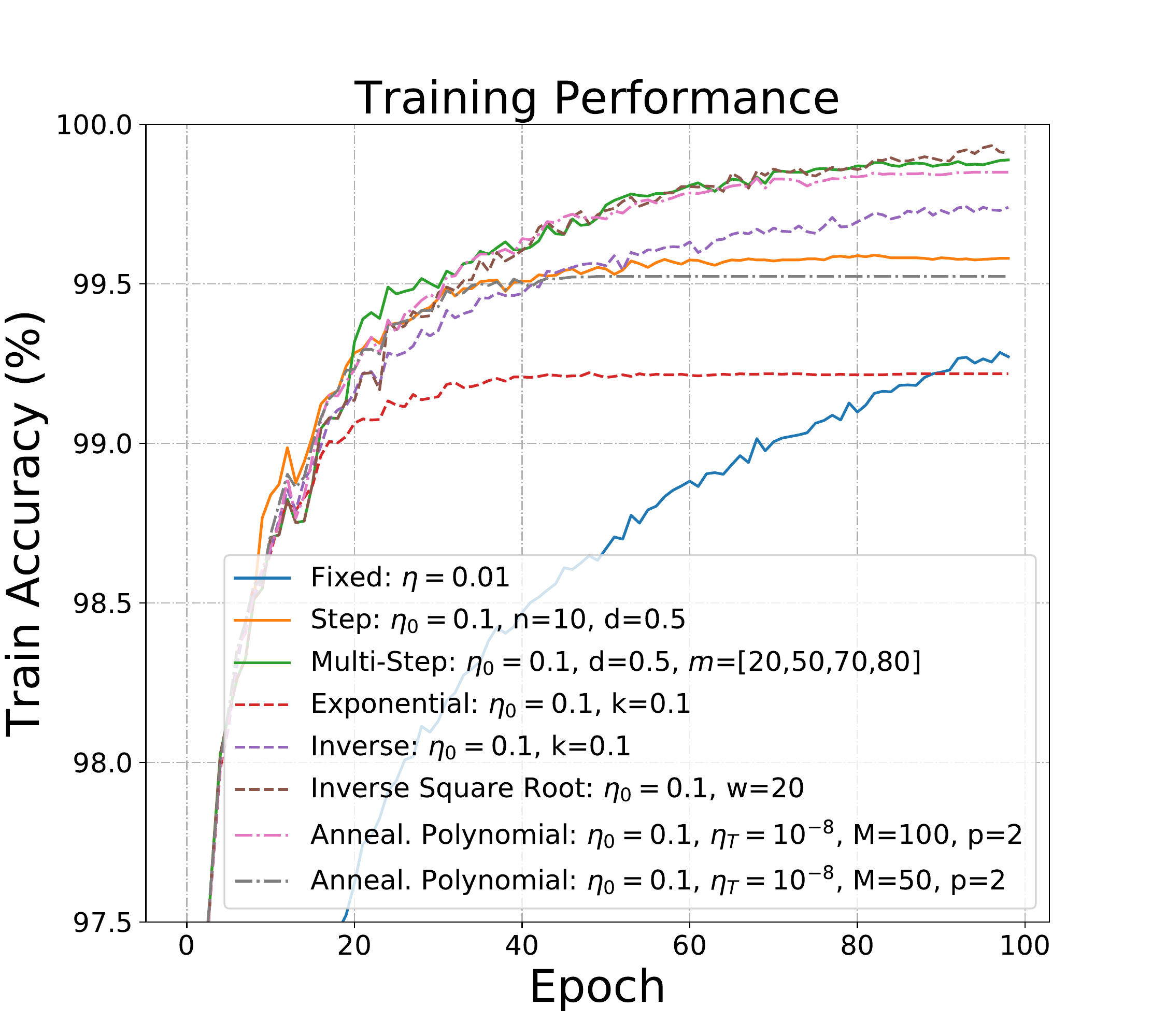}}
	\subfigure[Test loss.]{\label{fig:mnist_scheduler_3}
		\includegraphics[width=0.47\linewidth]{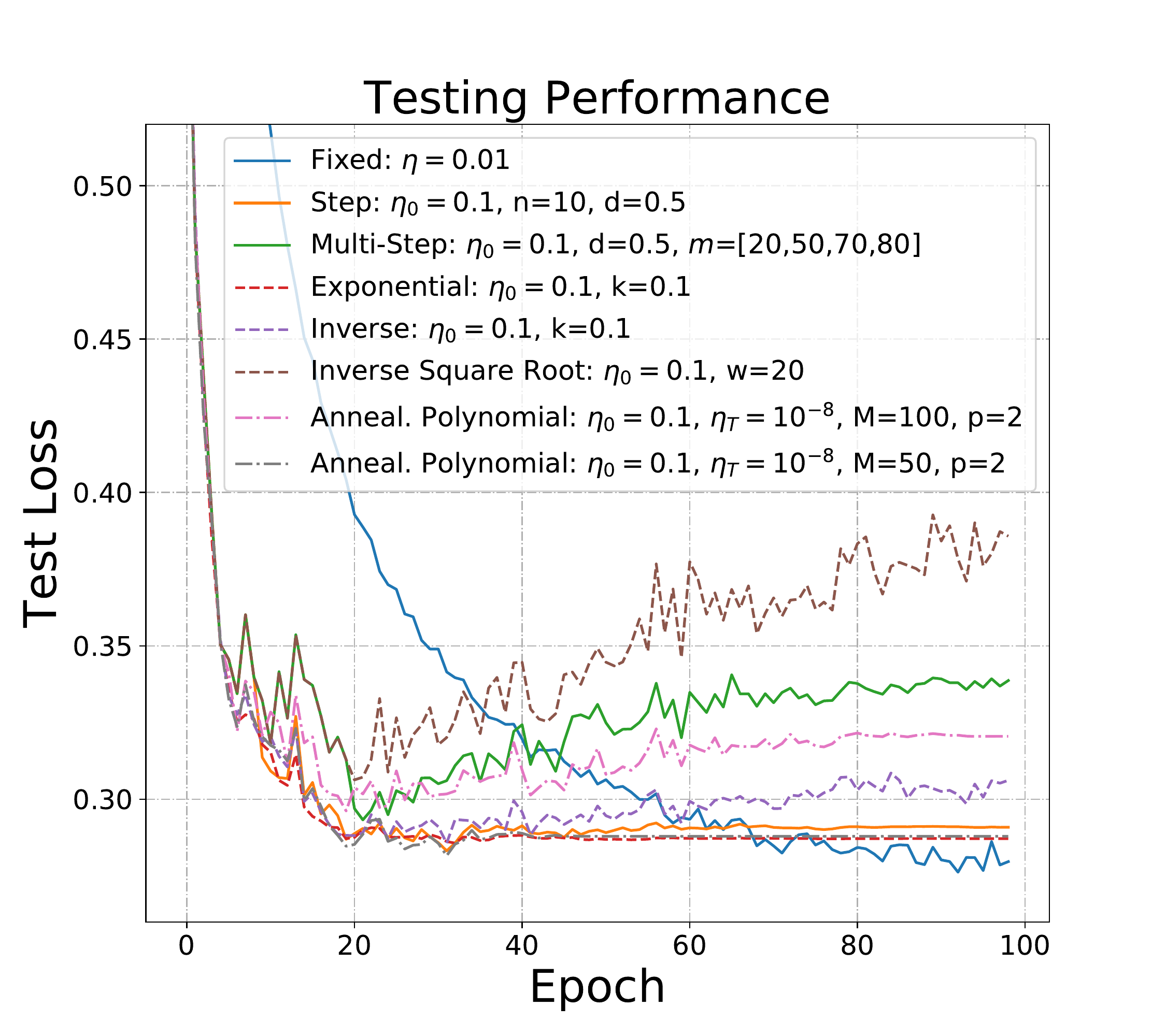}}
	\subfigure[Test accuracy.]{\label{fig:mnist_scheduler_4}
		\includegraphics[width=0.47\linewidth]{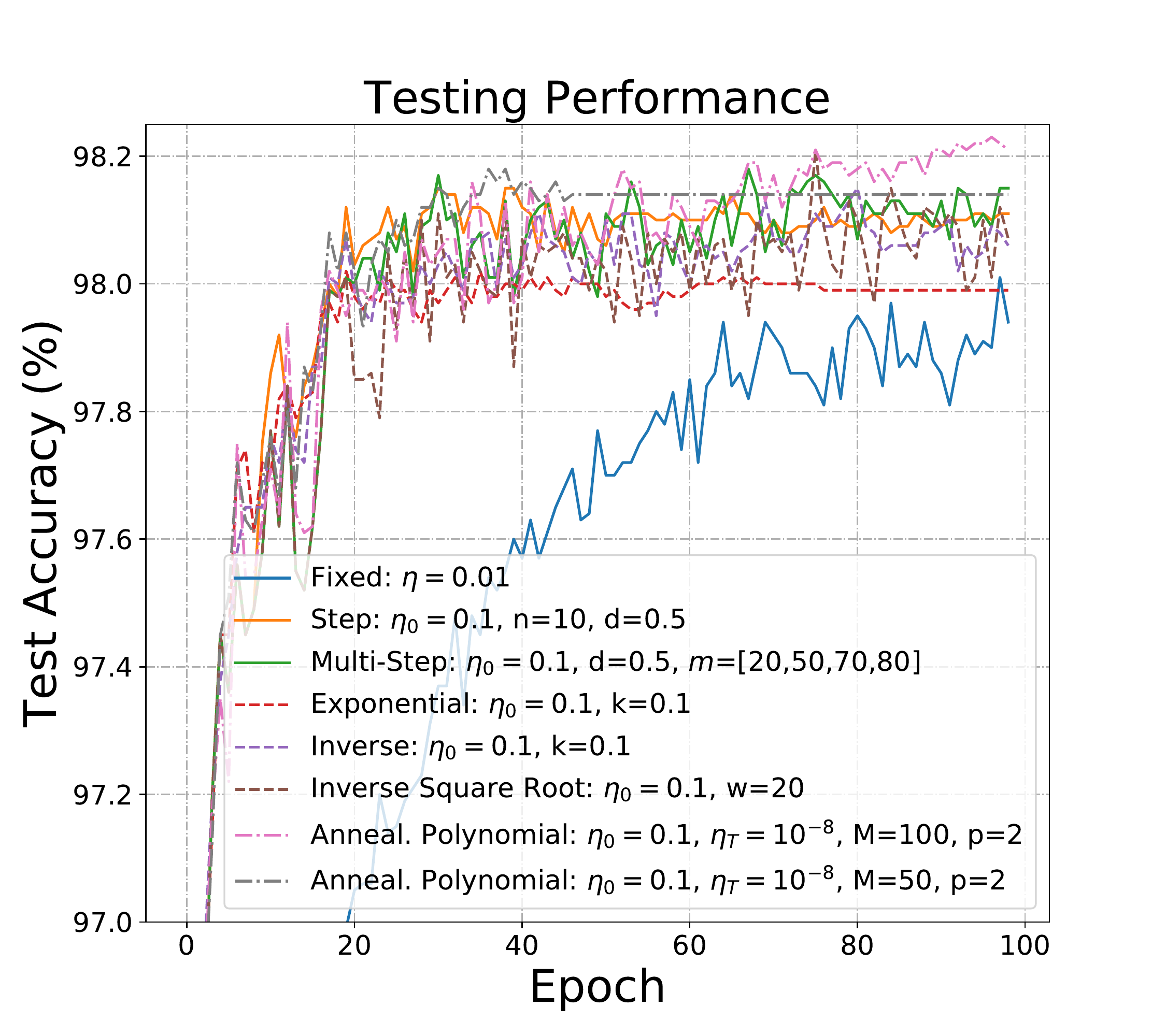}}
	\caption{Training and test performance with different learning rate schemes.}
	\label{fig:mnist_scheduler_1234}
\end{figure}

\paragraph{Toy example.}
To assess the impact of different schedulers, we utilize a toy example involving the training of a multi-layer perceptron (MLP) on the MNIST digit classification set  \citep{lecun1998mnist} \footnote{It has a training set of 60,000 examples, and a test set of 10,000 examples.}. Figure~\ref{fig:mnist_scheduler_1234} presents the training and test performance in terms of \textit{negative log-likelihood loss}. 
The parameters for various schedulers are detailed in Figure~\ref{fig:lr_step_decay} (for 100 epochs). We observe that the stochastic gradient descent method with fixed learning rate may lead to a continued reduction in test loss; however, its test accuracy may get stuck at a certain point. The toy example shows learning rate annealing schemes,  in general, can enhance optimization methods by guiding them towards better local minima with improved performance.

\index{Learning rate warmup}
\section{Learning Rate Warmup}
The concept of warmup in training neural networks receive attention in recent years \citep{he2016deep, goyal2017accurate, smith2019super}. Further insights into the efficacy of the warmup scheduler in neural machine translation (NML) can be found in the comprehensive discussion by  \citet{popel2018training}. 
The learning rate annealing schedulers can be utilized on both epoch- and step-basis. However, the learning rate warmup schemes are typically applied in the step context, where the total number of steps equals the product of the number of epochs and the number of steps per epoch, as aforementioned \citep{vaswani2017attention, howard2018universal}. Note that with this scheduler, early stopping should typically be avoided.
In the rest of this section, we delve into two commonly used warmup policies, namely, the slanted triangular learning rates (STLR) and the Noam methods.

\paragraph{Slanted Triangular Learning Rates (STLR).} STLR is a learning rate schedule that first linearly increases the learning rate over some number of epochs and then linearly decays it over the remaining epochs. The rate at iteration $t$ is computed as follows: 
$$
\begin{aligned}
	cut &= \ceil{T\cdot frac} ;\\
	p &= 
	\left\{
	\begin{aligned}
		&t/cut , \gap &\text{if }t<cut;\\
		& 1 - \frac{t-cut }{cut\cdot (1/frac-1)}, \gap &\text{otherwise};
	\end{aligned} 
	\right.\\
	\eta_t &= \eta_{\text{max}} \cdot \frac{1+p\cdot (ratio - 1)}{ratio },\\
\end{aligned}
$$
where $T$ is the number of training iterations (the product of the number of epochs and the number of updates per epoch), $frac$ is the fraction of iterations we want to increase the learning rate, $cut$ is the iteration when we switch from increasing to decreasing the learning rate, $p$ is the fraction of the number of iterations we have increased or decreased the learning rate respectively, $ratio$ specifies how much smaller the lowest learning rate is from the maximum learning rate $\eta_{\text{max}}$. In practice, the default values are $frac=0.1$, $ratio=32$, and $\eta_{\text{max}}=0.01$ \citep{howard2018universal}.

\paragraph{Noam.} The Noam scheduler is originally used in neural machine translation (NML) tasks and is proposed in \citet{vaswani2017attention}. This corresponds to increasing the learning rate linearly for the first ``warmup\_steps" 
training steps and decreasing it thereafter proportionally to the inverse square root of the
step number, scaled by the inverse square root of the dimensionality of the model (linear warmup for a given number of steps followed by exponential decay). Given the warmup steps $w$ and the model size $d_{\text{model}}$ (representing the hidden size parameter which dominates the number of parameters in the model), the learning rate $\eta_t$ at step $t$ can be calculated by 
$$
\eta_t = \alpha \cdot  \frac{1}{\sqrt{d_{\text{model}}}}\cdot  \min \left(\frac{1}{\sqrt{t}} ,  \frac{t}{w^{3/2}}\right),
$$
where $\alpha$ is a smoothing factor. In the original paper, the warmup step $w$ is set to $w=4000$. While in practice, $w=25000$ can be a good choice.

\begin{figure}[h]
	\centering
	\includegraphics[width=0.6\textwidth]{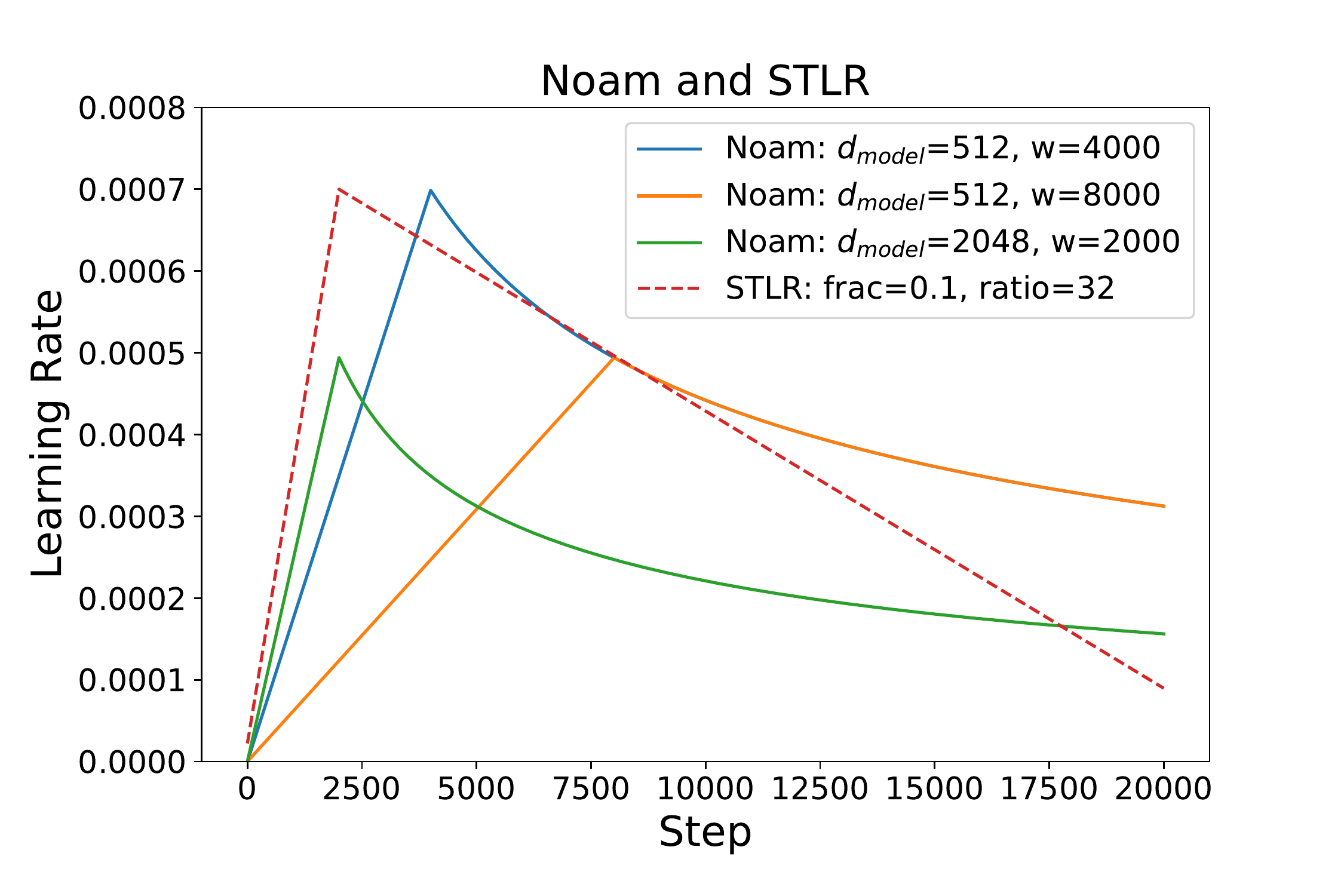}
	\caption{Comparison of Noam and STLR schedulers.}
	\label{fig:lr_noam}
\end{figure}
Moreover, in rare cases, the model size is occasionally  set to be the same as the warmup steps, resulting in what is known as the \textit{warmup Noam scheduler}:
$$
\eta_t = \alpha \cdot \frac{1}{\sqrt{w}} \cdot  \min \left(\frac{1}{\sqrt{t}} ,  \frac{t}{w^{3/2}}\right).
$$
\noindent Figure~\ref{fig:lr_noam} compares STLR and Noam schedulers with various parameters. We may observe that, in general, the Noam scheduler decays slower when the warmup phase finishes compared to the STLR.

%

\section{Cyclical Learning Rate (CLR) Policy}\label{section:cyclical-lr}
The cyclical learning rate is a generalization of warmup and decay policies (Noam scheme or STLR policy  typically involve only one cycle). The essence of this learning rate policy comes from the observation that a temporary increase in the learning rate might have a short-term negative effect and yet achieves a long-term beneficial effect. This observation leads to the idea of letting the learning rate vary within a range of values rather than adopting a stepwise fixed or exponentially decreasing value, where minimum and maximum boundaries are set to make the learning rate vary between them. The simplest function to adopt this idea is the triangular window function that linearly increases and then linearly decreases \citep{smith2017cyclical}.

\citet{dauphin2014identifying, dauphin2015equilibrated} argue that the difficulty in minimizing the loss arises from saddle points (toy example in Figure~\ref{fig:quadratic_saddle}, p.~\pageref{fig:quadratic_saddle}) rather than poor local minima. Saddle points have small gradients that slow the pace of the learning process. However, increasing the learning rate enables more rapid traversal of saddle point plateaus. In this scenario, a cyclical learning rate policy with periodical increasing and decreasing of the learning rate between minimum and maximum boundaries is reasonable. The minimum and maximum boundaries are problem-specific. 
Typically, the model is run for several epochs with different learning rates ranging from low to high values, known as the \textit{learning rate range test}. 
In this case, plotting the accuracy versus learning rate helps identify suitable minimum and maximum boundaries: when the accuracy starts to increase and when the accuracy slows, becomes ragged, or starts to fall, the two of which constitute good choices for the minimum and maximum boundaries.

The cyclical learning rate policies fall into two categories: the one based on iteration, and the one based on epoch. The former one implements  the annealing and warmup at each iteration, while the latter one does so on an epoch-basis \footnote{The total number of iterations equals the product of the number of epochs and the number of updates per epoch.}. However, there is no significance distinction  between the two; any policy can be applied in either one of the two fashions. In the following paragraphs, we will discuss the update policies based on their original proposals.

\paragraph{Triangular, Triangular2, and Exp Range.} 
The \textit{triangular} policy involves a linear increase and decrease of the learning rate. 
Given the initial learning rate $\eta_0$ (the lower boundary in the cycle), the maximum learning rate $\eta_{\max}$, and the step size $s$ (number of training iterations per half cycle), the learning rate $\eta_t$ at iteration $t$ can be obtained by:
$$
triangular: \gap 
\left\{
\begin{aligned}
	{cycle}&= \floor{1+\frac{t}{2s}};\\
	x &= \text{abs}\left(\frac{t}{s} - 2\times {cycle}+1\right);\\
	\eta_t &=\eta_0 + (\eta_{\max}-\eta_0) \cdot \max(0, 1-x),\\
\end{aligned}
\right.
$$
where the calculated \textit{cycle} indicates which cycle iteration $t$ is in.
The same as the \textit{triangular} policy, the \textit{triangular2} policy cuts in half at the end of each cycle:
$$
triangular2:\gap \eta_t =\eta_0 + (\eta_{\max}-\eta_0) \cdot \max(0, 1-x) \cdot \frac{1}{2^{\text{cycle}-1}}.
$$
Less aggressive than the \textit{triangular2} policy, the amplitude of a cycle in \textit{exp\_range} policy is scaled exponentially based on $\gamma^t$, where $\gamma<1$ is the scaling constant:
$$
exp\_range:\gap \eta_t =\eta_0 + (\eta_{\max}-\eta_0) \cdot \max(0, 1-x) \cdot \gamma^{t}.
$$ 
A comparison of these three policies is presented in Figure~\ref{fig:lr_triangular}. In practice, the step size $s$  is typically set to $2\sim 10$ times the number of iterations in an epoch \citep{smith2017cyclical}.

\begin{figure}[h]
	\centering
	\includegraphics[width=0.6\textwidth]{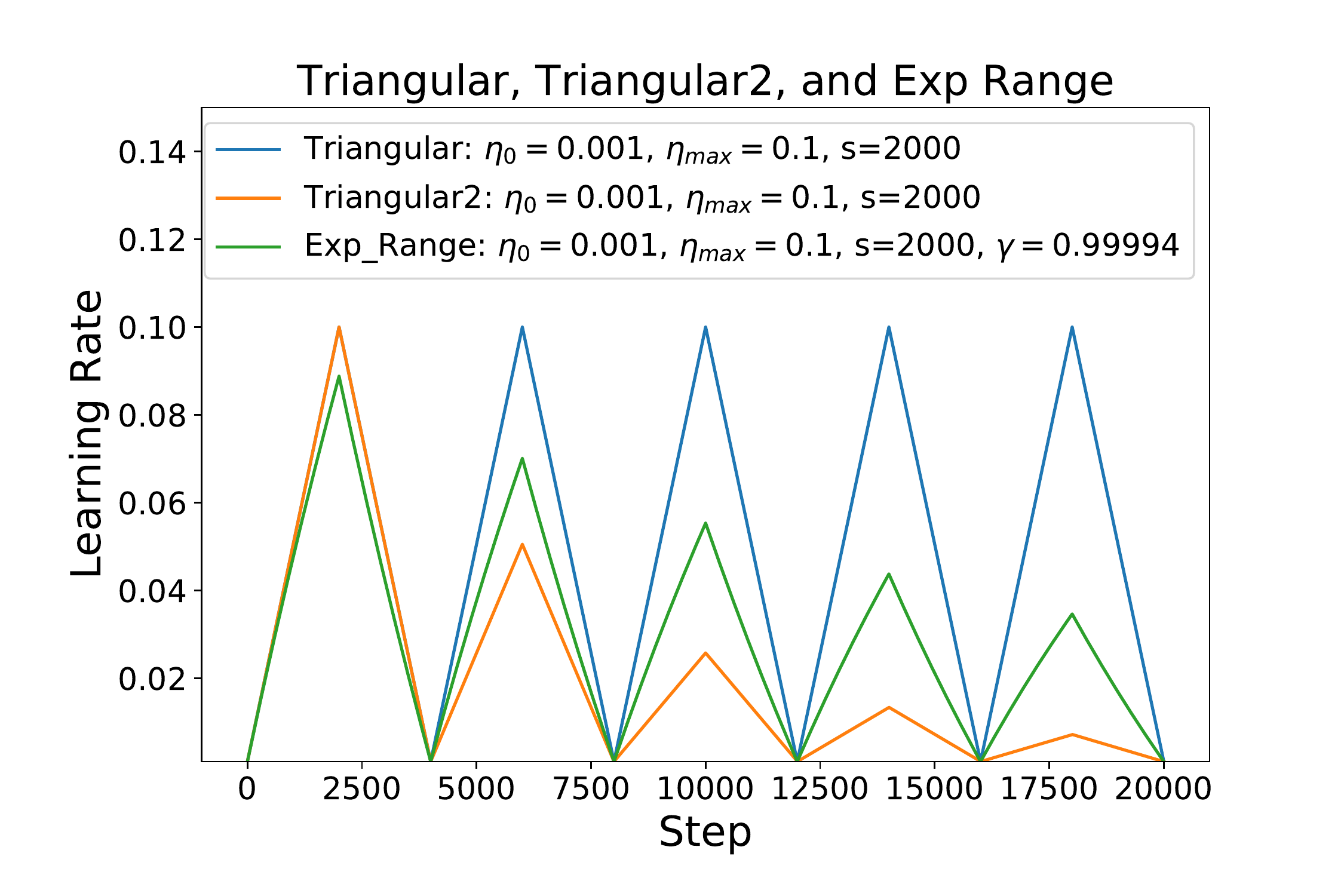}
	\caption{Demonstration of \textit{triangular}, \textit{triangular2}, and \textit{exp\_range} schedulers.}
	\label{fig:lr_triangular}
\end{figure}

\paragraph{Cyclical cosine.} 

The \textit{Cyclical cosine} is a type of learning rate scheduler that initiates with a high learning rate, rapidly decreases it to a minimum value, and then quickly increases it again.
The resetting of the learning rate acts as a simulated restart of the learning process and the re-use of good weights as the starting point of the restart is referred to as a ``warm restart" in contrast to a ``cold restart," where a new set of small random numbers may be used as a starting point \citep{loshchilov2016sgdr, huang2017snapshot}. The learning rate $\eta_t$ at iteration $t$ is calculated as follows:
$$
\eta_t = \frac{\eta_0}{2} \left( \cos \left(  \frac{\pi \,\, \text{mod}(t-1,\ceil{T/M} ) }{\ceil{T/M}}  \right)+1\right),
$$
where $T$ is the total number of training iterations (note the original paper takes the iterations as epochs in this sense \citep{loshchilov2016sgdr}), $M$ is the number of cycles, and $\eta_0$ is the initial learning rate. The scheduler anneals the learning rate from its initial value $\eta_0$ to a small learning rate approaching 0 over the course of a cycle. That is, we split the training process into $M$ cycles as shown in Figure~\ref{fig:lr_cosine}, each of which starts with a large learning rate $\eta_0$ and then gets annealed to a small learning rate.
The provided equation facilitates a rapid decrease in the learning rate, encouraging the model to converge towards its first local minimum after a few epochs. The optimization then continues at a larger learning rate that can perturb the model and dislodge it from the minimum \footnote{The goal of the procedure is similar to the perturbed SGD that can help escape from saddle points \citep{jin2017escape, du2017gradient}.}. The iterative procedure is then repeated several times to achieve multiple convergences. 
In practice, the iteration $t$ usually refers to the $t$-th epoch. More generally, any learning rate with general function $f$ in the following form can have a similar effect:
$$
\eta_t = f(\text{mod}(t-1, \ceil{T/M})).
$$
Moreover, the learning rate can be set for each batch rather than before each epoch to introduce more nuance to the updates \citep{huang2017snapshot}.


\begin{figure}[h]
	\centering  
	\vspace{-0.35cm} 
	\subfigtopskip=2pt 
	\subfigbottomskip=2pt 
	\subfigcapskip=-5pt 
	\subfigure[Cyclical Cosine.]{\label{fig:lr_cosine}
		\includegraphics[width=0.48\linewidth]{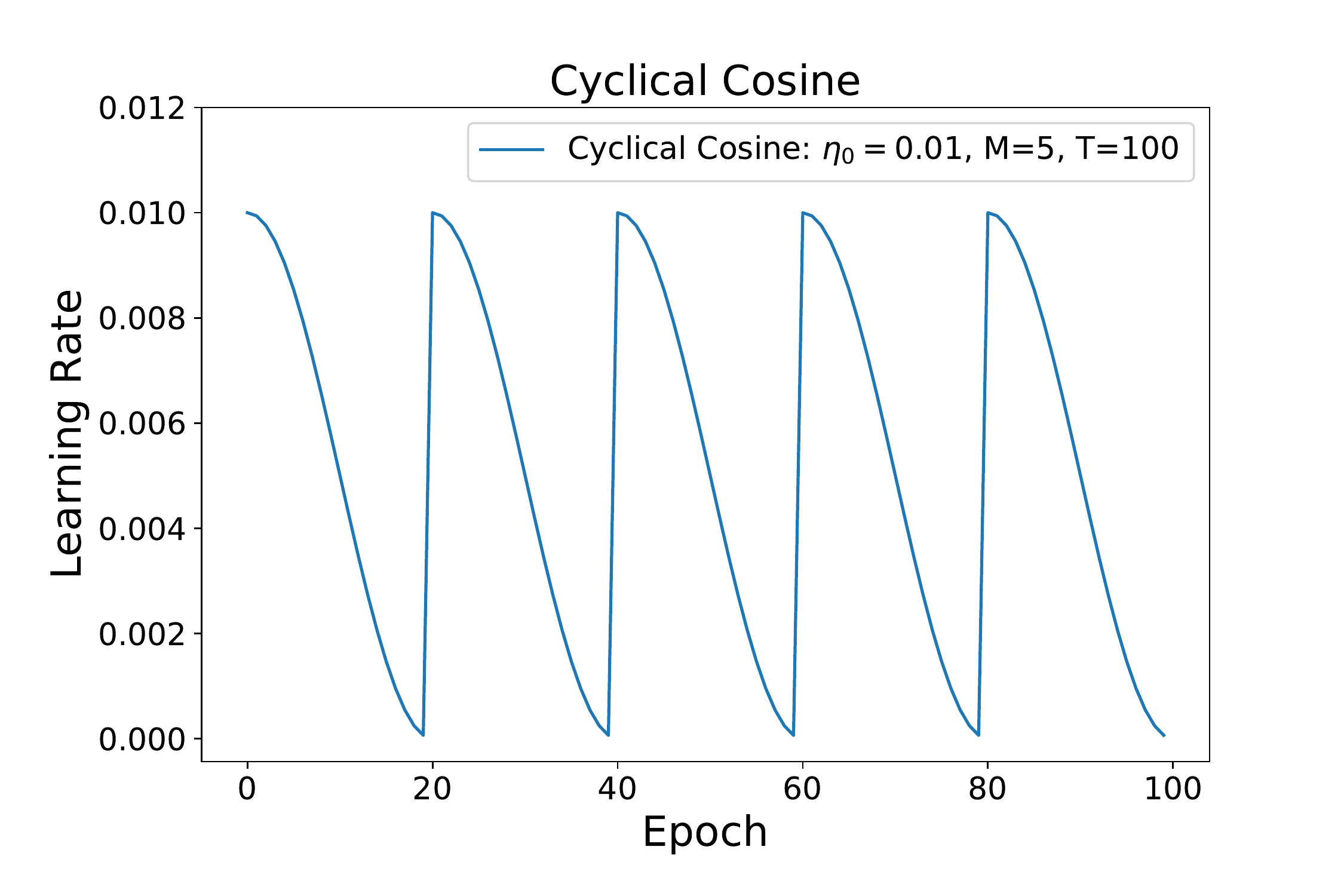}}
	\subfigure[Cyclical step.]{\label{fig:lr_cyclical_step}
		\includegraphics[width=0.48\linewidth]{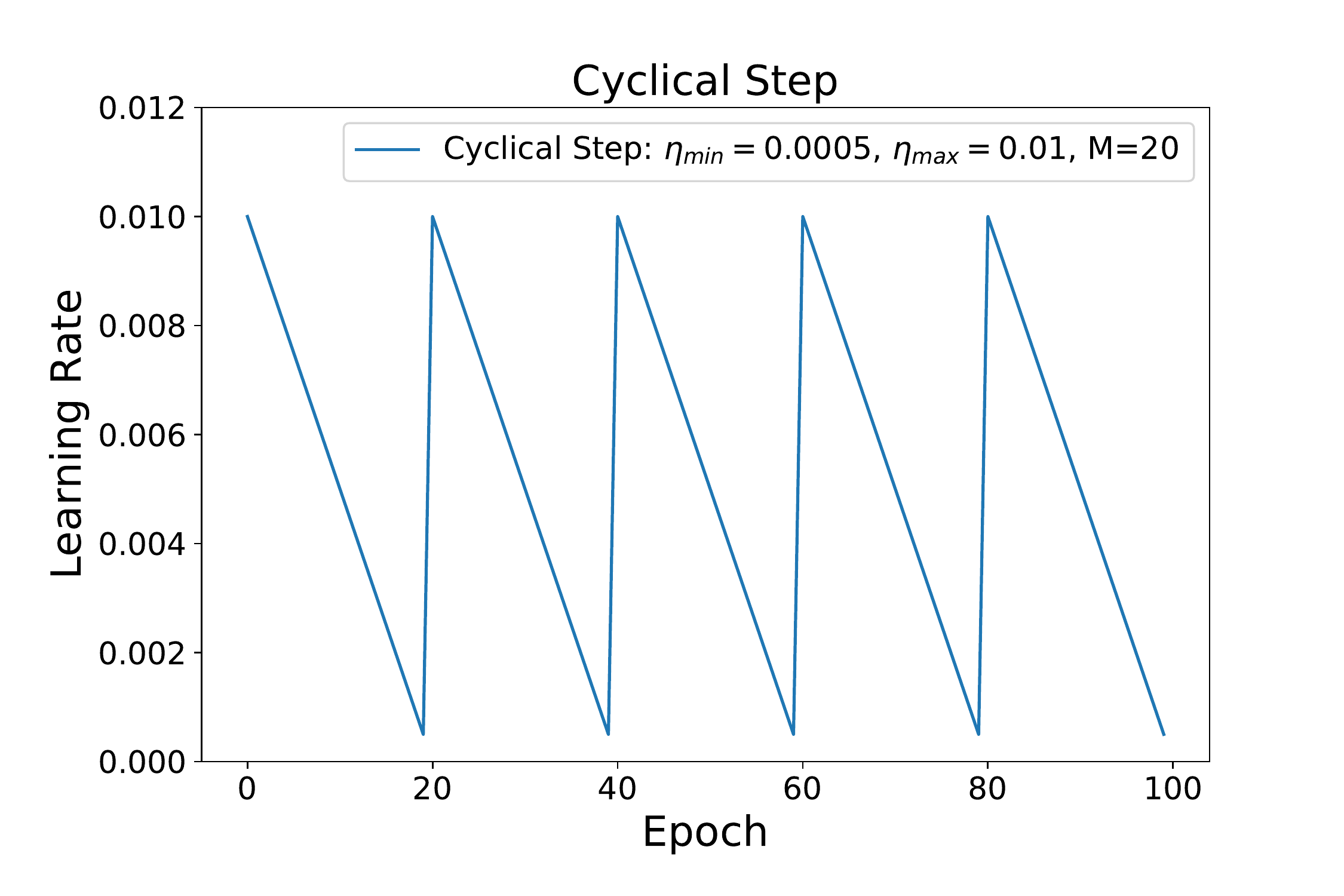}}
	\caption{Cyclical cosine and cyclical step learning rate policies.}
	\label{fig:lr_cosine_cyclical_step}
\end{figure}

\paragraph{Cyclical step.} Similar to the cyclical cosine scheme, the \textit{cyclical step learning rate policy} combines a linear learning rate decay with warm restarts \citep{mehta2019espnetv2}:
$$
\eta_t  =\eta_{\text{max}} - (t \,\, \text{mod}\,\, M) \cdot \eta_{\text{min}}.
$$
where in the original paper, $t$ refers to the epoch count, $\eta_{\text{min}}$ and $\eta_{\text{max}}$ are the ranges for the learning rate, and $M$ is the cycle length after which the learning rate will restart. The learning rate scheme can be seen as a variant of the cosine learning policy as discussed above and the comparison between the two policies is shown in Figure~\ref{fig:lr_cosine_cyclical_step}. In practice, $\eta_{\text{min}}=0.1$, $\eta_{\text{max}}=0.5$, and $M=5$ are set as default values in the original paper. 

\paragraph{Cyclical polynomial.}
The \textit{cyclical polynomial} is a variant of the \textit{annealing polynomial decay} (Eq.~\eqref{equation:annealing_polynomial}) scheme, where the difference is that the cyclical polynomial scheme employs a cyclical warmup similar to the $exp\_range$ policy. Given the iteration number $t$, the initial learning rate $\eta_0$, the final learning rate, $\eta_T$, and the maximal decay number $M<T$, the rate can be calculated by:
$$
\begin{aligned}
	& decay\_batch = M\cdot \ceil{\frac{t}{M}} \\
	& \eta_t= (\eta_0-\eta_T)\cdot \left(1-\frac{t}{decay\_batch+\epsilon}\right)^{p}+\eta_T,
\end{aligned}
$$
where $\epsilon=1e-10$ is applied for better conditioning when $t=0$. 
Figure~\ref{fig:lr_cyclic_polynomial_decay} presents the cyclical polynomial scheme with various parameters.

\begin{figure}[h]
	\centering
	\includegraphics[width=0.6\textwidth]{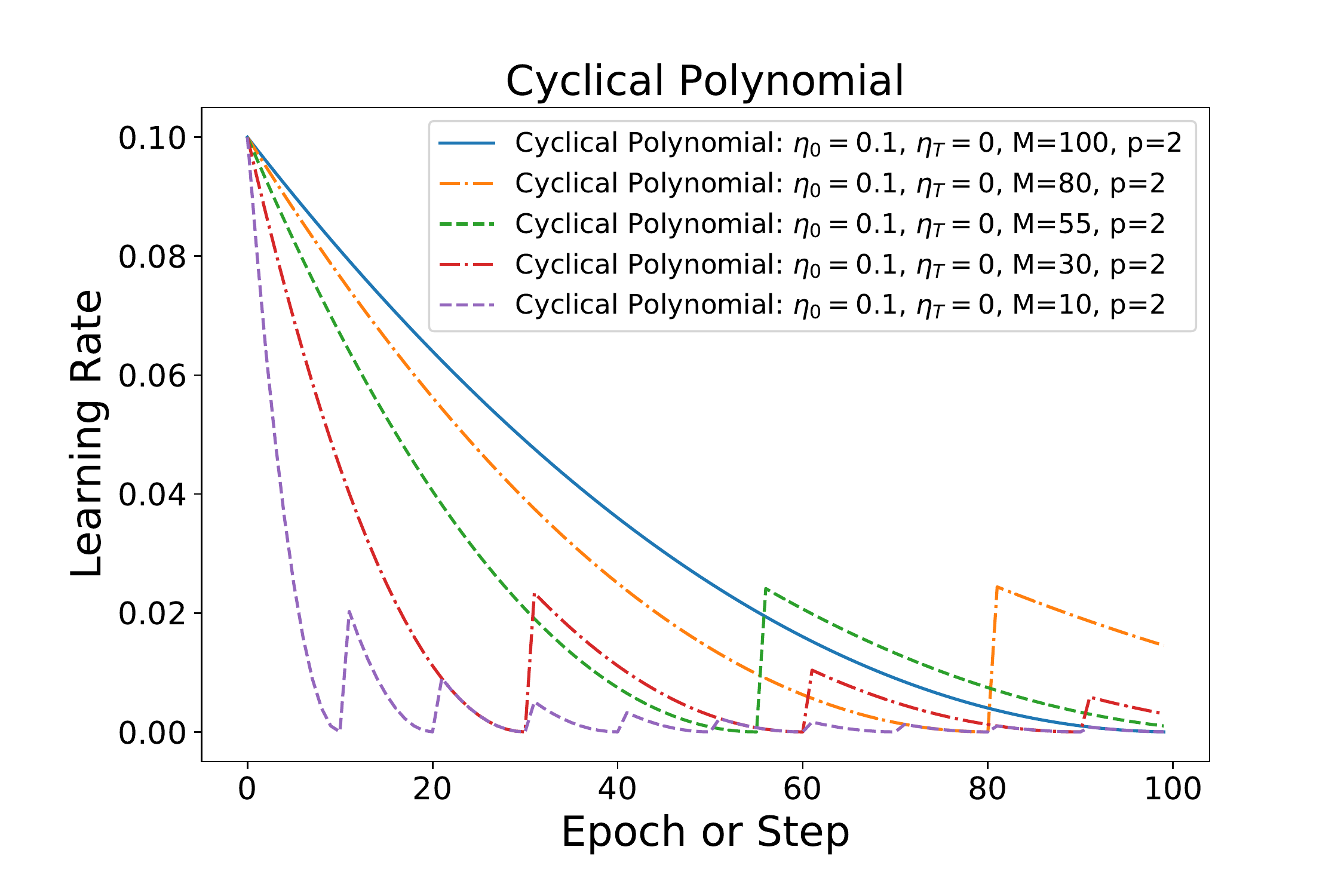}
	\caption{Demonstration of cyclical polynomial scheduler with various parameters.}
	\label{fig:lr_cyclic_polynomial_decay}
\end{figure}

%
%

%% file: chapter-stocopt.tex
\newpage
\clearchapter{Stochastic Optimizer}
\begingroup
\hypersetup{linkcolor=winestain,
linktoc=page,  
}
\minitoc \newpage
\endgroup
\index{Stochastic optimizer}
\section{Stochastic Optimizer}
\lettrine{\color{caligraphcolor}O}
Over the years, stochastic gradient-based optimization has emerged as a fundamental method in various fields of science and engineering, including computer vision and automatic speech recognition processing \citep{krizhevsky2012imagenet, hinton2012deep, graves2013speech}. Stochastic gradient descent (SGD) and deep neural network (DNN) play a core role in training stochastic objective functions. When a new deep neural network is developed for
a given task, some hyper-parameters related to the training of the network must be chosen heuristically. For each possible combination of structural hyper-parameters, a new network is typically
trained from scratch and evaluated over and over again. While much progress has been made on hardware (e.g., Graphical Processing Units) and software (e.g., cuDNN) to speed up the training time of a single structure
of a DNN, the exploration of a large set of possible structures remains very slow, making the need of a stochastic optimizer that is insensitive to hyper-parameters. Efficient stochastic optimizers, therefore, play a crucial role in training deep neural networks.
\begin{table}[h]
\begin{tabular}{llll|lll}
	\hline
	Method         & Year & Papers \gap \gap  &  & Method                  & Year & Papers \\ \hline
	Adam           & 2014 & 7532   &  & AdamW                   & 2017 & 45     \\
	SGD            & 1951 & 1212   &  & Local SGD               & 2018 & 41     \\
	RMSProp        & 2013 & 293    &  & Gravity                 & 2021 & 37     \\
	Adafactor      & 2018 & 177    &  & AMSGrad                 & 2019 & 35     \\
	Momentum       & 1999 & 130    &  & LARS                    & 2017 & 31     \\
	LAMB           & 2019 & 126    &  & MAS                     & 2020 & 26     \\
	AdaGrad        & 2011 & 103    &  & DFA                     & 2016 & 23     \\
	Deep Ensembles & 2016 & 69     &  & Nesterov momentum    & 1983 & 22     \\
	FA             & 2014 & 46     &  & Gradient Sparsification & 2017 & 20     \\ \hline
\end{tabular}
\caption{Data retrieved on April 27th, 2022 via https://paperswithcode.com/.}
\label{table:stochastic-optimizers}
\end{table}

There are several variants of SGD to use heuristics for estimating a good learning rate at each iteration of the progress. These methods either attempt to accelerate learning when suitable or to slow down learning near a local minimum. In this section, we introduce a few stochastic optimizers that are in the two categories. Table~\ref{table:stochastic-optimizers} provides an overview of the number of papers utilizing these optimizers for specific tasks and their publication dates. For additional comprehensive reviews, one can also check \citet{zeiler2012adadelta}, \citet{ruder2016overview}, \citet{goodfellow2016deep}, and many others.



\section{Momentum }
If the cost surface is not spherical, learning can be quite slow because the learning rate must be kept small to prevent divergence along the steep curvature directions \citep{polyak1964some, rumelhart1986learning, qian1999momentum, sutskever2013importance}. 
The SGD with momentum (that can be applied to full batch or mini-batch learning) attempts to use the previous step to speed up learning when suitable such that it enjoys better convergence rates in deep networks.
The main idea behind the momentum method is to speed up the learning along dimensions where the gradient consistently points in the same direction; and to slow the pace along dimensions in which the sign of the gradient continues to change. Figure~\ref{fig:momentum_gd} shows a set of updates for vanilla GD, where we can find the update along dimension $x_1$ is consistent; and the move along dimension $x_2$ continues to change in a zigzag pattern. The GD with momentum keeps track of past parameter updates with an exponential decay, and the update method has the following form from iteration $t-1$ to iteration $t$:
\begin{equation}
\begin{aligned}
\Delta \bx_t &= \rho \Delta \bx_{t-1} - \eta \frac{\partial L(\bx_t)}{\partial \bx_t},\\
\end{aligned}
\end{equation}
where the algorithm remembers the latest update and adds it to the present update by multiplying a parameter $\rho$, called the \textit{momentum parameter}, blending the present update with the past update. 
That is, the amount we change the parameter is proportional to the negative gradient plus the previous weight change; the added \textit{momentum term} acts as both a smoother and an accelerator.
The momentum parameter $\rho$ works as a \textit{decay constant}, where $\Delta\bx_1$ may have an effect on $\Delta \bx_{100}$; however, its effect is decayed by this decay constant. In practice, the momentum parameter $\rho$ is usually set to be 0.9 by default.
Momentum simulates the concept of inertia in physics. This means that in each iteration, the update mechanism is not only related to the gradient descent, which refers to the \textit{dynamic term}, but also maintains a component that is related to the direction of the last update iteration, which refers to the momentum.

\begin{figure}[h]
\centering  
\vspace{-0.35cm} 
\subfigtopskip=2pt 
\subfigbottomskip=2pt 
\subfigcapskip=-5pt 
\subfigure[
A two-dimensional surface plot for quadratic convex function.
]{\label{fig:alsgd12}
\includegraphics[width=0.47\linewidth]{./imgs/momentum_surface.pdf}}
\subfigure[The contour plot of $L(\bx)$. The red dot is the optimal point.]{\label{fig:alsgd22}
\includegraphics[width=0.44\linewidth]{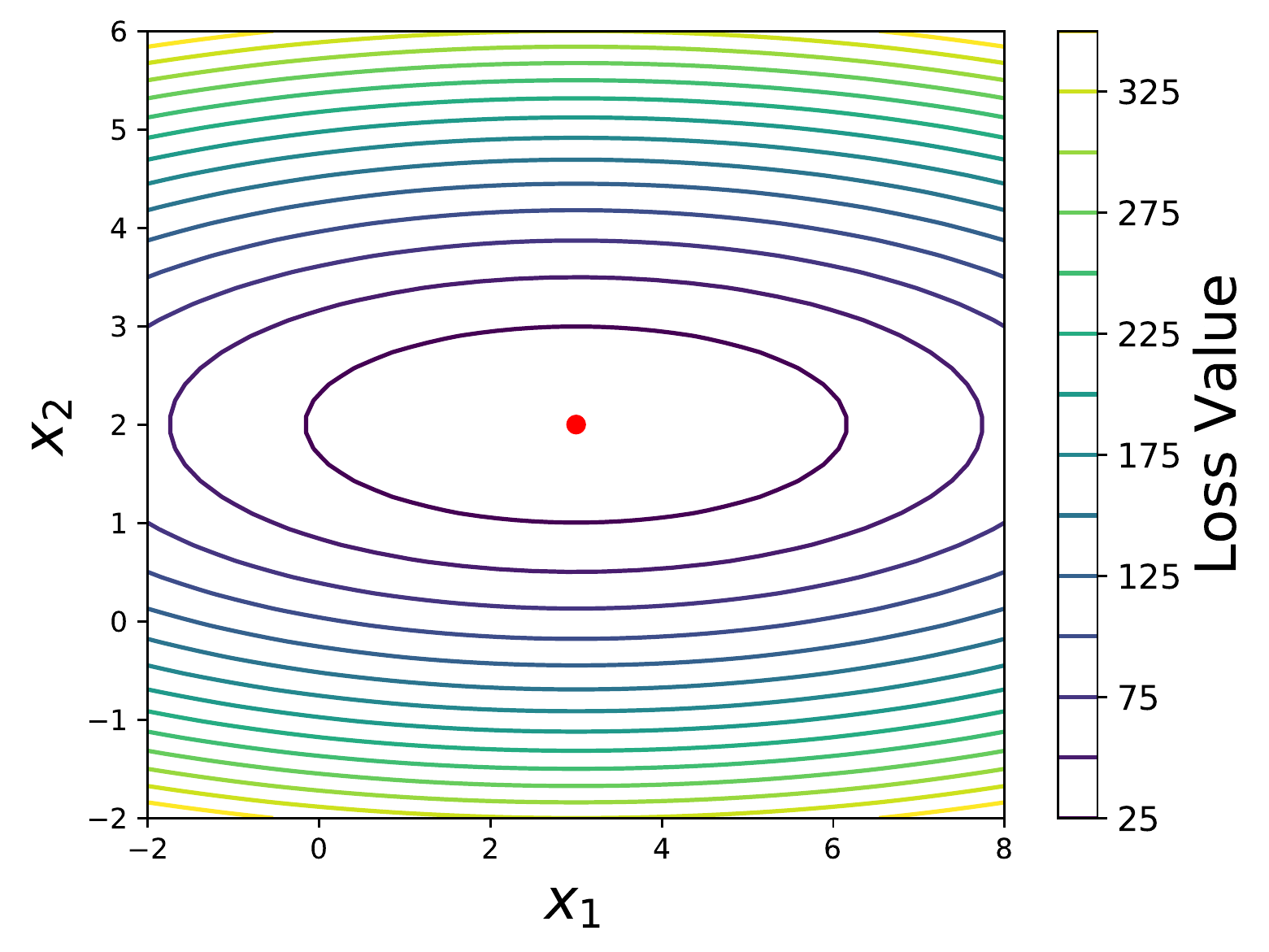}}
\caption{Figure~\ref{fig:alsgd12} shows a function surface and a contour plot (\textcolor{mylightbluetext}{blue}=low, \textcolor{mydarkyellow}{yellow}=high), where the upper graph is the surface, and the lower one is the projection of it (i.e., contour). The quadratic function is from parameters $\bA=\begin{bmatrix}
	4 & 0\\ 0 & 40
\end{bmatrix}$, $\bb=[12,80]^\top $, and $c=103$. Or equivalently, $L(\bx)=2(x_1-3)^2 + 20(x_2-2)^2+5$ and $\frac{\partial L(\bx)}{\partial \bx}=[4x_1-12, 8x_2-16]^\top$.}
\label{fig:momentum-contour}
\end{figure}
Momentum exhibits superior performance, particularly in the presence of a ravine-shaped loss curve. 
A ravine refers to an area where the surface curves are significantly steeper in one dimension than in another (see the surface and contour curve in Figure~\ref{fig:momentum-contour}, i.e., a long, narrow valley). 
Ravines are common near local minima in deep neural networks, and vanilla GD or SGD has trouble navigating them. As shown by the toy example in Figure~\ref{fig:momentum_gd}, GD tends to oscillate across the narrow ravine since the negative gradient will point down one of the steep sides rather than along the ravine towards the optimum. Momentum helps accelerate gradients in the correct direction and dampens oscillations, as evident in the example shown in  Figure~\ref{fig:momentum_mum}. 

\begin{figure}[h]
\centering  
\vspace{-0.35cm} 
\subfigtopskip=2pt 
\subfigbottomskip=2pt 
\subfigcapskip=-5pt 
\subfigure[Optimization without momentum. A higher learning rate
may result in larger parameter updates in the dimension across the valley (direction of $x_2$) which could lead to oscillations back and forth across the valley.]{\label{fig:momentum_gd}
\includegraphics[width=0.47\linewidth]{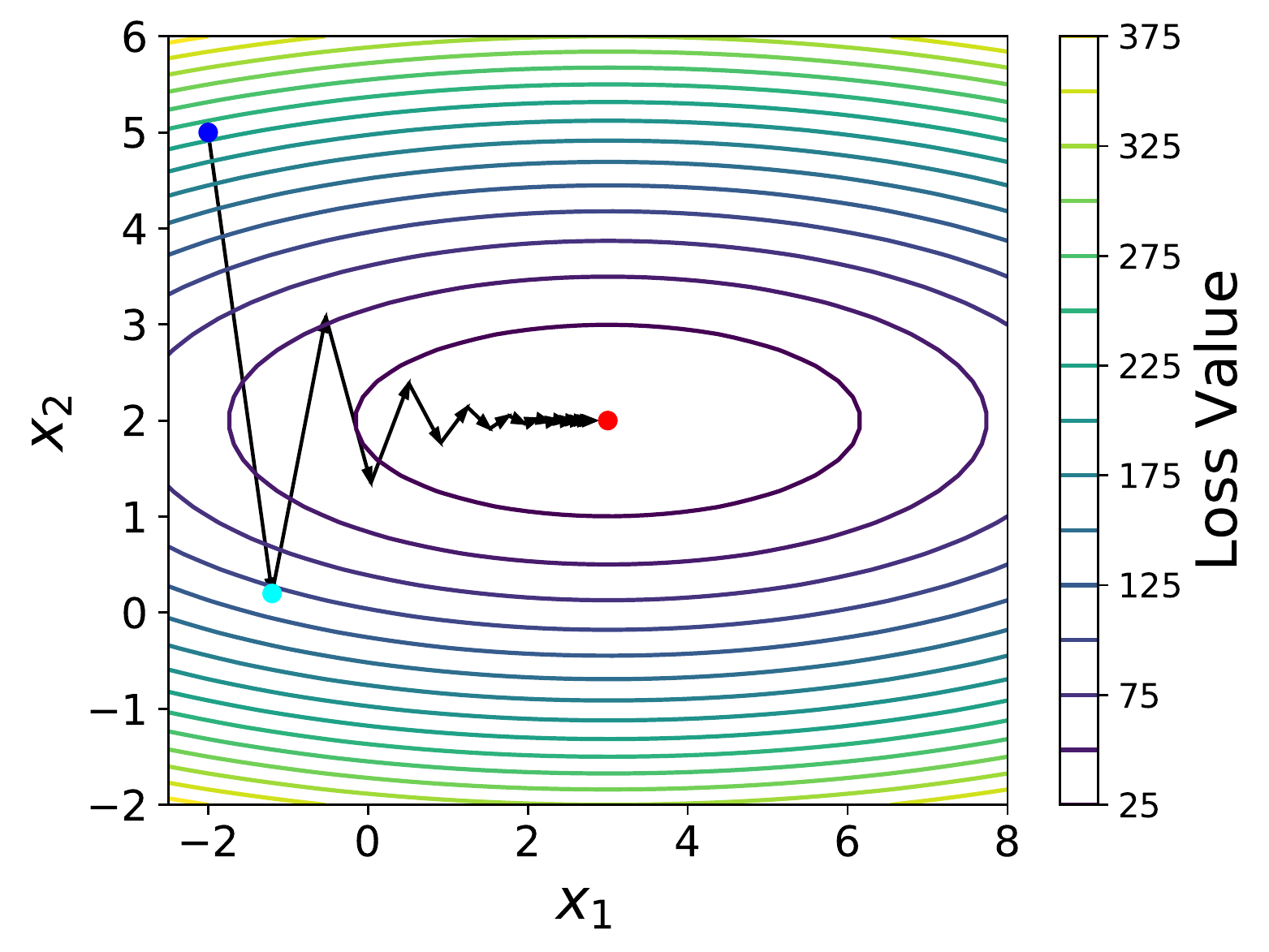}}
\subfigure[Optimization with momentum. Though the gradients along the valley (direction of $x_1$) are much smaller than the gradients across the valley (direction of $x_2$), they are
typically in the same direction. Thus, the momentum term
accumulates to speed up movement, dampens oscillations, and causes us to barrel through narrow valleys, small humps and (local) minima.]{\label{fig:momentum_mum}
\includegraphics[width=0.47\linewidth]{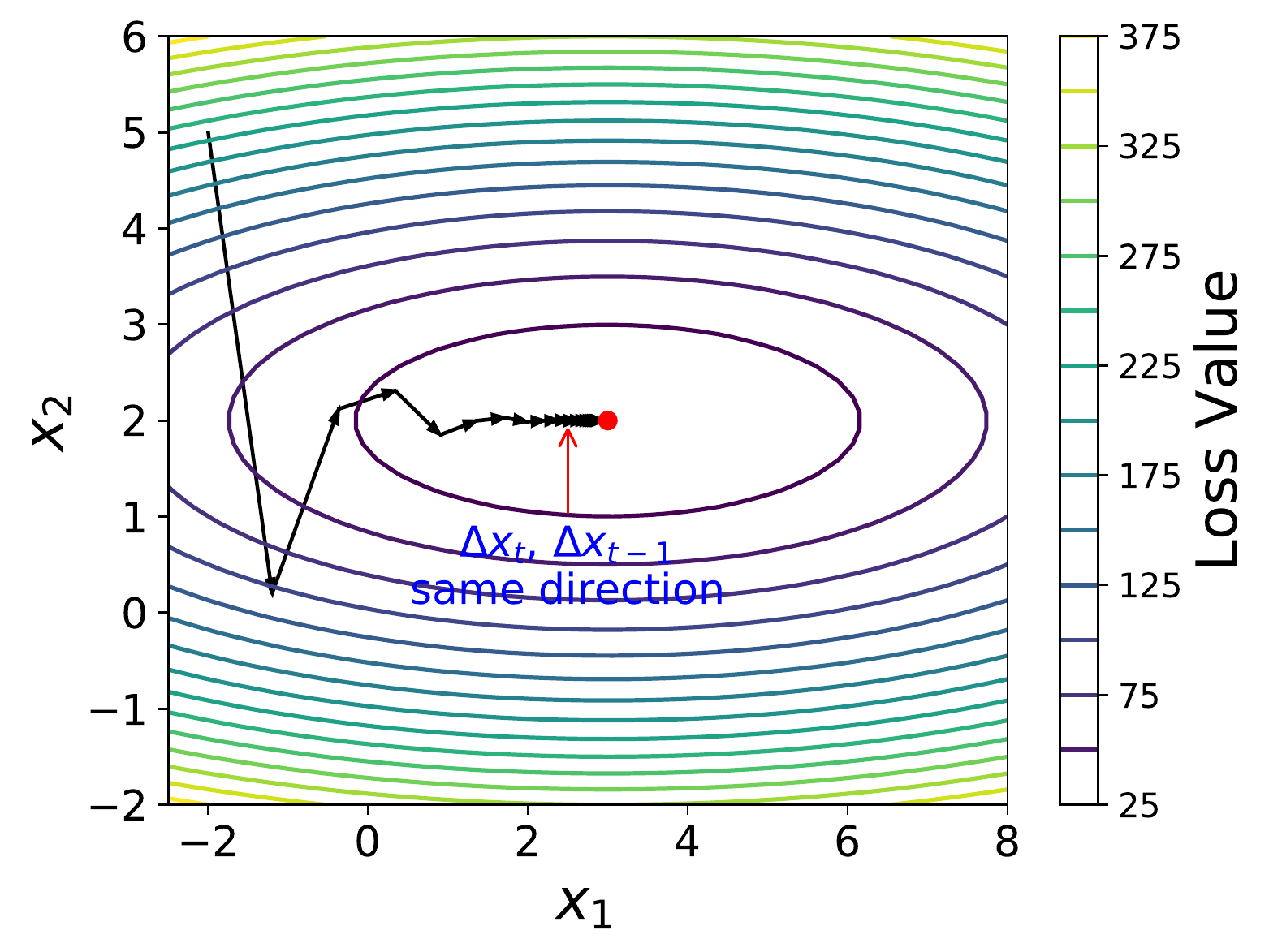}}
\caption{The loss function is shown in Figure~\ref{fig:momentum-contour}. The starting point is $[-2, 5]^\top$. After 5 iterations, the squared loss from vanilla GD is 42.72, and the loss from GD with momentum is 35.41 in this simple case. The learning rates $\eta$ are set to be 0.04 in both cases.}
\label{fig:momentum_gd_compare}
\end{figure}

As mentioned earlier, the momentum method is achieved by incorporating a fraction $\rho$ of the update vector from the previous time step into the current update vector. When $\Delta \bx_t$ and $\Delta \bx_{t-1}$ are in the same direction, the momentum accelerates the update step (e.g., the \textcolor{mylightbluetext}{blue} arrow regions in Figure~\ref{fig:momentum_mum}).
Conversely, if they are in the opposite directions, the algorithm tends to update in the former direction if $\bx$ has been updated in this direction for many iterations. 
To be more concrete, in Figure~\ref{fig:momentum_gd},  considering the \textcolor{mylightbluetext}{blue} starting point, and then looking at the \textcolor{cyan}{cyan} point we get to after one step in the step of the update without momentum, they have gradients that are pretty much equal and opposite. As a result, the gradient across the ravine has been canceled out. But the gradient along the ravine has not canceled out. 
Therefore, along the ravine, we're going to keep
building up speed, and so, after the momentum method has settled down, it'll
tend to go along the bottom of the ravine.

From this figure, the problem with the vanilla GD is that the gradient is big in the direction in which we only want to travel a small distance; and the gradient is small in the direction in which we want to travel a large distance. However, one can easily find that the momentum term helps average out the oscillation along the short axis while at the same time adds up contributions along the long axis. In other words, although it starts off by following the gradient, however, when it has velocity, it no longer does steepest descent. We call this \textit{momentum}, which makes it keep going in the previous direction.

\index{Spectral decomposition}
\index{Eigenvalue decomposition}
\index{Quadratic form}
\subsubsection{Quadratic Form in Momentum}\label{section:quadratic-in-momentum}
Following the discussion of the quadratic form in GD (Section~\ref{section:quadratic_vanilla_GD}, p.~\pageref{section:quadratic_vanilla_GD})  and steepest descent (Section~\ref{section:quadratic-in-steepestdescent}, p.~\pageref{section:quadratic-in-steepestdescent}), we now discuss the quadratic form in GD with momentum. The update is:
$$
\begin{aligned}
\Delta \bx_t &= \rho\Delta \bx_{t-1} - \eta\nabla L(\bx_t); \\
\bx_{t+1} &= \bx_t + \Delta \bx_t,
\end{aligned}
$$
where $\nabla L(\bx_t) = \bA\bx_t - \bb$ if $\bA$ is symmetric for the quadratic form. The update becomes
$$
\begin{aligned}
\Delta \bx_t &= \rho\Delta \bx_{t-1} - \eta(\bA\bx_t - \bb); \\
\bx_{t+1} &= \bx_t + \Delta \bx_t.
\end{aligned}
$$
Again, define the iterate vectors as follows:
$$
\left\{
\begin{aligned}
\by_t &= \bQ^\top(\bx_t - \bx_\star);\\
\bz_t &= \bQ^\top  \Delta \bx_t,
\end{aligned}
\right.
$$
where $\bx_\star = \bA^{-1}\bb$ under the assumption that $\bA$ is nonsingular and PD as aforementioned, and $\bA=\bQ\bLambda\bQ^\top$ is the spectral decomposition of matrix $\bA$. This construction leads to the following update rule:
$$
\begin{aligned}
\bz_t &= \rho \bz_{t-1} - \eta\bLambda \by_t; \\
\by_{t+1} &= \by_t + \bz_t,
\end{aligned}
$$
or,  after rearrangement:
$$
\begin{bmatrix}
\bz_t \\
\by_{t+1}
\end{bmatrix}
= 
\begin{bmatrix}
\rho \bI & -\eta \bLambda \\
\rho \bI & -\eta \bLambda + \bI
\end{bmatrix}
\begin{bmatrix}
\bz_{t-1} \\
\by_{t}
\end{bmatrix}.
$$
And this leads to the per-dimension update:
\begin{equation}\label{equation:momentum-quadra-generalformula}
\begin{bmatrix}
z_{t,i} \\
y_{t+1,i}
\end{bmatrix}
= 
\begin{bmatrix}
\rho  & -\eta\lambda_i \\
\rho  & 1-\eta \lambda_i 
\end{bmatrix}^t
\begin{bmatrix}
z_{0,i} \\
y_{1,i}
\end{bmatrix}=
\bB^t
\begin{bmatrix}
z_{0,i} \\
y_{1,i}
\end{bmatrix},
\end{equation}
where $z_{t,i}$ and $y_{t,i}$ are $i$-th element of $\bz_t$ and $\by_t$, respectively, and $\bB=\begin{bmatrix}
\rho  & -\eta\lambda_i \\
\rho  & 1-\eta \lambda_i 
\end{bmatrix}$.
Note here, $\bz_0$ is initialized as a zero vector, and $\by_1$ is initialized as $\bQ^\top (\bx_1-\bx_\star)$, where $\bx_1$ represents the initial parameter.
Suppose the eigenvalue decomposition (Theorem 11.1 in \citet{lu2022matrix} or Appendix~\ref{appendix:eigendecomp}, p.~\pageref{appendix:eigendecomp}) of $\bB$ admits 
$$
\bB = \bC\bD \bC^{-1},
$$
where the columns of $\bC$ contain eigenvectors of $\bB$, and $\bD=\diag(\alpha,\beta)$ is a diagonal matrix  containing the eigenvalues of $\bB$. Then $\bB^t = \bC\bD^t\bC^{-1}$. 
Alternatively, the eigenvalues of $\bB$ can be calculated by solving $\det(\bB-\alpha\bI)=0$:
$$
\alpha, \beta = \frac{(\rho+1-\eta\lambda_i) \pm \sqrt{(\rho+1-\eta\lambda_i)^2 -4\rho}}{2}.
$$
We then have by \citet{williams1992n} that
$$
\bB^t=
\left\{
\begin{aligned}
&\alpha^t \frac{\bB-\beta\bI}{\alpha-\beta} - \beta^t \frac{\bB-\alpha\bI}{\alpha-\beta}, \gap &\text{if $\alpha\neq \beta$};\\
&\alpha^{t-1}(t\bB - (t-1)\alpha\bI), \gap &\text{if $\alpha=\beta$}.
\end{aligned}
\right.
$$
Substituting  into Eq.~\eqref{equation:momentum-quadra-generalformula} yields the following expression:
$$
\begin{bmatrix}
z_{t,i} \\
y_{t+1,i}
\end{bmatrix}
=
\bB^{t}
\begin{bmatrix}
z_{0,i} \\
y_{1,i}
\end{bmatrix},
$$
where the rate of convergence is controlled by the slower one, $\max\{|\alpha|, |\beta|\}$; when $\max\{|\alpha|, |\beta|\}<1$, the GD with momentum is guaranteed to converge. 
In the case of  $\rho=0$, the momentum reduces to the vanilla GD, the condition for convergence becomes
$$
\max\{|\alpha|, |\beta|\} = |1-\eta\lambda_i | <1, \gap \forall \,\, i\in \{1,2,\ldots, d\},
$$
which aligns with  that in Eq.~\eqref{equation:vanillagd-quandr-rate-chgoices} (p.~\pageref{equation:vanillagd-quandr-rate-chgoices}).

Following the same example illustrated in Figure~\ref{fig:momentum-contour} and Figure~\ref{fig:momentum_gd_compare}, where $\bA=\begin{bmatrix}
4 & 0\\ 0 & 40
\end{bmatrix}$ with eigenvalues $\lambda_1=4$ and $\lambda_2=40$, and the matrix $\bB$ in Eq.~\eqref{equation:momentum-quadra-generalformula} is defined as 
$$
\bB_1 = 
\begin{bmatrix}
\rho & -4 \eta\\
\rho & 1-4\eta
\end{bmatrix}
\gap \text{and}\gap 
\bB_2 = 
\begin{bmatrix}
\rho & -40\eta \\
\rho & 1-40\eta
\end{bmatrix},
$$
respectively. Then it can be shown that when $\eta=0.04, \rho=0.8$, the rate of convergence is approximately  $0.89$; Figure~\ref{fig:momentum_rho8} displays the updates for 20 iterations, though the motion is in a zigzag pattern, it can still converge. However, when $\eta=0.04, \rho=1$, the rate of convergence is equal to 1; Figure~\ref{fig:momentum_rho10} shows the updates for 20 iterations, the movement diverges, even  as it traverses the optimal point.

\begin{figure}[h]
	\centering  
	\vspace{-0.35cm} 
	\subfigtopskip=2pt 
	\subfigbottomskip=2pt 
	\subfigcapskip=-5pt 
	\subfigure[Momentum $\rho=0.2$, convergence rate$\approx0.79$.]{\label{fig:momentum_rho2}
		\includegraphics[width=0.31\linewidth]{./imgs/mom_surface_lrate40_gd-1_xy-2-5_mom-2.pdf}}
	\subfigure[Momentum $\rho=0.8$, convergence rate$\approx0.89$.]{\label{fig:momentum_rho8}
		\includegraphics[width=0.31\linewidth]{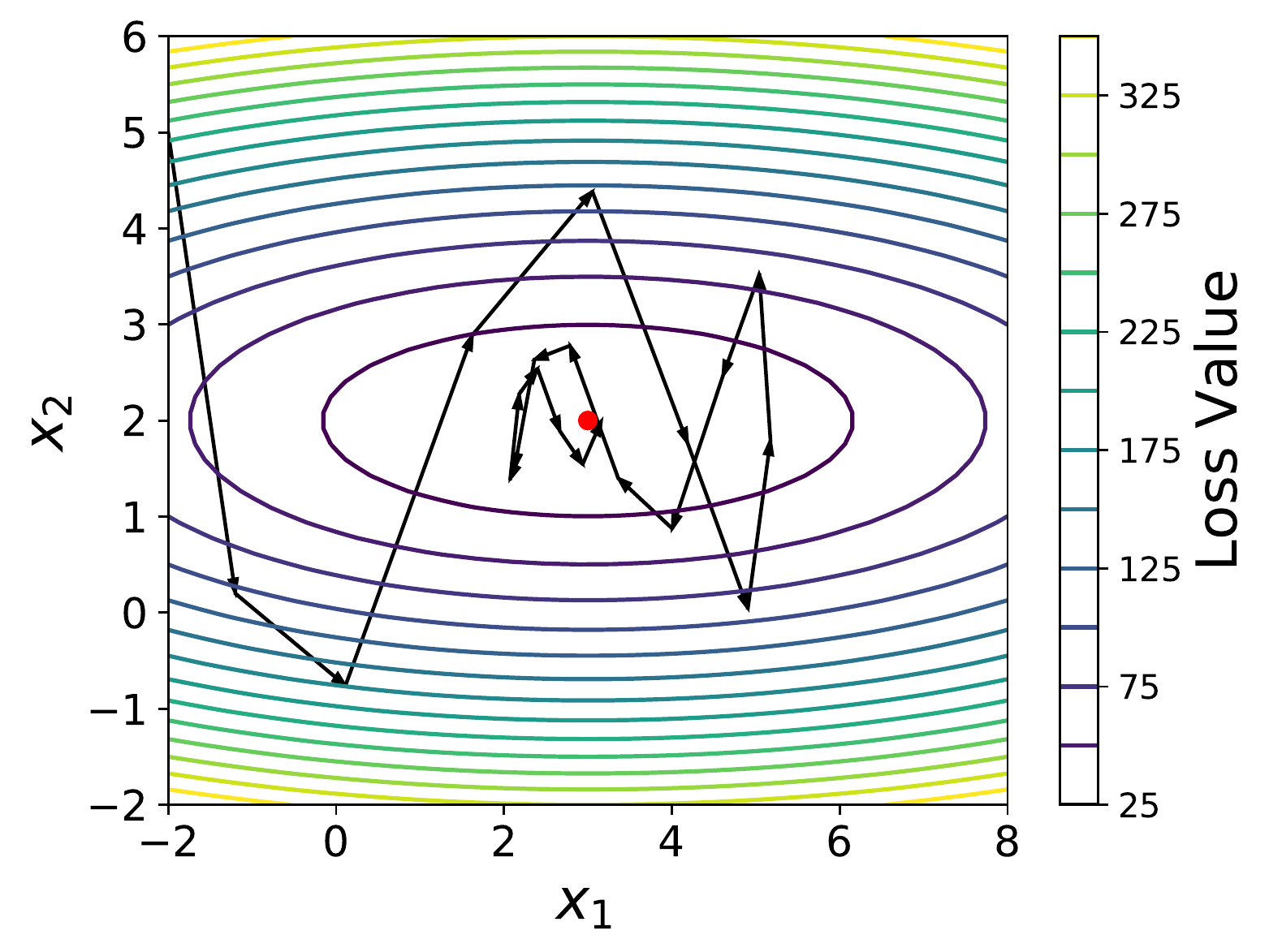}}
	\subfigure[Momentum $\rho=1$, convergence rate=1.]{\label{fig:momentum_rho10}
		\includegraphics[width=0.31\linewidth]{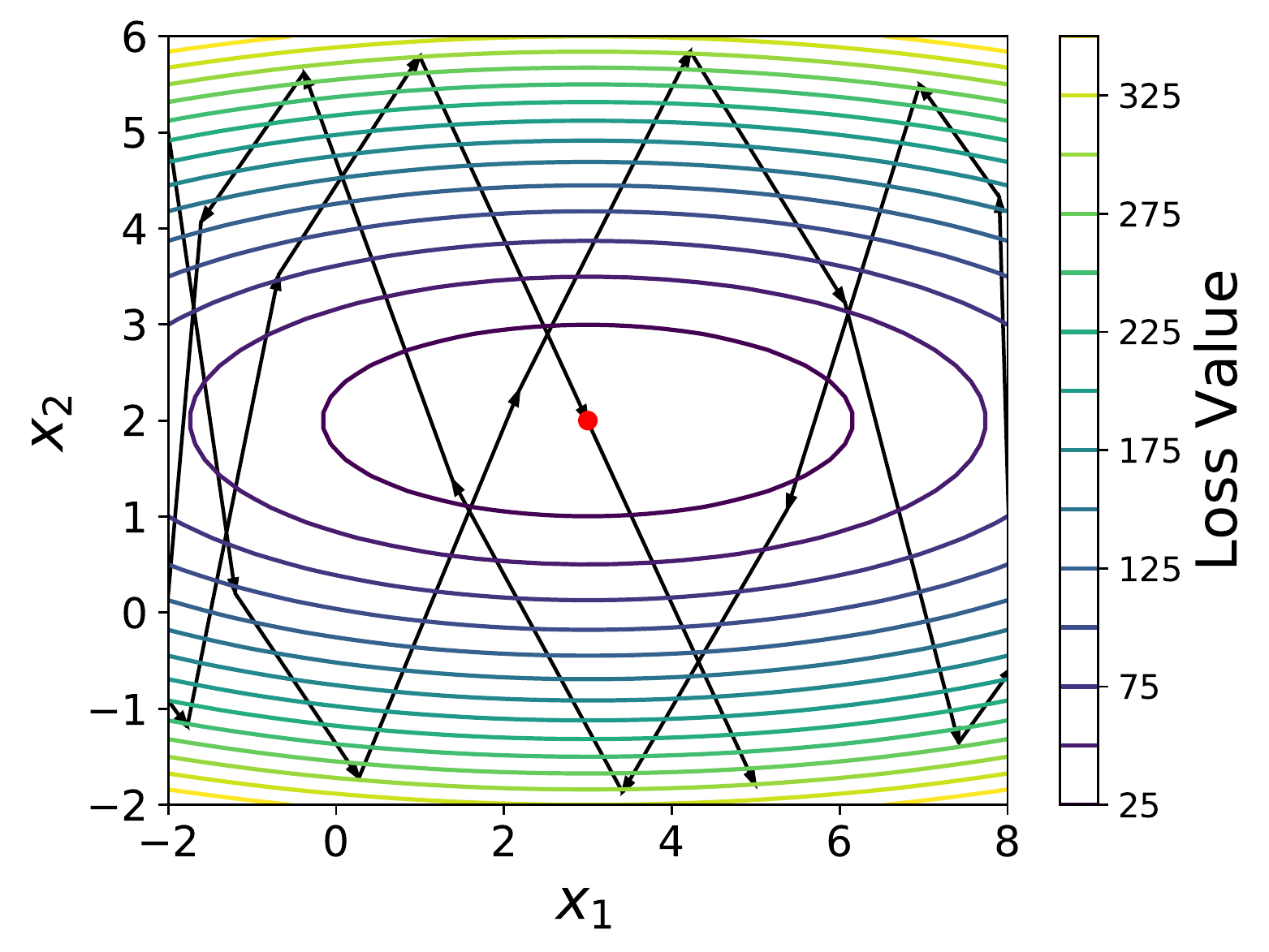}}
	\caption{Momentum creates its own oscillations. The learning rates $\eta$ are set to be 0.04 for all scenarios.}
	\label{fig:momentum_own_oscilation}
\end{figure}

\section{Nesterov Momentum}
Nesterov momentum, also known as Nesterov accelerated gradient (NAG), represents a slightly different version of the momentum update and has recently been gaining popularity. The core idea behind Nesterov momentum is that, when the current parameter vector is at some position $\bx_t$, then examining  the momentum update we discussed in the previous section. 
We know that the momentum term alone (i.e., ignoring the second term with the gradient) is about to nudge the parameter vector by $\rho \Delta \bx_{t-1}$. Therefore, if we are about to compute the gradient, we can treat the future approximate position $\bx_{t} + \rho \Delta \bx_{t-1}$ as a lookahead--this is a point in the vicinity of where we are soon going to end up. Hence, it makes sense to compute the gradient at $\bx_{t} + \rho \Delta \bx_{t-1}$ instead of at the old position $\bx_{t}$. Finally, the step takes on the following form:
$$
\begin{aligned}
	\Delta \bx_t &= \rho \Delta \bx_{t-1} - \eta \frac{\partial L(\bx_{t} + \rho \Delta \bx_{t-1})}{\partial \bx}.
\end{aligned}
$$

\begin{figure}[h]
	\centering  
	\vspace{-0.35cm} 
	\subfigtopskip=2pt 
	\subfigbottomskip=2pt 
	\subfigcapskip=-5pt 
	\subfigure[Momentum: evaluate gradient at the current position $\bx_t$, and the  momentum is about to carry us to the tip of the green arrow.]{\label{fig:nesterov1}
		\includegraphics[width=0.47\linewidth]{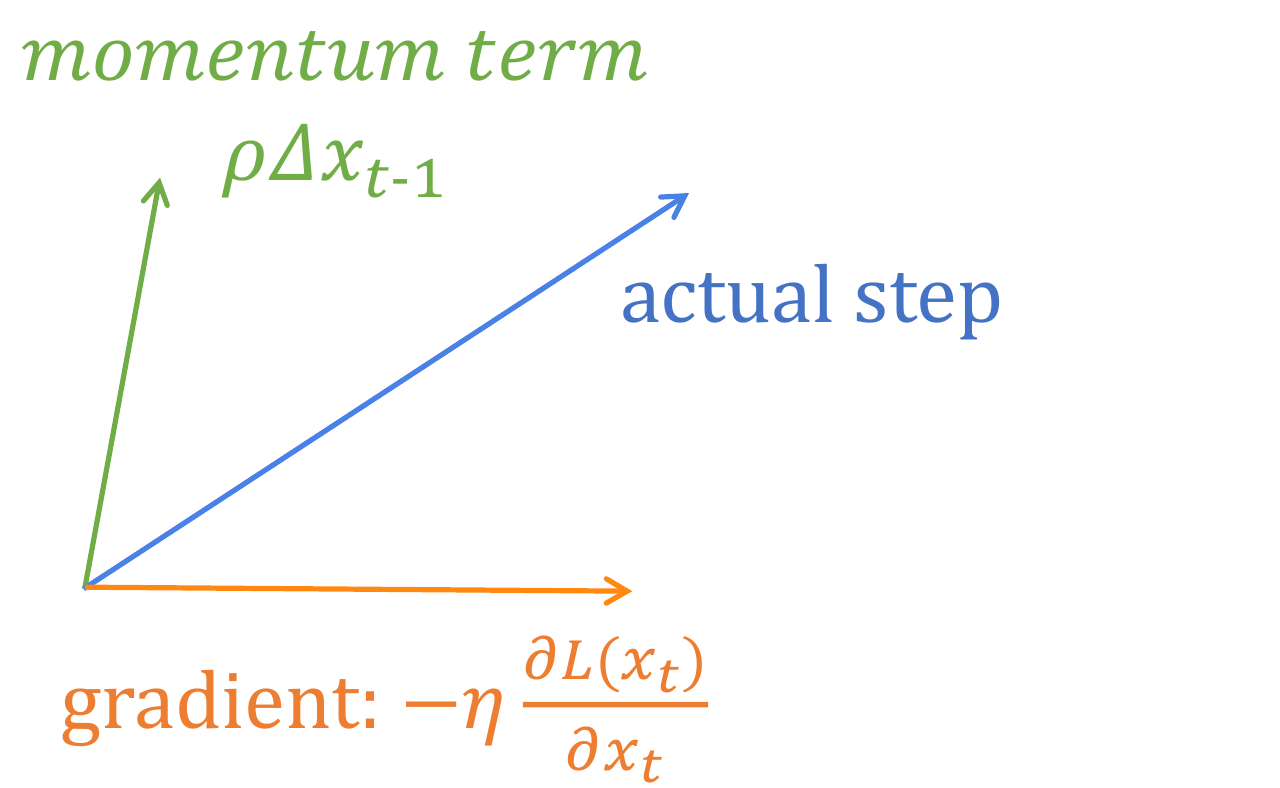}}
	\subfigure[Nesterov momentum: evaluate the gradient at this ``looked-ahead" position.]{\label{fig:nesterov2}
		\includegraphics[width=0.47\linewidth]{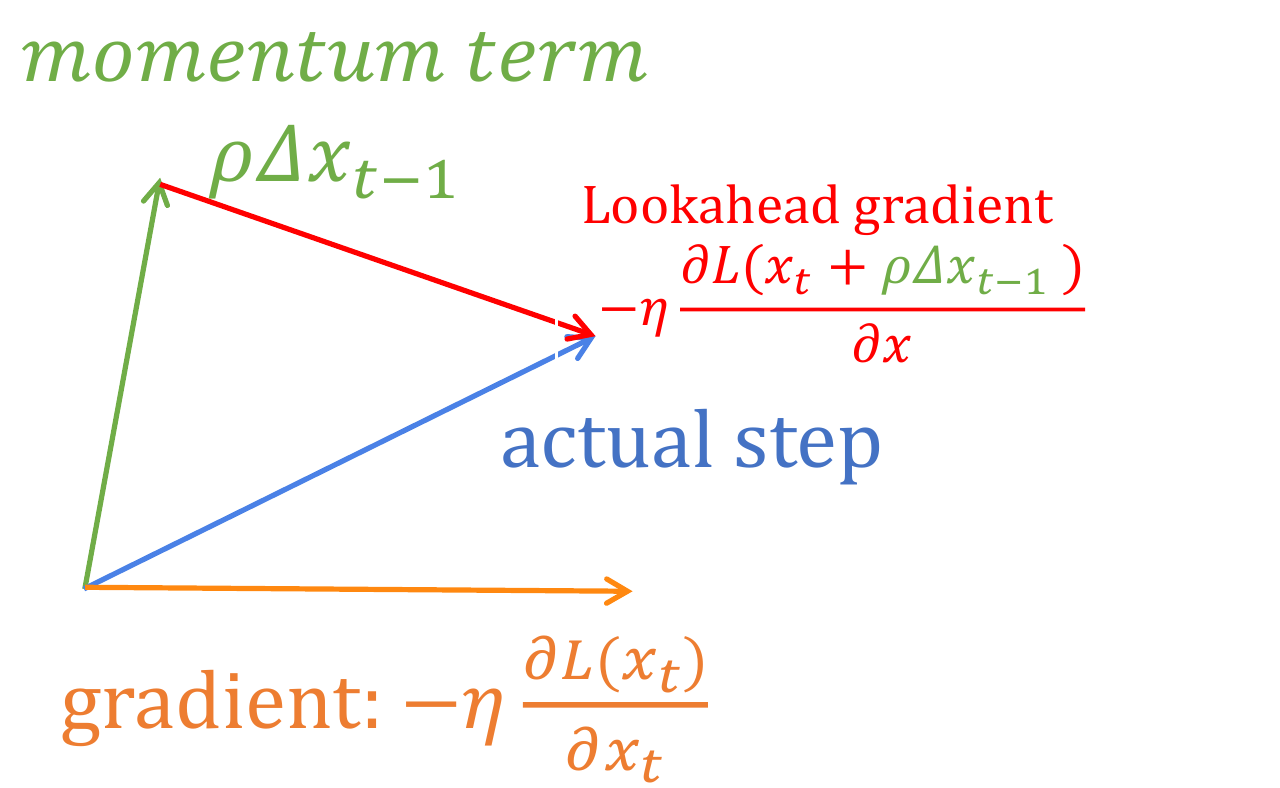}}
	\caption{Comparison of momentum and Nesterov momentum.}
	\label{fig:momentum_nesterov_comp}
\end{figure}

Figure~\ref{fig:momentum_nesterov_comp} shows the difference between momentum and Nesterov momentum approaches. This important difference is thought to counterbalance too high velocities by “peeking ahead” actual objective values in the candidate search direction. In other words, one first makes a big jump in the direction of the previously accumulated gradient; then measures the gradient where you end up and make a \textbf{correction}. But in standard momentum, one first jumps by current gradient, then makes a big jump in the direction of the previously accumulated gradient. 
To draw  a metaphor, it turns out, if you're going to gamble, it's much better to gamble and then make a correction, than to make a correction and then gamble \citep{hinton2012neural2}.
\citet{sutskever2013importance} show that Nesterov momentum has a provably better bound than gradient descent in convex, non-stochastic objectives settings.

\section{AdaGrad}
The learning rate annealing procedure modifies a single global learning rate that is applied  to all dimensions of the parameters (Section~\ref{section:learning-rate-anneal}, p.~\pageref{section:learning-rate-anneal}). \citet{duchi2011adaptive} proposed a method called AdaGrad where the learning rate is updated on a per-dimension basis. 
The learning rate for each parameter depends on the history of gradient updates of that parameter in a way such that parameters with a scarce history of updates can be updated faster by using a larger learning rate. In other words, parameters that have not been updated much in the past are more likely to have higher learning rates now. Denoting the element-wise vector multiplication between $\ba$ and $\bb$ by $\ba\odot\bb$, formally, the AdaGrad has the following update step:
\begin{equation}
	\begin{aligned}
		\Delta \bx_t &= -  \frac{ \eta}{ \sqrt{\sum_{\tau=1}^t \bg_{\tau}^2 +\epsilon} }\odot \bg_{t} ,
	\end{aligned}
\end{equation}
where $\epsilon$ is a smoothing term to better condition the division, $\eta$ is a global learning rate shared by all dimensions, $\bg_\tau^2$ indicates the element-wise square $\bg_\tau\odot \bg_\tau$, and the denominator computes the $\ell_2$ norm of a sum of all previous squared gradients in a per-dimension fashion. 
Though the global learning rate $\eta$ is shared by all dimensions, each dimension has its own dynamic learning rate controlled by the $\ell_2$ norm of accumulated gradient magnitudes. Since this dynamic learning rate grows with the inverse of the accumulated gradient magnitudes, larger gradient magnitudes have smaller learning rates, and smaller absolute values of gradients have larger learning rates. Therefore, the aggregated squared magnitude of the partial derivative with respect to each parameter over the course of the algorithm in the denominator has the same effects as the learning rate annealing.

One advantage of AdaGrad lies in that it is very easy to implement, the code snippet in the following is its implementation  by Python:
\begin{python}
	# Assume the gradient dx and parameter vector x 
	cache += dx**2
	x += - learning_rate * dx / np.sqrt(cache + 1e-8)
\end{python}
On the other hand, AdaGrad partly eliminates the need to tune the learning rate 
controlled by the accumulated gradient magnitude. 
However, AdaGrad faces a significant drawback related to the unbounded accumulation of squared gradients in the denominator. Since every added term is positive, the accumulated sum 
keeps growing or exploding during every training step. This in turn causes the per-dimension learning rate to shrink and eventually decrease throughout training and become infinitesimally small, eventually falling to zero and stopping training any more. 
Moreover, since the magnitudes of gradients are factored out in
AdaGrad, this method can be sensitive to the initialization
of the parameters and the corresponding gradients. If the initial magnitudes of the gradients are large or infinitesimally huge, the per-dimension learning rates will be low for the remainder of training. 
This can be partly combated by increasing the global learning rate, making the AdaGrad method sensitive to the choice of learning rate. 
Further, AdaGrad assumes the parameter with fewer updates should favor a larger learning rate; and one with more movement should employ a smaller learning rate. This makes it
consider only  the information from squared gradients or the absolute value of the gradients. And thus AdaGrad does not include information from the total move (i.e., the sum of updates; in contrast to the sum of absolute updates).

To be more succinct, AdaGrad exhibits  the following primary drawbacks:
1) the continual decay of learning rates throughout training;
2) the need for a manually selected global learning rate;
3) considering only the absolute value of gradients.

\section{RMSProp}

RMSProp is an extension of AdaGrad designed to overcome the main weakness of AdaGrad \citep{hinton2012neural, zeiler2012adadelta}. The original idea of RMSProp is simple: it restricts the window of accumulated past gradients to some fixed size $w$ rather than $t$ (i.e., current time step). 
However, since storing $w$ previous squared gradients is inefficient, the RMSProp introduced in \citet{hinton2012neural, zeiler2012adadelta} implements this accumulation as an exponentially decaying average of the squared gradients. This is very similar to the idea of momentum term (or decay constant).

We first discuss the specific formulation of RMSProp.
Assume at time $t$,  the running average, denoted by $E[\bg^2]_t$, is computed as follows:
\begin{equation}\label{equation:adagradwin}
	E[\bg^2]_t = \rho E[\bg^2]_{t-1} + (1 - \rho) \bg_{t}^2,
\end{equation}
where $\rho$ is a decay constant similar to that used in the momentum method, and $\bg_t^2$ indicates the element-wise square $\bg_t\odot \bg_t$. In other words, the estimate is achieved by multiplying the current squared aggregate (i.e., the running estimate) by the decay constant $\rho$ and then adding $(1-\rho)$ times the current squared partial derivative. This running estimate is initialized to $\mathbf{0}$, which may introduce  some bias in early iterations; while the bias disappears over the long term. We notice that the old gradients decay exponentially over the course of the algorithm.

As Eq.~\eqref{equation:adagradwin} is just the root mean squared (RMS) error criterion of the gradients, we can replace
it with the criterion short-hand.
Let $\rms[\bg]_t = \sqrt{E[\bg^2]_t + \epsilon}$, where again a constant $\epsilon$ is added to better condition the denominator. Then the resulting step size can be obtained as follows:
\begin{equation}\label{equation:rmsprop_update}
	\Delta \bx_t=- \frac{\eta}{\rms[\bg]_t}  \odot \bg_{t},
\end{equation}
where again $\odot$ is the element-wise vector multiplication. 

As aforementioned, the form in Eq.~\eqref{equation:adagradwin} is originally from the exponential moving average (EMA). In the original form of EMA, $1-\rho$ is also known as the smoothing constant (SC), where the SC can be expressed as $\frac{2}{N+1}$ and the period $N$ can be thought of as the number of past values to do the moving average calculation \citep{lu2022exploring}:
\begin{equation}\label{equation:ema_smooting_constant}
	\text{SC}=1-\rho \approx \frac{2}{N+1}.
\end{equation}
The above Eq.~\eqref{equation:ema_smooting_constant} establishes a relationship among  different variables: the decay constant $\rho$, the smoothing constant (SC), and the period $N$.
For instance , if $\rho=0.9$, then $N=19$. That is, roughly speaking, $E[\bg^2]_t $ at iteration $t$ is approximately equal to the moving average of the past 19 squared gradients and the current one (i.e., the moving average of a total of 20 squared gradients).
The relationship in Eq.~\eqref{equation:ema_smooting_constant} though is not discussed in the original paper of \citet{zeiler2012adadelta}, it is important to decide the lower bound of the decay constant $\rho$. Typically, a time period of $N=3$ or 7 is thought to be a relatively small frame making the lower bound of decay constant $\rho=0.5$ or 0.75; as $N\rightarrow \infty$, the decay constant $\rho$ approaches $1$.


AdaGrad is designed to converge rapidly when applied to a convex function; while RMSProp performs better in nonconvex settings. When applied to a nonconvex function to train a neural network, the learning trajectory can pass through many different structures and eventually arrives at a region that is a locally convex bowl. AdaGrad shrinks the learning rate according to the entire history of the squared partial derivative leading to an infinitesimally small learning rate before arriving at such a convex structure. While RMSProp discards ancient squared gradients to address this problem. 

However, we can find that the RMSProp still only considers the absolute value of gradients and a fixed number of past squared gradients is not flexible. 
This limitation can cause a small learning rate near (local) minima as we will discuss in the sequel. 

The RMSProp is developed independently by Geoff Hinton in \citet{hinton2012neural} and by Matthew Zeiler in \citet{zeiler2012adadelta} both of which are stemming from the need to resolve AdaGrad's radically diminishing per-dimension learning rates. 
\citet{hinton2012neural} suggest setting $\rho$ to 0.9 and the global learning rate $\eta$ to default to $0.001$. The RMSProp further can be combined into the Nesterov momentum method \citep{goodfellow2016deep}, where the comparison between the two is presented in Algorithm~\ref{alg:rmsprop} and Algorithm~\ref{alg:rmsprop_nesterov}.

\begin{algorithm}[h] 
	\caption{RMSProp}
	\label{alg:rmsprop}
	\begin{algorithmic}[1]
		\State {\bfseries Input:} Initial parameter $\bx_1$, constant $\epsilon$;
		\State {\bfseries Input:} Global learning rate $\eta$, by default $\eta=0.001$;
		\State {\bfseries Input:} Decay constant $\rho$;
		\State {\bfseries Input:} Initial accumulated squared gradients $E[\bg^2]_{0} = \bzero $;
		\For{$t=1:T$ } 
		\State Compute gradient $\bg_t = \nabla L(\bx_{t})$;
		\State Compute running estimate $	E[\bg^2]_t = \rho E[\bg^2]_{t-1} + (1 - \rho) \bg_{t}^2;$
		\State Compute step $\Delta \bx_t =- \frac{\eta}{\sqrt{E[\bg^2]_t+\epsilon }}  \odot \bg_{t}$;
		\State Apply update $\bx_{t+1} = \bx_{t} + \Delta \bx_t$;
		\EndFor
		\State {\bfseries Return:} resulting parameters $\bx_t$, and the loss $L(\bx_t)$.
	\end{algorithmic}
\end{algorithm}
\begin{algorithm}[h] 
	\caption{RMSProp with Nesterov Momentum}
	\label{alg:rmsprop_nesterov}
	\begin{algorithmic}[1]
		\State {\bfseries Input:} Initial parameter $\bx_1$, constant $\epsilon$;
		\State {\bfseries Input:} Global learning rate $\eta$, by default $\eta=0.001$;
		\State {\bfseries Input:} Decay constant $\rho$, \textcolor{mylightbluetext}{momentum constant $\alpha$};
		\State {\bfseries Input:} Initial accumulated squared gradients $E[\bg^2]_{0} = \bzero $, and update step $\Delta \bx_{0}=\bzero$;
		\For{$t=1:T$ } 
		\State Compute interim update $\widetilde{\bx}_{t} = \bx_{t} + \alpha \Delta \bx_{t-1}$;
		\State Compute interim gradient $\bg_t = \nabla L(\textcolor{mylightbluetext}{\widetilde{\bx}_{t}})$;
		\State Compute running estimate $	E[\bg^2]_t = \rho E[\bg^2]_{t-1} + (1 - \rho) \bg_{t}^2;$
		\State Compute step $\Delta \bx_t =\textcolor{mylightbluetext}{\alpha \Delta \bx_{t-1}}- \frac{\eta}{\sqrt{E[\bg^2]_t +\epsilon }}  \odot \bg_{t}$;
		\State Apply update $\bx_{t+1} = \bx_{t} + \Delta \bx_t$;
		\EndFor
		\State {\bfseries Return:} resulting parameters $\bx_t$, and the loss $L(\bx_t)$.
	\end{algorithmic}
\end{algorithm}

\section{AdaDelta}\label{section:adadelta}
\citet{zeiler2012adadelta} further shows an inconsistency in the units of the step size in RMSProp (so as the vanilla SGD, the momentum, and the AdaGrad). To overcome this weakness, and draw from the correctness of the second-order method (further discussed in Section~\ref{section:seconr-methods}, p.~\pageref{section:seconr-methods}), the author considers rearranging Hessian to determine the quantities involved.
It is well known that though the calculation of Hessian or approximation to the Hessian matrix is a tedious and computationally expensive task, the curvature information it provides proves valuable for optimization, and the units in Newton's method are well matched. Given the Hessian matrix $\bH$, the update step in Newton's method can be described as follows \citep{becker1988improving, dauphin2014identifying}:
\begin{equation}
	\Delta\bx_t \propto - \bH^{-1} \bg_t \propto \frac{\frac{\partial L(\bx_t)}{\partial \bx_t}}{\frac{\partial^2 L(\bx_t)}{\partial \bx^2}}.
\end{equation}
This implies 
\begin{equation}
	\frac{1}{\frac{\partial^2 L(\bx_t)}{\partial \bx_t^2}} = \frac{\Delta \bx_t}{\frac{\partial L(\bx_t)}{\partial \bx_t}},
\end{equation}
i.e., the units of the Hessian matrix can be approximated by the right-hand side term of the above equation. Since the RMSProp update in Eq.~\eqref{equation:rmsprop_update} already involves $\rms[\bg]_t$ in the denominator, i.e., the units of the gradients. Introducing an additional  unit of the order of $\Delta \bx_t$ in the numerator can match the same order as Newton's method. To do this, define another exponentially decaying average of the update steps:
\begin{equation}
	\begin{aligned}
		\rms[\Delta \bx]_t &= 	\sqrt{E[\Delta \bx^2]_t } \\
		&= \sqrt{\rho E[\Delta \bx^2]_{t-1} + (1 - \rho) \Delta \bx_{t}^2  }.
	\end{aligned}
\end{equation}
Since the value of $\Delta \bx_t$ for the current iteration is unknown, and the curvature can be assumed to be locally smoothed, making it suitable to approximate $\rms[\Delta \bx]_t$ by $\rms[\Delta \bx]_{t-1}$. 
So we can use an estimation of $\frac{1}{\frac{\partial^2 L(\bx_t)}{\partial \bx_t^2}} $ to substitute for the computationally expensive $\bH^{-1}$:
\begin{equation}
	\frac{\Delta \bx_t}{\frac{\partial L(\bx_t)}{\partial \bx_t}} \sim \frac{\rms[\Delta \bx]_{t-1}}{\rms[\bg]_t}.
\end{equation}
%
%
This presents an approximation to the diagonal Hessian, using only RMS measures of $\bg$ and $\Delta \bx$, and results in the update step whose units are matched:
\begin{equation}
	\Delta \bx_t = -\frac{\rms[\Delta \bx]_{t-1}}{\rms[\bg]_t} \odot \bg_t.
\end{equation}
The idea of AdaDelta, derived  from the second-order method, alleviates the annoying choosing of learning rate. 
Meanwhile, a web demo developed by Andrej Karpathy can be explored to find the convergence rates among SGD, SGD with momentum, AdaGrad, and AdaDelta \footnote{see https://cs.stanford.edu/people/karpathy/convnetjs/demo/trainers.html.}.

A crucial consideration when employing the RMSProp or AdaDelta method is to carefully notice that though the accumulated squared gradients in the denominator can compensate for the per-dimension learning rates, if we save the checkpoint of the neural networks at the end of some epochs and want to re-tune the parameter by loading the weights from the checkpoint, the first few batches of the re-tuning can perform poorly since there are not enough squared gradients to smooth the denominator. 
A particular example is shown in Figure~\ref{fig:er-rmsprop_epochstart}, where we save the weights and load them after each epoch; loss deterioration is observed after each epoch.
While this doesn't significantly impact overall training progress as the loss can still go down from that point, a more effective  choice involves saving the $E[\bg^2]_t$ along with the weights of the neural networks. 
\begin{figure}[h]
	\centering
	\includegraphics[width=0.6\textwidth]{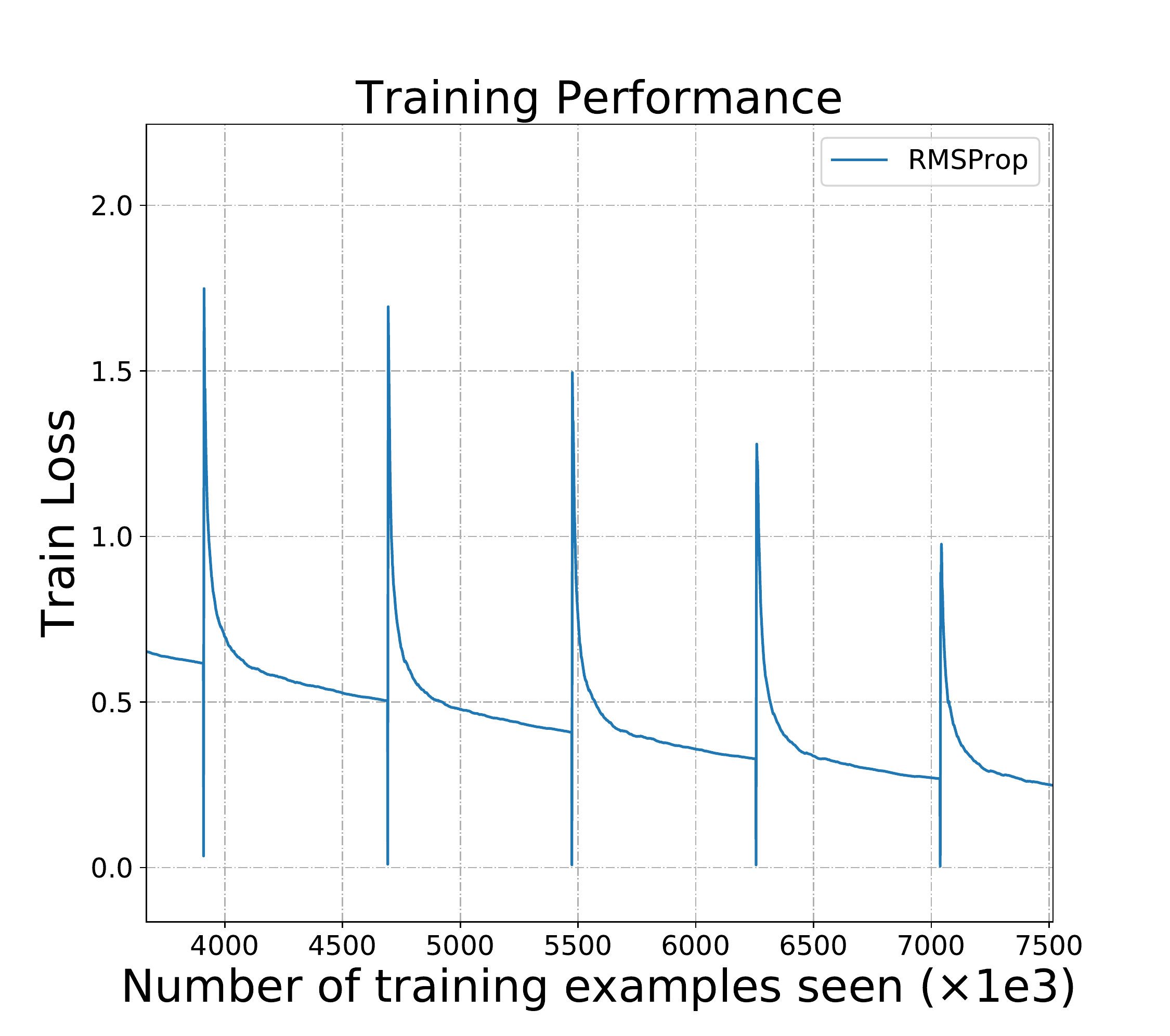}
	\caption{Demonstration of tuning parameter after each epoch by loading the weights. We save the weights and load them after each epoch such that there are step-points while re-training after each epoch. That is, loss deterioration is observed after each epoch.}
	\label{fig:er-rmsprop_epochstart}
\end{figure}
\index{Loss deterioration}
%
%
%
%
%

\index{Effective ratio}
\index{Exponential moving average}
\section{AdaSmooth}\label{section:adaer}
In this section, we will discuss the effective ratio, derived from previous updates in the stochastic optimization process.
We will elucidate its application to achieve  adaptive learning rates per-dimension via the flexible smoothing constant, hence the name AdaSmooth.
The idea presented in the section is derived from the RMSProp method in order to improve two main drawbacks of the method: 1) consider only the absolute value of the gradients rather than the total movement in each dimension; 2) the need for manually selection of hyper-parameters. 
\paragraph{Effective Ratio (ER).}
\citet{kaufman2013trading, kaufman1995smarter} suggested replacing the smoothing constant in the EMA formula with a constant based on the \textit{efficiency ratio} (ER). And the ER is shown to provide promising results for financial forecasting via classic quantitative strategies, wherein the ER of the closing price is calculated to decide the trend of the underlying asset \citep{lu2022exploring}. This indicator is designed to measure the \textit{strength of a trend}, defined within a range from -1.0 to +1.0, with a larger magnitude indicating a larger upward or downward trend.
Recently, \citet{lu2022reducing} show that the ER can be utilized to reduce overestimation and underestimation in time series forecasting.
Given the window size $M$ and a series $\{h_1, h_2, \ldots, h_T\}$, it is calculated with a simple formula:
\noindent
\begin{equation}
	\begin{aligned}
		e_t  &= \frac{s_t}{n_t}= \frac{h_{t} - h_{t-M}}{\sum_{i=0}^{M-1} |h_{t-i} - h_{t-1-i}|}= \frac{\text{Total move for a period}}{\text{Sum of absolute move for each bar}},
	\end{aligned}
\end{equation}
where $e_t$ is the ER of the series at time $t$. 
At a strong trend (i.e., the input series is moving in a certain direction, either up or down), the ER will approach 1 in absolute value; if there is no directed movement, it will be a little more than 0. 

Instead of calculating the ER base on the closing price of the underlying asset, we want to calculate the ER of the moving direction in the update methods for each parameter. And in the descent methods, we care more about how much each parameter moves apart from its initial point in each period, either moving positively or negatively. So here we only consider the absolute value of the ER. To be specific, the ER for the parameters in the  method is calculated as follows:
\begin{equation}\label{eqution:signoiase-er-delta}
	\begin{aligned}
		\be_t  = \frac{\bs_t}{\bn_t}&= \frac{| \bx_t -  \bx_{t-M}|}{\sum_{i=0}^{M-1} | \bx_{t-i} -  \bx_{t-1-i}|}= \frac{| \sum_{i=0}^{M-1} \Delta \bx_{t-1-i}|}{\sum_{i=0}^{M-1} | \Delta \bx_{t-1-i}|},
	\end{aligned}
\end{equation}
where $\be_t \in \real^d$, and its $i$-th element $e_{t,i}$ is in the range of $ [0, 1]$ for all $i$ in $[1,2,\ldots, d]$. A large value of $e_{t,i}$ indicates the descent method in the $i$-th dimension is moving consistently in a certain direction; while a small value approaching 0 means the parameter in the $i$-th dimension is moving in a zigzag pattern, alternating between positive and negative movements. In practice, and across all  our experiments,  $M$ is selected to be the batch index for each epoch. That is, $M=1$ if the training is in the first batch of each epoch; and $M=M_{\text{max}}$ if the training is in the last batch of the epoch, where $M_{\text{max}}$ is the maximal number of batches per epoch. In other words, $M$ ranges in $[1, M_{\text{max}}]$ for each epoch. Therefore, the value of $e_{t,i}$ indicates the movement of the $i$-th parameter in the most recent epoch. Or even more aggressively, the window can range from 0 to the total number of batches seen during the entire training progress. The adoption of the adaptive window size $M$ rather than a fixed one has the benefit that we do not need to keep the past $M+1$ steps $\{ \bx_{t-M},  \bx_{t-M+1}, \ldots,  \bx_t\}$ to calculate the signal and noise vectors $\{\bs_t,\bn_t\}$ in Eq.~\eqref{eqution:signoiase-er-delta} since they can be obtained in an accumulated fashion.

\paragraph{AdaSmooth.}\label{section:adaer-after-er}
If the ER in magnitude of each parameter is small (approaching 0), the movement in this dimension is zigzag, the  AdaSmooth method tends to use a long period average as the scaling constant to slow down the movement in that dimension. When the absolute ER per-dimension is large (tend to 1), the path in that dimension is moving in a certain direction (not zigzag), and the learning actually is happening and the descent is moving in a correct direction, where the learning rate should be assigned to a relatively large value for that dimension. Thus, the AdaSmooth tends to choose a small period which leads to a small compensation in the denominator; since the gradients in the closer periods are small in magnitude when it's near the (local) minima. A particular example is shown in Figure~\ref{fig:er-explain}, where the descent is moving in a certain direction, and the gradient in the near periods is small in magnitude; if we choose a larger period to compensate for the denominator, the descent will be slower due to the large factored denominator.
In short, we want a smaller period to calculate the exponential average of the squared gradients in Eq.~\eqref{equation:adagradwin} if the update is moving in a certain direction without a zigzag pattern; while when the parameter is updated in a zigzag fashion, the period for the exponential average should be larger \citep{lu2022adasmooth}.

\begin{figure}[h]
	\centering
	\includegraphics[width=0.6\textwidth]{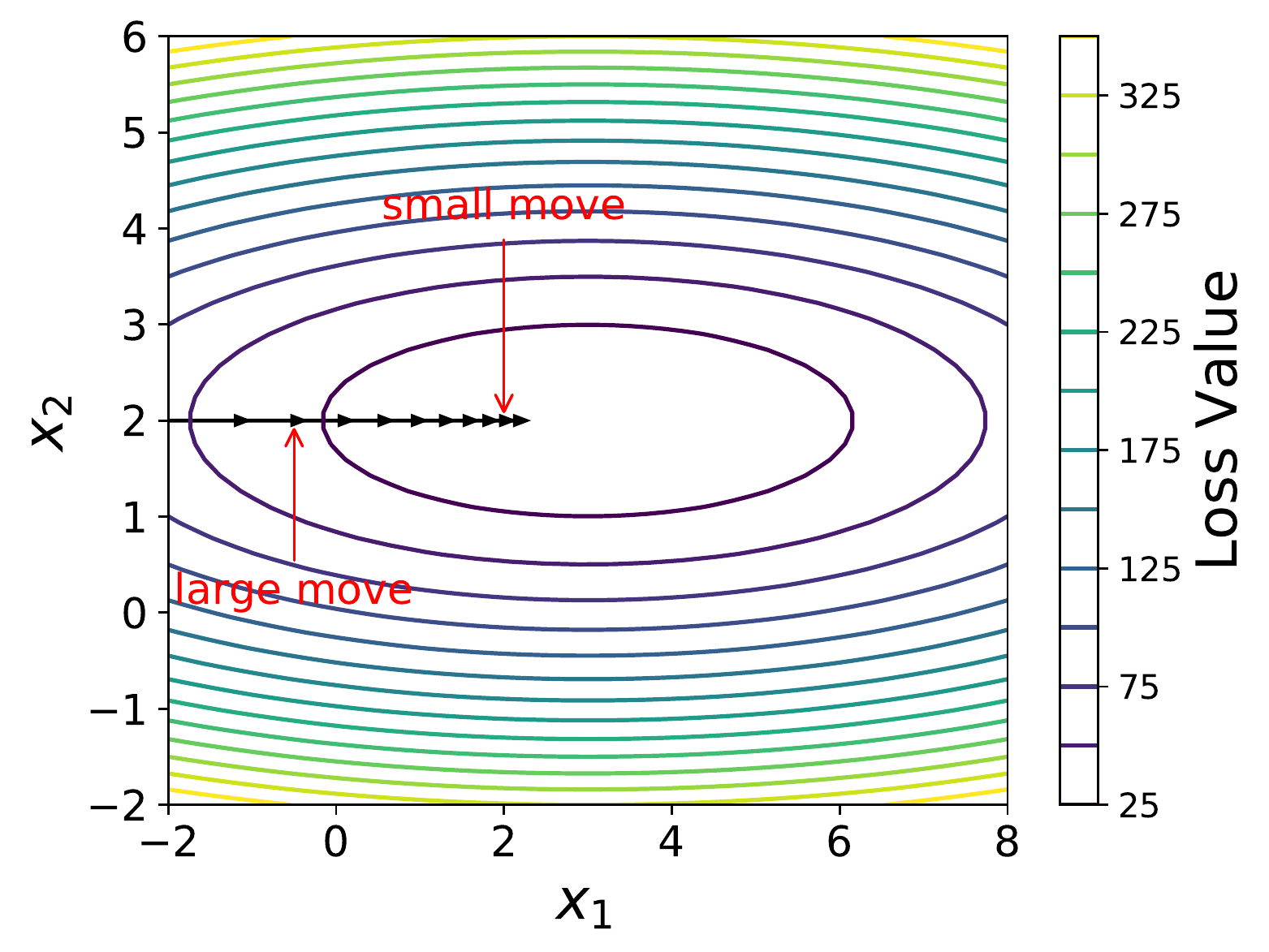}
	\caption{Demonstration of how the effective ratio works. Stochastic optimization tends to move a large step when it is far from the (local) minima; and a relatively small step when it is close to the (local) minima.}
	\label{fig:er-explain}
\end{figure}

The obtained value of ER is incorporated into the exponential smoothing formula.  
To enhance our approach, we aim to dynamically adjust the time period $N$ discussed in Eq.~\eqref{equation:ema_smooting_constant} to be a smaller value when the ER tends to 1 in absolute value; or a larger value when the ER moves towards 0. When $N$ is small, $\text{SC} $ is known as a ``\textit{fast SC}"; otherwise, $\text{SC} $ is known as a ``\textit{slow SC}". 

For example, let the small time period be $N_1=3$, and the large time period be $N_2=199$.
The smoothing ratio for the fast movement must align with that of  EMA with period $N_1$ (``fast SC" = $\frac{2}{N_1+1}$ = 0.5), and for the period of no trend EMA period must be equal to $N_2$ (``slow SC" = $\frac{2}{N_2+1}$ = 0.01). 
Thus the new changing smoothing constant is introduced, called the ``\textit{scaled smoothing constant}" (SSC), denoted by a vector $\bc_t\in \real^d$:
$$
\bc_t =  ( \text{fast SC} - \text{slow SC}) \times \be_t   + \text{slow SC}.
$$
By Eq.~\eqref{equation:ema_smooting_constant}, we can define the \textit{fast decay constant} $\rho_1=1-\frac{2}{N_1+1}$, and the \textit{slow decay constant} $\rho_2 = 1-\frac{2}{N_2+1}$. Then the scaled smoothing constant vector can be obtained by:
$$
\bc_t =  ( \rho_2- \rho_1) \times \be_t   + (1-\rho_2),
$$
where the smaller $\be_t$, the smaller $\bc_t$.
For a more efficient influence of the obtained smoothing constant on the averaging period, Kaufman recommended squaring it.
The final calculation formula then follows:
\begin{equation}\label{equation:squared-ssc}
	E[\bg^2]_t  =  \bc_t^2 \odot \bg_{t}^2  +  \left(1-\bc_t^2 \right)\odot E[\bg^2]_{t-1}.
\end{equation}
or after rearrangement:
$$
E[\bg^2]_t = E[\bg^2]_{t-1}+ \bc_t^2 \odot (\bg_{t}^2 -  E[\bg^2]_{t-1}).
$$
We notice that $N_1=3$ is a small period to calculate the average (i.e., $\rho_1=1-\frac{2}{N_1+1}=0.5$) such that the EMA sequence will be noisy if $N_1$ is less than 3. Therefore, the minimum value of $\rho_1$ in practice is set to be greater than 0.5 by default. While $N_2=199$ is a large period to compute the average (i.e., $\rho_2=1-\frac{2}{N_2+1}=0.99$) such that the EMA sequence almost depends only on the previous value, leading to the default value of $\rho_2$ no larger than 0.99. Experimental study will reveal that the AdaSmooth update will be insensitive to the hyper-parameters in the sequel. We also carefully notice that when $\rho_1=\rho_2$, the AdaSmooth algorithm recovers to the RMSProp algorithm with decay constant $\rho=1-(1-\rho_2)^2$ since we square it in Eq.~\eqref{equation:squared-ssc}. After developing the AdaSmooth method, we realize the main idea behind it is similar to that of SGD with momentum: to speed up (compensate less in the denominator) the learning along dimensions where the gradient consistently points in the same direction; and to slow the pace (compensate more in the denominator) along dimensions in which the sign of the gradient continues to change.

\index{Saddle point}
As discussed in the cyclical learning rate section (Section~\ref{section:cyclical-lr}, p.~\pageref{section:cyclical-lr}), \citet{dauphin2014identifying, dauphin2015equilibrated} argue that the difficulty in minimizing the loss arises from saddle points rather than poor local minima. Saddle points, characterized by small gradients, impede the learning process.  
However, an adaptive smoothing procedure for the learning rates per-dimension can naturally find these saddle points and compensate less in the denominator or ``increase" the learning rates when the optimization is in these areas, allowing more rapid traversal of saddle point plateaus. When applied to a nonconvex function to train a neural network, the learning trajectory may pass through many different structures and eventually arrive at a region that is a locally convex bowl. AdaGrad shrinks the learning rate according to the entire history of the squared partial derivative and may have made the learning rate too small before arriving at such a convex structure. RMSProp partly solves this drawback since it uses an exponentially decaying average to discard ancient squared gradients, making it more robust in a nonconvex setting compared to the AdaGrad method. The AdaSmooth goes further in two points: 1) when it's close to a saddle point, a small compensation in the denominator can help it escape the saddle point; 2) when it's close to a locally convex bowl, the small compensation further makes it converge faster.


Empirical evidence shows the ER used in the simple moving average (SMA) with a fixed windows size $w$ can also reflect the trend of the series/movement in quantitative strategies \citep{lu2022exploring}. However, this again needs to store $w$ previous squared gradients in the AdaSmooth case, making it inefficient and we shall not adopt this extension.

\paragraph{AdaSmoothDelta.}\label{section:adasmoothdelta}
We observe that the ER can also be applied to the AdaDelta setting:
\begin{equation}\label{equation:adasmoothdelta}
	\Delta \bx_t = -\frac{\sqrt{E[\Delta \bx^2]_t}}{\sqrt{E[\bg^2]_t+\epsilon}} \odot \bg_t,
\end{equation}
where 
\begin{equation}\label{equation:adasmoothdelta111}
	E[\bg^2]_t  =  \bc_t^2 \odot \bg_{t}^2  +  \left(1-\bc_t^2 \right)\odot E[\bg^2]_{t-1} ,
\end{equation}
and 
\begin{equation}\label{equation:adasmoothdelta222}
	E[\Delta \bx^2]_t = (1-\bc_t^2) \odot \Delta \bx^2_t+ \bc_t^2 \odot E[\Delta \bx^2]_{t-1},
\end{equation}
in which case the difference in $E[\Delta \bx^2]_t$ is to choose a larger period when the ER is small. This is reasonable in the sense that $E[\Delta \bx^2]_t$ appears in the numerator, while $E[\bg^2]_t$ is in the denominator of Eq.~\eqref{equation:adasmoothdelta}, making their compensation towards different directions. Alternatively, a fixed decay constant can be applied for $E[\Delta \bx^2]_t$:
$$
E[\Delta \bx^2]_t = (1-\rho_2)  \Delta \bx^2_t+ \rho_2  E[\Delta \bx^2]_{t-1},
$$
The AdaSmoothDelta optimizer introduced above further alleviates the need for a hand specified global learning rate, which is conventionally set to $\eta=1$ in the Hessian context. However, due to the adaptive smoothing constants in Eq.~\eqref{equation:adasmoothdelta111} and \eqref{equation:adasmoothdelta222}, the $E[\bg^2]_t $ and $E[\Delta \bx^2]_t$ are less locally smooth, making it less insensitive to the global learning rate than the AdaDelta method. Therefore, a smaller global learning rate, e.g., $\eta=0.5$ is favored in AdaSmoothDelta. The full procedure for computing AdaSmooth is then formulated in Algorithm~\ref{algo:adasmooth}.

\begin{algorithm}[tb]
	\caption{Computing AdaSmooth: the  AdaSmooth algorithm. All operations on vectors are element-wise. Good default settings for the tested tasks are $\rho_1=0.5, \rho_2=0.99, \epsilon=1e-6, \eta=0.001$; see Section~\ref{section:adaer-after-er} or Eq.~\eqref{equation:ema_smooting_constant} for a detailed discussion on the explanation of the decay constants' default values. 
		The AdaSmoothDelta iteration can be calculated in a similar way.
	}
	\label{alg:computer-adaer}
	\begin{algorithmic}[1]
		\State {\bfseries Input:} Initial parameter $\bx_1$, constant $\epsilon$;
		\State {\bfseries Input:} Global learning rate $\eta$, by default $\eta=0.001$;
		\State {\bfseries Input:} Fast decay constant $\rho_1$, slow decay constant $\rho_2$;
		\State {\bfseries Input:} Assert $\rho_2>\rho_1$, by default $\rho_1=0.5$, $\rho_2=0.99$;
		\For{$t=1:T$ } 
		\State Compute gradient $\bg_t = \nabla L(\bx_t)$;
		\State Compute ER $\be_t=\frac{| \bx_t -  \bx_{t-M}|}{\sum_{i=0}^{M-1} | \Delta \bx_{t-1-i}|}$ ;
		\State Compute scaled smoothing vector $\bc_t =  ( \rho_2- \rho_1) \times \be_t   + (1-\rho_2)$;
		\State Compute normalization term $E[\bg^2]_t  =  \bc_t^2 \odot \bg_{t}^2  +  \left(1-\bc_t^2 \right)\odot E[\bg^2]_{t-1} ;$
		\State Compute step $\Delta \bx_t =- \frac{\eta}{\sqrt{E[\bg^2]_t+\epsilon}}  \odot \bg_{t}$;
		\State Apply update $\bx_{t+1} = \bx_{t} + \Delta \bx_t$;
		\EndFor
		\State {\bfseries Return:} resulting parameters $\bx_t$, and the loss $L(\bx_t)$.
	\end{algorithmic}\label{algo:adasmooth}
\end{algorithm}

We have discussed the step-points problem when reloading weights from checkpoints in the RMSProp or AdaDelta methods. However, this issue is less severe in the AdaSmooth setting as a typical example shown in Figure~\ref{fig:er-rmsprop_epochstart22}, where a smaller loss deterioration is observed in the AdaSmooth example compared to the RMSProp case.
\begin{figure}[h]
	\centering
	\includegraphics[width=0.6\textwidth]{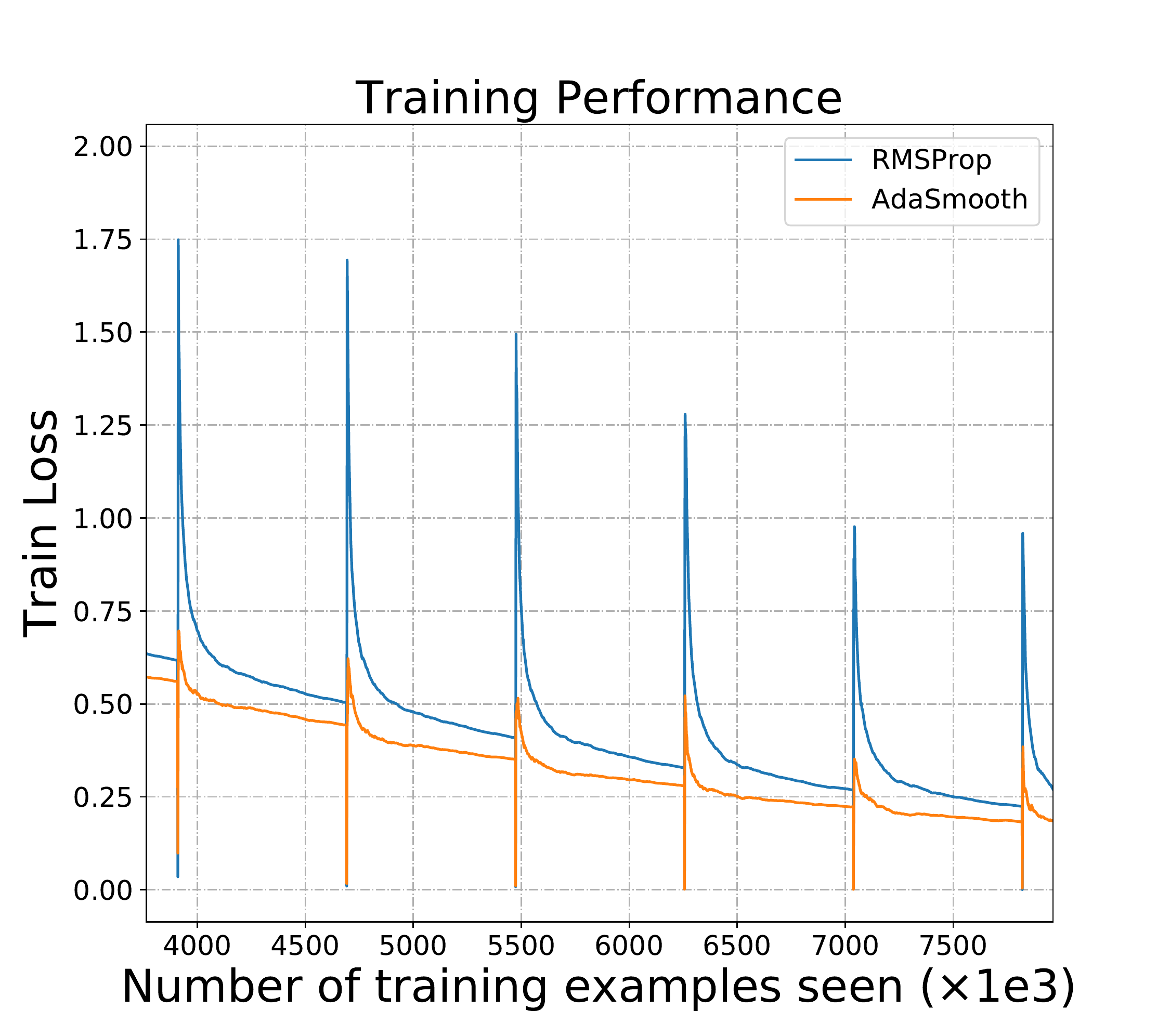}
	\caption{Demonstration of tuning parameter after each epoch by loading the weights. We save the weights and load them after each epoch such that there are step-points while re-training after each epoch. This issue is less sever in the AdaSmooth case than in the RMSProp method. A smaller loss deterioration is observed in the AdaSmooth example than that of the RMSProp case.}
	\label{fig:er-rmsprop_epochstart22}
\end{figure}
\index{Loss deterioration}

\begin{figure}[!h]
	\centering
	\subfigure[MNIST training Loss]{\includegraphics[width=0.47\textwidth, ]{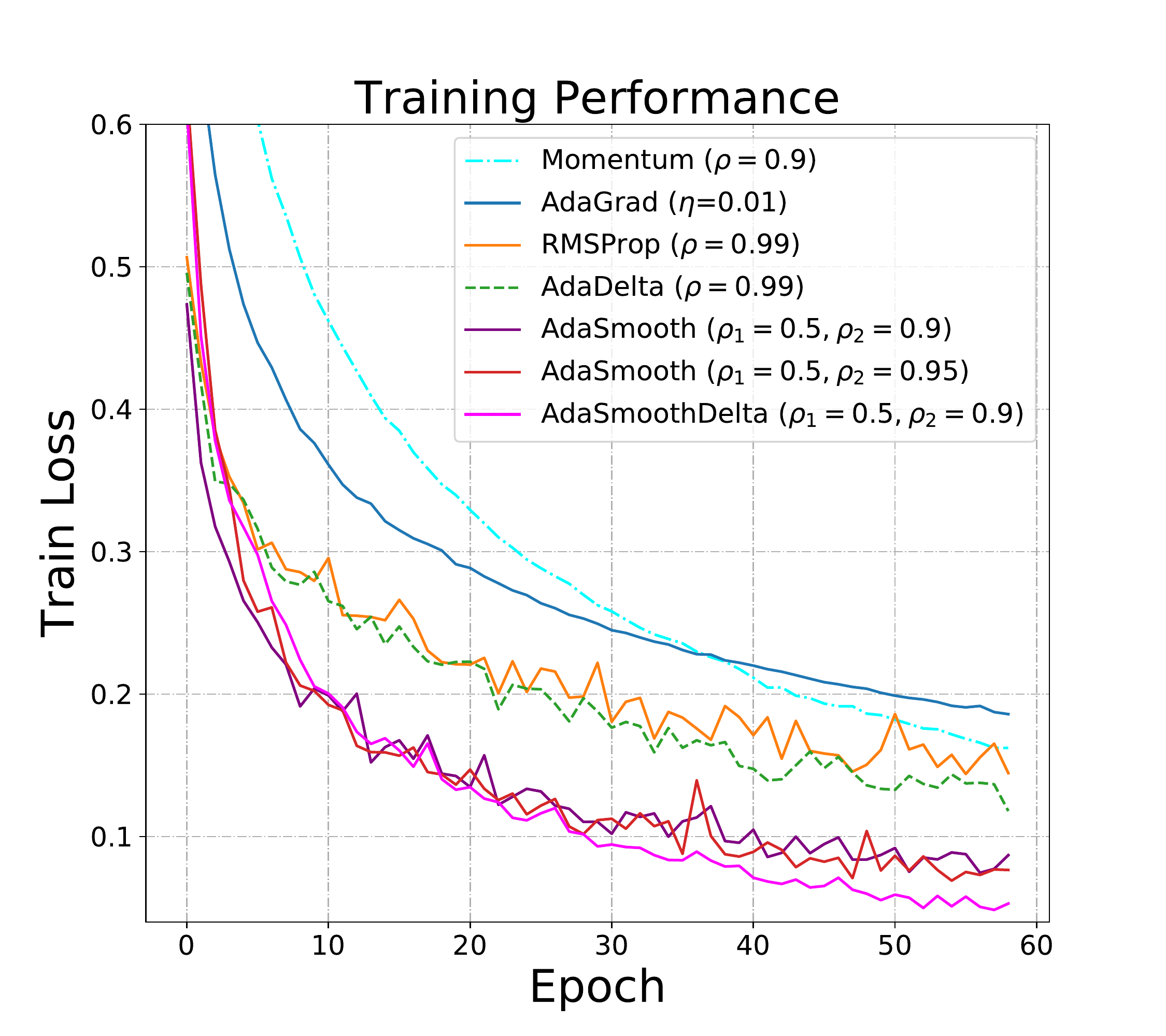} \label{fig:mnist_loss_train_mlp}}
	\subfigure[Census Income training loss]{\includegraphics[width=0.47\textwidth]{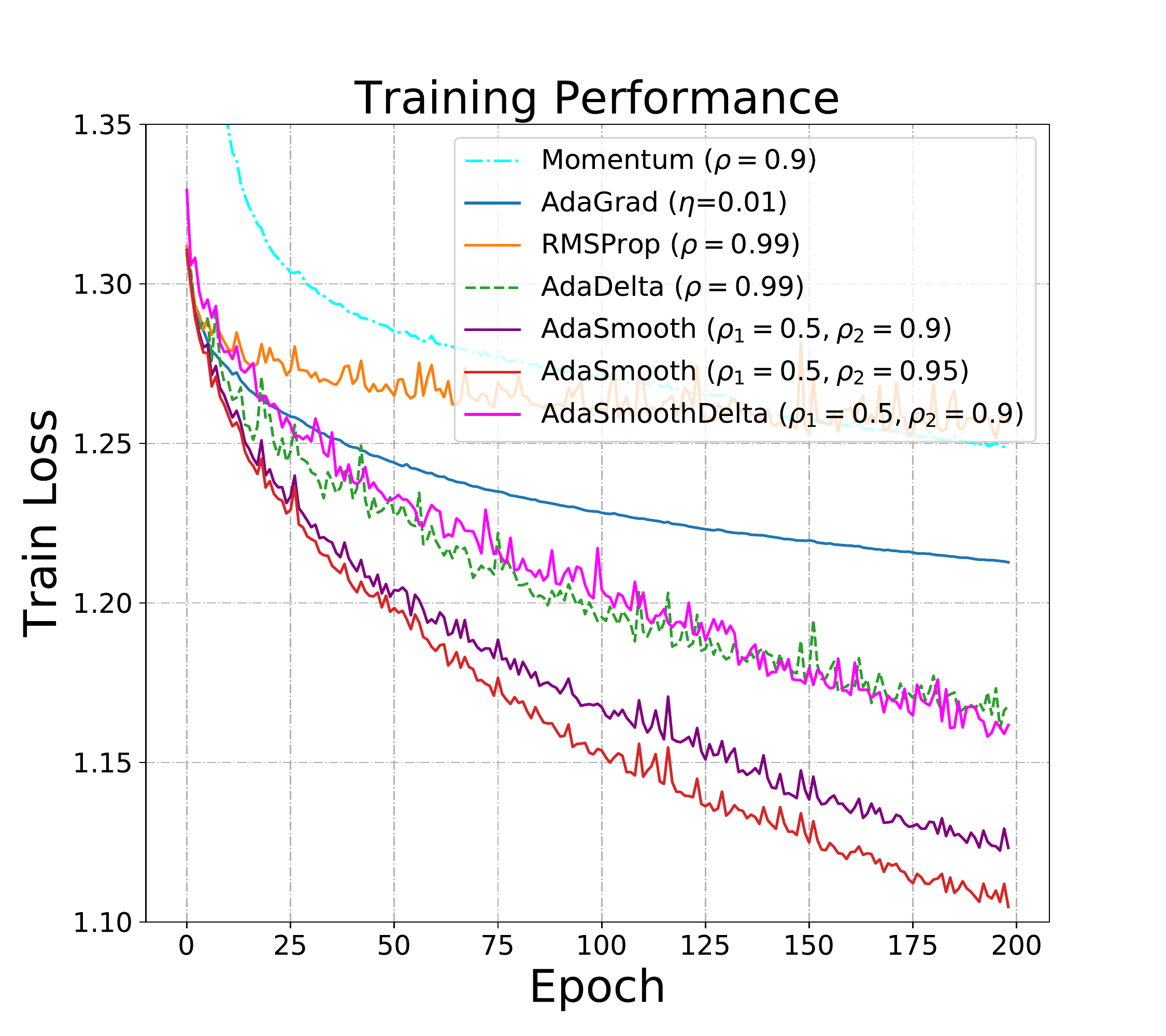} \label{fig:census_loss_train_mlp}}
	\caption{\textbf{MLP:} Comparison of descent methods on MNIST digit and Census Income data sets for 60 and 200 epochs with MLP.}
	\label{fig:mnist_census_loss_MLP}
\end{figure}

\paragraph{Example: multi-layer perceptron.}

To see the difference between the discussed algorithms by far,
we conduct experiments with different machine learning models; and different data sets including 
real handwritten digit classification task, MNIST \citep{lecun1998mnist} \footnote{It has a training set of 60,000 examples, and a test set of 10,000 examples.}, and Census Income \footnote{Census income data has 48842 number of samples and 70\% of them are used as the training set in our case: https://archive.ics.uci.edu/ml/datasets/Census+Income.} data sets are used.
In all scenarios, the same parameter initialization is adopted when training with different stochastic optimization algorithms. We compare the results in terms of convergence speed and generalization. 

Multi-layer perceptrons (MLP, a.k.a., multi-layer neural networks) are powerful tools for solving machine learning tasks, finding internal linear and nonlinear features behind the model inputs and outputs. We adopt the simplest MLP structure: an input layer, a hidden layer, and an output layer. We notice that rectified linear unit (Relu) outperforms Tanh, Sigmoid, and other nonlinear units in practice, making it the default nonlinear function in our structures. Since dropout has become a core tool in training neural networks \citep{srivastava2014dropout}, we adopt 50\% dropout noise to the network architecture during training to prevent overfitting.   
To be more concrete, the detailed architecture for each fully connected layer is described by F$(\langle \textit{num outputs} \rangle:\langle \textit{activation function} \rangle)$; and for a dropout layer is described by
DP$(\langle \textit{rate} \rangle)$. Then the network structure we use can be described as follows:
\begin{equation}
	\text{F(128:Relu)DP(0.5)F(\text{num of classes}:Softmax)}.
\end{equation}
All methods are trained on mini-batches of 64 images per batch for 60 or 200 epochs through the training set. Setting the hyper-parameter to $\epsilon=1e-6$.
If not especially mentioned, the global learning rates are set to $\eta=0.001$ in all scenarios. While a relatively large learning rate ($\eta=0.01$) is used for AdaGrad method due to  its accumulated decaying effect; learning rate for the AdaDelta method is set to 1 as suggested by \citet{zeiler2012adadelta} and for the AdaSmoothDelta method is set to 0.5 as discussed in Section~\ref{section:adasmoothdelta}.
In Figure~\ref{fig:mnist_loss_train_mlp} and ~\ref{fig:census_loss_train_mlp}, we compare SGD with momentum, AdaGrad, RMSProp, AdaDelta, AdaSmooth, and AdaSmoothDelta in optimizing the training set losses for MNIST and Census Income data sets, respectively. The SGD with momentum method does the worst in this case. AdaSmooth performs slightly better than AdaGrad and RMSProp in the MNIST case and much better than the latters in the Census Income case. AdaSmooth shows fast convergence from the initial epochs while continuing to reduce the training losses in both the two experiments. We here show two sets of slow decay constant for AdaSmooth, i.e., ``$\rho_2=0.9$" and ``$\rho_2=0.95$". Since we square the scaled smoothing constant in Eq.~\eqref{equation:squared-ssc}, when $\rho_1=\rho_2=0.9$, the AdaSmooth recovers to RMSProp with $\rho=0.99$ (so as the AdaSmoothDelta and AdaDelta case). In all cases, the AdaSmooth results perform better while there is almost no difference between the results of AdaSmooth with various hyper-parameters in the MLP model. Table~\ref{fig:mlp_table_perform} shows the best training set accuracy for different algorithms.
While we notice the best test set accuracies for various algorithms are very close; we only present the best ones for the first 5 epochs in Table~\ref{fig:mlp_table_perform-test}. 
In all scenarios, the AdaSmooth method converges slightly faster than other optimization methods in terms of the test accuracy for this toy example.

\begin{table}[!h]
	\centering
	\begin{tabular}{lll}
		\hline
		Method &\gap MNIST &\gap  Census \\ \hline
		SGD with Momentum ($\rho=0.9$) &\gap 98.64\% &\gap 85.65\%\\
		AdaGrad ($\eta$=0.01) &\gap 98.55\%&\gap 86.02\%\\
		RMSProp ($\rho=0.99$) &\gap 99.15\%&\gap 85.90\%\\
		AdaDelta ($\rho=0.99$) &\gap 99.15\%&\gap 86.89\%\\
		AdaSmooth ($\rho_1=0.5, \rho_2=0.9$) &\gap \textbf{99.34}\%&\gap \textbf{86.94}\%\\
		AdaSmooth ($\rho_1=0.5, \rho_2=0.95$) &\gap \textbf{99.45}\%&\gap \textbf{87.10}\%\\
		AdaSmoothDelta ($\rho_1=0.5, \rho_2=0.9$) &\gap \textbf{99.60}\%&\gap {86.86}\%\\
		\hline
	\end{tabular}
	\caption{\textbf{MLP}: Best in-sample evaluation in training accuracy (\%).}
	\label{fig:mlp_table_perform}
\end{table}

\begin{table}[!h]
	\centering
	\begin{tabular}{lll}
		\hline
		Method &\gap MNIST &\gap  Census \\ \hline
		SGD with Momentum ($\rho=0.9$) &\gap 94.38\%&\gap  83.13\%\\
		AdaGrad ($\eta$=0.01) &\gap 96.21\%& \gap84.40\%\\
		RMSProp ($\rho=0.99$) &\gap 97.14\%& \gap84.43\%\\
		AdaDelta ($\rho=0.99$) &\gap 97.06\%&\gap84.41\%\\
		AdaSmooth ($\rho_1=0.5, \rho_2=0.9$) &\gap 97.26\%&\gap 84.46\%\\
		AdaSmooth ($\rho_1=0.5, \rho_2=0.95$) &\gap 97.34\%&\gap 84.48\%\\
		AdaSmoothDelta ($\rho_1=0.5, \rho_2=0.9$) &\gap 97.24\% &\gap \textbf{84.51}\%\\
		\hline
	\end{tabular}
	\caption{\textbf{MLP}: Best out-of-sample evaluation in test accuracy for the first 5 epochs. }
	\label{fig:mlp_table_perform-test}
\end{table}


\section{Adam}
Adaptive moment estimation (Adam) is yet another adaptive learning rate optimization algorithm \citep{kingma2014adam}. The Adam algorithm uses a similar normalization by second-order information, the running estimates for squared gradient; however, it also incorporates first-order information into the update. In addition to storing the exponential moving average of past squared gradient (the second moment) like RMSProp, AdaDelta, and AdaSmooth, Adam also keeps an exponentially decaying average of the past gradients: 
\begin{equation}\label{equation:adam-updates}
	\begin{aligned}
		\bmm_t &=  \rho_1 \bmm_{t-1} + (1-\rho_1)\bg_t; \\
		\bv_t &= \rho_2 \bv_{t-1} +(1-\rho_2)\bg_t^2,
	\end{aligned}
\end{equation}
where $\bmm_t$ and $\bv_t$ are running estimates of the first moment (the mean) and the second moment (the uncentered variance) of the gradients, respectively. The drawback of RMSProp is that the running estimate $E[\bg^2]$ of the second-order moment is biased in the initial time steps since it is initialized to $\bzero$; especially when the decay constant is large (when $\rho$ is close to 1 in RMSProp). Observing the biases towards zero in Eq.~\eqref{equation:adam-updates} as well, Adam counteracts these biases by computing the bias-free moment estimates:
$$
\begin{aligned}
	\widehat{\bmm}_t &= \frac{\bmm_t}{1-\rho_1^t}; \\
	\widehat{\bv}_t &= \frac{\bv_t}{1-\rho_2^t}.
\end{aligned}
$$ 
The first and second moment estimates are then incorporated into the update step:
$$
\Delta \bx_t = - \frac{\eta}{\sqrt{\widehat{\bv}_t}+\epsilon} \odot  \widehat{\bmm}_t.
$$
And therefore the update becomes 
$$
\bx_{t+1} =\bx_t -  \frac{\eta}{\sqrt{\widehat{\bv}_t}+\epsilon} \odot  \widehat{\bmm}_t.
$$
In practice, \citet{kingma2014adam} suggest to use $\rho_1=0.9$, $\rho_2=0.999$, and $\epsilon=1e-8$ for the parameters by default.

\section{AdaMax}
Going further from Adam, \citet{kingma2014adam} notice the high-order moment:
$$
\bv_t = \rho_2^p \bv_{t-1} + (1-\rho_2^p) |\bg_t|^p,
$$
that is numerically unstable for large $p$ values, making $\ell_1$ and $\ell_2$ norms the common choices for updates. However, when $p\rightarrow \infty$, the $\ell_{\infty}$ also exhibits stable behavior. Therefore, the AdaMax admits the following moment update:
$$
\bu_t = \rho_2^\infty \bu_{t-1} + (1-\rho_2^\infty) |\bg_t|^\infty = \max (\rho_2\bu_{t-1}, |\bg_t|),
$$
where we do not need to correct for initialization bias in this case, and this yields the update step
$$
\Delta \bx_t = - \frac{\eta}{\bu_t} \odot  \widehat{\bmm}_t.
$$
In practice, \citet{kingma2014adam} suggest to use $\eta=0.002$, $\rho_1=0.9$, and $\rho_2=0.999$ as the default parameters.

\section{Nadam}
Nadam (Nesterov-accelerated Adam) combines the ideas of Adam and Nesterov momentum \citep{dozat2016incorporating}. We recall the momentum and NAG updates as follows:
$$
\boxed{
	\begin{aligned}
		&\text{Momentum:}\\
		&\bg_t = \nabla L(\bx_t);\\
		&\Delta \bx_t = \rho\Delta \bx_{t-1} - \eta \bg_t; \\
		&\bx_{t+1} = \bx_t +\Delta \bx_t,
	\end{aligned}
}
\gap 
\boxed{
	\begin{aligned}
		&\text{NAG:}\\
		&\bg_t = \nabla L(\bx_{t} + \rho \Delta \bx_{t-1});\\
		&\Delta \bx_t =\rho\Delta \bx_{t-1} - \eta\bg_t; \\
		&\bx_{t+1} = \bx_t +\Delta \bx_t,
	\end{aligned}
}
$$
\citet{dozat2016incorporating} first proposes a modification of NAG by using the current momentum vector to look ahead, which we call NAG$^\prime$ here. That is, applying the momentum update twice for each update:
$$
\boxed{
	\begin{aligned}
		&\text{NAG}^\prime:\\
		&\bg_t = \nabla L(\bx_t);\\
		&\Delta \bx_t = \rho\Delta \bx_{t-1} - \eta \bg_t; \\
		&\Delta \bx_t^\prime = \rho\Delta \bx_{t} - \eta \bg_t; \\
		&\bx_{t+1} = \bx_t +\Delta \bx_t^\prime .
	\end{aligned}
}
$$
By rewriting the Adam in the following form, where a similar modification according to NAG$^\prime$ leads to the Nadam update:
$$
\boxed{
	\begin{aligned}
		&\text{Adam}:\\
		&\bmm_t =  \rho_1 \bmm_{t-1} + (1-\rho_1)\bg_t;\\
		&\widehat{\bmm}_t = \rho_1 \frac{\bmm_{t-1} }{1-\rho_1^t} +\frac{1-\rho_1}{1-\rho_1^t} \bg_t; \\
		&\Delta \bx_t = - \frac{\eta}{\sqrt{\widehat{\bv}_t}+\epsilon} \odot  \widehat{\bmm}_t;\\
		&\bx_{t+1} = \bx_t +\Delta \bx_t,
	\end{aligned}
}
\leadto 
\boxed{
	\begin{aligned}
		&\text{Nadam}:\\
		&\bmm_t =  \rho_1 \bmm_{t-1} + (1-\rho_1)\bg_t;\\
		&\widehat{\bmm}_t = \rho_1 \frac{\bmm_{\textcolor{mylightbluetext}{t}} }{1-\rho_1^{\textcolor{mylightbluetext}{t+1}}} +\frac{1-\rho_1}{1-\rho_1^t} \bg_t; \\
		&\Delta \bx_t = - \frac{\eta}{\sqrt{\widehat{\bv}_t}+\epsilon} \odot  \widehat{\bmm}_t;\\
		&\bx_{t+1} = \bx_t +\Delta \bx_t.
	\end{aligned}
}
$$
However, the $\rho_1 \frac{\bmm_{t-1} }{1-\rho_1^t}$ in $\widehat{\bmm}_t$ of the Adam method can be replaced in a momentum fashion; by applying the same modification on NAG$^\prime$, the second version of Nadam has the following form (though it's not originally presented in \citet{dozat2016incorporating}):
$$
\boxed{
	\begin{aligned}
		&\text{Adam}^\prime:\\
		&\bmm_t =  \rho_1 \bmm_{t-1} + (1-\rho_1)\bg_t;\\
		&\widehat{\bmm}_t = \rho_1 \textcolor{mylightbluetext}{\widehat{\bmm}_{t-1}} +\frac{1-\rho_1}{1-\rho_1^t} \bg_t; \\
		&\Delta \bx_t = - \frac{\eta}{\sqrt{\widehat{\bv}_t}+\epsilon} \odot  \widehat{\bmm}_t;\\
		&\bx_{t+1} = \bx_t +\Delta \bx_t^\prime,
	\end{aligned}
}
\leadto 
\boxed{
	\begin{aligned}
		&\text{Nadam}^\prime:\\
		&\widehat{\bmm}_t = \rho_1 \frac{\bmm_{t-1} }{1-\rho_1^t} +\frac{1-\rho_1}{1-\rho_1^t} \bg_t; \\
		&\widehat{\bmm}_t^\prime = \rho_1 \textcolor{mylightbluetext}{\widehat{\bmm}_{t}} +\frac{1-\rho_1}{1-\rho_1^t} \bg_t; \\
		&\Delta \bx_t = - \frac{\eta}{\sqrt{\widehat{\bv}_t}+\epsilon} \odot  \widehat{\bmm}_t^\prime;\\
		&\bx_{t+1} = \bx_t +\Delta \bx_t^\prime.
	\end{aligned}
}
$$

\index{Saddle point}
\section{Problems in SGD}\label{section:c-problem}
The introduced optimizers for stochastic gradient descent  are widely used optimization algorithms, especially in training machine learning models. However, it has its challenges and potential issues. Here are some common problems associated with SGD:
\paragraph{Saddle points.}
When the Hessian of loss function is positive definite, then the optimal point $\bx_\star$ with vanishing gradient must be a local minimum. Similarly, when the Hessian is negative definite, the point is a local maximum; when the Hessian has both positive and negative eigenvalues, the point is a saddle point (see later discussion in Eq.~\eqref{equation:reparametrization-newton}, p.~\pageref{equation:reparametrization-newton}). 
The stochastic optimizers discussed above in practice are first order optimization algorithms: they only look at the gradient information, and never explicitly compute the Hessian. Such algorithms may get stuck at saddle points (toy example in Figure~\ref{fig:quadratic_saddle}, p.~\pageref{fig:quadratic_saddle}). In the algorithms presented earlier, including vanilla update, AdaGrad, AdaDelta, RMSprop, and others, this issue may arise. 
AdaSmooth may have chance to go out of the saddle points as argued in Section~\ref{section:adaer}. 
On the other hand, the mechanism of momentum and Nesterov momentum help point $\bx$ to go over the local minimum or saddle point because they have a term of previous step size (in general), but make the model more difficult to converge, especially when the momentum term $\rho$ is large. 

\paragraph{Low speed in SGD.}
However, although claimed in Rong Ge's post \footnote{http://www.offconvex.org/2016/03/22/saddlepoints/}, it is potential to converge to saddle points with high error rate, \citet{dauphin2014identifying} and Benjamin Recht's post \footnote{http://www.offconvex.org/2016/03/24/saddles-again/} point out that it is in fact super hard to converge to a saddle point if one picks a random initial point and run SGD. 
This is because the typical problem for both local minima and saddle points is that they are often surrounded by plateaus of small curvature in the error surface. In the SGD algorithms we discuss above, they are repelled away from a saddle point to a lower error by following the directions of negative curvature. In other words, there will be no so called saddle points problem in SGD algorithms. However, this repulsion can occur slowly due to the plateau with small curvature. 

While, for the second-order methods, e.g., Newton's method, it does not treat saddle points appropriately. This is partly because Newton's method is designed to rapidly descend plateaus surrounding local minima by rescaling gradient steps by the inverse eigenvalues of the Hessian matrix (we will see shortly in the sequel).

As argued in \citet{dauphin2014identifying}, random Gaussian error functions over large $d$ dimensions are increasingly likely to have saddle points rather than local minima as $d$ increases. And the ratio of the number of saddle points to local minima increases exponentially with the dimensionality $d$. The authors also argue that it is saddle points rather than local minima that provide a fundamental impediment to rapid high dimensional non-convex optimization. In this sense, local minima with high errors are exponentially rare in the dimensionality of the problem. So, the computation will be slow in SGD algorithms to escape from small curvature plateaus.

\paragraph{First-order method to escape from saddle point.}
The post \footnote{http://www.offconvex.org/2016/03/22/saddlepoints/} by Rong Ge introduces a first-order method to escape from saddle point. He claims that saddle points are very unstable: if we put a ball on a saddle point, then slightly perturb it, the ball is likely to fall to a local minimum, especially when the second-order term $\frac{1}{2} \Delta \bx^\top \bH \Delta \bx$ (see later discussion in Eq.~\eqref{equation:newton-derive}, p.~\pageref{equation:newton-derive}) is significantly smaller than 0 (i.e., there is a steep direction where the function value decreases, and assume we are looking for a local minimum), which is called a \textit{strict saddle function} in Rong Ge's post. In this case, we can use \textit{noisy gradient descent}:
$$
\bx_{t+1} = \bx_t + \Delta\bx + \bepsilon.
$$
where $\bepsilon$ is a noise vector that has zero mean $\mathbf{0}$. Actually, this is the basic idea of SGD, which uses the gradient of a mini batch rather than the true gradient. However, the drawback of the stochastic gradient descent is not the direction, but the size of the step along each eigenvector. The step, along any eigen-direction $\bq_i$, is given by $-\lambda_i \Delta {v_i}$ (see later discussion in Section~\ref{section:newtonsmethod}, p.~\pageref{section:newtonsmethod}, feel free to skip this paragraph at a first reading), when the steps taken in the direction with small absolute value of eigenvalues, the step is small. To be more concrete, as an example where the curvature of the error surface may not be the same in all directions. If there is a long and narrow valley in the error surface, the component of the gradient in the direction that points along the base of the valley is very small; while the component perpendicular to the valley walls is quite large even though we have to move a long distance along the base and a small distance perpendicular to the walls. This phenomenon can be seen in Figure~\ref{fig:momentum_gd} (though it's partly solved by SGD with momentum).

We normally move by making a step that is some constant times the negative gradient rather than a step of constant length in the direction of the negative gradient. This means that in steep regions (where we have to be careful not to make our steps too large), we move quickly; and in shallow regions (where we need to move in big steps), we move slowly. This phenomenon again contributes to the slower convergence of SGD methods compared to second-order methods.


%
%

%% file: chapter-second.tex
\newpage
\clearchapter{Second-Order Methods}
\begingroup
\hypersetup{linkcolor=winestain,
linktoc=page,  
}
\minitoc \newpage
\endgroup
\section{Second-Order Methods}\label{section:seconr-methods}
\lettrine{\color{caligraphcolor}W}
We previously addressed the derivation of AdaDelta based on the consistency of units in second-order methods  (Section~\ref{section:adadelta}, p.~\pageref{section:adadelta}). In this section, we provide a brief overview of Newton's method and its typical variations, including damped Newton's method and Levenberg gradient descent. Furthermore, we also derive the conjugate gradient from scratch that utilizes second-order information to capture the curvature shape of the loss surface in order to favor a faster convergence.

\section{Newton's Method}\label{section:newtonsmethod}
Newton's method is an optimization policy that employs Taylor's expansion to approximate the loss function with a quadratic form, providing an estimation of  the minimum location based on the approximated quadratic equation.
By Taylor's formula (Appendix~\ref{appendix:taylor-expansion}, p.~\pageref{appendix:taylor-expansion}) and disregarding  derivatives of higher order, the loss function $L(\bx+\Delta \bx)$ can be approximated by
\begin{equation}\label{equation:newton-derive}
L(\bx+\Delta \bx) \approx L(\bx) +\Delta \bx^\top \nabla L(\bx) + \frac{1}{2} \Delta \bx^\top \bH \Delta \bx,
\end{equation}
where $\bH$ is the Hessian of the loss function $L(\bx)$ with respect to $\bx$. The optimal point (minimum point) of Eq.~\eqref{equation:newton-derive} is then obtained at 
$$
\bx_\star = \bx - \bH^{-1} \nabla L(\bx).
$$
That is, the update step is rescaled by the inverse of Hessian. 
Newton's update can be intuitively understood as utilizing curvature information from the Hessian:  when the curvature is steep, the inverse Hessian can scale more making the step smaller; while the curvature is flat, the inverse Hessian scales less resulting in a larger update.
However, for nonlinear $L(\bx)$, achieving the minimum in a single step is not feasible. 
Similar to the stochastic optimization methods we have introduced in previous sections, Newton's method is updated iteratively as formulated in Algorithm~\ref{alg:newton_method} \citep{roweis1996levenberg, goodfellow2016deep}. 

\begin{algorithm}[H] 
\caption{Newton's Method}
\label{alg:newton_method}
\begin{algorithmic}[1]
	\State {\bfseries Input:} Initial parameter $\bx_1$;
	\For{$t=1:T$ } 
	\State Compute gradient $\bg_t = \nabla L(\bx_{t})$;
	\State Compute Hessian $\bH_t = \nabla^2 L(\bx_{t});$
	\State Compute inverse Hessian $\bH_t^{-1}$;
	\State Compute update step $\Delta \bx_t = -\bH_t^{-1}\bg_t$;
	\State Apply update $\bx_{t+1} = \bx_{t} + \Delta \bx_t$;
	\EndFor
	\State {\bfseries Return:} resulting parameters $\bx_t$, and the loss $L(\bx_t)$.
\end{algorithmic}
\end{algorithm}

The computational complexity of Newton's method arises from the computation of the inverse of the Hessian at each training iteration.
The number of entries in the Hessian matrix is squared in the number of parameters ($\bx\in \real^d, \bH\in \real^{d\times d}$), making the complexity of the inverse be of order $\mathcal{O}(d^3)$ \citep{trefethen1997numerical, boyd2018introduction, lu2021numerical}. As a consequence, only networks with a small number of parameters, e.g., shallow neural networks or multi-layer perceptrons, can be trained by Newton's method in practice.

\index{Spectral decomposition}
\index{Reparametrization}
\paragraph{Reparametrization of the space around critical point.}
A \textit{critical point or stationary point} is a point $\bx$ where the gradient of $L(\bx)$ vanishes (suppose $L(\bx)$ is differentiable over the region of $L(\bx)$). A useful reparametrization of the loss function $L$ \textbf{around critical points} is derived from Taylor's expansion. Since the gradient vanishes, Eq.~\eqref{equation:newton-derive} can be written as
\begin{equation}
L(\bx+\Delta \bx) \approx L(\bx)  + \frac{1}{2} \Delta \bx^\top \bH \Delta \bx.
\end{equation}
Since Hessian $\bH$ is symmetric, it admits spectral decomposition (Theorem 13.1 in \citet{lu2022matrix} or Appendix~\ref{appendix:spectraldecomp}, p.~\pageref{appendix:spectraldecomp}):
$$
\bH=\bQ\bLambda\bQ^\top \in \real^{d\times d},
$$ 
where the columns of $\bQ = [\bq_1, \bq_2, \ldots , \bq_d]$ are eigenvectors of $\bH$ and are mutually orthonormal, and the entries of $\bLambda = \diag(\lambda_1, \lambda_2, \ldots , \lambda_d)$ are the corresponding real eigenvalues of $\bA$. We define the following vector $\Delta\bv$:
$$
\Delta \bv=
\begin{bmatrix}
\bq_1^\top \\ \vdots \\ \bq_d^\top 
\end{bmatrix}
\Delta\bx=
\bQ^\top \Delta\bx.
$$
Then the reparametrization follows
\begin{equation}\label{equation:reparametrization-newton}
\begin{aligned}
	L(\bx+\Delta \bx) &\approx L(\bx)  + \frac{1}{2} \Delta \bx^\top \bH \Delta \bx\\
	&= L(\bx)  +  \frac{1}{2} \Delta \bv^\top \bLambda \Delta\bv 
	= L(\bx)  +  \frac{1}{2}  \sum_{i=1}^{d} \lambda_i (\Delta v_i)^2,
\end{aligned}
\end{equation}
where $\Delta v_i$ is the $i$-th element of $\Delta\bv$. A conclusion on the type of the critical point follows immediately from the reparametrization:
\begin{itemize}
\item If all eigenvalues are nonzero and positive, then the critical point is a local minimum;

\item If all eigenvalues are nonzero and negative, then the critical point is a local maximum;

\item If the eigenvalues are nonzero, and both positive and negative eigenvalues exist, then the critical point is a saddle point.
\end{itemize}

In vanilla GD, if an eigenvalue $\lambda_i$ is positive (negative), then the step moves towards (away) from $\bx$ along $\Delta \bv$, guiding the GD towards optimal $\bx_\star$. The step along any direction $\bq_i$ is given by $-\lambda_i\Delta v_i$. 
In contrast, in Newton's method, the step is rescaled by the inverse Hessian, making the step along direction $\bq_i$ scaled into $-\Delta v_i$. This may cause problem when the eigenvalue is negative, resulting in the step moving in the \textit{opposite} direction compared to  vanilla GD \citep{dauphin2014identifying}. 

The reparametrization shows that rescaling the gradient along the direction of each eigenvector can result in wrong direction when the eigenvalue $\lambda_i$ is negative. This suggests the rescale by its magnitude, i.e., scale by $1/|\lambda_i|$ rather than $1/\lambda_i$, preserving sign of the gradient and addressing the slowness issue of GD at the same time. From Eq.~\eqref{equation:reparametrization-newton}, when both positive and negative eigenvalues exist, both vanilla GD and Newton's method may get stuck in saddle points, leading to suboptimal performance. However, scaling by $1/|\lambda_i|$ can partly solve this problem, since the movement around the saddle point can either increase  or decrease the loss from this rescaling, rather than stay where it is \citep{nocedal1999numerical, dauphin2014identifying}.

\section{Damped Newton's Method}
Newton's method addresses the slowness problem by rescaling the gradients in each direction with the inverse of the corresponding eigenvalue, yielding the step $\Delta \bx_t = -\bH_t^{-1}\bg_t$ at iteration $t$. However, this approach can result in moving in an undesired  direction when the eigenvalue is negative, causing the  Newton's step to proceed along the eigenvector in a direction opposite to the gradient descent step and increasing the error. 

To mitigate this issue, damping the Hessian is proposed, where negative curvature is eliminated by introducing a constant $\alpha$ to its diagonal, yielding the step $\Delta \bx_t= - (\mathbf{H}+\alpha \mathbf{I})^{-1} \bg_t$. We can view $\alpha$ as the tradeoff between Newton's method and vanilla GD. When $\alpha$ is small, it is closer to Newton's method; when $\alpha$ is large, it is closer to vanilla GD. In this case, we get the step $-\frac{\lambda_i}{\lambda_i + \alpha}\Delta \bg_t$. 
However, obviously, the drawback of damping Newton's method is evident{\textemdash}it may result in a small step size across many eigen-directions due to the influence of the large damping factor $\alpha$.

\section{Levenberg (-Marquardt) Gradient Descent}
The quadratic rule is not universally better since it assumes a linear approximation of $L(\bx)$, which is only valid when it's in proximity to a minimum. The \textit{Levenberg gradient descent} goes further by combining the idea of damped Newton's method and vanilla GD. Initially, we can apply a steepest descent type method until we approach a minimum, prompting switching to the quadratic rule. The distance from a minimum is described by evaluating the loss \citep{levenberg1944method}. If the loss is increasing, the quadratic approximation is not working well and we are not likely near a minimum, yielding a larger $\alpha$ in damped Newton's method; while if the loss is decreasing, the quadratic approximation is working well and we are approaching a minimum, yielding a smaller $\alpha$ in damped Newton's method.

Marquardt improved this method by incorporating the local curvature information. In this modification,  one replaces the identity matrix in Levenberg's method by the diagonal of the Hessian, resulting in the \textit{Levenberg-Marquardt gradient descent} \citep{marquardt1963algorithm}:
$$
\Delta \bx_t= - \bigg(\mathbf{H}+\alpha \cdot \diag(\bH)\bigg)^{-1} \bg_t.
$$
The Levenberg-Marquardt gradient descent method is nothing more than a heuristic method since it is not optimal for any well defined criterion of speed or error measurements, and it is merely a well thought out optimization procedure. However, it is an optimization method that works extremely well in practice, especially for medium sized nonlinear models. 


\index{Conjugate gradient}
\index{Hessian-orthogonal}
\section{Conjugate Gradient}\label{section:conjugate-descent}
We have shown that the vanilla GD (employing the negative gradient as the descent direction) can move back and forth in a zigzag pattern when applied in a quadratic bowl (a ravine-shaped loss curve, see example in Figure~\ref{fig:quadratic_vanillegd_contour8}, p.~\pageref{fig:quadratic_vanillegd_contour8}). 
This zigzag behavior becomes more pronounced if the learning rate is guaranteed by line search (Section~\ref{section:line-search}, p.~\pageref{section:line-search}), since the gradient is orthogonal to the previous update step (Lemma~\ref{lemm:linear-search-orghonal}, p.~\ref{lemm:linear-search-orghonal} and example in Figure~\ref{fig:conjguatecy_zigzag2}, p.~\pageref{fig:conjguatecy_zigzag2}). The choice of orthogonal descent directions fails to preserve the minimum along the previous search directions, and the line search will undermine the progress we have already achieved in the direction of the previous line search. This results in the zigzag pattern in  movement.

Instead of favoring a descent direction that is orthogonal to previous search direction (i.e., $\bd_{t+1}^\top \bd_t=0$), the \textit{conjugate descent} selects a search direction that is \textit{Hessian-orthogonal} (i.e., $\bd_{t+1}^\top\bH\bd_t=0$, or conjugate with respect to $\bH$). This choice ensures that the movement is compensated by the curvature information of the loss function. In Figure~\ref{fig:conjugate_tile_A_orthogonals}, we show examples of Hessian-orthogonal pairs when the eigenvalues of the Hessian matrix $\bH$ are different or identical. When the eigenvalues of the Hessian matrix are the same, the Hessian-orthogonality reduces to trivial orthogonal cases (this can be shown by the spectral decomposition of the Hessian matrix, where the orthogonal transformation does not alter the orthogonality \citep{lu2021numerical}).

\begin{figure}[h]
\centering  
\vspace{-0.35cm} 
\subfigtopskip=2pt 
\subfigbottomskip=2pt 
\subfigcapskip=-5pt 
\subfigure[Surface plot: Hessian with different eigenvalues.]{\label{fig:conjugate_tile_A_orthogonal_diffeigenv}
	\includegraphics[width=0.481\linewidth]{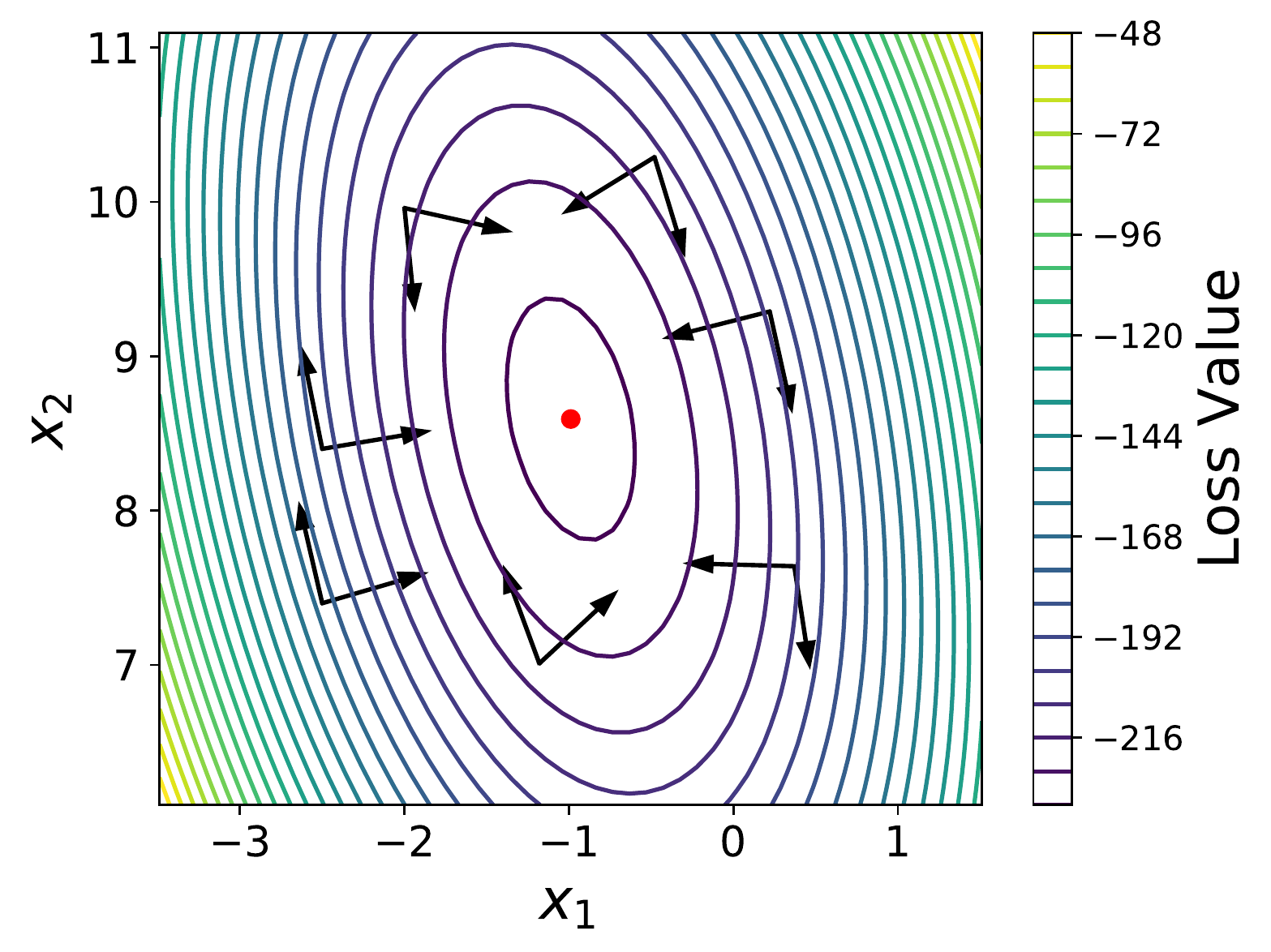}}
\subfigure[Surface plot: Hessian with same eigenvalues.]{\label{fig:conjugate_tile_A_orthogonal_sameeigenv}
	\includegraphics[width=0.481\linewidth]{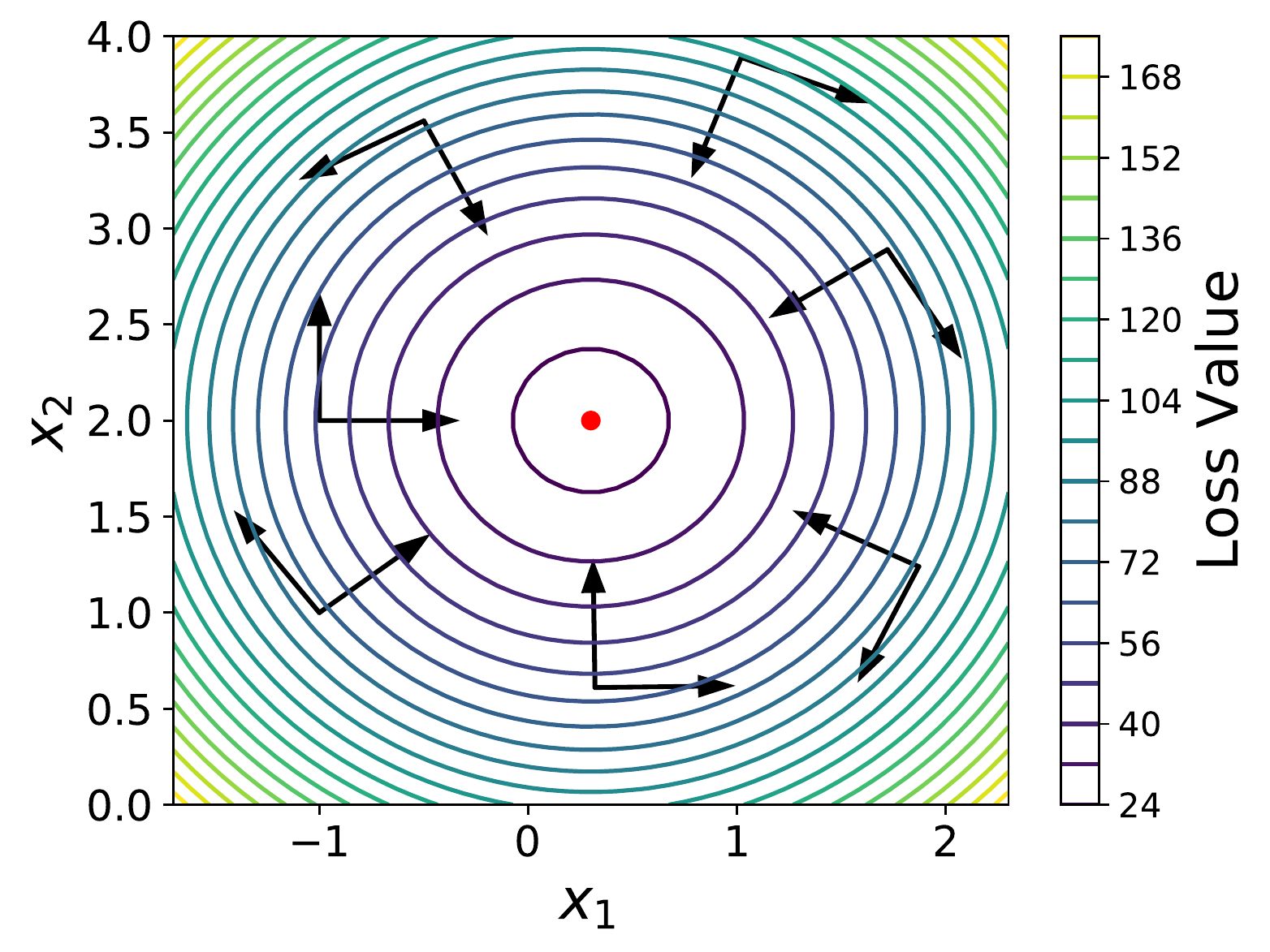}}
\caption{Illustration of $\bH$-orthogonal for different Hessian matrices in two-dimensional case:
	$\bH=\begin{bmatrix}
		40 & 5 \\ 7 & 10
	\end{bmatrix}$ for Fig~\ref{fig:conjugate_tile_A_orthogonal_diffeigenv}
	and 
	$\bH=\begin{bmatrix}
		40 & 0 \\ 0 & 40
	\end{bmatrix}$ 
	for Fig~\ref{fig:conjugate_tile_A_orthogonal_sameeigenv}. The $\bH$-orthogonal pairs are orthogonal when $\bH$ has identical eigenvalues.}
\label{fig:conjugate_tile_A_orthogonals}
\end{figure}

We now provide the formal definition of conjugacy as follows:
\begin{definition}[Conjugacy]\label{definition:conjugacy}
Given a positive definite matrix $\bA\in \real^{d\times d}$, the vectors $\bu, \bv\in \real^d$ are \textit{conjugate} with respect to $\bA$ if $\bu,\bv\neq \bzero$ and $\bu^\top\bA\bv = 0$.  
\end{definition}
In the method of \textit{conjugate gradient} (CG), we determine a descent direction that is conjugate to the previous search direction with respect to the Hessian matrix $\bH$, ensuring that the new update step will not undo the progress made in the previous directions:
$$
\bd_{t} = -\nabla L(\bx_t) + \beta_t \bd_{t-1},
$$
where $\beta_t$ is a coefficient controlling how much the previous direction would add back to the current search direction.
Three commonly used methods to compute the coefficient are as follows:
$$
\begin{aligned}
\text{Fletcher-Reeves:\gap } &\beta_t^F = \frac{\nabla L(\bx_t)^\top \nabla L(\bx_t)}{\nabla L(\bx_{t-1})^\top \nabla L(\bx_{t-1})},\\
\text{Polak-Ribi\`ere:\gap } &\beta_t^P = \frac{\bigg(\nabla L(\bx_t) - \nabla L(\bx_{t-1})\bigg)^\top \nabla L(\bx_t)}{\nabla L(\bx_{t-1})^\top \nabla L(\bx_{t-1})},\\
\text{Hestenes–Stiefel:\gap } &\beta_t^H = \frac{\bigg(\nabla L(\bx_t) - \nabla L(\bx_{t-1})\bigg)^\top \nabla L(\bx_t)}{\bigg(\nabla L(\bx_t) - \nabla L(\bx_{t-1})\bigg)^\top
	\bd_{t-1}}.
\end{aligned}
$$
In the case of a quadratic loss function, the conjugate gradient ensures that the gradient along the previous direction does not increase in magnitude \citep{shewchuk1994introduction, nocedal1999numerical, iserles2009first, goodfellow2016deep}. The full procedure of the conjugate gradient method is formulated in Algorithm~\ref{alg:conjugate_descent}, where it is observed that the first step of conjugate gradient is identical to a step of steepest descent when the learning rate is calculated by exact line search since $\beta_1=0$.


\begin{algorithm}[H] 
\caption{Fletcher-Reeves Conjugate Gradient}
\label{alg:conjugate_descent}
\begin{algorithmic}[1]
	\State {\bfseries Input:} Initial parameter $\bx_1$;
	\State {\bfseries Input:} Initialize $\bd_0 =\bzero $ and $\bg_0 = \bd_0+\epsilon$;
	\For{$t=1:T$ } 
	\State Compute gradient $\bg_t = \nabla L(\bx_{t})$;
	\State Compute coefficient $\beta_{t} = \frac{ \bg_t^\top \bg_t}{\bg_{t-1}^\top \bg_{t-1}}$ (\text{Fletcher-Reeves });
	\State Compute descent direction $\bd_t = -\bg_{t} +\beta_{t}  \bd_{t-1}$;
	\State Fixed learning rate $\eta_t=\eta$ or find it by line search: $\eta_t = \arg\min L(\bx_t + \eta \bd_t)$;
	\State Compute update step $\Delta \bx_t = \eta_t\bd_t$;
	\State Apply update $\bx_{t+1} = \bx_{t} + \Delta \bx_t$;
	\EndFor
	\State {\bfseries Return:} resulting parameters $\bx_t$, and the loss $L(\bx_t)$.
\end{algorithmic}
\end{algorithm}

\begin{figure}[h]
\centering  
\vspace{-0.35cm} 
\subfigtopskip=2pt 
\subfigbottomskip=2pt 
\subfigcapskip=-5pt 
\subfigure[Vanilla GD, fixed $\eta=0.08$.]{\label{fig:cgm_conjugate8}
	\includegraphics[width=0.22\linewidth]{./imgs/steepest_gd_mom-0_lrate-8.pdf}}
\subfigure[Steepest descent.]{\label{fig:cgm_zigzag2}
	\includegraphics[width=0.22\linewidth]{./imgs/steepest_gd_bisection.pdf}}
\subfigure[Conjugate descent, fixed $\eta=0.06$.]{\label{fig:cgm_conjugate2}
	\includegraphics[width=0.22\linewidth]{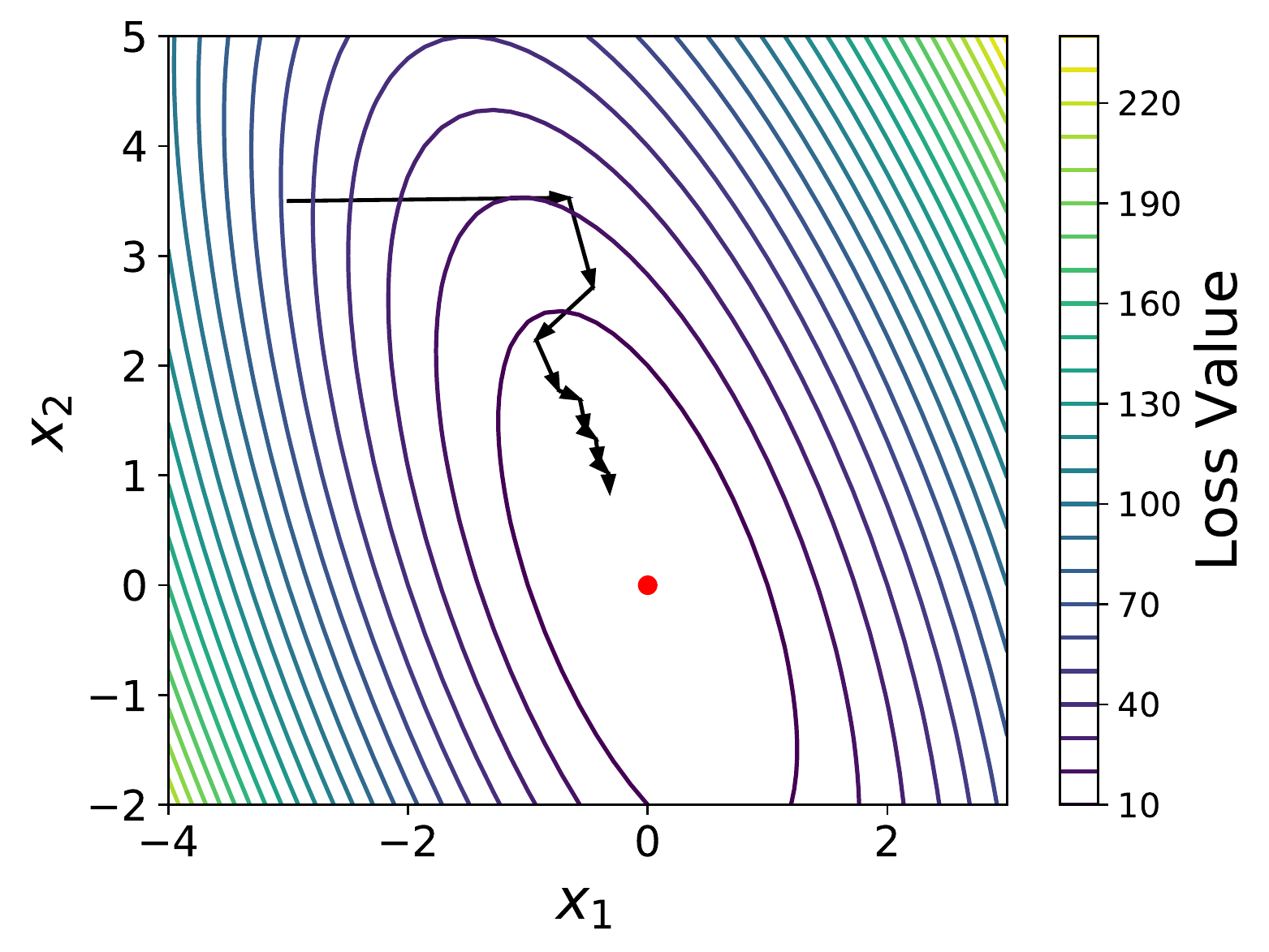}}
\subfigure[Conjugate descent, exact line search.]{\label{fig:cgm_conjugate3}
	\includegraphics[width=0.22\linewidth]{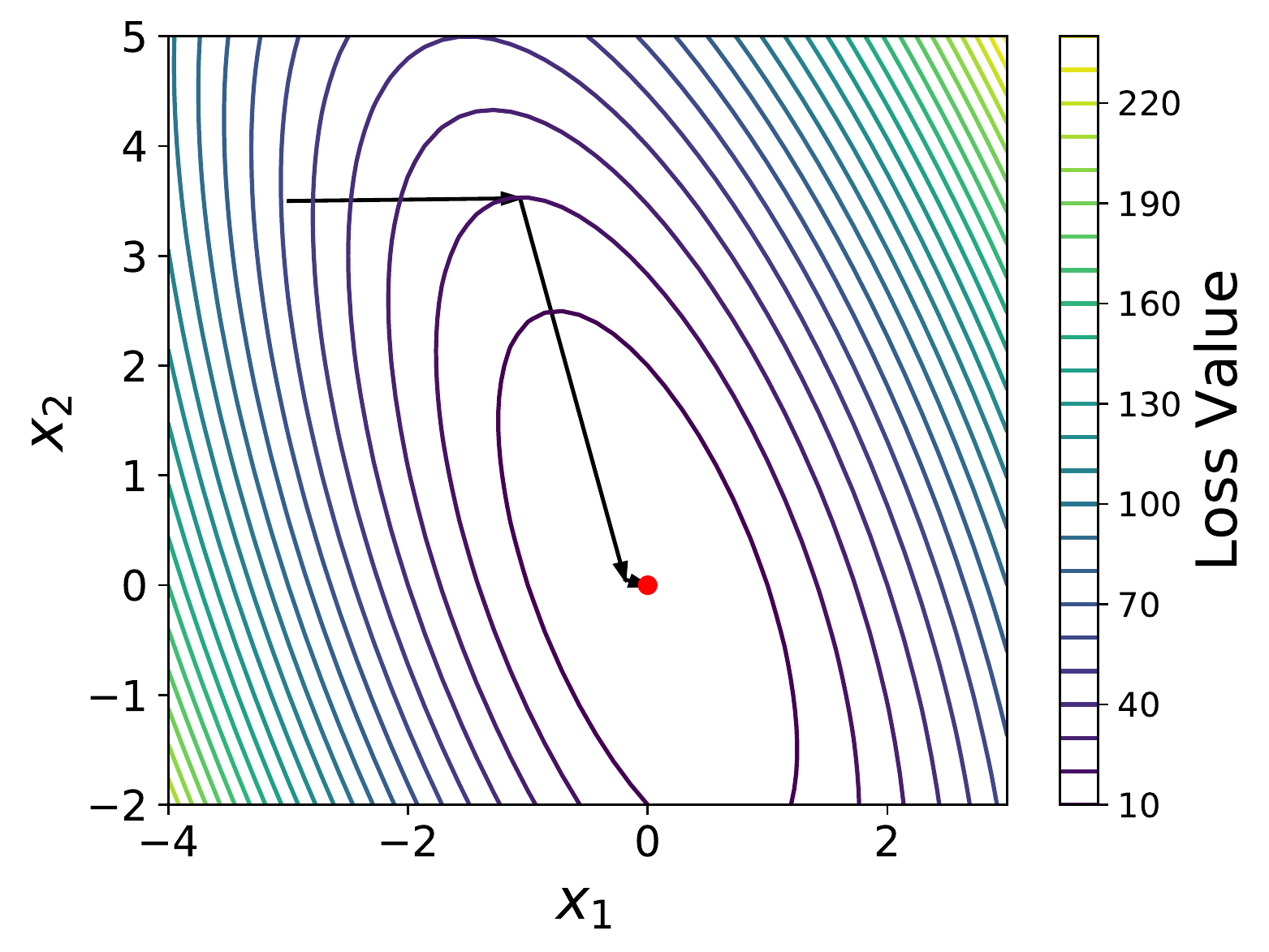}}
\caption{Illustration for the vanilla GD, steepest descent, and CG of quadratic form with $\bA=\begin{bmatrix}
		20 & 7 \\ 5 & 5
	\end{bmatrix}$, $\bb=\bzero$, and $c=0$. The starting point to descent is $\bx_1=[-3,3.5]^\top$.}
\label{fig:cgm-zigzzag}
\end{figure}

In the subsequent sections, we base our derivation of the conjugate gradient on the assumption of a symmetric positive definite $\bA$;  
however, it can be readily adapted to asymmetric matrices. A comparison among vanilla GD, steepest descent, and conjugate gradient is shown in Figure~\ref{fig:cgm-zigzzag}, where we observe that the updates in CG have less zigzag pattern than vanilla GD and steepest descent.

\index{Quadratic form}
\subsubsection{Quadratic Form in Conjugate Direction (CD) Method}
Following the discussion of the quadratic form in GD (Section~\ref{section:quadratic_vanilla_GD}, p.~\pageref{section:quadratic_vanilla_GD}), the quadratic form in steepest descent (Section~\ref{section:quadratic-in-steepestdescent}, p.~\pageref{section:quadratic-in-steepestdescent}), and the quadratic form in momentum (Section~\ref{section:quadratic-in-momentum}, p.~\pageref{section:quadratic-in-momentum}), we turn our attention to  the quadratic form in CG. 
To introduce the \textit{conjugate gradient (CG)} method, we proceed by an exploration of the \textit{conjugate direction (CD)} method, where the distinction between them will become evident in the subsequent discussions. 
According to the definition of conjugacy (Definition~\ref{definition:conjugacy}), it is easy to show that any set of vectors $\{\bd_1, \bd_2, \ldots, \bd_d\}\in \real^d$ satisfying this property with respect to the symmetric positive definite Hessian matrix $\bH=\frac{1}{2}(\bA^\top+\bA)$: 
\begin{equation}\label{equation:conjguate-definis}
\bd_i^\top \bH\bd_j, \gap \forall i\neq j,
\end{equation}
is also linearly independent. That is, the set span the entire $\real^d$ space:
$$
\text{span}\{\bd_1, \bd_2, \ldots, \bd_d\} = \real^d.
$$
Given the initial parameter $\bx_1$ and a set of \textit{conjugate directions} $\{\bd_1, \bd_2, \ldots, \bd_d\}$ (defined in Eq.~\eqref{equation:conjguate-definis}), the update at time $t$ is given by 
\begin{equation}\label{equation:conjugate_direction-update}
\bx_{t+1} = \bx_t +\eta_t \bd_t,
\end{equation}
where $\eta_t$ is the learning rate at time $t$ and is obtained by minimizing the one-dimensional quadratic function $J(\eta)=L(\bx_t+\eta\bd_t)$, as presented in Eq.~\eqref{equation:eta-gd-steepest} (p.~\pageref{equation:eta-gd-steepest}):
\begin{equation}\label{equation:conjugate_direction-update2}
\eta_t = - \frac{\bd_t^\top \bg_t}{ \bd_t^\top \bA\bd_t } \gap \text{with}\gap  \bg_t = \nabla L(\bx_t).
\end{equation}
Then, we can establish the following theorem that the updates following from the conjugate directions will converge in $d$ steps (the dimension of the parameter) when $\bA$ is symmetric positive definite.
\begin{theorem}[Converge in $d$ Steps]\label{theorem:conjudate_direc_d-steps}
For any initial parameter $\bx_1$, the sequence $\{\bx_t\}$ generated by the conjugate direction algorithm Eq.~\eqref{equation:conjugate_direction-update} converges to the solution $\bx_\star$ in at most $d$ steps {when $\bA$ is symmetric positive definite}.
\end{theorem}
\begin{proof}[of Theorem~\ref{theorem:conjudate_direc_d-steps}]
Since conjugate directions $\{\bd_1, \bd_2, \ldots, \bd_d\}$ span the entire $\real^d$ space, the initial error vector $\be_1 = \bx_1 - \bx_\star$ (Definition~\ref{definition:error-gd-}, p.~\pageref{definition:error-gd-}) then can be expressed as a linear combination of the conjugate directions:
\begin{equation}\label{equation:dsteps-symmetric}
	\be_1 = \bx_1-\bx_\star = \gamma_1\bd_1 + \gamma_2\bd_2+\ldots+\gamma_d\bd_d.
\end{equation}
The $\gamma_i$'s can be obtained by the following equation:
$$
\begin{aligned}
	\bd_t^\top \bH\be_1 &= \sum_{i=1}^{d} \gamma_i \bd_t^\top \bH\bd_i=\gamma_t \bd_t^\top \bH\bd_t &\text{(by conjugacy, Eq.~\eqref{equation:conjguate-definis})}\\
	\leadto \gamma_t &= \frac{\bd_t^\top \bH\be_1}{\bd_t^\top \bA\bd_t}= \frac{\bd_t^\top \bH(\be_1+ \sum_{i=1}^{t-1}\eta_i \bd_i )}{\bd_t^\top \bH\bd_t} &\text{(by conjugacy, Eq.~\eqref{equation:conjguate-definis})}\\
	&=\frac{\bd_t^\top \bH\be_t}{\bd_t^\top \bH\bd_t}.
\end{aligned}
$$
When $\bA$ is  symmetric and nonsingular, we have $\be_t = \bx_t-\bx_\star  = \bx_t-\bA^{-1}\bb$ and $\bH=\bA$. It can be shown that $\gamma_t$ is equal to $\frac{\bd_t^\top \bg_t}{\bd_t^\top \bA\bd_t}$. This is exactly the same form (in magnitude) as the learning rate at time $t$ in steepest descent: $\gamma_t = -\eta_t$ (see Eq.~\eqref{equation:eta-gd-steepest}, p.~\pageref{equation:eta-gd-steepest}). Substituting into Eq.~\eqref{equation:dsteps-symmetric}, it follows that 
$$
\bx_\star = \bx_1 + \eta_1\bd_1+\eta_2\bd_2+\ldots +\eta_d\bd_d.
$$
Moreover, we have updates by Eq.~\eqref{equation:conjugate_direction-update} that
$$
\begin{aligned}
	\bx_{d+1} &= \bx_d + \eta_d\bd_d  \\
	&=\bx_{d-1} +\eta_{d-1}\bd_{d-1} + \eta_d\bd_d \\
	&= \ldots \\
	&= \bx_1 + \eta_1\bd_1 + \eta_2\bd_2+\ldots+\eta_d\bd_d = \bx_\star,
\end{aligned}
$$
which completes the proof.
\end{proof}

The above theorem states the conjudate direction given by Eq.~\eqref{equation:conjugate_direction-update} converges in $d$ steps, i.e., $\bx_{d+1}$ minimizes the quadratic function $L(\bx)=\frac{1}{2}\bx^\top\bA\bx-\bb^\top\bx+c$ over the entire space $\real^d$. Furthermore, we can prove at each iteration $t\leq d$, the update $\bx_{t+1}$ minimizes the quadratic function over a subspace of $\real^d$.

\begin{theorem}[Expanding Subspace Minimization]\label{theorem:expanding_subspace_minimization}
For any initial parameter $\bx_1$, the sequence $\{\bx_t\}$ is generated by the conjugate direction algorithm Eq.~\eqref{equation:conjugate_direction-update}. Then it follows that 
\begin{equation}\label{equation:expanding_subspace_minimization_zero}
	\bg_{t+1}^\top \bd_i=0, \gap \forall i=1,2,\ldots, t, \text{ and } t\in \{1,2,\ldots, d\},
\end{equation}
where $\bg_t = \bA\bx_t - \bb$ (i.e., the gradient when $\bA$ is symmetric),
and $\bx_{t+1}$ is the minimizer of $L(\bx)=\frac{1}{2}\bx^\top\bA\bx-\bb^\top\bx+c$ with symmetric positive definite $\bA$ over the subspace
\begin{equation}\label{equation:space_d_t}
	\mathbb{D}_t=\left\{\bx | \bx=\bx_1 + \text{span}\{\bd_1, \bd_2, \ldots, \bd_t\}\right\}.
\end{equation}
\end{theorem}
\begin{proof}[of Theorem~\ref{theorem:expanding_subspace_minimization}]
We first prove $\bg_{t+1}^\top\bd_i$ by induction. When $t=1$, since $\eta_1$ is obtained to minimize $J(\eta)=L(\bx_1+\eta\bd_1)$, by Lemma~\ref{lemm:linear-search-orghonal} (p.~\pageref{lemm:linear-search-orghonal}), we have $\bg_2 = \nabla L(\bx_2)$ that is orthogonal to $\bd_1$. Suppose now for general $t-1$, the induction hypothesis is satisfied with $\bg_{t}^\top\bd_i=0$ for $i=0,1,\ldots, t-1$. The $\bg_t$ has the following update
\begin{equation}\label{equation:conjguate-redsidual-update}
	\begin{aligned}
		\bg_{t+1} &= \bA\bx_{t+1}-\bb  =  \bA (\bx_t+\eta_t\bd_t) -\bb  \\
		&= \bg_t + \eta_t\bA\bd_t.
	\end{aligned}
\end{equation}
By conjugacy and the induction hypothesis, we have $\bg_{t+1}^\top \bd_i=0$ for $i=\{0,1,\ldots,t-1\}$. If we further prove this is also true for $\bg_{t+1}^\top\bd_t$, we complete the proof. This follows again from Lemma~\ref{lemm:linear-search-orghonal} (p.~\pageref{lemm:linear-search-orghonal}), the current gradient is orthogonal to the previous search direction $\bd_t$.

For the second part, we define $f(\bm{\eta}) = L(\bx_1+\eta_1\bd_1+\eta_2\bd_2+\ldots +\eta_t\bd_t)$, which is a strictly convex quadratic function over $\bm{\eta}=[\eta_1, \eta_2, \ldots, \eta_t]^\top$ such that 
$$
\frac{\partial f(\bm{\eta})}{\partial \eta_i} = 0, \gap \forall i=1,2,\ldots, t.
$$
This implies 
$$
\underbrace{\nabla L(\bx_1 +\eta_1\bd_1+\eta_2\bd_2+\ldots +\eta_t\bd_t}_{\nabla L(\bx_{t+1})})^\top \bd_i = 0,  \gap \forall i=1,2,\ldots, t.
$$
That is, $\bx_{t+1}\in \{\bx | \bx=\bx_1 + \text{span}\{\bd_1, \bd_2, \ldots, \bd_t\}\}$ is the minimizer of $L(\bx)$.
\end{proof}

\index{Quadratic form}
\subsubsection{Quadratic Form in Conjugate Gradient (CG) Method}
We have mentioned that the conjugate gradient (CG) method differs from the conjugate descent (CD) method. The distinction lies  in the fact  that the CG method computes a new vector $\bd_{t+1}$ using only the previous vector $\bd_t$ rather than the entire sequence $\{\bd_1, \bd_2, \ldots, \bd_t\}$. 
And the resulting $\bd_{t+1}$ will automatically be conjugate to the sequence in this sense. In the CG method, each search direction $\bd_t$ is chosen to be a linear combination of negative gradient $-\bg_t$ (the search direction in steepest descent) and the previous direction $\bd_{t-1}$:
\begin{equation}\label{equation:cd_gradient_direction}
\bd_{t} =  -\bg_t + \beta_t \bd_{t-1}.
\end{equation}
This yields 
$$
\beta_t = \frac{\bg_t^\top\bA\bd_{t-1}}{\bd_{t-1}^\top \bA\bd_{t-1}}
.
$$
This choice of $\beta_t$ and $\bd_t$ actually results in the conjugate sequence $\{\bd_1, \bd_2, \ldots, \bd_t\}$. To see this, we first provide the definition of the \textit{Krylov subspace of degree $t$ for vector $\bv$ with respect to matrix $\bA$}:
$$
\mathcal{K}(\bv; t) = \text{span}\{\bv, \bA\bv, \ldots, \bA^{t-1}\bv\}.
$$

\begin{theorem}[Converge in $d$ Steps]\label{theorem:conjudate_CD_d-steps}
For any initial parameter $\bx_1$, the sequence $\{\bx_t\}$ generated by the conjugate descent algorithm, with search direction generated by Eq.~\eqref{equation:cd_gradient_direction}, converges to the solution $\bx_\star$ in at most $d$ steps {when $\bA$ is symmetric positive definite}. The result follows from the following claims:
\begin{align}
	\bg_t^\top \bg_i &= 0, \gap \text{for $i=\{1,2,\ldots, t-1\}$};\label{equation:conjudate_CD_d1}\\
	\text{span}\{\bg_1, \bg_2, \ldots, \bg_t\} &= \text{span}\{\bg_1, \bA\bg_1, \ldots, \bA^{t-1}\bg_1\}=\mathcal{K}(\bg_1; t);\label{equation:conjudate_CD_d2}\\
	\text{span}\{\bd_1, \bd_2, \ldots, \bd_t\} &= \text{span}\{\bg_1, \bA\bg_1, \ldots, \bA^{t-1}\bg_1\}=\mathcal{K}(\bg_1; t);\label{equation:conjudate_CD_d3}\\
	\bd_t^\top\bA\bd_i &= 0, \gap \text{for $i=\{1,2,\ldots, t-1\}$},\label{equation:conjudate_CD_d4}
\end{align}
where Eq.~\eqref{equation:conjudate_CD_d4} indicates the sequence $\{\bd_t\}$ is conjugate.
\end{theorem}
\begin{proof}[of Theorem~\ref{theorem:conjudate_CD_d-steps}]
The proof proceeds through induction.  Eq.~\eqref{equation:conjudate_CD_d2} and Eq.~\eqref{equation:conjudate_CD_d3} are trivial when $t=1$. 
Assuming that  for $t$, Eq.~\eqref{equation:conjudate_CD_d2}, Eq.~\eqref{equation:conjudate_CD_d3}, and Eq.~\eqref{equation:conjudate_CD_d4} hold as well; if we can show the  equations still hold for $t+1$, then we complete the proof. By the induction hypothesis, we have 
$$
\bg_t \in \text{span}\{\bg_1, \bA\bg_1, \ldots, \bA^{t-1}\bg_1\}, 
\gap 
\bd_t \in \text{span}\{\bg_1, \bA\bg_1, \ldots, \bA^{t-1}\bg_1\}.
$$
Left multiplying by $\bA$, it follows that
\begin{equation}\label{equation:dstesps_induction00}
	\bA\bd_t \in \text{span}\{\bA\bg_1, \bA^2\bg_1, \ldots, \bA^{t}\bg_1\}.
\end{equation}
Since 
\begin{equation}\label{equation:dstesps_induction11}
	\begin{aligned}
		\bg_{t+1} &= \bA\bx_{t+1}-\bb = \bA(\bx_t+\Delta \bx_t)-\bb\\
		&=\bA(\bx_t+\eta_t \bd_t) -\bb = \bg_t + \eta_t\bA\bd_t,
	\end{aligned}
\end{equation}
then we have 
\begin{equation}\label{equation:dstesps_induction2}
	\bg_{t+1}\in \text{span}\{\bg_1,\bA\bg_1, \bA^2\bg_1, \ldots, \bA^{t}\bg_1\}.
\end{equation}
Combining Eq.~\eqref{equation:dstesps_induction2} and Eq.~\eqref{equation:conjudate_CD_d2}, we have 
$$
\text{span}\{\bg_1, \bg_2, \ldots, \bg_{t}, \bg_{t+1}\} \subset \text{span}\{\bg_1, \bA\bg_1, \ldots, \bA^{t-1}\bg_1, \bA^t\bg_1\}.
$$
To see the reverse inclusion, by Eq.~\eqref{equation:conjudate_CD_d3}, it follows that 
$$
\bA^t \bg_1 = \bA(\bA^{t-1}\bg_1) \in \text{span}\{\bA\bd_1, \bA\bd_2, \ldots, \bA\bd_t\}.
$$
Again, by Eq.~\eqref{equation:dstesps_induction11}, we have $\bA\bd_t = (\bg_{t+1}-\bg_t)/\eta_t$. Therefore,
$$
\bA^t \bg_1 \in \text{span}\{\bg_1, \bg_2, \ldots, \bg_t, \bg_{t+1}\} .
$$
Combining with Eq.~\eqref{equation:conjudate_CD_d2}, we have 
$$
\text{span}\{\bg_1, \bA\bg_1, \ldots, \bA^{t-1}\bg_1, \bA^{t}\bg_1\}
\subset 
\text{span}\{\bg_1, \bg_2, \ldots, \bg_t, \bg_{t+1}\}.
$$
Therefore, Eq.~\eqref{equation:conjudate_CD_d2} holds for $t+1$. Eq.~\eqref{equation:conjudate_CD_d3} follows similarly and also holds for $t+1$.

To see how Eq.~\eqref{equation:conjudate_CD_d4} holds for $t+1$, we have 
$$
\bd_{t+1}^\top\bA\bd_i = (-\bg_{t+1} + \beta_{t+1}\bd_{t})^\top\bA\bd_i.
$$
By Theorem~\ref{theorem:expanding_subspace_minimization}, we have 
\begin{equation}\label{equation:dstesps_induction_argue41}
	\bg_{t+1}^\top\bd_i=0 \text{ for } i\in \{1,2,\ldots, t\}.
\end{equation}
Furthermore, by Eq.~\eqref{equation:dstesps_induction00} and Eq.~\eqref{equation:conjudate_CD_d3}, we have 
\begin{equation}\label{equation:dstesps_induction_argue42}
	\bA\bd_i \in \text{span}\{\bA\bg_1, \bA^2\bg_1, \ldots, \bA^{i}\bg_1\}\subset 
	\text{span}\{\bd_1, \bd_2, \ldots, \bd_i, \bd_{i+1}\}.
\end{equation}
Combining Eq.~\eqref{equation:dstesps_induction_argue41} and Eq.~\eqref{equation:dstesps_induction_argue42}, it then follows that 
$$
\bd_{t+1}^\top\bA\bd_i=0,\gap \text{ for }i\in \{1,2,\ldots,t-1\}.
$$
We need to further demonstrate $\bd_{t+1}^\top\bA\bd_t=0$, a result that is evident given the intentional design of the algorithm to satisfy this condition.

To establish the validity of  Eq.~\eqref{equation:conjudate_CD_d1}, we have $\bd_i = -\bg_i+\beta_i\bd_{i-1}$. Therefore, $\bg_i \in \text{span}\{\bd_i,\bd_{i-1}\}$. Furthermore, employing Eq.~\eqref{equation:dstesps_induction_argue41}, we prove $\bg_t^\top\bg_i=0$ for $i\in\{1,2,\ldots, t-1\}$.
\end{proof}

Therefore, the CG method developed by Eq.~\eqref{equation:cd_gradient_direction} that creates conjugate directions $\bd_{t+1}\bA\bd_t=0$ indeed finds a conjugate set $\bd_{t+1}\bA\bd_i=0$ for $i\in\{1,2,\ldots,t\}$. By Theorem~\ref{theorem:conjudate_direc_d-steps}, the CG method thus converges in at most $d$ steps (when $\bA$ is symmetric PD). The complete procedure is then formulated in Algorithm~\ref{alg:vanilla_conjugate_descent}.

\begin{algorithm}[H] 
\caption{Vanilla Conjugate Gradient Method for Quadratic Function}
\label{alg:vanilla_conjugate_descent}
\begin{algorithmic}[1]
	\State {\bfseries Require:} Symmetric positive definite $\bA\in \real^{d\times d}$;
	\State {\bfseries Input:} Initial parameter $\bx_1$;
	\State {\bfseries Input:} Initialize $\bd_0 =\bzero $ and $\bg_0 = \bd_0+\epsilon$;
	\For{$t=1:d$ } 
	\State Compute gradient $\bg_t = \nabla L(\bx_{t})$;
	\State Compute coefficient $\beta_{t} = \frac{ \bg_t^\top\bA \bd_{t-1}}{\bg_{t-1}^\top \bA\bg_{t-1}}$;
	\State Compute descent direction $\bd_t = -\bg_{t} +\beta_{t}  \bd_{t-1}$;
	\State Learning rate $\eta_t = - \frac{\bd_t^\top \bg_t}{ \bd_t^\top \bA\bd_t }$;
	\State Compute update step $\Delta \bx_t = \eta_t\bd_t$;
	\State Apply update $\bx_{t+1} = \bx_{t} + \Delta \bx_t$;
	\EndFor
	\State {\bfseries Return:} resulting parameters $\bx_t$, and the loss $L(\bx_t)$.
\end{algorithmic}
\end{algorithm}

\index{Complexity}
\index{Flops}

To further reduce the complexity of the CG algorithm, we  introduce the notion of floating operation (flop) counts. We follow the classical route and count the number of floating-point operations (flops) that the algorithm requires. Each addition, subtraction, multiplication, division, and square root is considered one flop. Note that we have the convention that an assignment operation is not counted as one flop.
The calculation of the complexity extensively relies on the complexity of the multiplication of two matrices so that we formulate the finding in the following lemma.
\begin{lemma}[Vector Inner Product Complexity]
Given two vectors $\bv,\bw\in \real^{n}$. The  inner product of the two vectors $\bv^\top\bw$ is given by $\bv^\top\bw=v_1w_1+v_2w_2+\ldots v_nw_n$, involving $n$ scalar multiplications and $n-1$ scalar additions. Therefore, the complexity for the inner product is $2n-1$ flops.
\end{lemma}

The matrix multiplication thus relies on the complexity of the inner product.
\begin{lemma}[Matrix Multiplication Complexity]\label{lemma:matrix-multi-complexity}
For matrix $\bA\in\real^{m\times n}$ and $\bB\in \real^{n\times k}$, the complexity of the multiplication $\bC=\bA\bB$ is $mk(2n-1)$ flops.
\end{lemma}
\begin{proof}[of Lemma~\ref{lemma:matrix-multi-complexity}]
We notice that each entry of $\bC$ involves a vector inner product that requires $n$ multiplications and $n-1$ additions. And there are $mk$ such entries, which leads to the conclusion.
\end{proof}

By Theorem~\ref{theorem:conjudate_CD_d-steps}, we can replace the formula for calculating learning rate into:
$$
\eta_t = - \frac{\bd_t^\top \bg_t}{ \bd_t^\top \bA\bd_t }
\leadto 
\eta_t = - \frac{\textcolor{mylightbluetext}{\bg_t}^\top \bg_t}{ \bd_t^\top \bA\bd_t }.
$$
According to Eq.~\eqref{equation:conjguate-redsidual-update}, it follows that $\eta_t\bA\bd_t = \bg_{t+1}-\bg_t$. Combining with Eq.~\eqref{equation:expanding_subspace_minimization_zero} and Eq.~\eqref{equation:conjudate_CD_d1}, $\beta_t$ can also be expressed as 
$$
\beta_t 
=
-\frac{\bg_t^\top\bg_t}{\bd_{t-1}^\top \bg_{t-1}}
=\frac{\bg_t^\top\bg_t}{\bg_{t-1}^\top \bg_{t-1}}
.
$$
This reduces the computational complexity from $\mathcal{O}(4d^2)$ to $\mathcal{O}(4d)$ flops. This practical CG method is then outlined in Algorithm~\ref{alg:practical_conjugate_descent}.

\begin{algorithm}[H] 
\caption{Practical Conjugate Gradient Method for Quadratic Function}
\label{alg:practical_conjugate_descent}
\begin{algorithmic}[1]
	\State {\bfseries Require:} Symmetric positive definite $\bA\in \real^{d\times d}$;
	\State {\bfseries Input:} Initial parameter $\bx_1$;
	\State {\bfseries Input:} Initialize $\bd_0 =\bzero $ and $\bg_0 = \bd_0+\epsilon$;
	\For{$t=1:d$ } 
	\State Compute gradient $\bg_t = \nabla L(\bx_{t})$;
	\State Compute coefficient $\beta_{t} =  \frac{\bg_t^\top\bg_t}{\bg_{t-1}^\top \bg_{t-1}}$; \Comment{set $\beta_1=0$ by convention}
	\State Compute descent direction $\bd_t = -\bg_{t} +\beta_{t}  \bd_{t-1}$;
	\State Learning rate $\eta_t =- \frac{{\bg_t}^\top \bg_t}{ \bd_t^\top \bA\bd_t }$;
	\State Compute update step $\Delta \bx_t = \eta_t\bd_t$;
	\State Apply update $\bx_{t+1} = \bx_{t} + \Delta \bx_t$;
	\EndFor
	\State {\bfseries Return:} resulting parameters $\bx_t$, and the loss $L(\bx_t)$.
\end{algorithmic}
\end{algorithm}

\index{Spectral decomposition}
\index{Quadratic form}
\index{Symmetry}
\index{Positive definite}
\subsubsection{Convergence Analysis for Symmetric Positive Definite Quadratic}
We further discuss the convergence results of the CG method. 
According to Eq.~\eqref{equation:conjudate_CD_d3}, there exists a set of $\{\sigma_1,\sigma_2,\ldots,\sigma_t\}$ coefficients such that 
\begin{equation}\label{equation:cg-convergence-xt1}
\begin{aligned}
	\bx_{t+1} &=\bx_1 + \eta_1\bd_1+\eta_2\bd_2+\ldots+\eta_t\bd_t  \\
	&= \bx_1 + \sigma_1\bg_1+\sigma_2\bA \bg_1+\ldots+\sigma_t\bA^{t-1}\bg_1\\
	&= \bx_1 + P^{\textcolor{mylightbluetext}{\star}}_{t-1}(\bA)\bg_1,
\end{aligned}
\end{equation}
where $P^{\textcolor{mylightbluetext}{\star}}_{t-1}(\bA) = \sigma_1\bI+\sigma_2\bA +\ldots+\sigma_t\bA^{t-1}$ is a polynomial of degree $t-1$ with coefficients $\{\sigma_1, \sigma_2, \ldots, \sigma_t\}$. 
This polynomial is a special case of a polynomial of degree $t-1$ with random coefficients $\{\omega_1, \omega_2, \ldots, \omega_t\}$, denoted by $P_{t-1}(\bA) = \omega_1\bI+\omega_2\bA +\ldots+\omega_t\bA^{t-1}$. (Note that $P_{t-1}$ can take either a scalar or a matrix as its argument).
Suppose the symmetric positive definite $\bA$ admits the spectral decomposition (Theorem 13.1 in \citet{lu2022matrix} or Appendix~\ref{appendix:spectraldecomp}, p.~\pageref{appendix:spectraldecomp}):
$$
\bA=\bQ\bLambda\bQ^\top  \in \real^{d\times d} \leadto \bA^{-1} = \bQ\bLambda^{-1}\bQ^\top,
$$ 
where the columns of $\bQ = [\bq_1, \bq_2, \ldots , \bq_d]$ are eigenvectors of $\bA$ and are mutually orthonormal, and the entries of $\bLambda = \diag(\lambda_1, \lambda_2, \ldots , \lambda_d)$ with $ \lambda_1\geq \lambda_2\geq \ldots\geq \lambda_d>0$ are the corresponding eigenvalues of $\bA$, which are real and ordered by magnitude (the eigenvalues are positive due to the positive definiteness assumption of $\bA$). It then follows that any eigenvector of $\bA$ is also an eigenvector of $P_{t-1}(\bA)$:
$$
P_{t-1}(\bA) \bq_i = P_{t-1}(\lambda_i) \bq_i, \gap \forall i\in \{1,2,\ldots, d\}.
$$
Moreover, since the eigenvectors span the entire space $\real^d$, there exists a set of $\{\nu_1,\nu_2,\ldots, \nu_d\}$ coefficients such that the initial error vector $\be_1$ can be expressed as 
\begin{equation}\label{equation:cg-convergence-xt2}
\be_1=\bx_1 - \bx_\star = \sum_{i=1}^{d} \nu_i \bq_i,
\end{equation}
where $\bx_1$ is the initial parameter. Combining Eq.~\eqref{equation:cg-convergence-xt1} and Eq.~\eqref{equation:cg-convergence-xt2}, this yields the update of the error vector:
\begin{equation}\label{equation:cg-convergence-xt3}
\begin{aligned}
	\be_{t+1}&=\bx_{t+1} - \bx_\star \\\
	&=\bx_1 +  P^{\textcolor{mylightbluetext}{\star}}_{t-1}(\bA)\bg_1-\bx_\star\\
	&=\bx_1 +  P^{\textcolor{mylightbluetext}{\star}}_{t-1}(\bA)(\bA\bx_1 - \bA\underbrace{\bA^{-1} \bb}_{\bx_\star})-\bx_\star\\
	&=\bx_1 +  P^{\textcolor{mylightbluetext}{\star}}_{t-1}(\bA)\bA(\bx_1 - \bx_\star)-\bx_\star\\
	&=\bigg(\bI+P^{\textcolor{mylightbluetext}{\star}}_{t-1}(\bA)\bA\bigg) (\bx_1 - \bx_\star)\\
	&=\bigg(\bI+P^{\textcolor{mylightbluetext}{\star}}_{t-1}(\bA)\bA\bigg) \sum_{i=1}^{d} \nu_i \bq_i= \sum_{i=1}^{d}\bigg(1+ \lambda_i P^{\textcolor{mylightbluetext}{\star}}_{t-1}(\bA)\bigg) \nu_i\bq_i
\end{aligned}
\end{equation}
To further discuss the convergence results, we  need to use the notion of \textit{energy norm} for error vector $\norm{\be}_{\bA} = (\be^\top\bA\be)^{1/2}$ as defined in Section~\ref{section:general-converg-steepest} (p.~\pageref{section:general-converg-steepest}). 
It can be shown that minimizing $\norm{\be_t}_{\bA}$ is equivalent to minimizing $L(\bx_t)$ by Eq.~\eqref{equation:energy-norm-equivalent} (p.~\pageref{equation:energy-norm-equivalent}).

\begin{remark}[Polynomial Minimization]
Since we proved in Theorem~\ref{theorem:expanding_subspace_minimization} that $\bx_{t+1}$ minimizes $L(\bx)$ over the subspace $\mathbb{D}_t$ defined in Eq.~\eqref{equation:space_d_t}, it also minimizes the energy norm $\norm{\be}_{\bA}$ over the subspace $\mathbb{D}_t$ at iteration $t$. It then follows that $P^{\textcolor{mylightbluetext}{\star}}_{t-1}(\bA)$ minimizes over the space of all possible polynomials of degree $t-1$:
$$
P^{\textcolor{mylightbluetext}{\star}}_{t-1}(\bA)
=\mathop{\arg\min}_{P_{t-1}(\bA)}  \norm{\bx_1 +  P_{t-1}(\bA)\bg_1-\bx_\star}_{\bA}.
$$
\end{remark}
Then the update of the squared energy norm can be obtained by
$$
\begin{aligned}
\norm{\be_{t+1}}_{\bA}^2 &= \be_{t+1}^\top \bA\be_{t+1} =\be_{t+1}^\top \left(\sum_{i=1}^{d}\lambda_i \bq_i\bq_i^\top\right) \be_{t+1} \\
& = \sum_{i=1}^{d} \lambda_i (\be_{t+1}^\top\bq_i)^2 \\
&=\sum_{i=1}^{d} \lambda_i \left(\bq_i^\top \bigg(\sum_{j=1}^{d}\bigg(1+ \lambda_j P^{\textcolor{mylightbluetext}{\star}}_{t-1}(\bA)\bigg) \nu_j\bq_j\bigg)\right)^2 &\text{(by Eq.~\eqref{equation:cg-convergence-xt3})}\\
&=\sum_{i=1}^{d}  \bigg(1+ \lambda_i P^{\textcolor{mylightbluetext}{\star}}_{t-1}(\lambda_i)\bigg)^2  \lambda_i\nu_i^2 
&\text{($\bq_i^\top\bq_j=0$ if $i\neq j$)} \\
&=\mathop{\min}_{P_{t-1}} \sum_{i=1}^{d}  \bigg(1+ \lambda_i  P_{t-1}(\lambda_i)\bigg)^2  \lambda_i\nu_i^2  \\
&\leq m_t \sum_{i=1}^{d} \lambda_i\nu_i^2    &\text{($m_t = \mathop{\min}_{P_{t-1}}\mathop{\max}_{1\leq j\leq d} (1+ \lambda_j P_{t-1}(\lambda_j))^2 $)} \\
&\leq m_t \cdot \norm{\be_{1}}_{\bA}^2.
\end{aligned}
$$
Therefore, the rate of convergence for the CG method is controlled by 
\begin{equation}\label{equation:cg-convergence-xt5}
m_t = \mathop{\min}_{P_{t-1}}\mathop{\max}_{1\leq j\leq d}(1+ \lambda_j P_{t-1}(\lambda_j))^2.
\end{equation}

\subsection*{Special Case: $\bA$ Has Only $r$ Distinct Eigenvalues}
We then consider some special cases. Firstly, we want to show the CG method terminates in exactly $r$ iterations if the symmetric positive definite $\bA$ has only $r$ distinct eigenvalues. To establish this, suppose $\bA$ has distinct eigenvalues $\mu_1<\mu_2<\ldots<\mu_r$. And we define a polynomial $Q_r(\lambda)$ by 
$$
Q_r(\lambda) =\frac{(-1)^r}{\mu_1\mu_2\ldots\mu_r} (\lambda-\mu_1)(\lambda-\mu_2)\ldots (\lambda-\mu_r),
$$
where $Q_r(\lambda_i)=0$ for $i=\{1,2,\ldots, d\}$ and $Q_r(0)=1$. Therefore, it follows that the polynomial 
$$
R_{r-1}(\lambda) = \frac{Q_r(\lambda)-1}{\lambda}
$$
is a polynomial of degree $r-1$ with a root at $\lambda=0$. Setting $t-1=r-1$ in Eq.~\eqref{equation:cg-convergence-xt5}, we have
$$
\begin{aligned}
0&\leq m_{r}=\mathop{\min}_{P_{r-1}}\mathop{\max}_{1\leq j\leq d}(1+ \lambda_j P_{r-1}(\lambda_j))^2\\
&=\mathop{\max}_{1\leq j\leq d}(1+ \lambda_j R_{r-1}(\lambda_j))^2 
= \mathop{\max}_{1\leq j\leq d} Q_r^2(\lambda_i) = 0.
\end{aligned}
$$
As a result, $m_{r}$=0, and $\norm{\be_{r+1}}_{\bA}=0$, implying $\bx_{r+1} = \bx_\star$ and the algorithm terminates at iteration $r$. A specific example is shown in Figure~\ref{fig:conjugate_specialcases}, where Figure~\ref{fig:conjugate_specialcases_2eigenvalue} terminates in two steps since it has two distinct eigenvalues, and Figure~\ref{fig:conjugate_specialcases_1eigenvalue} terminates in just one step as it has one distinct eigenvalue.

\begin{figure}[h]
\centering  
\vspace{-0.35cm} 
\subfigtopskip=2pt 
\subfigbottomskip=2pt 
\subfigcapskip=-5pt 
\subfigure[CG, 2 distinct eigenvlaues. Finish in 2 steps.]{\label{fig:conjugate_specialcases_2eigenvalue}
	\includegraphics[width=0.481\linewidth]{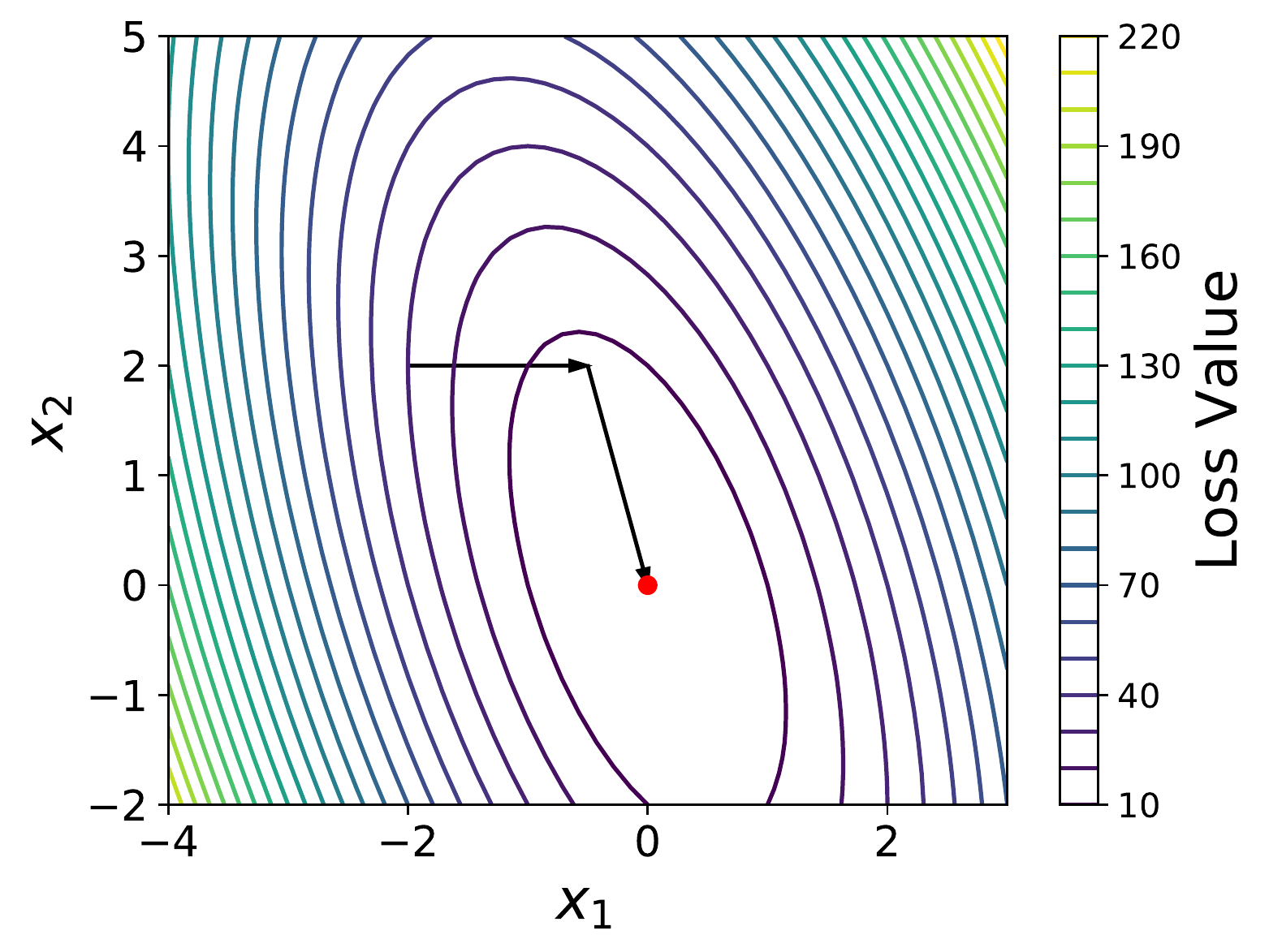}}
\subfigure[CG, 1 distinct eigenvalue. Finish in 1 step.]{\label{fig:conjugate_specialcases_1eigenvalue}
	\includegraphics[width=0.481\linewidth]{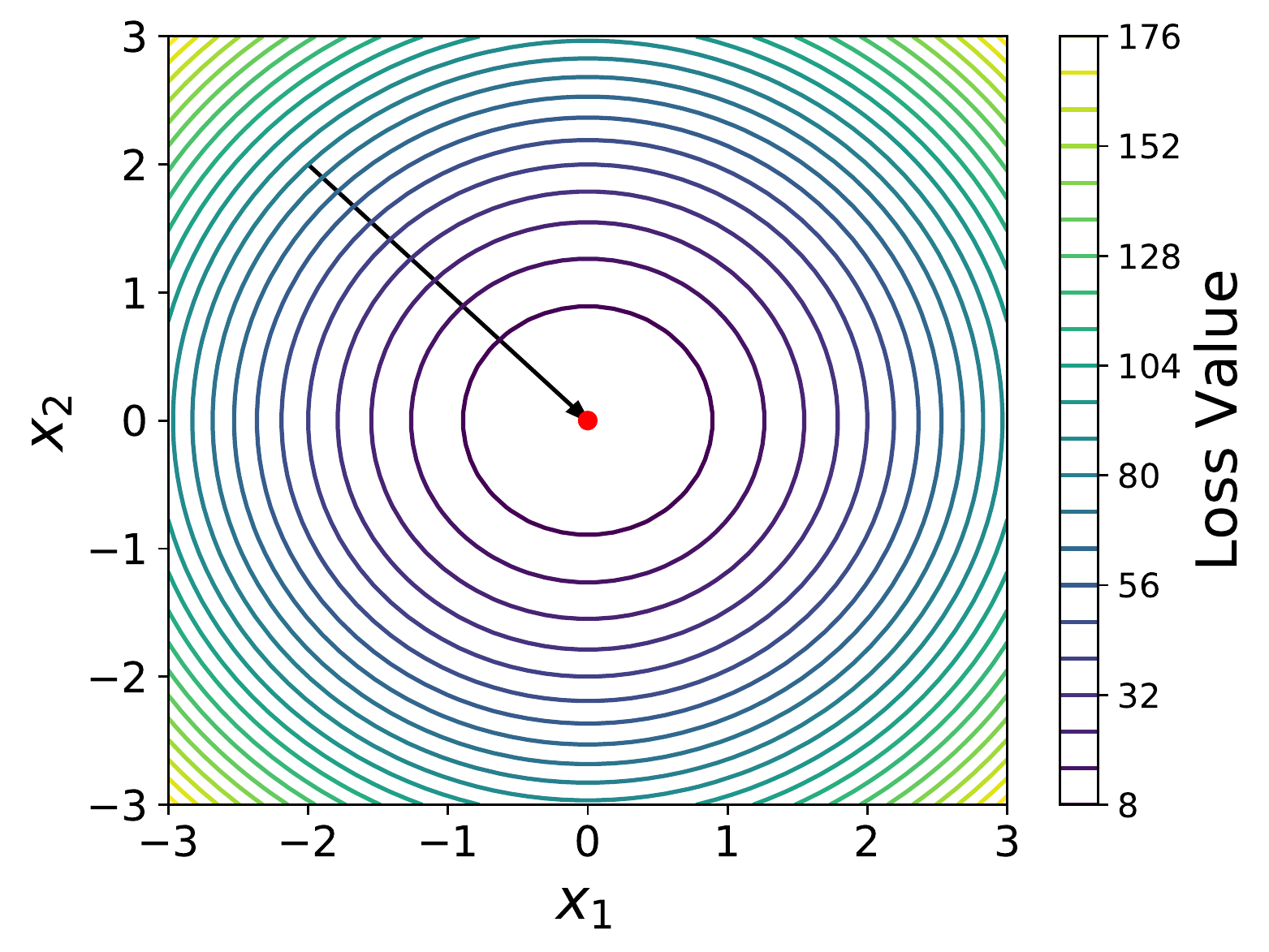}}
\caption{Illustration of special cases for CG with exact line search of quadratic forms. $\bA=\begin{bmatrix}
		20 & 5 \\ 5 & 5
	\end{bmatrix}$, $\bb=\bzero$, $c=0$, and starting point to descent is $\bx_1=[-2, 2]^\top$ for Fig~\ref{fig:conjugate_specialcases_2eigenvalue}. $\bA=\begin{bmatrix}
		20 & 0 \\ 0 & 20
	\end{bmatrix}$, $\bb=\bzero$, $c=0$, and starting point to descent is $\bx_1=[-2, 2]^\top$ for Fig~\ref{fig:conjugate_specialcases_1eigenvalue}.}
\label{fig:conjugate_specialcases}
\end{figure}

\subsection*{Closed Form by Chebyshev Polynomials}
It can be shown that Eq.~\eqref{equation:cg-convergence-xt5} is minimized by a Chebyshev polynomial, given by
$$
1+ \lambda_j P_{t-1}(\lambda_j) = \frac{T_{t}\left( \frac{\lambda_{\max} + \lambda_{\min} - 2\lambda}{\lambda_{\max}-\lambda_{\min}} \right) }
{T_{t}\left( \frac{\lambda_{\max} + \lambda_{\min} }{\lambda_{\max}-\lambda_{\min}} \right) },
$$
where $T_t(w) = \frac{1}{2} \left[ (w+\sqrt{w^2+1})^t + (w-\sqrt{w^2-1})^t\right]$ represents the Chebyshev polynomial of degree $t$.
\begin{proof}
To see this, we can express the $m_t$ in Eq.~\eqref{equation:cg-convergence-xt5} as
\begin{equation}\label{equation:cg-convergence-xt5_rewrite}
	m_t = \mathop{\min}_{P_{t-1}}\mathop{\max}_{1\leq j\leq d}(1+ \lambda_j P_{t-1}(\lambda_j))^2 = \mathop{\min}_{P_{t-1}}\mathop{\max}_{1\leq j\leq d} (\widetilde{P}_{t}(\lambda_i))^2,
\end{equation}
where $\widetilde{P}_{t}(\lambda) = 1+ \lambda P_{t-1}(\lambda)=1+w_1\lambda + \ldots+w_t\lambda^t$ is a special polynomial of degree  $t$ with $\widetilde{P}_{t}(0)=1$. We note that the Chebyshev polynomial can be expressed on the interval $w\in [-1,1]$ as
$$
T_t(w) =\cos(t \cos^{-1} w), \gap w\in [-1,1] \leadto |T_t(w)| \leq 1,\gap \text{if } w\in [-1,1].
$$
It is observable that $\widetilde{P}_{t}(\lambda)$ oscillates within the range $\pm {T_{t}\left( \frac{\lambda_{\max} + \lambda_{\min} }{\lambda_{\max}-\lambda_{\min}} \right) }^{-1}$ over the domain $[\lambda_{\min}, \lambda_{\max}]$. Suppose there exists a polynomial $S_t(\lambda)$ of degree $t$ such that $S_t(0)=1$ and $S_t$ is better than $\widetilde{P}_t$ on the domain $[\lambda_{\min}, \lambda_{\max}]$. It then follows that the $S_t-\widetilde{P}_t$ has a zero at $\lambda=0$ and other $t$ zeros on the range $[\lambda_{\min}, \lambda_{\max}]$, making it has $t+1$ zeros, which leads to a contradiction. Therefore, $\widetilde{P}_t$ is the optimal polynomial of degree $t$.
This completes the proof.
\end{proof}

Therefore, it follows that
$$
\begin{aligned}
\norm{\be_{t+1}}_{\bA} &\leq T_t\left(  \frac{\lambda_{\max} + \lambda_{\min}}{\lambda_{\max} - \lambda_{\min}}   \right)^{-1} \cdot\norm{\be_1}_{\bA} \\
&=T_t\left(  \frac{\kappa+1}{\kappa-1}   \right)^{-1} \cdot \norm{\be_1}_{\bA}\\
&= 2\left[ \left(\frac{\sqrt{\kappa}+1}{\sqrt{\kappa}-1}  \right)^t +
\left(\frac{\sqrt{\kappa}-1}{\sqrt{\kappa}+1}  \right)^t
\right]^{-1} \cdot\norm{\be_1}_{\bA},
\end{aligned}
$$
where $\kappa = \frac{\lambda_{\max}}{\lambda_{\min}}$ is the condition number, and $\left(\frac{\sqrt{\kappa}-1}{\sqrt{\kappa}+1}  \right)^t \rightarrow 0$ as iteration $t$ grows. A weaker inequality can be obtained by 
$$
\norm{\be_{t+1}}_{\bA} \leq  2 \left(\frac{\sqrt{\kappa}-1}{\sqrt{\kappa}+1}  \right)^t
\cdot\norm{\be_1}_{\bA}.
$$
Figure~\ref{fig:rate_convergen_conjugae_comparison}  compares the rate of convergence of steepest descent and CG per iteration. It is observed that  CG exhibits significantly faster convergence compared to steepest descent.

\begin{figure}[h]
\centering  
\vspace{-0.35cm} 
\subfigtopskip=2pt 
\subfigbottomskip=2pt 
\subfigcapskip=-5pt 
\subfigure[Rate of convergence for steepest descent per iteration (same as Figure~\ref{fig:rate_convergen_steepest}, p.~\pageref{fig:rate_convergen_steepest}). The $y$-axis is $\frac{\kappa-1}{\kappa+1}$.]{\label{fig:rate_convergen_steepest1}
	\includegraphics[width=0.481\linewidth]{./imgs/rate_convergen_steepest.pdf}}
\subfigure[Rate of convergence for CG per iteration. The $y$-axis is $\frac{\sqrt{\kappa}-1}{\sqrt{\kappa}+1}$.]{\label{fig:rate_convergen_conjugate}
	\includegraphics[width=0.481\linewidth]{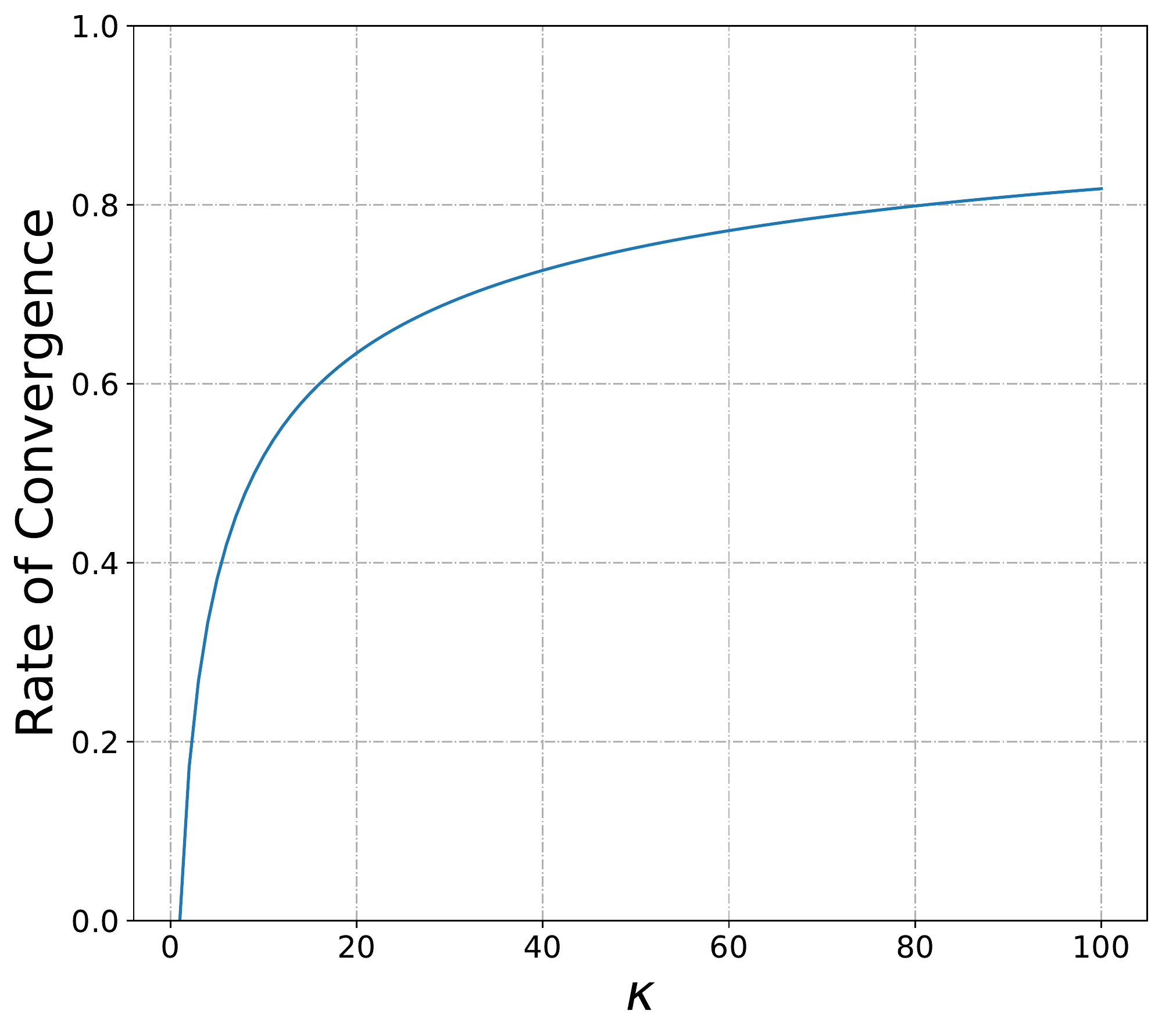}}
\caption{Illustration of the rate of convergence for CG and steepest descent.}
\label{fig:rate_convergen_conjugae_comparison}
\end{figure}

\index{Rate of convergence}
\subsection*{Preconditioning}
Since the smaller the condition number $\kappa$, the faster the convergence (Figure~\ref{fig:rate_convergen_conjugate}). We can accelerate the convergence of CG by transforming the linear system to improve the eigenvalue distribution of $\bA${\textemdash}the procedure is known as \textit{preconditioning}. The variable $\bx$ is transformed to $\widehat{\bx}$ via a nonsingular matrix $\bP$, satisfying
$$
\begin{aligned}
\whbx &= \bP\bx;\\
\whL(\whbx) &=\frac{1}{2}\whbx^\top (\bP^{-\top} \bA\bP^{-1})\whbx - (\bP^{-\top}\bb)^\top \whbx +c.
\end{aligned}
$$
When $\bA$ is symmetric, the solution of $\whL(\whbx)$ is equivalent to the solution of the linear equation
$$
\begin{aligned}
(\bP^{-\top} \bA\bP^{-1})\whbx &= \bP^{-\top}\bb \\
&\leadto \bP^{-\top}\bA\bx=\bP^{-\top}\bb \\
&\leadto \bA\bx=\bb.
\end{aligned}
$$
That is, we can solve $\bA\bx=\bb$ indirectly by solving $\bP^{-\top}\bA\bx=\bP^{-\top}\bb$.
Therefore, the rate of convergence of the quadratic form $\whL(\whbx)$ depends on the condition number of $\bP^{-\top} \bA\bP^{-1}$, which can be controlled by the nonsingular matrix $\bP$. 
Intuitively,  preconditioning is a procedure to stretch the quadratic form to make it more spherical so that the eigenvalues are clustered in a smaller range. A specific example is given in Figure~\ref{fig:conjugate_specialcases} that we want to transform the elliptical contour in Figure~\ref{fig:conjugate_specialcases_2eigenvalue} into the spherical contour in Figure~\ref{fig:conjugate_specialcases_1eigenvalue}.
Based on Algorithm~\ref{alg:practical_conjugate_descent}, the preconditioned CG method is formulated in Algorithm~\ref{alg:predicition_CG}.

\begin{algorithm}[h] 
\caption{Transformed-Preconditioned CG for Quadratic Functions}
\label{alg:predicition_CG}
\begin{algorithmic}[1]
	\State {\bfseries Require:} Symmetric positive definite $\bA\in \real^{d\times d}$;
	\State {\bfseries Input:} Initial parameter $\whbx_1$;
	\State {\bfseries Input:} Initialize $\whbd_0 =\bzero $ and $\whbg_0 = \whbd_0+\epsilon$;
	\For{$t=1:d$ } 
	\State Compute gradient $\whbg_t = \nabla \whL(\whbx_{t}) = (\bP^{-\top} \bA\bP^{-1})\whbx- \bP^{-\top}\bb $; \Comment{$=\textcolor{mylightbluetext}{\bP^{-\top}}\bg_t$}
	\State Compute coefficient $\widehat{\beta}_{t} =  \frac{\whbg_t^\top\whbg_t}{\whbg_{t-1}^\top \whbg_{t-1}}
	$; \Comment{$=
		\frac{\bg_t^\top \textcolor{mylightbluetext}{(\bP^\top\bP)^{-1}}\bg_t}
		{\bg_{t-1}^\top\textcolor{mylightbluetext}{(\bP^\top\bP)^{-1}} \bg_{t-1}}$}
	\State Compute descent direction $\whbd_t = -\whbg_{t} +\widehat{\beta}_{t}  \whbd_{t-1}$;
	\Comment{$=-\textcolor{mylightbluetext}{\bP^{-\top}}\bg_{t} +\widehat{\beta}_{t}  \whbd_{t-1}$}
	\State Learning rate 
	$\widehat{\eta}_t =- \frac{{\whbg_t}^\top \whbg_t}{ \whbd_t^\top (\bP^{-\top} \bA\bP^{-1})\whbd_t }
	$;
	\Comment{$=
		- \frac{{\bg_t}^\top \textcolor{mylightbluetext}{(\bP^\top\bP)^{-1}}\bg_t}{ \whbd_t^\top (\bP^{-\top} \bA\bP^{-1})\whbd_t }$}
	\State Compute update step $\Delta \whbx_t = \widehat{\eta}_t\whbd_t$;
	\State Apply update $\whbx_{t+1} = \whbx_{t} + \Delta \whbx_t$;
	\EndFor
	\State {\bfseries Return:} resulting parameters $\textcolor{mylightbluetext}{\bx_t=\bP^{-1}\whbx_t}$, and the loss $L(\bx_t)$.
\end{algorithmic}
\end{algorithm}

However, the procedure in Algorithm~\ref{alg:predicition_CG} is not desirable since we need to transform $\bx$ into $\whbx=\bP\bx$ and untransformed back by $\bx=\bP^{-1}\whbx$ as highlighted in the blue texts of Algorithm~\ref{alg:predicition_CG}. This introduces additional computational overhead. Let $\bM=\bP^\top\bP$, Algorithm~\ref{alg:untransformed_predicition_CG} is proposed to formulate the untransformed-preconditioned CG, which proves to be more efficient  than Algorithm~\ref{alg:predicition_CG}.

\begin{algorithm}[h] 
\caption{Untransformed-Preconditioned CG for Quadratic Functions}
\label{alg:untransformed_predicition_CG}
\begin{algorithmic}[1]
	\State {\bfseries Require:} Symmetric positive definite $\bA\in \real^{d\times d}$;
	\State {\bfseries Input:} Initial parameter $\bx_1$;
	\State {\bfseries Input:} Initialize $\bd_0 =\bzero $ and $\bg_0 = \bd_0+\epsilon$;
	\For{$t=1:d$ } 
	\State Compute gradient $\bg_t = \nabla L(\bx_{t})$;
	\Comment{Same as that of Algorithm~\ref{alg:practical_conjugate_descent}}
	\State Compute coefficient $\widehat{\beta}_{t} =  \frac{\bg_t^\top\textcolor{mylightbluetext}{\bM^{-1}}\bg_t}{\bg_{t-1}^\top\textcolor{mylightbluetext}{\bM^{-1}} \bg_{t-1}}$; \Comment{Same as that of Algorithm~\ref{alg:predicition_CG}}
	\State Compute descent direction $\widetilde{\bd}_t = -\textcolor{mylightbluetext}{\bM^{-1}}\bg_{t} +\widehat{\beta}_{t}  \widetilde{\bd}_{t-1}$;
	\Comment{$=-\textcolor{mylightbluetext}{\bP^{-1}} \whbd_{t}$ in Algorithm~\ref{alg:predicition_CG}}
	\State Learning rate $\widehat{\eta}_t =- {({\bg_t}^\top\textcolor{mylightbluetext}{\bM^{-1}} \bg_t)}/{ (\widetilde{\bd}_t^\top \bA\widetilde{\bd}_t )}$; \Comment{Same as that of Algorithm~\ref{alg:predicition_CG}}
	\State Compute update step ${\Delta \bx}_t = \widehat{\eta}_t\widetilde{\bd}_t$;
	\Comment{$=-\textcolor{mylightbluetext}{\bP^{-1}} \Delta\whbx_{t}$ in Algorithm~\ref{alg:predicition_CG}}
	\State Apply update $\bx_{t+1} = \bx_{t} + {\Delta \bx}_t$;
	\EndFor
	\State {\bfseries Return:} resulting parameters $\bx_t$, and the loss $L(\bx_t)$.
\end{algorithmic}
\end{algorithm}

\begin{figure}[h]
\centering  
\vspace{-0.35cm} 
\subfigtopskip=2pt 
\subfigbottomskip=2pt 
\subfigcapskip=-5pt 
\subfigure[Contour plot of quadratic function with $\bA$.]{\label{fig:conjugate_precondition1}
	\includegraphics[width=0.481\linewidth]{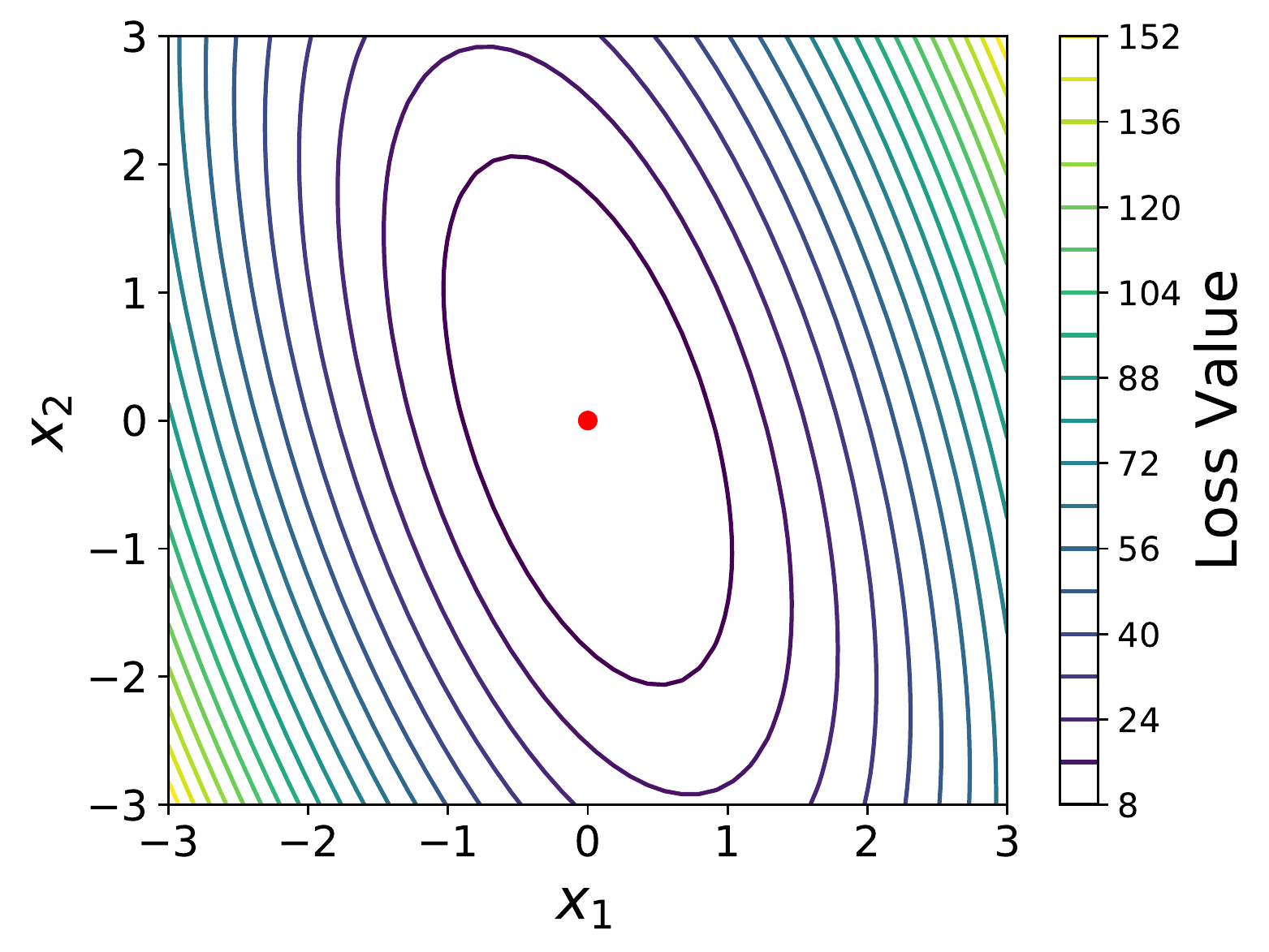}}
\subfigure[Contour plot of quadratic function with $\bP^{-\top} \bA\bP^{-1}$.]{\label{fig:conjugate_precondition2}
	\includegraphics[width=0.481\linewidth]{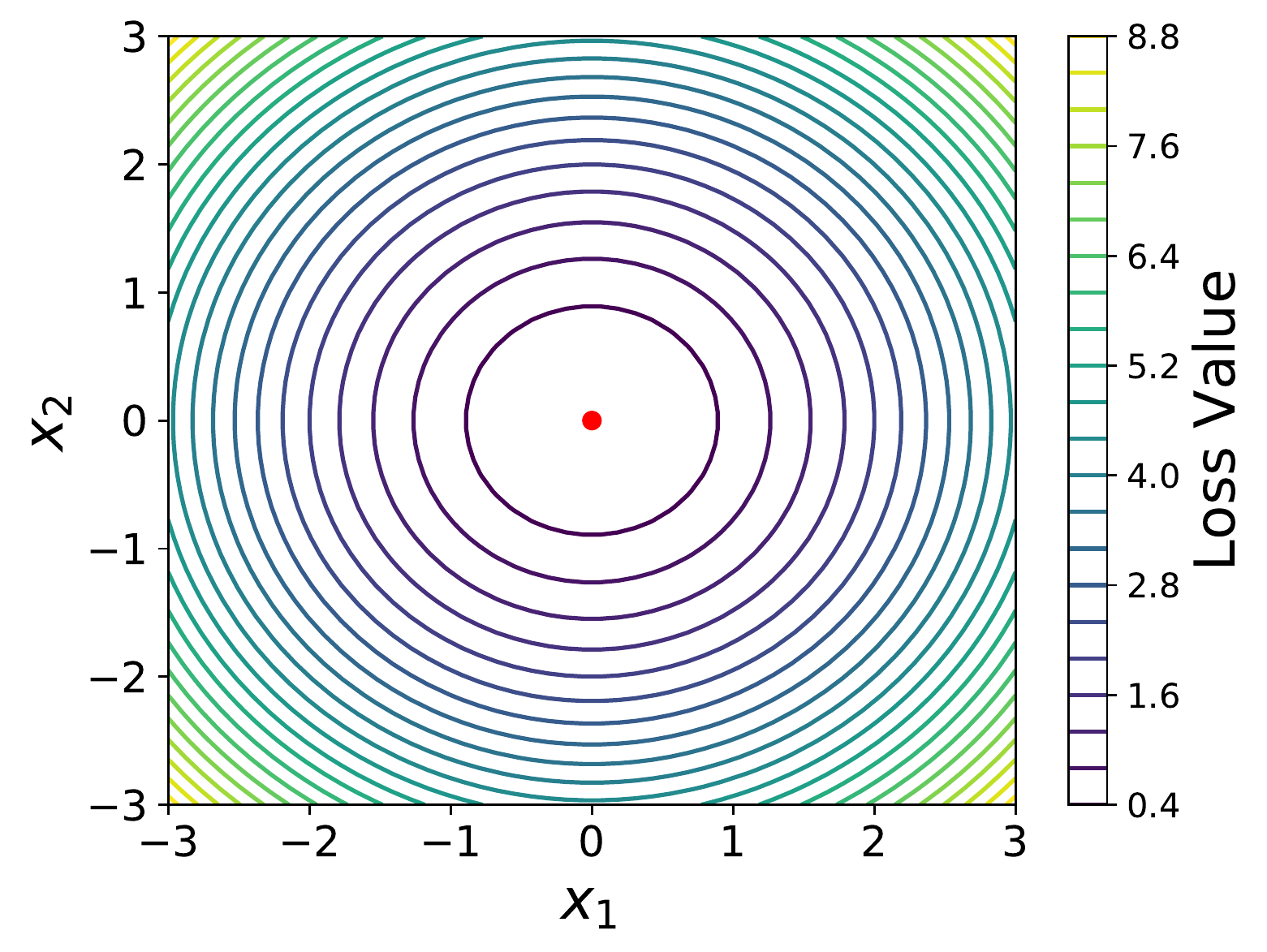}}
\caption{Illustration of preconditioning for $\bA=\begin{bmatrix}
		20&5 \\5&5
	\end{bmatrix}$. $\bP$ is obtained by the Cholesky decomposition such that $\bM=\bA=\bP^\top\bP$.}
\label{fig:conjugate_precondition13}
\end{figure}
\index{Cholesky decomposition}

\paragraph{Second perspective of preconditioning.}
The matrices $\bM^{-1}\bA$ and $\bP^{-\top} \bA\bP^{-1}$ have the same eigenvalues. To see this, suppose the eigenpair of $\bM^{-1}\bA$ is $(\bM^{-1} \bA) \bv =\lambda \bv$, it follows that
$$
(\bP^{-\top} \bA\bP^{-1}) (\bP\bv) = \bP^{-\top} \bA\bv = 
\bP\bP^{-1}\bP^{-\top} \bA\bv 
=\bP\bM^{-1}\bA\bv=\lambda (\bP\bv).
$$
Therefore, the preconditioning can be understood from two perspectives. While the second perspective is to solve $\bM^{-1}\bA\bx = \bM^{-1}\bb$, where the condition number is decided by matrix $\bM^{-1}\bA$.
The simplest preconditioner $\bM^{-1}$ is  a diagonal matrix whose diagonal entries are identical to those of $\bA$, known as \textit{diagonal preconditioning}, in which case  we scale the quadratic form along the coordinate axes. In contrast, the \textit{perfect preconditioner} is $\bM=\bA$ such that $\bM^{-1}\bA=\bI$, whose condition number is 1, in which case  the quadratic form is scaled along its eigenvector directions. In this sense, the $\bP$ can be obtained by the (pseudo) Cholesky decomposition (Theorem 2.1 in \citet{lu2022matrix} or Appendix~\ref{appendix:choleskydecomp}) such that $\bM=\bA=\bP^\top\bP$. Figure~\ref{fig:conjugate_precondition13} shows the perfect preconditioning on $\bM=\bA=\begin{bmatrix}
20&5 \\5&5\\
\end{bmatrix}$ such that the eigenvalues of $\bP^{-\top}\bA\bP^{-1}$ are identical and the condition number is thus equal to 1.

\index{Cholesky decomposition}

\subsubsection{General Conjugate Gradient Method}
We now revisit the general CG method as introduced in \citet{fletcher1964function}. The method has been previously formulated in Algorithm~\ref{alg:conjugate_descent}; we may notice the \textit{Fletcher-Reeves Conjugate Gradient} method (Algorithm~\ref{alg:conjugate_descent}) is just the same as the \textit{Practical Conjugate Gradient} method (Algorithm~\ref{alg:practical_conjugate_descent}) under the conditions of a strongly convex quadratic loss function and the use of an exact line search for the learning rate $\eta_t$.

To see why the Fletcher-Reeves Conjugate Gradient algorithm (Algorithm~\ref{alg:conjugate_descent}) works, the search direction $\bd_t$ must satisfy the descent condition (Remark~\ref{remark:descent_condition}, p.~\pageref{remark:descent_condition}) such that $ \bg_t^\top \bd_t<0$. The descent condition is satisfied when the learning rate is calculated by exact line search, in which case the gradient $\nabla L(\bx_t) = \bg_t$ is orthogonal to search direction $\bd_{t-1}$ (Lemma~\ref{lemm:linear-search-orghonal}, p.~\pageref{lemm:linear-search-orghonal}): $\bg_t^\top\bd_{t-1}=0$. Therefore, 
$$
\bg_t^\top \bd_t = \bg_t^\top (-\bg_t +\beta_t\bd_{t-1} ) = -\norm{\bg_t}^2 + \beta_t \bg_t^\top \bd_{t-1}<0
$$
when $\eta_t$ is determined by exact line search. However, when $\eta_t$ is fixed or calculated by inexact line search, the descent condition $\bg_t^\top\bd_t$ may not be satisfied. This problem, however, can be attacked by \textit{strong Wolfe conditions} \citep{nocedal1999numerical}; and we will not go into the details.

\paragraph{Polak-Ribi\`ere conjugate gradient.} We have mentioned previously that the $\beta_t$ can also be computed by the Polak-Ribi\`ere coefficient:
$$
\text{Polak-Ribi\`ere:\gap } \beta_t^P = \frac{\bigg(\nabla L(\bx_t) - \nabla L(\bx_{t-1})\bigg)^\top \nabla L(\bx_t)}{\nabla L(\bx_{t-1})^\top \nabla L(\bx_{t-1})}
=
\frac{(\bg_t - \bg_{t-1})^\top  \bg_t}{ \bg_{t-1}^\top \bg_{t-1}}
.
$$
When the loss function is strongly convex quadratic and the learning rate is chosen by exact line search, the Polak-Ribi\`ere coefficient $\beta_t^P$ is identical to the Fletcher-Reeves coefficient $\beta_t^F$ since $\bg_{t-1}^\top\bg_t=0$ by Theorem~\ref{theorem:conjudate_CD_d-steps}.

\paragraph{Hestenes–Stiefel conjugate gradient.} Hestenes–Stiefel coefficient is yet another variant of the Polak-Ribi\`ere coefficient:
$$
\text{Hestenes–Stiefel:\gap } \beta_t^H = \frac{\bigg(\nabla L(\bx_t) - \nabla L(\bx_{t-1})\bigg)^\top \nabla L(\bx_t)}{\bigg(\nabla L(\bx_t) - \nabla L(\bx_{t-1})\bigg)^\top
\bd_{t-1}}
=
\frac{(\bg_t - \bg_{t-1})^\top  \bg_t}{ (\bg_t - \bg_{t-1})^\top \bd_{t-1}}.
$$
When the loss function is strongly convex quadratic and the learning rate is chosen by exact line search, the Hestenes–Stiefel coefficient $\beta_t^H$ is identical to the Fletcher-Reeves coefficient $\beta_t^F$ since $\bg_{t-1}^\top\bg_t=0$ by Theorem~\ref{theorem:conjudate_CD_d-steps} and $\bg_t^\top\bd_{t-2}=\bg_{t-1}^\top\bd_{t-2}=0$ by Theorem~\ref{theorem:expanding_subspace_minimization}.

Moreover, numerical experiments show that the Polak-Ribi\`ere coefficient and Hestenes
–Stiefel coefficient are more robust than Fletcher-Reeves coefficient in nonconvex settings \citep{nocedal1999numerical}.

\newpage

%% file: chapter-append.tex
\appendix
\newpage
\clearchapter{Taylor’s Expansion}
\begingroup
\hypersetup{linkcolor=winestain,
linktoc=page,  
}
\minitoc \newpage
\endgroup
\section{Taylor’s Expansion}\label{appendix:taylor-expansion}
\index{Taylor's formula}
\begin{theorem}[Taylor’s Expansion with Lagrange Remainder]
Let $f(x): \real\rightarrow \real$ be $k$-times continuously differentiable on the closed interval $I$ with endpoints $x$ and $y$, for some $k\geq 0$. If $f^{(k+1)}$ exists on the interval $I$, then there exists a $x^\star \in (x,y)$ such that 
$$
\begin{aligned}
&\gap f(x) \\
&= f(y) + f^\prime(y)(x-y) + \frac{f^{\prime\prime}(y)}{2!}(x-y)^2+ \ldots + \frac{f^{(k)}(y)}{k!}(x-y)^k
+ \frac{f^{(k+1)}(x^\star)}{(k+1)!}(x-y)^{k+1}\\
&=\sum_{i=0}^{k} \frac{f^{(i)}(y)}{i!} (x-y)^i + \frac{f^{(k+1)}(x^\star)}{(k+1)!}(x-y)^{k+1}.
\end{aligned}
$$ 
The Taylor's expansion can be extended to a function of vector $f(\bx):\real^n\rightarrow \real$ or a function of matrix $f(\bX): \real^{m\times n}\rightarrow \real$.
\end{theorem}
The Taylor's expansion, or also known as the \textit{Taylor's series}, approximates the function $f(x)$ around the value of $y$ by a polynomial in a single indeterminate $x$. To see where  this series comes from, we recall from the elementary calculus course that the approximated function around $\theta=0$ for $\cos (\theta)$ is given by 
$$
\cos (\theta) \approx 1-\frac{\theta^2}{2}.
$$
That is, the $\cos  (\theta)$ is approximated by a polynomial with a degree of 2. Suppose we want to approximate $\cos  (\theta)$ by the more general polynomial with  a degree of 2 by $ f(\theta) = c_1+c_2 \theta+ c_3 \theta^2$. A intuitive idea is to match the gradients around the $0$ point. That is,
$$\left\{
\begin{aligned}
\cos(0) &= f(0); \\
\cos^\prime(0) &= f^\prime(0);\\
\cos^{\prime\prime}(0) &= f^{\prime\prime}(0);\\
\end{aligned}
\right.
\leadto 
\left\{
\begin{aligned}
1 &= c_1; \\
-\sin(0) &=0= c_2;\\
-\cos(0) &=-1= 2c_3.\\
\end{aligned}
\right.
$$
This makes $f(\theta) = c_1+c_2 \theta+ c_3 \theta^2 = 1-\frac{\theta^2}{2}$, and agrees with our claim that $\cos (\theta) \approx 1-\frac{\theta^2}{2}$ around the 0 point. We shall not give the details of the proof.

\newpage
\clearchapter{Matrix Decomposition}
\begingroup
\hypersetup{linkcolor=winestain,
linktoc=page,  
}
\minitoc \newpage
\endgroup
\section{Cholesky Decomposition}\label{appendix:choleskydecomp}
\lettrine{\color{caligraphcolor}P}
Positive definiteness or positive semidefiniteness is one of the highest accolades to which a matrix can aspire. 
In this section, we introduce decompositional approaches for positive definite matrices and we illustrate the most famous Cholesky decomposition as follows.

\index{Cholesky decomposition}
\begin{theoremHigh}[Cholesky Decomposition]\label{theorem:cholesky-factor-exist}
Every positive definite (PD) matrix $\bA\in \real^{n\times n}$ can be factored as 
$$
\bA = \bR^\top\bR,
$$
where $\bR \in \real^{n\times n}$ is an upper triangular matrix \textbf{with positive diagonal elements}. This decomposition is known as the \textit{Cholesky decomposition}  of $\bA$. $\bR$ is known as the \textit{Cholesky factor} or \textit{Cholesky triangle} of $\bA$.

Alternatively, $\bA$ can be factored as $\bA=\bL\bL^\top$, where $\bL=\bR^\top$ is a lower triangular matrix \textit{with positive diagonals}.

Specifically, the Cholesky decomposition is unique (Corollary~\ref{corollary:unique-cholesky-main}).
\end{theoremHigh}

The Cholesky decomposition is named after a French military officer and mathematician, Andr\'{e}-Louis Cholesky (1875{\textendash}1918), who developed the Cholesky decomposition in his surveying work. The Cholesky decomposition is used primarily to solve positive definite linear systems. 

We here only establish the existence of the Cholesky decomposition through an inductive approach. Although alternative methods, such as a consequence of the LU decomposition, exist for proving it \citep{lu2022matrix}.
\begin{proof}[of Theorem~\ref{theorem:cholesky-factor-exist}]
We will prove by induction that every $n\times n$ positive definite matrix $\bA$ has a decomposition $\bA=\bR^\top\bR$. The $1\times 1$ case is trivial by setting $R=\sqrt{A}$, thus, $A=R^2$. 

Suppose any $k\times k$ PD matrix $\bA_k$ has a Cholesky decomposition. If we prove that  any $(k+1)\times(k+1)$ PD matrix $\bA_{k+1}$ can also be factored as this Cholesky decomposition, then we complete the proof.

For any $(k+1)\times(k+1)$ PD matrix $\bA_{k+1}$, write out $\bA_{k+1}$ as
$$
\bA_{k+1} = \begin{bmatrix}
\bA_k & \bb \\
\bb^\top & d
\end{bmatrix}.
$$
We note that $\bA_k$ is PD. 
From the inductive hypothesis, it admits a Cholesky decomposition $\bA_k$ given by $\bA_k = \bR_k^\top\bR_k$. We can construct the upper triangular matrix 
$$
\bR=\begin{bmatrix}
\bR_k & \br\\
0 & s
\end{bmatrix},
$$
such that it follows that
$$
\bR_{k+1}^\top\bR_{k+1} = 
\begin{bmatrix}
\bR_k^\top\bR_k & \bR_k^\top \br\\
\br^\top \bR_k & \br^\top\br+s^2
\end{bmatrix}.
$$
Therefore, if we can prove $\bR_{k+1}^\top \bR_{k+1} = \bA_{k+1}$ is the Cholesky decomposition of $\bA_{k+1}$ (which requires the value $s$ to be positive), then we complete the proof. That is, we need to prove
$$
\begin{aligned}
\bb &= \bR_k^\top \br, \\
d &= \br^\top\br+s^2.
\end{aligned}
$$
Since $\bR_k$ is nonsingular, we have a unique solution for $\br$ and $s$ that 
$$
\begin{aligned}
\br &= \bR_k^{-\top}\bb, \\
s &= \sqrt{d - \br^\top\br} = \sqrt{d - \bb^\top\bA_k^{-1}\bb},
\end{aligned}
$$
where we assume $s$ is nonnegative. However, we need to further prove that $s$ is not only nonnegative, but also positive. Since $\bA_k$ is PD, from Sylvester's criterion (see \citet{lu2022matrix}), and the fact that if matrix $\bM$ has a block formulation: $\bM=\begin{bmatrix}
\bA & \bB \\
\bC & \bD 
\end{bmatrix}$, then $\det(\bM) = \det(\bA)\det(\bD-\bC\bA^{-1}\bB)$, we have
$$
\det(\bA_{k+1}) = \det(\bA_k)\det(d- \bb^\top\bA_k^{-1}\bb) =  \det(\bA_k)(d- \bb^\top\bA_k^{-1}\bb)>0.
$$
Because $ \det(\bA_k)>0$, we then obtain that $(d- \bb^\top\bA_k^{-1}\bb)>0$, and this implies $s>0$.
We complete the proof.
\end{proof}
%

\index{Uniqueness}
\begin{corollary}[Uniqueness of Cholesky Decomposition\index{Uniqueness}]\label{corollary:unique-cholesky-main}
The Cholesky decomposition $\bA=\bR^\top\bR$ for any positive definite matrix $\bA\in \real^{n\times n}$ is unique.
\end{corollary}
\begin{proof}[of Corollary~\ref{corollary:unique-cholesky-main}]
Suppose the Cholesky decomposition is not unique, then we can find two decompositions such that $\bA=\bR_1^\top\bR_1 = \bR_2^\top\bR_2$, which implies
$$
\bR_1\bR_2^{-1}= \bR_1^{-\top} \bR_2^\top.
$$
From the fact that the inverse of an upper triangular matrix is also an upper triangular matrix, and the product of two upper triangular matrices is also an upper triangular matrix, \footnote{Same for lower triangular matrices: the inverse of a lower triangular matrix is also a lower triangular matrix, and the product of two lower triangular matrices is also a lower triangular matrix.} we realize that the left-side of the above equation is an upper triangular matrix and the right-side of it is a lower triangular matrix. This implies $\bR_1\bR_2^{-1}= \bR_1^{-\top} \bR_2^\top$ is a diagonal matrix, and $\bR_1^{-\top} \bR_2^\top= (\bR_1^{-\top} \bR_2^\top)^\top = \bR_2\bR_1^{-1}$.
Let $\bLambda = \bR_1\bR_2^{-1}= \bR_2\bR_1^{-1}$ be the diagonal matrix. We notice that the diagonal value of $\bLambda$ is the product of the corresponding diagonal values of $\bR_1$ and $\bR_2^{-1}$ (or $\bR_2$ and $\bR_1^{-1}$). That is, for 
$$
\bR_1=\begin{bmatrix}
r_{11} & r_{12} & \ldots & r_{1n} \\
0 & r_{22} & \ldots & r_{2n}\\
\vdots & \vdots & \ddots & \vdots\\
0 & 0 & \ldots & r_{nn}
\end{bmatrix},
\qquad 
\bR_2=
\begin{bmatrix}
s_{11} & s_{12} & \ldots & s_{1n} \\
0 & s_{22} & \ldots & s_{2n}\\
\vdots & \vdots & \ddots & \vdots\\
0 & 0 & \ldots & s_{nn}
\end{bmatrix},
$$
we have,
$$
\begin{aligned}
\bR_1\bR_2^{-1}=
\begin{bmatrix}
\frac{r_{11}}{s_{11}} & 0 & \ldots & 0 \\
0 & \frac{r_{22}}{s_{22}} & \ldots & 0\\
\vdots & \vdots & \ddots & \vdots\\
0 & 0 & \ldots & \frac{r_{nn}}{s_{nn}}
\end{bmatrix}
=
\begin{bmatrix}
\frac{s_{11}}{r_{11}} & 0 & \ldots & 0 \\
0 & \frac{s_{22}}{r_{22}} & \ldots & 0\\
\vdots & \vdots & \ddots & \vdots\\
0 & 0 & \ldots & \frac{s_{nn}}{r_{nn}}
\end{bmatrix}
=\bR_2\bR_1^{-1}.
\end{aligned}
$$ 
Since both $\bR_1$ and $\bR_2$ have positive diagonals, this implies $r_{11}=s_{11}, r_{22}=s_{22}, \ldots, r_{nn}=s_{nn}$. And $\bLambda = \bR_1\bR_2^{-1}= \bR_2\bR_1^{-1}  =\bI$.
That is, $\bR_1=\bR_2$ and this leads to a contradiction. The Cholesky decomposition is thus unique. 
\end{proof}		

\section{Eigenvalue Decomposition}\label{appendix:eigendecomp}
\index{Eigenvalue decomposition}
\begin{theoremHigh}[Eigenvalue Decomposition]\label{theorem:eigenvalue-decomposition}
Any square matrix $\bA\in \real^{n\times n}$ with linearly independent eigenvectors can be factored as 
$$
\bA = \bX\bLambda\bX^{-1},
$$
where $\bX\in\real^{n\times n}$ contains the eigenvectors of $\bA$ as its columns, and $\bLambda\in\real^{n\times n}$ is a diagonal matrix $\diag(\lambda_1, $ $\lambda_2, \ldots, \lambda_n)$ with $\lambda_1, \lambda_2, \ldots, \lambda_n$ denoting the eigenvalues of $\bA$.
\end{theoremHigh}

Eigenvalue decomposition is also known as to diagonalize the matrix $\bA$. When no eigenvalues of $\bA$ are repeated, the corresponding eigenvectors are guaranteed to be linearly independent. Then $\bA$ can be diagonalized. It is essential to emphasize that without a set of $n$ linearly independent eigenvectors, diagonalization is not feasible. 

\begin{proof}[of Theorem~\ref{theorem:eigenvalue-decomposition}]
Let $\bX=[\bx_1, \bx_2, \ldots, \bx_n]$ be the linearly independent eigenvectors of $\bA$. Clearly, we have
$$
\bA\bx_1=\lambda_1\bx_1,\qquad \bA\bx_2=\lambda_2\bx_2, \qquad \ldots, \qquad\bA\bx_n=\lambda_n\bx_n.
$$
In the matrix form, we can express this as:
$$
\bA\bX = [\bA\bx_1, \bA\bx_2, \ldots, \bA\bx_n] = [\lambda_1\bx_1, \lambda_2\bx_2, \ldots, \lambda_n\bx_n] = \bX\bLambda.
$$
Since we assume the eigenvectors are linearly independent, then $\bX$ has full rank and is invertible. We obtain
$$
\bA = \bX\bLambda \bX^{-1}.
$$
This completes the proof.
\end{proof}

We will discuss some similar forms of eigenvalue decomposition in the spectral decomposition section, where the matrix $\bA$ is required to be symmetric, and the factor $\bX$ is not only nonsingular but also orthogonal. Alternatively, the matrix $\bA$ is required to be a \textit{simple matrix}, that is, the algebraic multiplicity and geometric multiplicity are the same for $\bA$, and $\bX$ will be a trivial nonsingular matrix that may not contain the eigenvectors of $\bA$.

A matrix decomposition in the form of $\bA =\bX\bLambda\bX^{-1}$ has a nice property that we can compute the $m$-th power efficiently.

\index{$m$-th Power}
\begin{remark}[$m$-th Power]\label{remark:power-eigenvalue-decom}
The $m$-th power of $\bA$ is $\bA^m = \bX\bLambda^m\bX^{-1}$ if the matrix $\bA$ can be factored as $\bA=\bX\bLambda\bX^{-1}$.
\end{remark}

We observe that the prerequisite for the existence of the eigenvalue decomposition is the linear independence of the eigenvectors of $\bA$. This condition is inherently fulfilled under specific circumstances.
\begin{lemma}[Different Eigenvalues]\label{lemma:diff-eigenvec-decompo}
Suppose the eigenvalues $\lambda_1, \lambda_2, \ldots, \lambda_n$ of $\bA\in \real^{n\times n}$ are all distinct, the associated eigenvectors are inherently linearly independent. Put differently, any square matrix with unique eigenvalues can be diagonalized.
\end{lemma}
\begin{proof}[of Lemma~\ref{lemma:diff-eigenvec-decompo}]
Suppose the eigenvalues $\lambda_1, \lambda_2, \ldots, \lambda_n$ are all different, and the eigenvectors $\bx_1,\bx_2, \ldots, \bx_n$ are linearly dependent. That is, there exists a nonzero vector $\bc = [c_1,c_2,\ldots,c_{n-1}]^\top$ satisfying
$$
\bx_n = \sum_{i=1}^{n-1} c_i\bx_{i}. 
$$
Then we have 
$$
\begin{aligned}
\bA \bx_n &= \bA (\sum_{i=1}^{n-1} c_i\bx_{i}) \\
&=c_1\lambda_1 \bx_1 + c_2\lambda_2 \bx_2 + \ldots + c_{n-1}\lambda_{n-1}\bx_{n-1}.
\end{aligned}
$$
and 
$$
\begin{aligned}
\bA \bx_n &= \lambda_n\bx_n\\
&=\lambda_n (c_1\bx_1 +c_2\bx_2+\ldots +c_{n-1} \bx_{n-1}).
\end{aligned}
$$
Combining  two equations above, we have 
$$
\sum_{i=1}^{n-1} (\lambda_n - \lambda_i)c_i \bx_i = \bzero .
$$
This leads to a contradiction since $\lambda_n \neq \lambda_i$ for all $i\in \{1,2,\ldots,n-1\}$, from which the result follows.
\end{proof}

%
%

\section{Spectral Decomposition}\label{appendix:spectraldecomp}\label{section:existence-of-spectral}
\index{Spectral decomposition}
\begin{theoremHigh}[Spectral Decomposition]\label{theorem:spectral_theorem}
A real matrix $\bA \in \real^{n\times n}$ is symmetric if and only if there exists an orthogonal matrix $\bQ$ and a diagonal matrix $\bLambda$ such that
\begin{equation*}
\bA = \bQ \bLambda \bQ^\top,
\end{equation*}
where the columns of $\bQ = [\bq_1, \bq_2, \ldots, \bq_n]$ are eigenvectors of $\bA$ and are mutually orthonormal, and the entries of $\bLambda=\diag(\lambda_1, \lambda_2, \ldots, \lambda_n)$ are the corresponding eigenvalues of $\bA$, which are real. And the rank of $\bA$ is equal to the number of nonzero eigenvalues. This is known as the \textit{spectral decomposition} or \textit{spectral theorem} for real symmetric matrix $\bA$. Specifically, we have the following properties:
\begin{enumerate}
\item A symmetric matrix has only \textbf{real eigenvalues};

\item The eigenvectors are orthogonal such that they can be chosen \textbf{orthonormal} by normalization;

\item The rank of $\bA$ is equal to the number of nonzero eigenvalues;

\item If the eigenvalues are distinct, the eigenvectors are linearly independent.
\end{enumerate}
\end{theoremHigh}

The above decomposition is called the spectral decomposition for real symmetric matrices and is often known as the \textit{spectral theorem}. We prove the theorem in several steps.

\begin{tcolorbox}[title={Symmetric Matrix Property 1 of 4},colback=\mdframecolorTheorem]
\begin{lemma}[Real Eigenvalues]\label{lemma:real-eigenvalues-spectral}
The eigenvalues of any symmetric matrix are all real. 
\end{lemma}
\end{tcolorbox}
\begin{proof}[of Lemma~\ref{lemma:real-eigenvalues-spectral}]
Suppose eigenvalue $\lambda$ of a symmetric matrix $\bA$ is a complex number $\lambda=a+ib$, where $a,b$ are real. Its complex conjugate is $\bar{\lambda}=a-ib$. Same for the corresponding complex eigenvector $\bx = \bc+i\bd$ and its complex conjugate $\bar{\bx}=\bc-i\bd$, where $\bc, \bd$ are real vectors. We then have the following property
$$
\bA \bx = \lambda \bx\qquad   \underrightarrow{\text{ leads to }}\qquad  \bA \bar{\bx} = \bar{\lambda} \bar{\bx}\qquad   \underrightarrow{\text{ transpose to }}\qquad  \bar{\bx}^\top \bA =\bar{\lambda} \bar{\bx}^\top.
$$
We take the dot product of the first equation with $\bar{\bx}$ and the last equation with $\bx$:
$$
\bar{\bx}^\top \bA \bx = \lambda \bar{\bx}^\top \bx, \qquad \text{and } \qquad \bar{\bx}^\top \bA \bx = \bar{\lambda}\bar{\bx}^\top \bx.
$$
Then we have the equality $\lambda\bar{\bx}^\top \bx = \bar{\lambda} \bar{\bx}^\top\bx$. Since $\bar{\bx}^\top\bx = (\bc-i\bd)^\top(\bc+i\bd) = \bc^\top\bc+\bd^\top\bd$ is a real number. Therefore, the imaginary part of $\lambda$ is zero and $\lambda$ is real.
\end{proof}

\begin{tcolorbox}[title={Symmetric Matrix Property 2 of 4},colback=\mdframecolorTheorem]
\begin{lemma}[Orthogonal Eigenvectors]\label{lemma:orthogonal-eigenvectors}
The eigenvectors  corresponding to distinct eigenvalues of any symmetric matrix are orthogonal. Therefore, we can normalize eigenvectors to make them orthonormal since $\bA\bx = \lambda \bx \underrightarrow{\text{ leads to } } \bA\frac{\bx}{\norm{\bx}} = \lambda \frac{\bx}{\norm{\bx}}$, which corresponds to the same eigenvalue.
\end{lemma}
\end{tcolorbox}
\begin{proof}[of Lemma~\ref{lemma:orthogonal-eigenvectors}]
Suppose eigenvalues $\lambda_1, \lambda_2$ correspond to eigenvectors $\bx_1, \bx_2$, satisfying $\bA\bx_1=\lambda \bx_1$ and $\bA\bx_2 = \lambda_2\bx_2$. We have the following equality:
$$
\bA\bx_1=\lambda_1 \bx_1 \leadto \bx_1^\top \bA =\lambda_1 \bx_1^\top \leadto \bx_1^\top \bA \bx_2 =\lambda_1 \bx_1^\top\bx_2,
$$
and 
$$
\bA\bx_2 = \lambda_2\bx_2 \leadto  \bx_1^\top\bA\bx_2 = \lambda_2\bx_1^\top\bx_2,
$$
which implies $\lambda_1 \bx_1^\top\bx_2=\lambda_2\bx_1^\top\bx_2$. Since eigenvalues $\lambda_1\neq \lambda_2$, the eigenvectors are orthogonal.
\end{proof}

In   Lemma~\ref{lemma:orthogonal-eigenvectors} above, we establish the orthogonality of eigenvectors associated with distinct eigenvalues of symmetric matrices. More generally, we prove the important theorem that eigenvectors corresponding to distinct eigenvalues of any matrix are linearly independent. 
\begin{theorem}[Independent Eigenvector Theorem]\label{theorem:independent-eigenvector-theorem}
If a matrix $\bA\in \real^{n\times n}$ has $k$ distinct eigenvalues, then any set of $k$ corresponding (nonzero) eigenvectors are linearly independent.
\end{theorem}
\begin{proof}[of Theorem~\ref{theorem:independent-eigenvector-theorem}]
We will prove by induction. Firstly, we will prove that any two eigenvectors corresponding to distinct eigenvalues are linearly independent. Suppose $\bv_1$ and $\bv_2$ correspond to distinct eigenvalues $\lambda_1$ and $\lambda_2$, respectively. Suppose further there exists a nonzero vector $\bx=[x_1,x_2] \neq \bzero $ satisfying 
\begin{equation}\label{equation:independent-eigenvector-eq1}
x_1\bv_1+x_2\bv_2=\bzero.
\end{equation}
That is, $\bv_1,\bv_2$ are linearly independent.
Multiply Eq.~\eqref{equation:independent-eigenvector-eq1} on the left by $\bA$, we get
\begin{equation}\label{equation:independent-eigenvector-eq2}
x_1 \lambda_1\bv_1 + x_2\lambda_2\bv_2 = \bzero.
\end{equation}
Multiply Eq.~\eqref{equation:independent-eigenvector-eq1} on the left by $\lambda_2$, we get 
\begin{equation}\label{equation:independent-eigenvector-eq3}
x_1\lambda_2\bv_1 + x_2\lambda_2\bv_2 = \bzero.
\end{equation}
Subtract Eq.~\eqref{equation:independent-eigenvector-eq2} from Eq.~\eqref{equation:independent-eigenvector-eq3}, we find
$$
x_1(\lambda_2-\lambda_1)\bv_1 = \bzero.
$$
Since $\lambda_2\neq \lambda_1$ and $\bv_1\neq \bzero$, we must have $x_1=0$. From Eq.~\eqref{equation:independent-eigenvector-eq1} and  $\bv_2\neq \bzero$, we must also have $x_2=0$, which arrives at a contradiction. Thus, $\bv_1$ and $\bv_2$ are linearly independent.

Now, suppose any $j<k$ eigenvectors are linearly independent, if we could prove that any $j+1$ eigenvectors are also linearly independent, we finish the proof. Suppose $\bv_1, \bv_2, \ldots, \bv_j$ are linearly independent, and $\bv_{j+1}$ is dependent on the first $j$ eigenvectors. That is, there exists a nonzero vector $\bx=[x_1,x_2,\ldots, x_{j}]\neq \bzero$ such that 
\begin{equation}\label{equation:independent-eigenvector-zero}
\bv_{j+1}=	x_1\bv_1+x_2\bv_2+\ldots+x_j\bv_j .
\end{equation}
Suppose the $j+1$ eigenvectors correspond to distinct eigenvalues $\lambda_1,\lambda_2,\ldots,\lambda_j,\lambda_{j+1}$.
Multiply Eq.~\eqref{equation:independent-eigenvector-zero} on the left by $\bA$, we get
\begin{equation}\label{equation:independent-eigenvector-zero2}
\lambda_{j+1} \bv_{j+1} = x_1\lambda_1\bv_1+x_2\lambda_2\bv_2+\ldots+x_j \lambda_j\bv_j .
\end{equation}
Multiply Eq.~\eqref{equation:independent-eigenvector-zero} on the left by $\lambda_{j+1}$, we get
\begin{equation}\label{equation:independent-eigenvector-zero3}
\lambda_{j+1} \bv_{j+1} = x_1\lambda_{j+1}\bv_1+x_2\lambda_{j+1}\bv_2+\ldots+x_j \lambda_{j+1}\bv_j .
\end{equation}
Subtract Eq.~\eqref{equation:independent-eigenvector-zero3} from Eq.~\eqref{equation:independent-eigenvector-zero2}, we find
$$
x_1(\lambda_{j+1}-\lambda_1)\bv_1+x_2(\lambda_{j+1}-\lambda_2)\bv_2+\ldots+x_j (\lambda_{j+1}-\lambda_j)\bv_j = \bzero. 
$$
From assumption, $\lambda_{j+1} \neq \lambda_i$ for all $i\in \{1,2,\ldots,j\}$, and $\bv_i\neq \bzero$ for all $i\in \{1,2,\ldots,j\}$. We must have $x_1=x_2=\ldots=x_j=0$, which leads to a contradiction. Then $\bv_1,\bv_2,\ldots,\bv_j,\bv_{j+1}$ are linearly independent. This completes the proof.
\end{proof}

A direct consequence of the  above theorem is as follows:
\begin{corollary}[Independent Eigenvector Theorem, CNT.]\label{theorem:independent-eigenvector-theorem-basis}
If a matrix $\bA\in \real^{n\times n}$ has $n$ distinct eigenvalues, then any set of $n$ corresponding eigenvectors form a basis for $\real^n$.
\end{corollary}

\begin{tcolorbox}[title={Symmetric Matrix Property 3 of 4},colback=\mdframecolorTheorem]
\begin{lemma}[Orthonormal Eigenvectors for Duplicate Eigenvalue]\label{lemma:eigen-multiplicity}
If $\bA$ has a duplicate eigenvalue $\lambda_i$ with multiplicity $k\geq 2$, then there exist $k$ orthonormal eigenvectors corresponding to $\lambda_i$.
\end{lemma}
\end{tcolorbox}
\begin{proof}[of Lemma~\ref{lemma:eigen-multiplicity}]
We note that there is at least one eigenvector $\bx_{i1}$ corresponding to $\lambda_i$. And for such eigenvector $\bx_{i1}$, we can always find additional $n-1$ orthonormal vectors $\by_2, \by_3, \ldots, \by_n$ so that $\{\bx_{i1}, \by_2, \by_3, \ldots, \by_n\}$ forms an orthonormal basis in $\real^n$. Put the $\by_2, \by_3, \ldots, \by_n$ into matrix $\bY_1$ and $\{\bx_{i1}, \by_2, \by_3, \ldots, \by_n\}$ into matrix $\bP_1$:
$$
\bY_1=[\by_2, \by_3, \ldots, \by_n] \qquad \text{and} \qquad \bP_1=[\bx_{i1}, \bY_1].
$$
We then have
$$
\bP_1^\top\bA\bP_1 = \begin{bmatrix}
\lambda_i &\bzero \\
\bzero & \bY_1^\top \bA\bY_1
\end{bmatrix}.
$$
As a result, $\bA$ and $\bP_1^\top\bA\bP_1$ are similar matrices such that they have the same eigenvalues since $\bP_1$ is nonsingular (even orthogonal here). We obtain 
$$
\det(\bP_1^\top\bA\bP_1 - \lambda\bI_n) =
\footnote{By the fact that if matrix $\bM$ has a block formulation: $\bM=\begin{bmatrix}
\bA & \bB \\
\bC & \bD 
\end{bmatrix}$, then $\det(\bM) = \det(\bA)\det(\bD-\bC\bA^{-1}\bB)$.
}
(\lambda_i - \lambda )\det(\bY_1^\top \bA\bY_1 - \lambda\bI_{n-1}).
$$
If $\lambda_i$ has multiplicity $k\geq 2$, then the term $(\lambda_i-\lambda)$ occurs $k$ times in the polynomial from the determinant $\det(\bP_1^\top\bA\bP_1 - \lambda\bI_n)$, i.e., the term occurs $k-1$ times in the polynomial from $\det(\bY_1^\top \bA\bY_1 - \lambda\bI_{n-1})$. In another word, $\det(\bY_1^\top \bA\bY_1 - \lambda_i\bI_{n-1})=0$ and $\lambda_i$ is an eigenvalue of $\bY_1^\top \bA\bY_1$. 

Let $\bB=\bY_1^\top \bA\bY_1$. Since $\det(\bB-\lambda_i\bI_{n-1})=0$, the null space of $\bB-\lambda_i\bI_{n-1}$ is not empty. Suppose $(\bB-\lambda_i\bI_{n-1})\bn = \bzero$, i.e., $\bB\bn=\lambda_i\bn$ and $\bn$ is an eigenvector of $\bB$. 

From $
\bP_1^\top\bA\bP_1 = \begin{bmatrix}
\lambda_i &\bzero \\
\bzero & \bB
\end{bmatrix},
$
we have $
\bA\bP_1 
\begin{bmatrix}
z \\
\bn 
\end{bmatrix} 
= 
\bP_1
\begin{bmatrix}
\lambda_i &\bzero \\
\bzero & \bB
\end{bmatrix}
\begin{bmatrix}
z \\
\bn 
\end{bmatrix}$, where $z$ is any scalar. From the left side of this equation, we have 
\begin{equation}\label{equation:spectral-pro4-right}
\begin{aligned}
\bA\bP_1 
\begin{bmatrix}
	z \\
	\bn 
\end{bmatrix} 
&=
\begin{bmatrix}
	\lambda_i\bx_{i1}, \bA\bY_1
\end{bmatrix}
\begin{bmatrix}
	z \\
	\bn 
\end{bmatrix} \\
&=\lambda_iz\bx_{i1} + \bA\bY_1\bn.
\end{aligned}
\end{equation}
And from the right side of the equation, we have 
\begin{equation}\label{equation:spectral-pro4-left}
\begin{aligned}
\bP_1
\begin{bmatrix}
	\lambda_i &\bzero \\
	\bzero & \bB
\end{bmatrix}
\begin{bmatrix}
	z \\
	\bn 
\end{bmatrix}
&=
\begin{bmatrix}
	\bx_{i1} & \bY_1
\end{bmatrix}
\begin{bmatrix}
	\lambda_i &\bzero \\
	\bzero & \bB
\end{bmatrix}
\begin{bmatrix}
	z \\
	\bn 
\end{bmatrix}\\
&=
\begin{bmatrix}
	\lambda_i\bx_{i1} & \bY_1\bB 
\end{bmatrix}
\begin{bmatrix}
	z \\
	\bn 
\end{bmatrix}\\
&= \lambda_i z \bx_{i1} + \bY_1\bB \bn \\
&=\lambda_i z \bx_{i1} + \lambda_i \bY_1 \bn.  \qquad (\text{Since $\bB \bn=\lambda_i\bn$})\\
\end{aligned}
\end{equation}
Combining Eq.~\eqref{equation:spectral-pro4-left} and Eq.~\eqref{equation:spectral-pro4-right}, we obtain 
$$
\bA\bY_1\bn = \lambda_i\bY_1 \bn,
$$
which means $\bY_1\bn$ is an eigenvector of $\bA$ corresponding to the eigenvalue $\lambda_i$ (same eigenvalue corresponding to $\bx_{i1}$). Since $\bY_1\bn$ is a linear combination of $\by_2, \by_3, \ldots, \by_n$, which are orthonormal to $\bx_{i1}$, the $\bY_1\bn$ can be chosen to be orthonormal to $\bx_{i1}$.

To conclude, if we have one eigenvector $\bx_{i1}$ corresponding to $\lambda_i$ whose multiplicity is $k\geq 2$, we could construct the second eigenvector by choosing one vector from the null space of $(\bB-\lambda_i\bI_{n-1})$, as constructed above. Suppose now, we have constructed the second eigenvector $\bx_{i2}$, which is orthonormal to $\bx_{i1}$.  
For such eigenvectors $\bx_{i1}$ and $\bx_{i2}$, we can always find additional $n-2$ orthonormal vectors $\by_3, \by_4, \ldots, \by_n$ so that $\{\bx_{i1},\bx_{i2}, \by_3, \by_4, \ldots, \by_n\}$ forms an orthonormal basis in $\real^n$. Put the $\by_3, \by_4, \ldots, \by_n$ into matrix $\bY_2$ and $\{\bx_{i1},\bx_{i2},  \by_3, \by_4, \ldots, \by_n\}$ into matrix $\bP_2$:
$$
\bY_2=[\by_3, \by_4, \ldots, \by_n] \qquad \text{and} \qquad \bP_2=[\bx_{i1}, \bx_{i2},\bY_1].
$$
We then have
$$
\bP_2^\top\bA\bP_2 = 
\begin{bmatrix}
\lambda_i & 0 &\bzero \\
0& \lambda_i &\bzero \\
\bzero & \bzero & \bY_2^\top \bA\bY_2
\end{bmatrix}
=
\begin{bmatrix}
\lambda_i & 0 &\bzero \\
0& \lambda_i &\bzero \\
\bzero & \bzero & \bC
\end{bmatrix},
$$
where $\bC=\bY_2^\top \bA\bY_2$ such that $\det(\bP_2^\top\bA\bP_2 - \lambda\bI_n) = (\lambda_i-\lambda)^2 \det(\bC - \lambda\bI_{n-2})$. If the multiplicity of $\lambda_i$ is $k\geq 3$, $\det(\bC - \lambda_i\bI_{n-2})=0$ and the null space of $\bC - \lambda_i\bI_{n-2}$ is not empty so that we can still find a vector from null space of $\bC - \lambda_i\bI_{n-2}$ and $\bC\bn = \lambda_i \bn$. Now we can construct a vector $\begin{bmatrix}
z_1 \\
z_2\\
\bn
\end{bmatrix}\in \real^n $, where $z_1, z_2$ are any scalar values, such that 
$$
\bA\bP_2\begin{bmatrix}
z_1 \\
z_2\\
\bn
\end{bmatrix} = \bP_2 
\begin{bmatrix}
\lambda_i & 0 &\bzero \\
0& \lambda_i &\bzero \\
\bzero & \bzero & \bC
\end{bmatrix}
\begin{bmatrix}
z_1 \\
z_2\\
\bn
\end{bmatrix}.
$$
Similarly, from the left side of the above equation, we will get $\lambda_iz_1\bx_{i1} +\lambda_iz_2\bx_{i2}+\bA\bY_2\bn$. From the right side of the above equation, we will get $\lambda_iz_1\bx_{i1} +\lambda_i z_2\bx_{i2}+\lambda_i\bY_2\bn$. As a result, 
$$
\bA\bY_2\bn = \lambda_i\bY_2\bn,
$$
where $\bY_2\bn$ is an eigenvector of $\bA$ and orthogonal to $\bx_{i1}, \bx_{i2}$. And it is easy to construct the eigenvector to be orthonormal to the first two.

The process can go on, and finally, we will find $k$ orthonormal eigenvectors corresponding to $\lambda_i$. 

In fact, the dimension of the null space of $\bP_1^\top\bA\bP_1 -\lambda_i\bI_n$ is equal to the multiplicity $k$. It also follows that if the multiplicity of $\lambda_i$ is $k$, there cannot be more than $k$ orthogonal eigenvectors corresponding to $\lambda_i$. Otherwise, it will come to the conclusion that we could find more than $n$ orthogonal eigenvectors, which would lead to a contradiction.
\end{proof}

The proof of the existence of the spectral decomposition is then trivial from the lemmas above. 

For any matrix multiplication, the rank of the multiplication result is not larger than the rank of the inputs.

\begin{lemma}[Rank of $\bA\bB$]\label{lemma:rankAB}
For any matrix $\bA\in \real^{m\times n}$ and  $\bB\in \real^{n\times k}$, then the matrix multiplication $\bA\bB\in \real^{m\times k}$ has $\rank$($\bA\bB$)$\leq$min($\rank$($\bA$), $\rank$($\bB$)).
\end{lemma}

\begin{proof}[of Lemma~\ref{lemma:rankAB}]
For matrix multiplication $\bA\bB$, we have 
\begin{itemize}
\item All rows of $\bA\bB$ are combinations of rows of $\bB$, the row space of $\bA\bB$ is a subset of the row space of $\bB$. Thus, $\rank$($\bA\bB$)$\leq$$\rank$($\bB$).

\item All columns of $\bA\bB$ are combinations of columns of $\bA$, the column space of $\bA\bB$ is a subset of the column space of $\bA$. Thus, $\rank$($\bA\bB$)$\leq$$\rank$($\bA$).
\end{itemize}

Therefore, $\rank$($\bA\bB$)$\leq$min($\rank$($\bA$), $\rank$($\bB$)).
\end{proof}
\begin{tcolorbox}[title={Symmetric Matrix Property 4 of 4},colback=\mdframecolorTheorem]
\begin{lemma}[Rank of Symmetric Matrices]\label{lemma:rank-of-symmetric}
If $\bA$ is an $n\times n$ real symmetric matrix, then rank($\bA$) =
the total number of nonzero eigenvalues of $\bA$. 
In particular, $\bA$ has full rank if and only if $\bA$ is nonsingular. Further, $\cspace(\bA)$ is the linear space spanned by the eigenvectors of $\bA$ that correspond to nonzero eigenvalues.
\end{lemma}
\end{tcolorbox}
\begin{proof}[of Lemma~\ref{lemma:rank-of-symmetric}]
For any symmetric matrix $\bA$, we have $\bA$, in spectral form, as $\bA = \bQ \bLambda\bQ^\top$ and also $\bLambda = \bQ^\top\bA\bQ$. Since we have shown in Lemma~\ref{lemma:rankAB} that the rank of the multiplication $\rank$($\bA\bB$)$\leq$min($\rank$($\bA$), $\rank$($\bB$)).

$\bullet$ From $\bA = \bQ \bLambda\bQ^\top$, we have $\rank(\bA) \leq \rank(\bQ \bLambda) \leq \rank(\bLambda)$;

$\bullet$ From $\bLambda = \bQ^\top\bA\bQ$, we have $\rank(\bLambda) \leq \rank(\bQ^\top\bA) \leq \rank(\bA)$, 

The inequalities above give us a contradiction. And thus $\rank(\bA) = \rank(\bLambda)$, which is the total number of nonzero eigenvalues.

Since $\bA$ is nonsingular if and only if all of its eigenvalues are nonzero, $\bA$ has full rank if and only if $\bA$ is nonsingular.
\end{proof}

\index{$m$-th Power}
Similar to the eigenvalue decomposition, we can compute the $m$-th power of matrix $\bA$ via the spectral decomposition more efficiently.
\begin{remark}[$m$-th Power]\label{remark:power-spectral}
The $m$-th power of $\bA$ is $\bA^m = \bQ\bLambda^m\bQ^\top$ if the matrix $\bA$ can be factored as the spectral decomposition $\bA=\bQ\bLambda\bQ^\top$.
\end{remark}
